\pdfoutput=1
\documentclass[a4paper, 11pt, dvipsnames, headings=medium,numbers=noenddot]{scrartcl}
\usepackage[top=2.5cm,bottom=3cm,left=2.5cm,right=2.5cm]{geometry}
\usepackage{caption}
\makeatletter
\renewcommand{\@seccntformat}[1]{\csname the#1\endcsname.\ }
\makeatother
\usepackage[english]{babel}
\usepackage[T1]{fontenc}
\usepackage[utf8]{inputenc}
\usepackage{graphicx}
\usepackage{natbib}
\usepackage{microtype}
\usepackage{subfigure}
\usepackage{booktabs}
\usepackage[table]{xcolor}

\usepackage{amsmath}
\usepackage{amssymb}
\usepackage{mathtools}
\usepackage{amsthm}
\usepackage{mathrsfs}

\setcitestyle{round}

\definecolor{mydarkblue}{rgb}{0,0.08,0.45}
\usepackage[colorlinks, anchorcolor=white, linkcolor=mydarkblue, urlcolor=mydarkblue, citecolor=mydarkblue]{hyperref}

\makeatletter
\def\thm@space@setup{\thm@preskip=0pt
\thm@postskip=0pt}
\makeatother
\newtheoremstyle{newstyle}      
{2pt plus 1pt minus 0.5pt} 
{4pt plus 1pt minus 0.5pt} 
{\mdseries} 
{} 
{\bfseries} 
{.} 
{ } 
{} 

\theoremstyle{newstyle}

\allowdisplaybreaks

\usepackage[capitalize,noabbrev]{cleveref}

\newtheorem{theorem}{Theorem}[section]
\newtheorem{proposition}[theorem]{Proposition}
\newtheorem{lemma}[theorem]{Lemma}
\newtheorem{fact}[theorem]{Fact}
\newtheorem{corollary}[theorem]{Corollary}
\theoremstyle{definition}
\newtheorem{definition}[theorem]{Definition}

\newtheorem{example}[theorem]{Example}
\newtheorem{remark}[theorem]{Remark}

\usepackage{enumitem} 
\usepackage{multirow}
\usepackage{makecell}

\usepackage{wrapfig}

\usepackage{amsmath,amsfonts,bm}

\def\eqref#1{equation~\ref{#1}}

\def\1{\bm{1}}

\DeclareMathAlphabet{\mathsfit}{\encodingdefault}{\sfdefault}{m}{sl}
\SetMathAlphabet{\mathsfit}{bold}{\encodingdefault}{\sfdefault}{bx}{n}

\def\gA{{\mathcal{A}}}
\def\gB{{\mathcal{B}}}
\def\gC{{\mathcal{C}}}
\def\gD{{\mathcal{D}}}
\def\gE{{\mathcal{E}}}

\def\gN{{\mathcal{N}}}

\def\gP{{\mathcal{P}}}
\def\gQ{{\mathcal{Q}}}
\def\gR{{\mathcal{R}}}
\def\gS{{\mathcal{S}}}
\def\gT{{\mathcal{T}}}
\def\gU{{\mathcal{U}}}
\def\gV{{\mathcal{V}}}
\def\gW{{\mathcal{W}}}
\def\gX{{\mathcal{X}}}
\def\gY{{\mathcal{Y}}}
\def\gZ{{\mathcal{Z}}}

\newcommand*{\ldblbrace}{\{\mskip-5mu\{}
\newcommand*{\rdblbrace}{\}\mskip-5mu\}}
\newcommand*{\Ldblbrace}{\left\{\mskip-7mu\left\{}
\newcommand*{\Rdblbrace}{\right\}\mskip-7mu\right\}}
\newcommand*{\dis}{{\operatorname{dis}}}
\newcommand*{\disR}{\operatorname{dis}^\mathsf{R}}
\newcommand*{\disH}{\operatorname{dis}^\mathsf{H}}
\newcommand*{\hash}{\mathsf{hash}}
\newcommand*{\op}{\mathsf{op}}
\newcommand*{\agg}{\mathsf{agg}}
\newcommand*{\walk}{\mathsf{walk}}
\newcommand*{\pool}{\mathsf{Pool}}
\newcommand*{\twist}{\mathsf{twist}}
\newcommand*{\meta}{\mathsf{Meta}}

\usepackage[textsize=tiny]{todonotes}

\title{\LARGE\normalfont\bfseries A Complete Expressiveness Hierarchy for Subgraph GNNs via Subgraph Weisfeiler-Lehman Tests}

\author{Bohang Zhang,\quad Guhao Feng\footnote{The two authors contributed equally. The order is determined by rolling a dice.} ,\quad Yiheng Du$^{*}$,\quad Di He,\quad Liwei Wang\\
\fontsize{11pt}{0pt}{\texttt {zhangbohang@pku.edu.cn}\qquad  \texttt{\{fenguhao,duyiheng\}@stu.pku.edu.cn}} \\
\fontsize{11pt}{0pt}{\texttt {dihe@pku.edu.cn}\qquad\qquad\qquad \texttt {wanglw@cis.pku.edu.cn}}\\
Peking University
}

\date{\vspace{10pt}\textbf{Abstract}\vspace{-35pt}}

\begin{document}
\maketitle
\begin{abstract}
Recently, subgraph GNNs have emerged as an important direction for developing expressive graph neural networks (GNNs). While numerous architectures have been proposed, so far there is still a limited understanding of how various design paradigms differ in terms of expressive power, nor is it clear what design principle achieves maximal expressiveness with minimal architectural complexity. To address these fundamental questions, this paper conducts a systematic study of general node-based subgraph GNNs through the lens of Subgraph Weisfeiler-Lehman Tests (SWL). Our central result is to build a \emph{complete hierarchy} of SWL with \emph{strictly} growing expressivity. Concretely, we prove that any node-based subgraph GNN falls into one of the six SWL equivalence classes, among which $\mathsf{SSWL}$ achieves the maximal expressive power. We also study how these equivalence classes differ in terms of their practical expressiveness such as encoding graph distance and biconnectivity. Furthermore, we give a tight expressivity upper bound of all SWL algorithms by establishing a close relation with \emph{localized} versions of WL and Folklore WL (FWL) tests. Our results provide insights into the power of existing subgraph GNNs, guide the design of new architectures, and point out their limitations by revealing an inherent gap with the 2-FWL test. Finally, experiments demonstrate that $\mathsf{SSWL}$-inspired subgraph GNNs can significantly outperform prior architectures on multiple benchmarks despite great simplicity.
\end{abstract}
\vspace{-10pt}

\section{Introduction}
Graph neural networks (GNNs), especially equivariant message-passing neural networks (MPNNs), have become the dominant approach for learning on graph-structured data \citep{gilmer2017neural,hamilton2017inductive,kipf2017semisupervised,velivckovic2018graph}. Despite their great simplicity and scalability, one major drawback of MPNNs lies in the limited expressiveness \citep{xu2019powerful,morris2019weisfeiler}. 
This motivated a variety of subsequent works to develop provably more expressive architectures, among which subgraph GNNs have emerged as a new trend \citep{cotta2021reconstruction,you2021identity,zhang2021nested,bevilacqua2022equivariant,zhao2022stars,papp2022theoretical,frasca2022understanding,qian2022ordered,huang2022boosting}. 

Broadly speaking, a general (node-based) subgraph GNN first transforms an input graph $G$ into a collection of subgraphs, each of which is associated with a unique node in $G$. It then computes a feature representation for each node of each subgraph through a series of equivariant message-passing layers. Finally, it outputs a representation of graph $G$ by pooling all these subgraph node features. Subgraph GNNs have received great attention partly due to their elegant structure, enhanced expressiveness, message-passing-based inductive bias, and superior empirical performance \citep{frasca2022understanding,zhao2022stars}.

\looseness=-1 One central question in subgraph GNNs lies in how to design simple yet expressive equivariant layers. Starting from the most basic design where each node only interacts with its local neighbors in the own subgraph \citep{cotta2021reconstruction,qian2022ordered}, recent works have developed a rich family of (cross-graph) aggregation operations \citep{bevilacqua2022equivariant,zhao2022stars,frasca2022understanding}. In particular, \citet{frasca2022understanding} gave a unified characterization of the design space of subgraph GNNs based on 2-IGN \citep{maron2019invariant,maron2019provably}, which contains \emph{dozens} of atomic aggregations. 
However, it is generally unclear whether the added aggregations can \emph{theoretically} improve a model's expressiveness as it becomes increasingly complex. So far, a systematic investigation and comparison of various possible aggregation schemes in terms of expressiveness is still lacking. More fundamentally, for both theory and practice, \emph{is there a canonical design principle of subgraph GNNs that achieves the maximal expressiveness with the least model complexity?}

\looseness=-1 \textbf{A complete hierarchy of subgraph GNNs}. In this paper, we comprehensively study the aforementioned questions through the lens of Subgraph Weisfeiler-Lehman Tests (SWL), a class of color refinement algorithms abstracted from subgraph GNNs in distinguishing non-isomorphic graphs. Each SWL consists of three ingredients: (a) graph generation policy, (b) message-passing aggregation scheme, and (c) final pooling scheme. Among commonly used graph generation policies, we mainly focus on the canonical \emph{node marking} SWL as it theoretically achieves the best expressive power despite simplicity (\cref{thm:node_marking}). Our central result is to build a complete hierarchy for all node marking SWL with various aggregation schemes and pooling schemes. Concretely, we prove that any node-based subgraph GNN falls into one of the six SWL equivalence classes and establish \emph{strict expressivity inclusion relationships} between different classes (see \cref{thm:swl_hierarchy,thm:separation,fig:hierarchy}). In particular, our result highlights that, by including symmetrically two basic local aggregations, the corresponding SWL (called $\mathsf{SSWL}$) has theoretically achieved the maximal expressive power. Our result thus provides a clear picture of the power and limitation of existing architectures, settling a series of open problems raised in \citet{bevilacqua2022equivariant,frasca2022understanding,qian2022ordered,zhao2022stars} (see \cref{sec:discussion} for discussions).

\looseness=-1 \textbf{Related to practical expressiveness}. 
We provide concrete evidence that subgraph GNNs with better theoretical expressivity are also stronger in terms of their ability to compute fundamental graph properties. 
Inspired by the recent work of \citet{zhang2023rethinking}, we prove that the $\mathsf{PSWL}$ (defined in \cref{thm:swl_hierarchy}) is \emph{strictly more powerful} than a variant of the Generalized Distance WL proposed in their paper, which incorporates both the \emph{shortest path distance} and the \emph{hitting time distance} (\cref{def:hitting_time_distance}). Our result \emph{unifies} and \emph{extends} the findings in \citet{zhang2023rethinking} and implies that all SWL algorithms stronger than $\mathsf{PSWL}$ are able to encode both \emph{distance} and \emph{biconnectivity} properties. In contrast, we give counterexamples to show that \emph{neither} of these basic graph properties can be fully encoded in vanilla SWL.

\textbf{Localized (Folklore) WL tests}. Similar to the classic WL and Folklore WL tests \citep{weisfeiler1968reduction,cai1992optimal}, node marking SWL corresponds to a natural class of computation models for graph canonization \citep{immerman1990describing}. All SWL algorithms have $O(n^2)$ memory complexity and $O(nm)$ computational complexity (per iteration) for a graph of $n$ vertices and $m$ edges. Owing to the improved computational efficiency over classic 2-FWL/3-WL (i.e. $O(n^3)$), a better understanding of what can/cannot be achieved under this complexity class is arguably an important research question. We answer this question by first establishing a close relation between SWL and \emph{localized} versions of 2-WL \citep{morris2020weisfeiler} and 2-FWL tests, both of which have the same complexity as SWL. We then derive a number of key results: $(\mathrm{i})$~The strongest $\mathsf{SSWL}$ is \emph{as powerful as} localized 2-WL. This builds a surprising link between the works of \citet{frasca2022understanding} and \citet{morris2020weisfeiler}. $(\mathrm{ii})$~Despite the same complexity, there is an \emph{inherent gap} between localized 2-WL and localized 2-FWL. $(\mathrm{iii})$~There is an \emph{inherent gap} between localized 2-FWL and classic 2-FWL. Consequently, our results settle a fundamental open problem raised in \citet{frasca2022understanding} about whether subgraph GNNs can match the power of 2-FWL, and further implies that \emph{subgraph GNNs even do not reach the maximal expressiveness in the model class of $O(nm)$ complexity}. This reveals an intrinsic limitation of the subgraph GNN model class and points out a new direction for improvement.

\looseness=-1 \textbf{Technical Contributions}. Actually, it is quite challenging to find a principled class of hard graphs that can reveal the expressivity gap of different SWL/FWL-type algorithms. As a main technical contribution, we develop a novel analyzing framework inspired by \citet{cai1992optimal} based on \emph{pebbling games}, where we considerably extend the game originally designed for FWL to all types of SWL and localized 2-WL/2-FWL algorithms. The game viewpoint offers deep insights into the power of different algorithms, through which we can skillfully construct a collection of nontrivial counterexample graphs to prove all strict separation results in this paper. We believe the proposed games and counterexamples may be of independent value in future work.

\textbf{Practical Contributions}. Our theoretical insights can also guide in designing simple, efficient, yet powerful subgraph GNN architectures. In particular, the proposed $\mathsf{SSWL}$ corresponds to an elegant design principle with only 3 atomic equivariant aggregation operations, yet the resulting model is strictly more powerful than all prior node-based subgraph GNNs. Empirically, we verify $\mathsf{SSWL}$-based subgraph GNNs on several benchmark datasets, showing that they can significantly outperform prior architectures despite fewer model parameters and great simplicity.

\section{Formalizing Subgraph GNNs}
\label{sec:subgraph_gnn}

\textbf{Notations}. We use $\{\ \}$ to denote sets and use $\ldblbrace\ \rdblbrace$ to denote multisets. The cardinality of (multi)set $\gS$ is denoted as $|\gS|$. In this paper, we consider finite, undirected, simple, \emph{connected} graphs, and we use $G=(\gV_G,\gE_G)$ to denote such a graph with vertex set $\gV_G$ and edge set $\gE_G$. Each edge in $\gE_G$ is expressed as a set $\{u,v\}$ containing two distinct vertices in $\gV_G$. Given a vertex $u$, denote its \emph{neighbors} as $\gN_G(u):=\{v\in\gV_G:\{u,v\}\in\gE_G\}$. Similarly, the \emph{$k$-hop neighbors} of $u$ is denoted as $\gN_G^k(u):=\{v\in\gV_G:\dis_G(u,v)\le k\}$, where $\dis_G(u,v)$ is the shortest path distance between $u$ and $v$. In particular, $\gN_G^1(u)=\gN_G(u)\cup\{u\}$. 

A general subgraph GNN processes an input graph $G$ following three steps: $(\mathrm{i})$ generating subgraphs, $(\mathrm{ii})$ equivariant message-passing, and $(\mathrm{iii})$ final pooling. Below, we separately describe each of these components.

\textbf{Node-based graph generation policies}.  The first step is to generate a collection of subgraphs of $G$ based on a predefined graph generation policy $\pi$ and initialize node features in each subgraph. For \emph{node-based} subgraph GNNs, there are a total of $|\gV_G|$ subgraphs, and each subgraph is uniquely associated with a specific node $u\in\gV_G$, so that $\pi$ can be expressed as a mapping of the form $\pi(G)=\ldblbrace(G^u,\tilde h_G^u):u\in\gV_G\rdblbrace$. Here, all subgraphs $G^u=(\gV_G, \gE_G^u)$ share the vertex set $\gV_G$ but may differ in the edge set $\gE_G^u$. The mapping $\tilde h_G^u:\gV_G\to\mathbb R^d$ defines the initial node features, i.e., $\tilde h_G^u(v)$ is the initial feature of vertex $v$ in subgraph $G^u$. 

Various graph generation policies have been proposed in prior works, which differ in the choice of $\gE_G^u$ and $\tilde h_G^u$. For example, common choices of $\gE_G^u$ are: $(\mathrm{i})$~using the original graph ($\gE_G^u=\gE_G$), $(\mathrm{ii})$~node deletion ($\gE_G^u=\gE_G\backslash\{\{u,v\}:v\in\gN_G(u)\}$, which deletes all edges associated to node $u$), and $(\mathrm{iii})$~$k$-hop ego network 
($\gE_G^u=\{\{v,w\}\in\gE_G:v,w\in\gN_G^k(u)\}$).
To initialize node features, there are also three popular choices: $(\mathrm{i})$~constant node features, where $\tilde h_G^u(v)$ is the same for all $u,v\in\gV_G$; $(\mathrm{ii})$~node marking, where $\tilde h_G^u(v)$ depends only on whether $u=v$ or not; $(\mathrm{iii})$~distance encoding, where $\tilde h_G^u(v)$ depends on the shortest path distance between $u$ and $v$, i.e. $\dis_G(u,v)$.

In this paper, we mainly consider the canonical \emph{node marking} policy on the original graph due to its simplicity. Importantly, we will show in \cref{sec:node_marking_swl} that it already achieves the maximal expressiveness among all the above policies.

\textbf{Equivariant message-passing}. The main backbone of subgraph GNNs consists of $L$ stacked equivariant message-passing layers. For each network layer $l\in[L]$, the feature of each node $v$ in each subgraph $G^u$ is computed, which can be denoted as $h_G^{(l)}(u,v)$. At the beginning, $h_G^{(0)}(u,v)=\tilde h_G^u(v)$. Following \citet{frasca2022understanding}, we study arguably the most general design space that incorporates a broad class of possible message-passing aggregation operations\footnote{We note that there are still other possible equivariant operations that are not included in \cref{def:layer}, such as diagonal aggregations (e.g., $\sum_{w\in\gV_G}h(w,w)$) and composite aggregations (e.g., $\sum_{w\in \gV_G}\sum_{x\in\gN_G(v)}h(w,x)$ used in $\mathsf{ESAN}$). In particular, \citet{frasca2022understanding} recently proposed a powerful subgraph GNN framework called $\mathsf{ReIGN(2)}$, which contains a total of 39 atomic operations for node-marking policy (the number can be even larger for other policies). However, we prove that incorporating these operations does not bring extra expressiveness beyond the current framework (see \cref{sec:other_aggregation}). Here, we select the 8 basic operations in \cref{def:layer} mainly due to their fundamental nature, simplicity, and completeness.}.

\begin{definition}
\label{def:layer}
    A general subgraph GNN layer has the form
    \begin{equation*}
        h_G^{(l+1)}(u, v) = \sigma^{(l+1)}(\op_1(u, v, G, h_G^{(l)}),\cdots ,\op_r(u, v, G, h_G^{(l)})),
    \end{equation*}
    where $\sigma^{(l+1)}$ is an arbitrary (parameterized) continuous function, and each atomic operation $\op_i(u,v,G,h)$ can take any of the following expressions:
    \begin{itemize}[topsep=0pt]
    \raggedright
    \setlength{\itemsep}{0pt}
        \item Single-point: $h(u,v)$, $h(v,u)$, $h(u,u)$, or $h(v,v)$;
        \item Global: $\sum_{w\in\gV_G}h(u,w)$ or $\sum_{w\in\gV_G}h(w,v)$;
        \item Local: $\sum_{w\in\gN_{G^u}(v)}h(u,w)$ or $\sum_{w\in\gN_{G^v}(u)}h(w,v)$.
    \end{itemize}
    We assume that $h(u,v)$ is always present in some $\op_i$.
\end{definition}

It is easy to see that any GNN layer defined above is permutation \emph{equivariant}. Among them, two most basic atomic operations are $h(u,v)$ and $\sum_{w\in\gN_{G^u}(v)}h(u,w)$, which are applied in all prior subgraph GNNs. Without using further operations, the \emph{vanilla subgraph GNN} layer has the form
\begin{equation*}
    h_G^{(l+1)}(u,v)=\sigma^{(l+1)}\left(h_G^{(l)}(u,v),\sum_{w\in\gN_{G^u}(v)}h_G^{(l)}(u,w)\right).
\end{equation*}
Besides, several works have explored other aggregation operations, and we list a few representative examples below.
\begin{example}
    $(\mathrm{i})$~$\mathsf{ESAN}$ \citep{bevilacqua2022equivariant} additionally uses global aggregation $\sum_{w\in\gV_G}h(w,v)$. $(\mathrm{ii})$~$\mathsf{GNN}\text{-}\mathsf{AK}$ \citep{zhao2022stars} additionally uses single-point operation $h(v,v)$. It also uses global aggregation $\sum_{w\in\gV_G}h(u,w)$ when $u=v$. $(\mathrm{iii})$~$\mathsf{SUN}$ \citep{frasca2022understanding} additionally uses $h(u,u)$, $h(v,v)$, and both types of global aggregations.
\end{example}
    

\textbf{Final pooling layer}. The last step is to output a graph representation $f(G)$ based on all the collected features $\ldblbrace h_G^{(L)}(u,v):u,v\in\gV_G\rdblbrace$. There are two different ways to implement this, which differ in the order of pooling along the two dimensions $u,v$. The first approach, called \emph{vertex-subgraph pooling}, first pools all node features in each subgraph $G^u$ to obtain the subgraph representation, i.e., $f^\mathsf{S}(G,u):=\sigma^\mathsf{S}\left(\sum_{v\in\gV_G}h_G^{(L)}(u,v)\right)$, and then pools all subgraph representations to obtain the final output $f(G):=\sigma^\mathsf{G}(\sum_{u\in\gV_G}f^\mathsf{S}(G,u))$. Here, $\sigma^\mathsf{S}$ and $\sigma^\mathsf{G}$ can be any parameterized function. Most prior works follow this paradigm. In contrast, the second approach, called \emph{subgraph-vertex pooling}, first generates node representations $f^\mathsf{V}(G,v):=\sigma^\mathsf{V}(\sum_{u\in\gV_G}h_G^{(L)}(u,v))$ for each $v\in\gV_G$, and then pools all these node representations to obtain the graph representation, i.e., $f(G):=\sigma^\mathsf{G}(\sum_{v\in\gV_G}f^\mathsf{V}(G,v))$. This approach is adopted in \citet{qian2022ordered}.

\section{Subgraph Weisfeiler-Lehman Test}
\label{sec:swl}

To formally study the expressive power of subgraph GNNs, in this section we introduce the Subgraph WL Test (SWL), a class of color refinement algorithms for graph isomorphism test. Let $G=(\gV_G,\gE_G)$ and $H=(\gV_H,\gE_H)$ be two graphs.
As with subgraph GNNs, SWL first generates for each graph a collection of subgraphs and initializes color mappings based on a graph generation policy $\pi$. We denote the results as $\ldblbrace (G^u,\tilde \chi_G^u):u\in\gV_G\rdblbrace$ and $\ldblbrace (H^u,\tilde \chi_H^u):u\in\gV_H\rdblbrace$, where $\tilde\chi$ is the color mapping that can be constant, node marking or distance encoding (according to \cref{sec:subgraph_gnn}).

Given graph $G$, let $\chi_G^{(0)}(u,v):=\tilde\chi_G^u(v)$ for $u,v\in\gV_G$. SWL then refines the color of each $(u,v)$ pair using various types of aggregation operations defined as follows:

\begin{definition}
\label{def:swl}
    A general SWL iteration has the form
    \begin{equation*}
        \chi_G^{(t+1)}(u, v) = \hash(\agg_1(u, v, G, \chi_G^{(t)}),\cdots,\agg_r(u, v, G,\chi_G^{(t)})),
    \end{equation*}
     where $\hash$ is a perfect hash function and each $\agg_i(u,v,G,\chi)$ can take any of the following expressions:
    \begin{itemize}[topsep=0pt]
    \raggedright
    \setlength{\itemsep}{0pt}
        \item Single-point: $\chi(u,v)$, $\chi(v,u)$, $\chi(u,u)$, or $\chi(v,v)$;
        \item Global: $\ldblbrace \chi(u,w):w\in\gV_G\rdblbrace$ or $\ldblbrace \chi(w,v):w\in\gV_G\rdblbrace$.
        \item Local: $\ldblbrace \chi(u,w):w\in\gN_{G^u}(v)\rdblbrace$ or $\ldblbrace \chi(w,v):w\in\gN_{G^v}(u)\rdblbrace$.
    \end{itemize}
    We use symbols $\agg^\mathsf{P}_\mathsf{uv}$, $\agg^\mathsf{P}_\mathsf{vu}$, $\agg^\mathsf{P}_\mathsf{uu}$, $\agg^\mathsf{P}_\mathsf{vv}$, $\agg^\mathsf{G}_\mathsf{u}$, $\agg^\mathsf{G}_\mathsf{v}$, $\agg^\mathsf{L}_\mathsf{u}$, and $\agg^\mathsf{L}_\mathsf{v}$ to denote each of the 8 basic operations, respectively.
    We assume $\agg^\mathsf{P}_\mathsf{uv}$ is always present in some $\agg_i$. The set $\gA:=\{\agg_i:i\in[r]\}$ fully determines the SWL iteration and is called the \emph{aggregation scheme}. 
\end{definition}
\vspace{-3pt}

For each iteration $t$, the color mapping $\chi_G^{(t)}$ induces an equivalence relation and thus a \emph{partition} $\gP_G^{(t)}$ over the set $\gV_G\times \gV_G$. Since $\agg^\mathsf{P}_\mathsf{uv}$ is present in $\gA$, $\gP_G^{(t)}$ must get \emph{refined} as $t$ grows. Therefore, with a sufficiently large number of iterations $t\le |\gV_G|^2$, the color mapping becomes stable (i.e., inducing a \emph{stable partition}). Without abuse of notation, we denote the stable color mapping by $\chi_G$.

Finally, the representation of graph $G$, denoted as $c(G)$, is computed by hashing all colors $\chi_G(u,v)$ for $u,v\in\gV_G$. Parallel to the previous section, there are two different \emph{pooling paradigms} to implement this:
\begin{itemize}[topsep=0pt]
    \raggedright
    \setlength{\itemsep}{0pt}
    \item Vertex-subgraph pooling (abbreviated as $\mathsf{VS}$): $c(G)=\hash\left(\ldblbrace\hash(\ldblbrace \chi_G(u,v):v\in\gV_G\rdblbrace):u\in\gV_G\rdblbrace\right)$;
    \item Subgraph-vertex pooling (abbreviated as $\mathsf{SV}$): $c(G)=\hash\left(\ldblbrace\hash(\ldblbrace \chi_G(u,v):u\in\gV_G\rdblbrace):v\in\gV_G\rdblbrace\right)$.
\end{itemize}

We say SWL can distinguish a pair of graphs $G$ and $H$ if $c(G)\neq c(H)$. Similarly, given a subgraph GNN $f$, we say $f$ distinguishes graphs $G$ and $H$ if $f(G)\neq f(H)$. The following proposition establishes the connection between SWL and subgraph GNNs in terms of expressivity in distinguishing non-isomorphic graphs. 

\begin{proposition}
\label{thm:swl_and_subgraph_gnn}
    The expressive power of any subgraph GNN defined in \cref{sec:subgraph_gnn} is bounded by a corresponding SWL by matching the graph generation policy $\pi$, the aggregation scheme (between \cref{def:layer,def:swl}), and the pooling paradigm. Moreover, when considering bounded-size graphs, for any SWL algorithm, there exists a matching subgraph GNN with the same expressive power.
\end{proposition}
\vspace{-3pt}

\citet{qian2022ordered} first proved the above result for \emph{vanilla} subgraph GNNs without cross-graph aggregations. Here, we consider general aggregation schemes and give a unified proof of \cref{thm:swl_and_subgraph_gnn} in \cref{sec:swl_eq_subgraph_gnn}. Based on this result, we can focus on studying the expressive power of SWL in subsequent analysis.

\section{Expressiveness and Hierarchy of SWL}
\label{sec:swl_analysis}

In this section, we systematically study how different design paradigms impact the expressiveness of SWL algorithms. To begin with, we need the following set of terminologies:

\begin{definition}
    Let $\mathsf{A}_1$ and $\mathsf{A}_2$ be two color refinement algorithms, and denote $c_i(G)$, $i\in \{1,2\}$ as the graph representation computed by $\mathsf{A}_i$ for graph $G$. We say:
    \begin{itemize}[topsep=0pt]
    \setlength{\itemsep}{0pt}
        \item $\mathsf{A}_1$ is \emph{more powerful} than $\mathsf{A}_2$, denoted as $\mathsf{A}_2\preceq \mathsf{A}_1$, if for any pair of graphs $G$ and $H$, $c_1(G)=c_1(H)$ implies $c_2(G)=c_2(H)$.
        \item $\mathsf{A}_1$ is \emph{as powerful as} $\mathsf{A}_2$, denoted as $\mathsf{A}_1\simeq\mathsf{A}_2$, if both $\mathsf{A}_1\preceq \mathsf{A}_2$ and $\mathsf{A}_2\preceq \mathsf{A}_1$ hold.
        \item $\mathsf{A}_1$ is \emph{strictly more powerful} than $\mathsf{A}_2$, denoted as $\mathsf{A}_2\prec \mathsf{A}_1$, if $\mathsf{A}_2\preceq \mathsf{A}_1$ and $\mathsf{A}_2\not\simeq \mathsf{A}_1$, i.e., there exist graphs $G$, $H$ such that $c_1(G)\neq c_1(H)$ and $c_2(G)=c_2(H)$.
        
        \item $\mathsf{A}_1$ and $\mathsf{A}_2$ are \emph{incomparable}, denoted as $\mathsf{A}_1 \nsim \mathsf{A}_2$, if neither $\mathsf{A}_1\preceq \mathsf{A}_2$ nor $\mathsf{A}_2\preceq \mathsf{A}_1$ holds.
    \end{itemize} 
\end{definition}
\vspace{-3pt}

\subsection{The canonical form: node marking SWL test}
\label{sec:node_marking_swl}

The presence of many different graph generation policies complicates our subsequent analysis. Interestingly, however, we show the simple \emph{node marking} policy (on the original graph) already achieves the maximal power among all policies considered in \cref{sec:subgraph_gnn} under mild assumptions.

\begin{proposition}
\label{thm:node_marking}
    Consider any SWL algorithm $\mathsf{A}$ that contains the two basic aggregations $\agg^\mathsf{P}_\mathsf{uv}$ and $\agg^\mathsf{L}_\mathsf{u}$ in \cref{def:swl}. Denote $\hat{\mathsf{A}}$ as the corresponding algorithm obtained from $\mathsf{A}$ by replacing the graph generation policy $\pi$ to node marking (on the original graph). Then, $\mathsf{A}\preceq\hat{\mathsf{A}}$.
\end{proposition}

We give a proof in \cref{sec:policy}, which is based on the following finding: when the special node mark is propagated by SWL with local aggregation, the color of each node pair $(u,v)$ can encode its distance $\dis_G(u,v)$ (\cref{thm:node_marking_implies_distance}), and the structure of $k$-hop ego network is also encoded.

Note that for the node marking policy, all subgraphs are just the original graph ($G^u=G$), which simplifies our analysis. We hence focus on the simple yet expressive node marking policy in subsequent sections. The following notations will be frequently used:
\begin{definition}
\label{def:node_marking_swl}
    Denote $\mathsf{A}(\gA,\pool)$ as the node marking SWL test with aggregation scheme $\gA\cup \{\agg^\mathsf{P}_{\mathsf{uv}}\}$ and pooling paradigm $\pool$, where $\pool\in\{\mathsf{VS},\mathsf{SV}\}$, and
    \begin{equation*}
    \setlength{\abovedisplayskip}{3pt}
    \setlength{\belowdisplayskip}{3pt}
        \gA\subset\{\agg^\mathsf{P}_{\mathsf{uu}}, \agg^\mathsf{P}_{\mathsf{vv}}, \agg^\mathsf{P}_{\mathsf{vu}}, \agg^\mathsf{G}_{\mathsf{u}}, \agg^\mathsf{G}_{\mathsf{v}}, \agg^\mathsf{L}_{\mathsf{u}}, \agg^\mathsf{L}_{\mathsf{v}}\}.
    \end{equation*}
    Here, we assume that $\agg^\mathsf{P}_{\mathsf{uv}}$ is always present in SWL.
\end{definition}

\subsection{Hierarchy of different algorithms}
\label{sec:swl_hierarchy}

As shown in \cref{def:node_marking_swl}, there are a large number of possible combinations of aggregation/pooling designs. In this subsection, we aim to build a complete hierarchy of SWL algorithms by establishing expressivity inclusion relations between different design paradigms. All proofs in this section are deferred to \cref{sec:proof_of_swl_hierarchy}.

We first consider the expressive power of different aggregation schemes. We have the following main theorem:
\begin{theorem}
\label{thm:aggregation}
    Under the notation of \cref{def:node_marking_swl}, for any $\gA$ and $\pool$, the following hold:
    \begin{itemize}[topsep=0pt]
    \raggedright
    \setlength{\itemsep}{0pt}
        \item $\mathsf{A}(\gA\cup\{\agg^\mathsf{G}_{\mathsf{u}}\},\pool)\preceq\mathsf{A}(\gA\cup\{\agg^\mathsf{L}_{\mathsf{u}}\},\pool)$ and $\mathsf{A}(\gA\cup\{\agg^\mathsf{L}_{\mathsf{u}}\},\pool)\simeq \mathsf{A}(\gA\cup\{\agg^\mathsf{L}_{\mathsf{u}},\agg^\mathsf{G}_{\mathsf{u}}\},\pool)$;
        \item $\mathsf{A}(\gA\cup\{\agg^\mathsf{P}_{\mathsf{uu}}\},\pool)\preceq\mathsf{A}(\gA\cup\{\agg^\mathsf{G}_{\mathsf{u}}\},\pool)$ and $\mathsf{A}(\gA\cup\{\agg^\mathsf{G}_{\mathsf{u}}\},\pool)\simeq\mathsf{A}(\gA\cup\{\agg^\mathsf{G}_{\mathsf{u}},\agg^\mathsf{P}_{\mathsf{uu}}\},\pool)$;
        \item $\mathsf{A}(\{\agg^\mathsf{L}_{\mathsf{u}},\agg^\mathsf{P}_{\mathsf{vu}}\},\pool)\simeq\mathsf{A}(\{\agg^\mathsf{L}_{\mathsf{u}},\agg^\mathsf{L}_{\mathsf{v}}\},\pool)\simeq\mathsf{A}(\{\agg^\mathsf{L}_{\mathsf{u}},\agg^\mathsf{L}_{\mathsf{v}},\agg^\mathsf{P}_{\mathsf{vu}}\},\pool)$.
    \end{itemize}
\end{theorem}
\cref{thm:aggregation} shows that local aggregation is more powerful than (and can \emph{express}) the corresponding global aggregation, while global aggregation is more powerful than (and can \emph{express}) the corresponding single-point aggregation. In addition, the ``transpose'' aggregation $\agg^\mathsf{P}_{\mathsf{vu}}$ is quite powerful: when combining a local aggregation $\agg^\mathsf{L}_{\mathsf{u}}$, it can express the other local aggregation $\agg^\mathsf{L}_{\mathsf{v}}$.

We next turn to the pooling paradigm. We first show that there is a symmetry (duality) between $u,v$ and the two types of pooling paradigms $\mathsf{VS}, \mathsf{SV}$.

\begin{proposition}
\label{thm:symmetry}
    Let $\gA$ be any aggregation scheme defined in \cref{def:node_marking_swl}. Denote $\gA^{\mathsf{u}\leftrightarrow\mathsf{v}}$ as the aggregation scheme obtained from $\gA$ by exchanging the element $\agg^\mathsf{P}_{\mathsf{uu}}$ with $\agg^\mathsf{P}_{\mathsf{vv}}$, exchanging $\agg^\mathsf{G}_{\mathsf{u}}$ with $\agg^\mathsf{G}_{\mathsf{v}}$, and exchanging $\agg^\mathsf{L}_{\mathsf{u}}$ with $\agg^\mathsf{L}_{\mathsf{v}}$. Then, $\mathsf{A}(\gA, \mathsf{VS})\simeq\mathsf{A}(\gA^{\mathsf{u}\leftrightarrow\mathsf{v}}, \mathsf{SV})$.
\end{proposition}

Based on the symmetry, one can easily extend \cref{thm:aggregation} to a variant that gives relations for $\agg^\mathsf{P}_\mathsf{vv}$, $\agg^\mathsf{G}_\mathsf{v}$, and $\agg^\mathsf{L}_\mathsf{v}$. Moreover, we have the following main theorem:
\begin{theorem}
\label{thm:pooling}
    Let $\gA$ be defined in \cref{def:node_marking_swl} with $\agg^\mathsf{L}_{\mathsf{u}}\in\gA$. Then, the following hold:
    \begin{itemize}[topsep=0pt]
    \setlength{\itemsep}{0pt}
        \item $\mathsf{A}(\gA,\mathsf{VS})\preceq\mathsf{A}(\gA,\mathsf{SV})$;
        \item If $\{\agg^\mathsf{G}_{\mathsf{v}},\agg^\mathsf{L}_{\mathsf{v}}\}\cap\gA\neq\emptyset$, then $\mathsf{A}(\gA,\mathsf{VS})\simeq\mathsf{A}(\gA,\mathsf{SV})$.
    \end{itemize}
\end{theorem}

\cref{thm:pooling} indicates that the subgraph-vertex pooling is always more powerful than the vertex-subgraph pooling, especially when the aggregation scheme is weak (e.g, the vanilla SWL). On the other hand, they become equally expressive for SWL with strong aggregation schemes.

Combined with the above three results, we have built a \emph{complete hierarchy} for the expressive power of all node marking SWL algorithms in \cref{def:node_marking_swl}. In particular, we show any SWL must fall into the following 6 types:
\begin{corollary}
\label{thm:swl_hierarchy}
    Let $\mathsf{A}(\gA,\pool)$ be any SWL defined in \cref{def:node_marking_swl} with at least one local aggregation, i.e. $\{\agg^\mathsf{L}_{\mathsf{u}},\agg^\mathsf{L}_{\mathsf{v}}\}\cap\gA\neq \emptyset$. Then, $\mathsf{A}(\gA,\pool)$ must be as expressive as one of the 6 SWL algorithms defined below: {\normalfont
    \begin{itemize}[topsep=0pt]
    \raggedright 
    \setlength{\itemsep}{0pt}
        \item (Vanilla SWL) $\mathsf{SWL(VS)}:=\mathsf{A}(\{\agg^\mathsf{L}_{\mathsf{u}}\},\mathsf{VS})$, $\mathsf{SWL(SV)}:=\mathsf{A}(\{\agg^\mathsf{L}_{\mathsf{u}}\},\mathsf{SV})$;
        \item (SWL with additional single-point aggregation) $\mathsf{PSWL(VS)}:=\mathsf{A}(\{\agg^\mathsf{L}_{\mathsf{u}},\agg^\mathsf{P}_{\mathsf{vv}}\},\mathsf{VS})$, $\mathsf{PSWL(SV)}:=\mathsf{A}(\{\agg^\mathsf{L}_{\mathsf{u}},\agg^\mathsf{P}_{\mathsf{vv}}\},\mathsf{SV})$;
        \item (SWL with additional global aggregation) $\mathsf{GSWL}:=\mathsf{A}(\{\agg^\mathsf{L}_{\mathsf{u}},\agg^\mathsf{G}_{\mathsf{v}}\},\mathsf{VS})$;
        \item (Symmetrized SWL) $\mathsf{SSWL}:=\mathsf{A}(\{\agg^\mathsf{L}_{\mathsf{u}},\agg^\mathsf{L}_{\mathsf{v}}\},\mathsf{VS})$.
    \end{itemize}}
    Moreover, we have
    \begin{gather*}
        \mathsf{SWL(VS)}\preceq\mathsf{SWL(SV)}\ \text{and}\ \mathsf{PSWL(VS)}\preceq\mathsf{PSWL(SV)},\\
        \mathsf{SWL(VS)}\preceq\mathsf{PSWL(VS)}\ \text{and}\ \mathsf{SWL(SV)}\preceq\mathsf{PSWL(SV)},\\
        \mathsf{PSWL(SV)}\preceq\mathsf{GSWL}\preceq\mathsf{SSWL}.
    \end{gather*}
\end{corollary}
\vspace{-5pt}

\cref{thm:swl_hierarchy} is significant in that it drastically reduces the problem of studying a large number of different SWL variants to the study of only 6 standard paradigms. Moreover, it implies that the simple $\mathsf{SSWL}$ already achieves the maximal expressive power among all SWL variants. A detailed discussion on how these standard paradigms relate to previously proposed subgraph GNNs will be made in \cref{sec:discussion}.

Yet, there are still two fundamental problems that are not answered in \cref{thm:swl_hierarchy}. \emph{First}, it remains unclear whether some SWL algorithm is \emph{strictly} more powerful than another. This question is particularly important for a better understanding of how global, local, and single-point aggregations vary in their expressive power brought to SWL.

\emph{Second}, a deep understanding of the \emph{limitation} of SWL algorithms is still open. While \citet{frasca2022understanding,qian2022ordered} recently discovered that the expressiveness of subgraph GNNs can be upper bounded by the standard 2-FWL (3-WL) test, it remains a mystery whether there is an inherent gap between 2-FWL and SWL (in particular, the strongest $\mathsf{SSWL}$). Note that the per-iteration complexity of SWL is $O(nm)$ for a graph of $n$ vertices and $m$ edges, which is remarkably lower than 2-FWL ($O(n^3)$ complexity), so it is reasonable to expect that 2-FWL is strictly more powerful. If this is the case, one may further ask: \emph{does SWL achieve the maximal power among all color refinement algorithms with complexity $O(nm)$}? We aim to fully address both the above fundamental questions in subsequent sections.

\section{Localized Folklore Weisfeiler-Lehman Tests}
\label{sec:lfwl}

\looseness=-1 In this section, we propose two novel types of WL algorithms based on the standard 2-dimensional Folklore Weisfeiler-Lehman test (2-FWL) \citep{weisfeiler1968reduction,cai1992optimal}, which turns out to be closely related to SWL. Recall that 2-FWL maintains a color for each vertex pair $(u,v)\in\gV_G\times\gV_G$. Initially, the color $\chi_G^{(0)}(u,v)$ depends on the isomorphism type of the subgraph induced by $(u,v)$, namely, depending on whether $u=v$, $\{u,v\}\in\gE_G$, or $\{u,v\}\notin\gE_G$. In each iteration $t$, the color is refined by the following update formula:
\begin{equation}
\label{eq:2fwl}
    \chi_G^{(t+1)}(u,v)=\hash(\chi_G^{(t)}(u,v),\walk(u,v,\gV_G,\chi_G^{(t)})),
\end{equation}
where we define 
\begin{equation}
\label{eq:walk}
    \walk(u,v,\gV,\chi):=\ldblbrace( \chi(u,w),\chi(w,v)):w\in\gV\rdblbrace.
\end{equation}
The color mapping $\chi_G^{(t)}$ stabilizes after a sufficiently large number of iterations $t\le |\gV_G|^2$. Denote the stable color mapping as $\chi_G$. 2-FWL finally outputs the graph representation $c(G):=\hash(\ldblbrace\chi_G(u,v):u,v\in\gV_G\rdblbrace)$.

One can see that each 2-FWL iteration has a complexity of $O(n^3)$ for a graph of $n$ vertices and $m$ edges, due to the need to enumerate all $w\in\gV_G$ for each pair $(u,v)$. For sparse graphs where $m=o(n^2)$, 2-FWL is inefficient and does not well-exploit the sparse nature of the graph. This inspires us to consider variants of 2-FWL that enumerate only the \emph{local neighbors}, such as $w\in\gN_G^1(v)$, by which the rich adjacency information is naturally incorporated in the \emph{update formula} (besides in the initial colors by the isomorphism type). We note that such an idea was previously explored in \citet{morris2020weisfeiler} (see \cref{sec:discussion} for further discussions). Importantly, the simple change substantially reduces the computational cost to $O(nm)$, which is the same as SWL. To this end, we define two novel FWL-type algorithms:

\begin{definition}
    Define $\mathsf{LFWL(2)}$ as the localized version of 2-FWL, which replaces $\gV_G$ in (\ref{eq:2fwl}) by $\gN_G^1(v)$. Define $\mathsf{SLFWL(2)}$ as the symmetrized version of $\mathsf{LFWL(2)}$, which replaces $\gV_G$ in (\ref{eq:2fwl}) by $\gN_G^1(u)\cup\gN_G^1(v)$. Finally, denote $\mathsf{FWL(2)}$ as the standard 2-FWL for consistency. 
\end{definition}

Note that $\mathsf{LFWL(2)}$ only exploits the local information of the vertex $v$, while $\mathsf{SLFWL(2)}$ uses all the local information of a vertex pair $(u,v)$ while still maintaining the $O(nm)$ cost. Therefore, one may expect that the latter is more powerful. Indeed, we have the following central result:

\begin{theorem}
\label{thm:2lfwl}
The following relations hold:
    \begin{itemize}[topsep=0pt]
    \setlength{\itemsep}{0pt}
        \item $\mathsf{LFWL(2)}\preceq\mathsf{SLFWL(2)}\preceq\mathsf{FWL(2)}$;
        \item $\mathsf{PSWL}(\mathsf{VS})\preceq\mathsf{LFWL(2)}$ and $\mathsf{SSWL} \preceq\mathsf{SLFWL(2)}$.
    \end{itemize}
\end{theorem}
\vspace{-3pt}

The proof is given in \cref{sec:proof_lfwl}. We now make several discussions regarding the significance of \cref{thm:2lfwl}. \emph{First}, $\mathsf{FWL(2)}$ is more powerful than its localized variants, confirming that there is indeed a trade-off between complexity and expressiveness. \emph{Second}, \cref{thm:2lfwl} reveals a close relationship between SWL and these localized 2-WL/2-FWL variants. In particular, $\mathsf{SLFWL(2)}$ is more powerful than all SWL algorithms \emph{despite the same computational cost}. Therefore, we obtain a \emph{tight} upper bound on the expressive power of subgraph GNNs with \emph{matching} complexity, which remarkably improves the previous 2-FWL upper bound \citep{frasca2022understanding,qian2022ordered}.

However, again, it is not known whether these localized 2-FWL variants are \emph{strictly more powerful} than SWL, nor do we know whether there is an intrinsic gap between 2-FWL and its localized variants. To thoroughly answer all of these questions, we need a new tool: the \emph{pebbling game}.

\section{Pebbling Game}
\label{sec:pebbling_game}
In this section, we develop a novel and unified analyzing framework for various SWL/FWL algorithms based on Ehrenfeucht-Fraïssé games \citep{ehrenfeucht1961application,fraisse}. The seminal paper of \citet{cai1992optimal} has used such games to prove the existence of counterexample graphs which $k$-FWL could not distinguish. Here, we vastly extend their result and show how pebbling games can be used to analyze \emph{all types} of SWL and localized FWL algorithms.

First consider any SWL algorithm $\mathsf{A}(\gA,\pool)$. The pebbling game is played on two graphs $G=(\gV_G,\gE_G)$ and $H=(\gV_H,\gE_H)$. Each graph is equipped with two different pebbles $u$ and $v$, both of which lie outside the graph initially. There are two players, the \emph{Spoiler} and the \emph{Duplicator}. To describe the game, we first introduce a basic game operation dubbed ``vertex selection''.

\begin{definition}[Vertex Selection]
\label{def:vertex_selection}
    \looseness=-1 Let $\gS_G\subset\gV_G$ and $\gS_H\subset\gV_H$ be given sets. Spoiler first freely chooses a non-empty subset $\gS^\mathsf{S}$ from either $\gS_G$ or $\gS_H$, and Duplicator should respond with a subset $\gS^\mathsf{D}$ from the other set, satisfying $|\gS^\mathsf{S}|=|\gS^\mathsf{D}|$. Duplicator loses the game if she has no feasible choice. Then, Spoiler can select any vertex $x^\mathsf{S}\in\gS^\mathsf{D}$, and Duplicator responds by selecting any vertex $x^\mathsf{D}\in\gS^\mathsf{S}$.
\end{definition}

\textbf{Initialization}. If $\pool=\mathsf{VS}$, the two players first select vertices $x^\mathsf{S}$ and $x^\mathsf{D}$ following the vertex selection procedure with $\gS_G=\gV_G$ and $\gS_H=\gV_H$. Spoiler places pebble $u$ on the selected vertex $x^\mathsf{S}$ and Duplicator places the other pebble $u$ on vertex $x^\mathsf{D}$. Next, Spoiler and Duplicator perform the vertex selection step again with $\gS_G=\gV_G$ and $\gS_H=\gV_H$ and place pebbles $v$ similarly. If $\pool=\mathsf{SV}$, the above procedure is analogous except that Spoiler/Duplicator places pebble $v$ in the first step and pebble $u$ in the second step. 

\textbf{Main loop}. The game then cyclically executes the following process. Depending on the SWL aggregation scheme $\gA$, Spoiler can freely choose one of the following ways to play:
\begin{itemize}[topsep=0pt]
    \setlength{\itemsep}{0pt}
    \item Local aggregation $\agg^\mathsf{L}_\mathsf{u}\in\gA$. Spoiler and Duplicator perform the vertex selection step with $\gS_G=\gN_G(v)$ and $\gS_H=\gN_H(v)$, where $\gN_G(v)$/$\gN_H(v)$ represents the set of vertices in graph $G$/$H$ adjacent to the vertex placed by pebble $v$. Spoiler moves pebble $v$ to the selected vertex $x^\mathsf{S}$, and Duplicator moves the other pebble $v$ to vertex $x^\mathsf{D}$.
    \item Global aggregation $\agg^\mathsf{G}_\mathsf{u}\in\gA$. Spoiler and Duplicator perform the vertex selection step with $\gS_G=\gV_G$ and $\gS_H=\gV_H$. Spoiler moves pebble $v$ to the selected vertex $x^\mathsf{S}$, and Duplicator moves the other pebble $v$ to vertex $x^\mathsf{D}$.
    \item \looseness=-1 Single-point aggregation $\agg^\mathsf{P}_\mathsf{uu}\in\gA$. Both players move pebble $v$ to the position of pebble $u$.
    \item Single-point aggregation $\agg^\mathsf{P}_\mathsf{vu}\in\gA$. Both players swap the position of pebbles $u$ and $v$.
\end{itemize}

The cases of $\agg^\mathsf{L}_\mathsf{v}$, $\agg^\mathsf{G}_\mathsf{v}$, $\agg^\mathsf{P}_\mathsf{vv}$ are similar (symmetric) to $\agg^\mathsf{L}_\mathsf{u}$, $\agg^\mathsf{G}_\mathsf{u}$, $\agg^\mathsf{P}_\mathsf{uu}$, so we omit them for clarity.

\looseness=-1 Spoiler wins the game if, after a certain round, the subgraph of $G$ induced by vertices placed by pebbles $u,v$ does not have the same isomorphism type as that of $H$. Duplicator wins the game if Spoiler cannot win after any number of rounds. Roughly speaking, Spoiler tries to find differences between graphs $G$ and $H$ using pebbles $u$ and $v$, while Duplicator strives to make these pebbles look the same in the two graphs. Our main result is stated as follows (see \cref{sec:proof_pebbling_game} for a proof):
\begin{theorem}
\label{thm:swl_pebble_game}
    Let $\mathsf{A}(\gA,\pool)$ be any SWL algorithm defined in \cref{def:node_marking_swl}, satisfying $\{\agg^\mathsf{L}_\mathsf{u},\agg^\mathsf{L}_\mathsf{v}\}\cap\gA\neq\emptyset$. Then, $\mathsf{A}(\gA,\pool)$ can distinguish a pair of graphs $G$ and $H$ if and only if Spoiler can win the corresponding pebbling game on graphs $G$ and $H$.
\end{theorem}

We next turn to FWL-type algorithms. The games are mostly similar to SWL but with a few subtle differences. There are also two pebbles $u,v$ for each graph. Here, the two players first places pebbles $u,v$ using just \emph{one} vertex selection step: Spoiler first chooses a non-empty subsets $\gS^\mathsf{S}$ from either $\gV_G\times \gV_G$ or $\gV_H\times \gV_H$, and  Duplicator should respond with a subset $\gS^\mathsf{D}$ from the other set, satisfying $|\gS^\mathsf{S}|=|\gS^\mathsf{D}|$. Then, Spoiler selects any vertex pair $(x^\mathsf{S}_\mathsf{u},x^\mathsf{S}_\mathsf{v})\in\gS^\mathsf{D}$, and Duplicator
responds by selecting $(x^\mathsf{D}_\mathsf{u},x^\mathsf{D}_\mathsf{v})\in\gS^\mathsf{S}$. Spoiler places pebbles $u$ and $v$ on $x^\mathsf{S}_\mathsf{u}$ and $x^\mathsf{S}_\mathsf{v}$, respectively. Duplicator places the other pebbles $u$ and $v$ on $x^\mathsf{D}_\mathsf{u}$ and $x^\mathsf{D}_\mathsf{v}$, respectively.

The game then cyclically executes the following process. First consider $\mathsf{LFWL(2)}$. In each round, the two players perform the vertex selection step with $\gS_G=\gN_G^1(v)$ and $\gS_H=\gN_H^1(v)$ and select vertices $x^\mathsf{S}$ and $x^\mathsf{D}$, respectively. Then it comes to the major difference from SWL: Spoiler can choose \emph{whether} to move pebble $u$ \emph{or} pebble $v$ to vertex $x^\mathsf{S}$, and Duplicator should move the same pebble in the other graph to $x^\mathsf{D}$. For $\mathsf{SLFWL(2)}$, the process is exactly the same as above except that the vertex selection is performed with $\gS_G=\gN_G^1(u)\cup\gN_G^1(v)$ and $\gS_H=\gN_H^1(u)\cup\gN_H^1(v)$. Finally, for the standard $\mathsf{FWL(2)}$, the vertex selection is performed with $\gS_G=\gV_G$ and $\gS_H=\gV_H$. Our main result is stated as follows (see \cref{sec:proof_pebbling_game} for a proof):

\begin{theorem}
\label{thm:fwl_pebble_game}
    $\mathsf{LFWL(2)}$/$\mathsf{SLFWL(2)}$/$\mathsf{FWL(2)}$ can distinguish a pair of graphs $G$ and $H$ if and only if Spoiler can win the corresponding pebbling game on graphs $G$ and $H$.
\end{theorem}

\cref{thm:swl_pebble_game,thm:fwl_pebble_game} build an interesting connection between WL algorithms and games. Importantly, the game viewpoint offers us a much clearer picture to sort out various complex aggregation/pooling paradigms and leads to the main result of this paper in the next section.

\section{Strict Separation Results}
\label{sec:separation_results}

Up to now, all results derived in this paper are of the form ``$\mathsf{A_1}\preceq\mathsf{A_2}$''. In this section, we will complete the analysis by proving that all relations $\preceq$ in \cref{thm:swl_hierarchy,thm:2lfwl} are actually the strict relations $\prec$. Formally, we will prove:

\begin{theorem}
\label{thm:separation}
    The following hold:
    \begin{itemize}[topsep=0pt]
    \raggedright
    \setlength{\itemsep}{0pt}
        \item $\mathsf{SWL}(\mathsf{VS})\prec\mathsf{SWL}(\mathsf{SV})$, $\mathsf{PSWL}(\mathsf{VS})\prec\mathsf{PSWL}(\mathsf{SV})$;
        \item $\mathsf{SWL}(\mathsf{VS})\prec\mathsf{PSWL}(\mathsf{VS})$, $\mathsf{SWL}(\mathsf{SV})\prec\mathsf{PSWL}(\mathsf{SV})$;
        \item $\mathsf{PSWL}(\mathsf{SV})\prec\mathsf{GSWL}\prec\mathsf{SSWL}$;
        \item $\mathsf{PSWL(VS)}\prec\mathsf{LFWL(2)}$, $\mathsf{SSWL}\prec\mathsf{SLFWL(2)}$;
        \item $\mathsf{LFWL(2)}\prec\mathsf{SLFWL(2)}\prec\mathsf{FWL(2)}$;
        \item $\mathsf{SWL}(\mathsf{SV})\nsim\mathsf{PSWL}(\mathsf{VS})$;
        \item $\mathsf{LFWL(2)}\nsim \mathsf{SWL(SV)}$, $\mathsf{LFWL(2)}\nsim \mathsf{PSWL(SV)}$, $\mathsf{LFWL(2)}\nsim \mathsf{GSWL}$, $\mathsf{LFWL(2)}\nsim \mathsf{SSWL}$.
    \end{itemize}
\end{theorem}
Due to space limitations, we can only present a brief proof sketch below, but we strongly encourage readers to browse the proof in \cref{sec:proof_separation}, where novel counterexamples for all these cases are constructed and analyzed using the pebbling game developed in \cref{sec:pebbling_game}. This is highly non-trivial and is a major technical contribution of this paper.

Our counterexamples are motivated by \citet{furer2001weisfeiler}. Given a base graph $F$, \citet{furer2001weisfeiler} gave a principled way to construct a pair of non-isomorphic but highly similar graphs $G(F)$ and $H(F)$ that cannot be distinguished by $k$-FWL. The key insight is that the difference between $G(F)$ and $H(F)$ is caused by a ``twist'' operation. (One can imagine the two graphs as a circle strip and its corresponding Möbius strip.) To distinguish the two graphs, Spoiler's only strategy is to fence out a twisted edge using his pebbles, similar to the strategy in Go. Yet, their analysis only applies to $k$-FWL algorithms. We considerably generalize F{\"u}rer's approach by noting that different SWL/FWL-type algorithms differ significantly in their ``surrounding'' capability in the pebbling game. Given two WL algorithms $\mathsf{A}_1,\mathsf{A}_2$ where we want to prove $\mathsf{A}_1\prec\mathsf{A}_2$, we can identify the extra surrounding capability of $\mathsf{A}_2$ and skillfully construct a base graph $F$ such that the extra power is \emph{necessary} to fence out a twisted edge. Here, the main challenge lies in constructing base graphs, which are given in \cref{fig:counterexample_pswlsv,fig:counterexample_swlvs,fig:counterexample_gswl,fig:counterexample_lfwl_sswl,fig:counterexample_slfwl,fig:counterexample_slfwl_2,fig:counterexample_fwl,fig:counterexample_pooling}. 

\begin{figure}[t]
    \centering
    \includegraphics[width=0.6\textwidth]{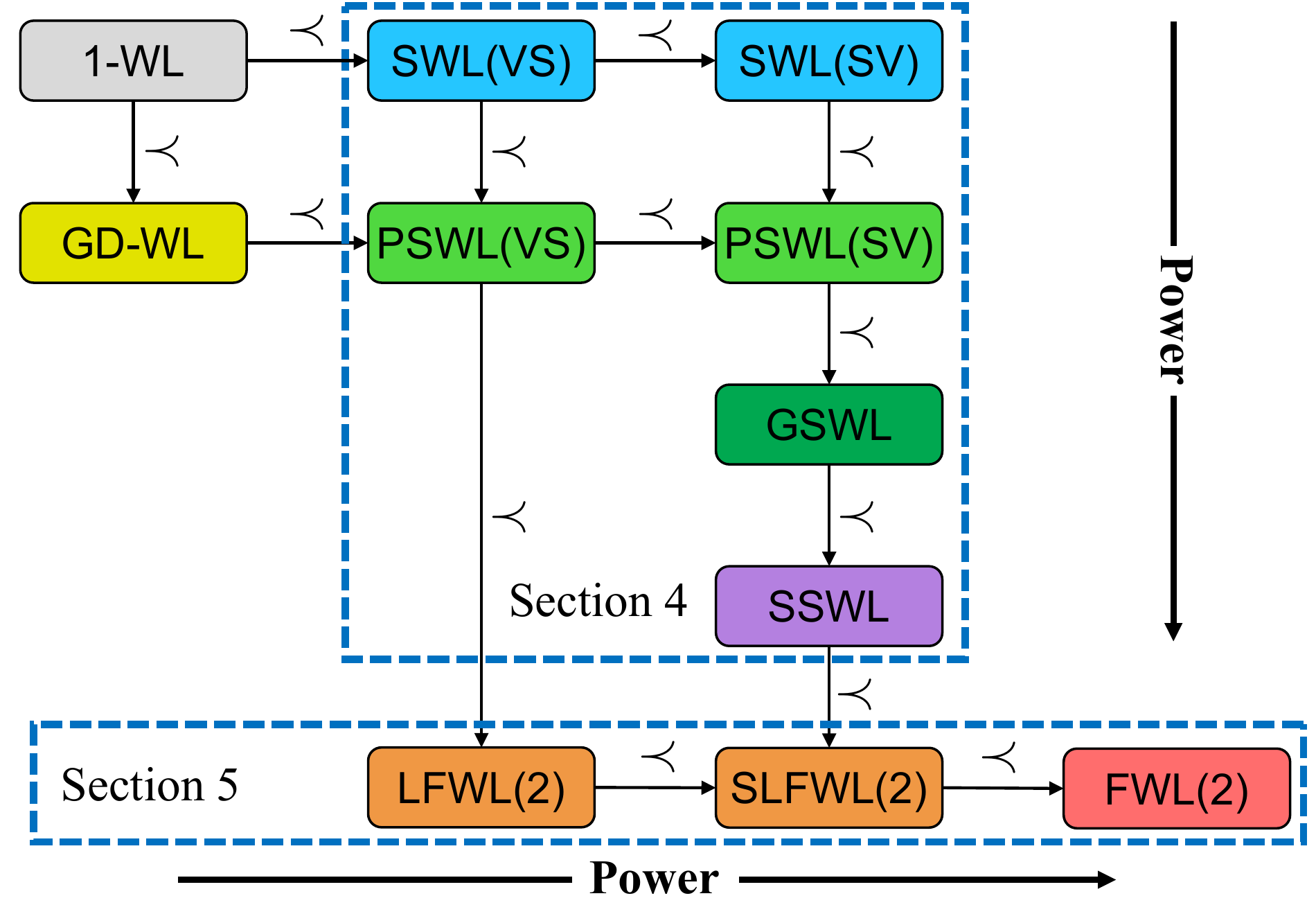}
    \caption{Expressiveness hierarchy of different WL algorithms.}
    \vspace{-5pt}
    \label{fig:hierarchy}
\end{figure}

In \cref{fig:hierarchy}, we give a clear illustration of the relationships between different SWL/FWL-type algorithms stated in \cref{thm:separation}, which forms a complete and elegant hierarchy. In the next section, we will give a detailed discussion of the significance of \cref{thm:separation} in the context of prior works.

\section{Discussions with prior works}
\label{sec:discussion}

The theoretical results in this paper can be directly used to analyze and compare the expressiveness of various subgraph GNNs in prior work. This is summarized in the following proposition:
\begin{proposition}
\label{thm:related_work_to_swl}
    Under the node marking policy, the following hold:
    \begin{itemize}[topsep=0pt]
    \setlength{\itemsep}{0pt}
        \item $\mathsf{ReconstructionGNN}$ \citep{cotta2021reconstruction}, $\mathsf{NGNN}$ \citep{zhang2021nested}, $\mathsf{IDGNN}$ \citep{you2021identity}, and $\mathsf{DS}\text{-}\mathsf{GNN}$ \citep{bevilacqua2022equivariant} are as expressive as  $\mathsf{SWL(VS)}$;
        \item $\mathsf{OSAN}$ \citep{qian2022ordered} is as expressive as $\mathsf{SWL(SV)}$;
        \item $\mathsf{GNN}\text{-}\mathsf{AK}$ \citep{zhao2022stars} is as expressive as $\mathsf{PSWL(VS)}$;
        \item $\mathsf{DSS}\text{-}\mathsf{GNN}$ (or $\mathsf{ESAN}$) \citep{bevilacqua2022equivariant}, $\mathsf{GNN}\text{-}\mathsf{AK}\text{-}\mathsf{ctx}$ \citep{zhao2022stars}, and $\mathsf{SUN}$ \citep{frasca2022understanding} are as expressive as $\mathsf{GSWL}$;
        \item $\mathsf{ReIGN(2)}$ \citep{frasca2022understanding} is as expressive as $\mathsf{SSWL}$.
    \end{itemize}
\end{proposition}

\begin{proof}
    The proof of $\mathsf{ReconstructionGNN}$, $\mathsf{NGNN}$, $\mathsf{IDGNN}$, $\mathsf{DS}\text{-}\mathsf{GNN}$, and $\mathsf{OSAN}$ follows by directly using \cref{thm:swl_hierarchy} since these subgraph GNNs fit our framework of \cref{def:layer}.  For other architectures such as $\mathsf{ESAN}$, $\mathsf{GNN}\text{-}\mathsf{AK}$, $\mathsf{GNN}\text{-}\mathsf{AK}\text{-}\mathsf{ctx}$, $\mathsf{SUN}$, and $\mathsf{ReIGN(2)}$, the proof can be found in \cref{sec:other_aggregation} (which is more involved as they use other atomic aggregations beyond \cref{def:layer}).
\end{proof}

\textbf{Regarding open problems in prior works.} Below, we show how our results can be used to settle a series of open problems raised before.
\begin{itemize}[topsep=0pt]
\setlength{\itemsep}{0pt}
    \item In \citet{bevilacqua2022equivariant}, the authors proposed two variants of WL algorithms, the $\mathsf{DS}\text{-}\mathsf{WL}$ and the $\mathsf{DSS}\text{-}\mathsf{WL}$. They conjectured that the latter is strictly more powerful than the former due to the introduced cross-graph aggregation. Very recently, \citet{zhang2023rethinking} gave the first evidence to this conjecture by proving that $\mathsf{DSS}\text{-}\mathsf{WL}$ can distinguish cut vertices using node colors while $\mathsf{DS}\text{-}\mathsf{WL}$ cannot. However, since identifying cut vertices is a \emph{node-level} task, it remains an open question when considering the standard \emph{graph-level} expressiveness, in particular, the task of distinguishing non-isomorphic graphs. Our result fully addressed the open question by showing that $\mathsf{DSS}\text{-}\mathsf{WL}$ is indeed strictly more powerful than $\mathsf{DS}\text{-}\mathsf{WL}$ in distinguishing non-isomorphic graphs.
    \item In \citet{zhao2022stars}, the authors proposed two GNN architectures: $\mathsf{GNN}\text{-}\mathsf{AK}$ and its extension $\mathsf{GNN}\text{-}\mathsf{AK}\text{-}\mathsf{ctx}$. $\mathsf{GNN}\text{-}\mathsf{AK}$ incorporates the so-called \emph{centroid encoding} and $\mathsf{GNN}\text{-}\mathsf{AK}\text{-}\mathsf{ctx}$ further incorporates the \emph{contextual encoding}. While the authors empirically showed the effectiveness of these encodings and found that $\mathsf{GNN}\text{-}\mathsf{AK}\text{-}\mathsf{ctx}$ can achieve much better performance on real-world tasks, they did not give a theoretical justification. Here, our result provides deep insights into the two models by indicating that $(\mathrm{i})$ with centroid encoding, $\mathsf{GNN}\text{-}\mathsf{AK}$ is strictly more powerful than vanilla subgraph GNNs; $(\mathrm{ii})$ with contextual encoding, $\mathsf{GNN}\text{-}\mathsf{AK}\text{-}\mathsf{ctx}$ is strictly more powerful than $\mathsf{GNN}\text{-}\mathsf{AK}$.
    \item Recently, \citet{qian2022ordered} proposed two classes of subgraph GNNs, the original $\mathsf{OSAN}$ and the vertex-subgraph $\mathsf{OSAN}$, which differ only in the final pooling paradigm. However, the authors did not discuss the relationship between the two types of architectures. Indeed, one may naturally guess that they have the same expressive power given the same GNN backbone. However, our result highlights that it is not the case: the original 1-$\mathsf{OSAN}$ is strictly more powerful than vertex-subgraph 1-$\mathsf{OSAN}$.
    \item Recently, \citet{frasca2022understanding} proposed a theoretically-inspired model called $\mathsf{ReIGN(2)}$, as well as a practical version called $\mathsf{SUN}$ that unifies prior node-based subgraph GNNs. The authors conjectured that these models are more powerful than prior architectures and may even match the power of 2-FWL. It is formally left as an important open problem to study the expressiveness lower bound of $\mathsf{ReIGN(2)}$ and $\mathsf{SUN}$ \citep[Appendix E]{frasca2022understanding}. In this paper, we fully settle the open problem by showing that: $(\mathrm{i})$ $\mathsf{ReIGN(2)}$ is indeed the strongest subgraph GNN model and is strictly more powerful than prior models; $(\mathrm{ii})$ However, $\mathsf{SUN}$ is just as powerful as the simpler $\mathsf{ESAN}$ although it incorporates many extra equivariant aggregation operations; $(\mathrm{iii})$ $\mathsf{ReIGN(2)}$ \emph{does not} achieve the 2-FWL expressiveness. Moreover, we point out an inherent gap between $\mathsf{ReIGN(2)}$ and 2-FWL, showing that $\mathsf{ReIGN(2)}$ even does not match $\mathsf{SLFWL(2)}$, a much weaker WL algorithm with the same complexity as $\mathsf{ReIGN(2)}$.
\end{itemize}

Finally, we note that \citet{frasca2022understanding} mentioned two basic atomic aggregations that are not included in prior subgraph GNNs: $\agg^\mathsf{L}_\mathsf{v}$ and $\agg^\mathsf{P}_\mathsf{vu}$ (see \cref{def:swl}). In this paper, we highlight that they are actually fundamental: incorporating either of them into the subgraph GNN layer can essentially improve the model's expressiveness.

\textbf{Discussions with \citet{morris2020weisfeiler}.} Our results also reveal a surprising relationship between the work of \citet{morris2020weisfeiler} and subgraph GNNs. In \citet{morris2020weisfeiler}, the authors proposed the so-called $\delta$-2-LWL, which can be seen as the symmetrized version of local 2-WL test. The update formula of $\delta$-2-LWL is written as follows:
\begin{equation}
\label{eq:delta-lwl}
    \chi_G^{(t+1)}(u,v)=\hash\left(\chi_G^{(t)}(u,v),\ldblbrace\chi_G^{(t)}(u,w):w\in\gN_G(v)\rdblbrace,\ldblbrace\chi_G^{(t)}(w,v):w\in\gN_G(u)\rdblbrace\right).
\end{equation}
An interesting finding is that $\delta$-2-LWL shares great similarities with $\mathsf{SSWL}$ in the update formula. Actually, while the two algorithms differ in the initial color and the final pooling paradigm, we can prove that $\delta$-2-LWL is \emph{as powerful as} $\mathsf{SSWL}$. We thus obtain the following key results:
\begin{itemize}
    \item Subgraph GNNs are also bounded by $\delta$-2-LWL. Moreover, the strongest subgraph GNN, such as $\mathsf{ReIGN(2)}$, matches the power of $\delta$-2-LWL. This builds an interesting link between the works of \citet{frasca2022understanding} and \citet{morris2020weisfeiler}.
    \item There is a fundamental gap between localized 2-WL and localized 2-FWL, despite the fact that both algorithms have the same computation/memory complexity. Such a result is perhaps surprising: it strongly contrasts to the relation between standard WL and FWL algorithms, where algorithms with equal computational complexity (e.g., $k$-FWL and $(k+1)$-WL) always have the \emph{same} expressive power. 
\end{itemize}

\section{Discussions on Practical Expressiveness}
\label{sec:practical_expressiveness}
Up to now, we have obtained precise expressivity relations for all pairs of SWL/FWL-type algorithms in distinguishing non-isomorphic graphs. From a practical perspective, however, one may still wonder whether/how GNNs designed based on a theoretically stronger WL algorithm can be more powerful in solving \emph{practical} graph problems. Here, we give concrete evidence that the power of different SWL algorithms does vary in terms of computing fundamental graph properties. In particular, WL algorithms with expressiveness over $\mathsf{PSWL}$ are capable of encoding \emph{distance} and \emph{biconnectivity} of a graph, while weaker algorithms like the vanilla SWL are unable to fully encode any of them.

Our result is motivated by the recent study of \citet{zhang2023rethinking}, who proposed a new class of WL algorithms called the Generalized Distance WL (GD-WL). Given graph $G=(\gV_G,\gE_G)$, GD-WL maintains a color $\chi_G(v)$ for each node $v\in\gV_G$, and the node color is updated according to the following formula:
\begin{equation*}
    \chi_G^{(t+1)}(v):= \hash(\ldblbrace (d_G(u,v), \chi_G^{(t)}(u)):u\in \gV\rdblbrace),
\end{equation*}
where $d_G(u,v)$ is a generalized distance between $u$ and $v$. \citet{zhang2023rethinking} proved that, by incorporating both the \emph{shortest path distance} (SPD) and the \emph{resistance distance} (RD), i.e., setting $d_G(u,v)=(\dis_G(u,v),\disR_G(u,v))$, the resulting GD-WL is provably expressive for all types of biconnectivity metrics, such as identifying cut vertices, cut edges, or distinguishing non-isomorphic graphs with different block cut trees. Surprisingly, we find that the $\mathsf{PSWL}$ has intrinsically (implicitly) encoded another type of GD-WL defined as follows:

\begin{definition}[Hitting time distance]
\label{def:hitting_time_distance}
    Define $\disH_G(u,v)$ to be the hitting time distance (HTD) from node $u$ to $v$ in graph $G$, i.e., the average number of edges passed in a random walk starting from $u$ and reaching $v$ for the first time.
\end{definition}
\begin{theorem}
\label{thm:pswl_gdwl}
    Let $d_G(u,v)=(\dis_G(u,v),\disH_G(u,v))$. Then, $\mathsf{GD}\text{-}\mathsf{WL}\prec\mathsf{PSWL}(\mathsf{VS})$.
\end{theorem}

Hitting time distance is closely related to resistance distance, in that $\disR_G(u,v)=(\disH_G(u,v)+\disH_G(v,u))/2|\gE_G|$ holds for any graph $G$ and nodes $u,v\in\gV_G$ \citep{chandra1996electrical}. In other words, RD can be seen as the symmetrized version of HTD (by ignoring the constant $1/|\gE_G|$). Moreover, we have the following theorem, showing that HTD-WL also resembles RD-WL in distinguishing vertex-biconnectivity:
\begin{theorem}
\label{thm:htd-wl}
    By setting $d_G=\disH_G$, the resulting HTD-WL is fully expressive for all vertex-biconnectivity metrics proposed in \citet{zhang2023rethinking}.
\end{theorem}
The proofs of \cref{thm:pswl_gdwl,thm:htd-wl} are given in \cref{sec:proof_practical_expressiveness}. Combining the two theorems readily leads to the following corollary:
\begin{corollary}
    $\mathsf{PSWL}(\mathsf{VS})$ is fully expressive for both edge-biconnectivity and vertex-biconnectivity.
\end{corollary}

On the other hand, we find that the vanilla SWL is unable to fully encode either SPD, HTD, or RD, as shown in the proposition below:

\begin{proposition}
    The following hold:
    \begin{itemize}[topsep=0pt]
    \setlength{\itemsep}{0pt}
        \item $\mathsf{SWL}(\mathsf{VS})\nsim\mathsf{SPD}\text{-}\mathsf{WL}$, $\mathsf{SWL}(\mathsf{SV})\nsim\mathsf{SPD}\text{-}\mathsf{WL}$;
        \item $\mathsf{SWL}(\mathsf{VS})\nsim\mathsf{HTD}\text{-}\mathsf{WL}$, $\mathsf{SWL}(\mathsf{SV})\nsim\mathsf{HTD}\text{-}\mathsf{WL}$;
        \item $\mathsf{SWL}(\mathsf{VS})\nsim\mathsf{RD}\text{-}\mathsf{WL}$,
    $\mathsf{SWL}(\mathsf{SV})\nsim\mathsf{RD}\text{-}\mathsf{WL}$.
    \end{itemize}
\end{proposition}

Besides, \citet{zhang2023rethinking} has shown that the $\mathsf{SWL}(\mathsf{VS})$ cannot identify cut vertices of a graph. Therefore, incorporating extra aggregation operations in the vanilla SWL \emph{does} essentially improve its practical expressiveness in computing basic graph properties like distance and biconnectivity.

\textbf{A further discussions with \citet{zhang2023rethinking}}. In \citet{zhang2023rethinking}, the authors showed that most prior GNN models are not expressive for biconnectivity metrics except $\mathsf{ESAN}$ \citep{bevilacqua2022equivariant}, which corresponds to $\mathsf{GSWL}$ in our framework. Here, we unify, justify, and extend their results/findings in the following aspects:
\begin{itemize}[topsep=0pt]
\setlength{\itemsep}{0pt}
    \item We show $\mathsf{ESAN}$ can identify cut vertices mainly because it encodes the generalized distance. This provides deep insights into $\mathsf{ESAN}$ and complements the finding that $\mathsf{ESAN}$ can encode SPD. From this perspective, we obtain an alternative and unified proof for $\mathsf{ESAN}$ in distinguishing vertex-biconnectivity. Moreover, we also prove that $\mathsf{ESAN}$ can distinguish the block cut-vertex tree, a new result that was not originally proved in \citet{zhang2023rethinking}.
    \item We strongly justify the introduced generalized distance and $\mathsf{GD}\text{-}\mathsf{WL}$ as a fundamental class of color refinement algorithms, since the reason why $\mathsf{ESAN}$ and other SWL variants can encode biconnectivity metrics simply lies in the fact that it is more powerful than $\mathsf{GD}\text{-}\mathsf{WL}$.
    \item In contrast, we prove that the weaker $\mathsf{SWL(VS)}$ (or $\mathsf{SWL(SV)}$) is not more powerful than either $\mathsf{SPD}\text{-}\mathsf{WL}$ or $\mathsf{RD}\text{-}\mathsf{WL}$. This explains and complements the finding in \citet{zhang2023rethinking} on why $\mathsf{DS}\text{-}\mathsf{WL}$ cannot identify cut vertices. We also partially answered the question in \citet{zhang2023rethinking} for $\mathsf{OS}\text{-}\mathsf{WL}$ \citep{qian2022ordered}.
    \item We show adding the global aggregation in vanilla SWL (like $\mathsf{ESAN}$) is not the only way to make it expressive for biconnectivity metrics. In particular, simply adding a single-point aggregation ($\mathsf{PSWL}$) already suffices.
\end{itemize}

\begin{remark}
    We suspect that the $\mathsf{PSWL}$ can also encode the resistance distance, but currently we can only prove that the strongest $\mathsf{SSWL}$ can encode RD (\cref{sec:proof_practical_expressiveness}). We leave this as an open problem for future work.
\end{remark}

\section{Experiments}
\label{sec:experiments}
Our theory also provides clear guidance in designing simple, efficient, yet powerful subgraph GNN architectures. In particular, we find that all the previously proposed practical node-based subgraph GNNs are bounded by $\mathsf{GSWL}$ (\cref{thm:related_work_to_swl}), which does not attain the maximal power in the SWL hierarchy. Instead, we decide to adopt the elegant, $\mathsf{SSWL}$-based subgraph GNN design principle, resulting in only 3 atomic equivariant aggregation operations; yet the corresponding model, called $\mathsf{GNN}\text{-}\mathsf{SSWL}$, is strictly more powerful than all prior node-based subgraph GNNs. We also design an extension of $\mathsf{GNN}\text{-}\mathsf{SSWL}$, denoted as $\mathsf{GNN}\text{-}\mathsf{SSWL}\text{+}$, by further incorporating the single-point aggregation $\agg^\mathsf{P}_\mathsf{vv}$ motivated by \cref{sec:practical_expressiveness}. While this does not improve the model's expressivity in theory, we find that it can often achieve better performance in real-world tasks. In addition, motivated by \cref{thm:node_marking}, the graph generation policy for both $\mathsf{GNN}\text{-}\mathsf{SSWL}$ and $\mathsf{GNN}\text{-}\mathsf{SSWL}\text{+}$ is chosen as the distance encoding on the original graph (which is as expressive as node marking). A detailed description of model configuration and training hyper-parameters is given in \cref{sec:exp_details}. Our code will be released at \texttt{\href{https://github.com/subgraph23/SWL}{https://github.com/subgraph23/SWL}}.

\begin{table*}[h]
\centering
\caption{Performance comparison of different GNN architectures on the Counting Substructure benchmark. We report the Mean Absolute Error (MAE), and use different background colors to distinguish different levels of MAE.}
\setlength{\tabcolsep}{4pt}
\small
\label{tab:exp_substructure}
\vspace{2pt}
\begin{tabular}{ll|cccccc}
\Xhline{1pt}
Model      & Reference                       & Triangle & Tailed Tri. & Star & 4-Cycle & 5-Cycle & 6-Cycle \\ \Xhline{0.75pt}
PPGN	   & \citet{maron2019provably}       & \cellcolor{cyan!20} 0.0089 & \cellcolor{cyan!20} 0.0096 & \cellcolor{green!20} 0.0148 & \cellcolor{cyan!20} 0.0090 & \cellcolor{green!20} 0.0137 & \cellcolor{green!20} 0.0167 \\
GNN-AK     & \citet{zhao2022stars}           & \cellcolor{red!20} 0.0934 & \cellcolor{red!20} 0.0751 & \cellcolor{green!20} 0.0168 & \cellcolor{red!20} 0.0726 & \cellcolor{red!20} 0.1102 & \cellcolor{red!20} 0.1063 \\
GNN-AK+    & \citet{zhao2022stars}           & \cellcolor{green!20} 0.0123 & \cellcolor{green!20} 0.0112 & \cellcolor{green!20} 0.0150 & \cellcolor{green!20} 0.0126 & \cellcolor{orange!20} 0.0268 & \cellcolor{red!20} 0.0584 \\
SUN (EGO+) & \citet{frasca2022understanding} & \cellcolor{cyan!20} 0.0079 & \cellcolor{cyan!20} 0.0080 & \cellcolor{cyan!20} \textbf{0.0064} & \cellcolor{green!20} 0.0105 & \cellcolor{green!20} 0.0170 & \cellcolor{red!20} 0.0550 \\ \hline
GNN-SSWL   & This paper                      & \cellcolor{cyan!20} 0.0098 & \cellcolor{cyan!20} 0.0090 & \cellcolor{cyan!20} 0.0089 & \cellcolor{green!20} \cellcolor{green!20} 0.0107 & \cellcolor{green!20} 0.0142 & \cellcolor{green!20} 0.0189 \\
GNN-SSWL+  & This paper                      & \cellcolor{cyan!20} \textbf{0.0064} & \cellcolor{cyan!20} \textbf{0.0067} & \cellcolor{cyan!20} 0.0078 & \cellcolor{cyan!20} \textbf{0.0079} & \cellcolor{green!20} \textbf{0.0108} & \cellcolor{green!20} \textbf{0.0154} \\ \Xhline{1pt}
\end{tabular}
\end{table*}

\textbf{Performance on Counting Substructure Benchmark}. Following \citet{zhao2022stars,frasca2022understanding}, we first consider the synthetic task of counting substructures. The result is presented in \cref{tab:exp_substructure}. It can be seen that our proposed models can solve all tasks almost completely and performs better than all prior node-based subgraph GNNs on most substructures, such as triangle, tailed triangle, 4-cycle, 5-cycle, and 6-cycle. In particular, our proposed models significantly outperform $\mathsf{GNN}\text{-}\mathsf{AK}\text{+}$ and $\mathsf{SUN}$ for counting 6-cycles. We suspect that $\mathsf{GSWL}$ is not expressive for counting 6-cycle while $\mathsf{SSWL}$ is expressive for this task, which may highlight a fundamental advantage of $\mathsf{SSWL}$ in practical scenarios when the ability to count 6-cycle is needed \citep{huang2022boosting}.

\begin{table*}[h]
\centering
\caption{Performance comparison of different subgraph GNNs on ZINC benchmark. The Mean Absolute Error (MAE) and the standard deviation are reported. We also list the WL equivalence class and the number of parameters/atomic aggregations for each model.}
\setlength{\tabcolsep}{4pt}
\small
\label{tab:exp_zinc}
\vspace{2pt}
\begin{tabular}{ll|lcc|cc}
\Xhline{1pt}
\multirow{2}{*}{Model} & \multirow{2}{*}{Reference} & \multirow{2}{*}{WL} & \multirow{2}{*}{\begin{tabular}[c]{@{}c@{}}\#\\Param.\end{tabular}} & \multirow{2}{*}{\begin{tabular}[c]{@{}c@{}}\#\\Agg.\end{tabular}} & \multicolumn{2}{c}{ZINC Test MAE} \\ 
 &  &  &  &  & Subset & Full \\ \Xhline{0.75pt}
GSN         & \citet{bouritsas2022improving} & -              & $\sim$500k & - & 0.101±0.010 & - \\
CIN (small) & \citet{bodnar2021cellular} & -              & $\sim$100k & - & 0.094±0.004 & 0.044±0.003 \\ \hline
SAN         & \citet{kreuzer2021rethinking} & -               & 509k     & - & 0.139±0.006 & - \\
K-Subgraph SAT & \citet{chen2022structure} & -               & 523k        & - & 0.094±0.008 & - \\
Graphormer & \citet{ying2021transformers} & $\mathsf{SPD}\text{-}\mathsf{WL}$ & 489k & -    & 0.122±0.006 & 0.052±0.005\\
URPE       & \citet{luo2022your}          & $\mathsf{SPD}\text{-}\mathsf{WL}$  & 492k & -    & 0.086±0.007 & 0.028±0.002\\
Graphormer-GD & \citet{zhang2023rethinking} & $\mathsf{GD}\text{-}\mathsf{WL}$ & 503k & -    & 0.081±0.009 & 0.025±0.004\\
GPS        & \citet{rampasek2022recipe} & -                & 424k & -    & \textbf{0.070±0.004} & - \\ \hline
NGNN       & \citet{zhang2021nested} & $\mathsf{SWL(VS)}$  & $\sim$500k & 2    & 0.111±0.003 & 0.029±0.001  \\
GNN-AK     & \citet{zhao2022stars}   & $\mathsf{PSWL(VS)}$ & $\sim$500k & 4    & 0.105±0.010 & - \\
GNN-AK+    & \citet{zhao2022stars}   & $\mathsf{GSWL}$     & $\sim$500k & 5    & 0.091±0.002 & - \\
ESAN       & \citet{bevilacqua2022equivariant} &  $\mathsf{GSWL}$     & $\sim$100k & 4    & 0.102±0.003 & 0.029±0.003 \\
ESAN       & \citet{frasca2022understanding} & $\mathsf{GSWL}$    & 446k   & 4    & 0.097±0.006 & 0.025±0.003    \\
SUN        & \citet{frasca2022understanding} & $\mathsf{GSWL}$    & 526k   & 12   & 0.083±0.003 & 0.024±0.003   \\\hline
GNN-SSWL   & This paper              & $\mathsf{SSWL}$     & 274k & 3    & 0.082±0.003 & 0.026±0.001  \\
GNN-SSWL+  & This paper              & $\mathsf{SSWL}$     & 387k & 4    & \textbf{0.070±0.005} & \textbf{0.022±0.002} \\ \Xhline{1pt}

\end{tabular}
\end{table*}

\textbf{Performance on ZINC benchmark}. We then validate our proposed models on the ZINC molecular property prediction benchmark \citep{dwivedi2020benchmarking}, a standard and widely-used task in the GNN community. We consider both ZINC-subset (12K selected graphs) and ZINC-full (250k graphs) and comprehensively compare our models with three types of baselines: $(\mathrm{i})$ subgraph GNNs, $(\mathrm{ii})$ substructure-based GNNs including GSN \citep{bouritsas2022improving} and CIN \citep{bodnar2021cellular}, and $(\mathrm{iii})$ latest strong baselines based on Graph Transformers. In particular, Graphormer-GD \citep{zhang2023rethinking} and GPS \citep{rampasek2022recipe} are two representative Graph Transformers that achieve state-of-the-art performance on ZINC benchmark. 

The result is presented in \cref{tab:exp_zinc}. \emph{First}, it can be observed that our proposed $\mathsf{GNN}\text{-}\mathsf{SSWL}$ already matches/outperforms the performance of all subgraph GNN baselines while being much simpler. In particular, compared with state-of-the-art $\mathsf{SUN}$ architecture, $\mathsf{GNN}\text{-}\mathsf{SSWL}$ requires only a quarter of atomic aggregations in each GNN layer and roughly half of the parameters, yet matches the performance of $\mathsf{SUN}$ on ZINC-subset. \emph{Second}, by further incorporating $\agg^\mathsf{P}_\mathsf{vv}$, $\mathsf{GNN}\text{-}\mathsf{SSWL}\text{+}$ significantly surpasses all subgraph GNN baselines and achieves state-of-the-art performance on both tasks. Note that our training time is also significantly faster than Graphormer-GD and GPS. \emph{Finally}, an interesting finding is that the performance of different subgraph GNN architectures shown in \cref{tab:exp_zinc} roughly aligns with their theoretical expressivity in the SWL hierarchy. This may further justify that designing theoretically more powerful subgraph GNNs can benefit real-world tasks as well.

\textbf{Other tasks}. We also conduct experiments on the OGBG-molhiv dataset \citep{hu2020open}. Due to space limit, the result is presented in \cref{sec:ogbg}.

\section{Related Work}
\label{sec:related_work}

Since \citet{xu2019powerful,morris2019weisfeiler} discovered the limited expressiveness of vanilla MPNNs, a large amount of works have been devoted to developing GNNs with better expressive power. Here, we briefly review the literature on expressive GNNs that are most relevant to this paper.

\textbf{Higher-order GNNs}. \citet{maron2019provably,maron2019universality,azizian2021expressive,geerts2022expressiveness} theoretically studied the question of designing provably expressive equivariant GNNs that match the power of $k$-FWL test for $k>1$. In this way, they build a hierarchy of GNNs with strictly growing expressivity (similar to this paper). A representative higher-order GNN architecture is called the $k$-IGN \citep{maron2019universality}: it stores a feature representation for each node $k$-tuple and updates these features using higher-order equivariant layers developed in \citet{maron2019invariant}. Recently, \citet{frasca2022understanding} proved that all node-based subgraph GNNs can be implemented by 3-IGN, which then implies that subgraph GNNs' expressive power is intrinsically bounded by 2-FWL \citep{geerts2022expressiveness}.

\textbf{Sparsity-aware GNNs}. One major drawback of higher-order GNNs is that the architectural design does not well-exploit the graph structural information, since the graph adjacency is only encoded in the initial node features. In light of this, subsequent works like \citet{morris2020weisfeiler,morris2022speqnets,zhao2022practical} incorporated this inductive bias directly into the network layers and designed \emph{local} versions of higher-order GNNs. For example, \citet{morris2020weisfeiler} developed the so-called $\delta$-$k$-LWL, which can be seen as a localized version of $k$-WL. \citet{morris2022speqnets} proposed the $(k,s)$-SpeqNets by considering only $k$-tuples whose vertices can be grouped into no more than $s$ connected components. \citet{zhao2022practical} concurrently proposed the $(k,s)$-SETGNN which is similar to $(k,s)$-SpeqNets. In this paper, we propose a class of localized $k$-FWL, which shares interesting similarities to $\delta$-$k$-LWL. Our major contribution is to establish complete relations between localized 2-FWL, $\delta$-2-LWL, and subgraph GNNs. 

\textbf{Subgraph GNNs}. Subgraph GNNs are an emerging class of higher-order GNNs that compute a feature representation for each subgraph-node pair. The earliest idea of subgraph GNNs may track back to \citet{cotta2021reconstruction,papp2021dropgnn}, which proposed to use node-deleted subgraphs and performed message-passing on each subgraph separately without cross-graph interaction. \citet{papp2022theoretical} argued to use node marking instead of node deletion for better expressive power. \citet{zhang2021nested} proposed the Nested GNN ($\mathsf{NGNN}$), a variant of subgraph GNNs that use $k$-hop ego nets with distance encoding. It further added the global aggregation $\agg^\mathsf{G}_\mathsf{u}$ to merge all node information in a subgraph when computing the feature of the root node of a subgraph. \citet{you2021identity} designed the $\mathsf{ID}\text{-}\mathsf{GNN}$, which is similar to $\mathsf{NGNN}$ and also uses $k$-hop ego nets as subgraphs. \citet{bevilacqua2022equivariant} developed a principled class of subgraph GNNs, called $\mathsf{ESAN}$, which first introduced the cross-graph global aggregation into the network design. \citet{zhao2022stars} concurrently proposed the $\mathsf{GNN}\text{-}\mathsf{AK}$ and its extension $\mathsf{GNN}\text{-}\mathsf{AK}\text{-}\mathsf{ctx}$, which also includes the cross-graph global aggregation. Recently, \citet{frasca2022understanding,qian2022ordered} first provided theoretical analysis of various node-based subgraph GNNs by proving that they are intrinsically bounded by 2-FWL. We note that besides node-based subgraph GNNs, one can also develop edge-based subgraph GNNs, which have been explored in \citet{bevilacqua2022equivariant,huang2022boosting}. Both works showed that the expressive power of edge-based subgraph GNNs can go beyond 2-FWL. Finally, we note that \citet{vignac2020building} proposed a GNN architecture that is somewhat \emph{similar} to the vanilla subgraph GNN, and the $\delta$-2-LWL proposed in \citet{morris2020weisfeiler} can also be seen as a subgraph GNN according to \cref{sec:discussion}.

\textbf{Practical expressivity of GNNs}. Another line of works sought to develop expressive GNNs from practical consideration. For example, \citet{furer2017combinatorial,chen2020can,arvind2020weisfeiler} studied the power of WL algorithms in counting graph substructures and pointed out that vanilla MPNNs cannot count/detect cycles, which may severely limit their practical performance in real-world tasks (e.g., in bio-chemistry). In light of this, \citet{bouritsas2022improving,barcelo2021graph} proposed to incorporate substructure counting (or homomorphism counting) into the initial node features to boost the expressiveness. \citet{bodnar2021topological,bodnar2021cellular} further proposed a message-passing framework that enables interaction between nodes, edges, and higher-order substructures. \citet{huang2022boosting} studied the cycle counting power of subgraph GNNs and proposed the I$^2$-GNN to count cycles of length no more than 6. Recently, \citet{puny2023equivariant} studied the expressive power of GNNs in expressing/approximating equivariant graph polynomials. They showed that computing equivariant polynomials generalizes the problem of counting substructures.

\looseness=-1 Besides cycle counting, several works explored other aspects of encoding basic graph properties. \citet{you2019position,li2020distance,ying2021transformers} proposed to use distance encoding to boosting the expressiveness of MPNNs or Graph Transformers. In particular, \citet{li2020distance} proposed to use a generalized distance called page-rank distance. \citet{balcilar2021breaking,kreuzer2021rethinking,lim2022sign} studies the expressive power of GNNs from the perspective of graph spectral \citep{cvetkovic1997eigenspaces}. Recently, \citet{zhang2023rethinking} discovered that most prior GNN architectures are not expressive for graph biconnectivity and built an interesting relation between biconnectivity and generalized distance. Here, we extend \citet{zhang2023rethinking} by giving a comprehensive characterization of which SWL equivalence class can encode distance and biconnectivity.

\section{Conclusion}

This paper gives a comprehensive and unified analysis of the expressiveness of subgraph GNNs. By building a complete expressiveness hierarchy, one can gain deep insights into the power and limitation of various prior works and guide designing more powerful GNN architectures. On the theoretical side, we reveal close relations between SWL, localized WL, and localized Folklore WL, and propose a unified analyzing framework via pebbling games. Our results address a series of previous open problems and highlight several new research directions to this field. On the practical side, we design a simple yet powerful subgraph GNN architecture that achieves superior performance on real-world tasks.

\subsection{Open directions}
\label{sec:open_problem}

We highlight several open directions for future work as follows.

\textbf{Regarding higher-order subgraph GNNs}. From a theoretical perspective, it is an interesting direction to generalize the results of this paper to \emph{higher-order} subgraph GNNs (which compute a feature representation for each node $k$-tuple). We note that such an idea has appeared in \citet{cotta2021reconstruction,qian2022ordered,papp2022theoretical}. However, none of these works explored the possible design space of \emph{cross-graph} aggregations. Since our results imply that these cross-graph aggregations do essentially improve the expressive power, it may be worthwhile to establish a complete hierarchy of higher-order subgraph GNNs. This may include the following questions: $(\mathrm{i})$~How many expressivity equivalence classes are there? $(\mathrm{ii})$~What are the expressivity inclusion relations between different equivalence classes? $(\mathrm{iii})$~\emph{What design principle achieves the maximal expressive power with a minimal number of atomic aggregations?} We conjecture that, by symmetrically incorporating $k$ local aggregations, the resulting $k$-order subgraph GNN achieves the maximal expressiveness and is as expressive as $\mathsf{ReIGN}(k)$ (by extending \citet{frasca2022understanding}).
    
\textbf{Regarding edge-based subgraph GNNs}. Another different perspective is to study edge-based subgraph GNNs, which compute a feature representation for each edge-node pair. Importantly, edge-based subgraph GNNs and the corresponding SWL are also a fundamental class of computation models with $O(nm)$ memory complexity and $O(m^2)$ computation complexity. For sparse graphs (i.e. $m=O(n)$), such a complexity is quite desirable and is close to that of node-based subgraph GNNs. Yet, it results in enhanced expressiveness: as shown in \citet{bevilacqua2022equivariant,huang2022boosting}, their proposed edge-based subgraph GNNs are not less powerful than 2-FWL. Therefore, we believe that characterizing the expressiveness hierarchy of edge-based subgraph GNNs is of both theoretical and practical interest. Another interesting topic is to build expressivity relations between node-based and edge-based subgraph GNNs.

\textbf{Regarding localized Folklore WL tests}. This paper proposed a novel class of color refinement algorithms called localized Folklore WL. Importantly, we show $\mathsf{SLFWL(2)}$ is strictly more powerful than all node-based subgraph GNNs despite the same complexity. Therefore, an interest question is whether we can design practical GNN architectures based on $\mathsf{SLFWL(2)}$ for both efficiency and better expressiveness. On the other hand, from a theoretical side, it may also be an interesting direction to study higher-order localized Folklore WL tests, in particular, $\mathsf{SLFWL}(k)$, due to its fundamental nature. We conjecture that $\mathsf{SLFWL}(k)$ is strictly more powerful than $\delta$-$k$-LWL \citep{morris2020weisfeiler} and strictly less powerful than standard $k$-FWL. Furthermore, \emph{does $\mathsf{SLFWL}(k)$ achieve the maximal expressive power among the algorithm class within $O(n^{k-1} m)$ computation cost?}

\textbf{Regarding practical expressiveness of $\mathsf{GSWL}$ and $\mathsf{SSWL}$}. This paper discusses the practical expressiveness of subgraph GNNs by showing an inherent gap between $\mathsf{SWL}$ and $\mathsf{PSWL}$ in terms of their ability to encode distance and biconnectivity of a graph. Yet, it remains an open problem how $\mathsf{PSWL}$, $\mathsf{GSWL}$, and $\mathsf{SSWL}$ differ in terms of their practical expressiveness for computing graph properties. This question is particularly important since recently proposed subgraph GNNs are typically bounded by $\mathsf{GSWL}$. Answering this question will thus highlight the power and limitation of prior architectures.

\newpage
\bibliography{ref}

\clearpage
\appendix

\renewcommand \thepart{} 
\renewcommand \partname{}
\part{Appendix} 
The Appendix is organized as follows:
\begin{itemize}[topsep=0pt]
    \setlength{\itemsep}{0pt}
    \item In \cref{sec:swl_eq_subgraph_gnn}, we give the missing proof of \cref{thm:swl_and_subgraph_gnn}, showing the equivalence between SWL and Subgraph GNNs.
    \item In \cref{sec:proof_of_swl_hierarchy}, we give all the missing proofs in \cref{sec:swl_analysis}, which we use to build a complete hierarchy of SWL algorithms. This part is technical and is divided into several subsections (from \cref{sec:proof_of_swl_hierarchy_preliminary,sec:policy,sec:proof_aggregation_schemes,sec:proof_pooling_paradigm,sec:proof_hierarchy_main}).
    \item In \cref{sec:other_aggregation}, we discuss several subgraph GNNs beyond our proposed framework (\cref{def:layer}), including $\mathsf{GNN}\text{-}\mathsf{AK}$, $\mathsf{GNN}\text{-}\mathsf{AK}\text{-}\mathsf{ctx}$, $\mathsf{ESAN}$, $\mathsf{SUN}$, and $\mathsf{ReIGN(2)}$. We show that each of these architectures still corresponds to an equivalent SWL algorithm in terms of expressive power.
    \item In \cref{sec:proof_lfwl}, we give the missing proof of \cref{thm:2lfwl}, showing the expressivity relationships between different SWL and localized FWL algorithms.
    \item In \cref{sec:proof_pebbling_game}, we give all the missing proofs in \cref{sec:pebbling_game}, bridging SWL/FWL-type algorithms and pebbling games.
    \item In \cref{sec:proof_separation}, we give the missing proof of \cref{thm:separation}. The proof is non-trivial and contains the main technical contribution of this paper. It is divided into three parts in \cref{sec:proof_separation_part1,sec:pebbling_furer,sec:proof_separation_part3} for readability.
    \item In \cref{sec:proof_practical_expressiveness}, we give all the missing proofs in \cref{sec:practical_expressiveness}, showing how various SWL algorithms differ in terms of their practical expressiveness such as encoding graph distance and biconnectivity.
    \item In \cref{sec:exp_details}, we provide experimental details to reproduce the results in \cref{sec:experiments}, as well as a comprehensive set of ablation studies.
\end{itemize}
\newpage

\section{The Equivalence between SWL and Subgraph GNNs}
\label{sec:swl_eq_subgraph_gnn}
This section aims to prove \cref{thm:swl_and_subgraph_gnn}. We restate the proposition below:

\textbf{\cref{thm:swl_and_subgraph_gnn}.} \emph{The expressive power of any subgraph GNN defined in \cref{sec:subgraph_gnn} is bounded by a corresponding SWL by matching the policy $\pi$, the aggregation scheme between \cref{def:layer,def:swl}, and the pooling paradigm. Moreover, when considering bounded-size graphs, for any SWL algorithm, there exists a matching subgraph GNN with the same expressive power.}

\begin{proof}
    We first prove that any subgraph GNN defined in \cref{sec:subgraph_gnn} is bounded by a corresponding SWL. To do this, we will prove the following result: given a pair of graphs $G=(\gV_G,\gE_G)$ and $H=(\gV_H,\gE_H)$, for any $t\in\mathbb N$ and any vertices $u,v\in\gV_G$ and $x,y\in\gV_H$, $\chi_G^{(t)}(u,v)=\chi_H^{(t)}(x,y) \implies h_G^{(t)}(u,v)=h_H^{(t)}(x,y)$, where $\chi$ and $h$ are defined in \cref{def:swl,def:layer}, respectively.

    We prove the result by induction over $t$. For the base case of $t=0$, the result clearly holds when the graph generation policy is the same between SWL and subgraph GNNs. Now assume the result holds for all $t\le T$, and we want to prove that it also holds for $t=T+1$. By \cref{def:swl}, $\chi_G^{(T+1)}(u,v)=\chi_H^{(T+1)}(x,y)$ is equivalent to
    $$\agg_i(u,v,G,\chi_G^{(T)})=\agg_i(x,y,H,\chi_H^{(T)}), \quad \forall i\in[r].$$
    We separately consider each type of aggregation operation:
    \begin{itemize}[topsep=0pt]
    \raggedright
    \setlength{\itemsep}{0pt}
        \item Single-point aggregation. Take $\agg^\mathsf{P}_\mathsf{vu}$ for example: $\agg^\mathsf{P}_\mathsf{vu}(u,v,G,\chi_G^{(T)})=\agg^\mathsf{P}_\mathsf{vu}(x,y,H,\chi_H^{(T)})$ implies $\chi_G^{(T)}(v,u)=\chi_H^{(T)}(y,x)$. By induction, we have $h_G^{(t)}(v,u)=h_H^{(t)}(y,x)$.
        \item Local aggregation. Take $\agg^\mathsf{L}_\mathsf{u}$ for example: $\agg^\mathsf{L}_\mathsf{u}(u,v,G,\chi_G^{(T)})=\agg^\mathsf{L}_\mathsf{u}(x,y,H,\chi_H^{(T)})$ implies
        \begin{align*}
            \ldblbrace\chi_G^{(T)}(u,w): w\in\gN_{G^u}(v)\rdblbrace=\ldblbrace\chi_H^{(T)}(x,z): z\in\gN_{H^x}(y)\rdblbrace.
        \end{align*}
        By induction, it is straightforward to see that
        \begin{align*}
            \ldblbrace h_G^{(T)}(u,w): w\in\gN_{G^u}(v)\rdblbrace=\ldblbrace h_H^{(T)}(x,z): z\in\gN_{H^x}(y)\rdblbrace.
        \end{align*}
        Therefore, $$\sum_{w\in\gN_{G^u}(v)}h_G^{(T)}(u,w)=\sum_{z\in\gN_{H^x}(y)}h_H^{(T)}(x,z).$$
        \item Global aggregation. This case is similar to the above one and we omit it for clarity.
    \end{itemize}
    Combining all these cases, we have
    $$\op_i(u,v,G,\chi_G^{(T)})=\op_i(x,y,H,\chi_H^{(T)}) \quad \forall i\in[r],$$
    and thus $h_G^{(T+1)}(u,v)=h_H^{(T+1)}(x,y)$. We have completed the induction step.

    Let $L$ be the number of layers in a subgraph GNN. Then $\chi_G^{(L)}(u,v)=\chi_H^{(L)}(x,y)$ implies $h_G^{(L)}(u,v)=h_H^{(L)}(x,y)$. Since $\agg^\mathsf{P}_\mathsf{uv}$ is always present, the stable color mapping $\chi$ satisfies that $\chi_G(u,v)=\chi_H(x,y)\implies\chi_G^{(L)}(u,v)=\chi_H^{(L)}(x,y)$. Therefore, $\chi_G(u,v)=\chi_H(x,y)\implies h_G^{(L)}(u,v)=h_H^{(L)}(x,y)$.

    Finally consider the pooling paradigm. As in the analysis of global aggregation, it can be concluded that $c(G)=c(H)$ implies $f(G)=f(H)$ where $c(G)$ and $f(G)$ represent the graph representation computed by SWL and subgraph GNN, respectively. We have finished the first part of the proof.

    It remains to prove that for any SWL algorithm, there exists a matching subgraph GNN with the same expressive power. The key idea is to ensure that whenever $\chi_G^{(t)}(u,v)\neq\chi_H^{(t)}(x,y)$, we have $h_G^{(t)}(u,v)\neq h_H^{(t)}(x,y)$. To achieve this, we rely on injective functions that take a set as input. When assuming that the size of the set is bounded, the injective property can be easily constructed using the approach proposed in \citet{maron2019provably}, called the power-sum multi-symmetric polynomials (PMP). We note that while \citet{maron2019provably} only focused on the case when the input belongs to sets of a \emph{fixed} size, it can be easily extended to our case for sets of different but bounded sizes by padding zero-elements.
    The \emph{summation} in PMP just coincides with the aggregation in \cref{def:layer}, and the power can be extracted by the function $\sigma^{(t)}$ in the previous layer. For more details, please refer to \citet{maron2019provably}.

    Finally, note that when the input graph has bounded size $N$, the SWL iteration must get stabled in no more than $N^2$ steps. Therefore, by using a sufficiently deep GNN (i.e., $L=N^2$), one can guarantee that $\chi_G(u,v)\neq\chi_H(x,y)$ implies $h_G^{(L)}(u,v)\neq h_H^{(L)}(x,y)$. This eventually yields that $c(G)\neq c(H)$ implies $f(G)\neq f(H)$, as desired.
\end{proof}

\section{Proof of Theorems in \cref{sec:swl_analysis}}
\label{sec:proof_of_swl_hierarchy}

This section contains all the missing proofs in \cref{sec:swl_analysis}.

\subsection{Preliminary}
\label{sec:proof_of_swl_hierarchy_preliminary}

We first introduce some basic terminologies and facts, which will be frequently used in subsequent proofs.

\begin{definition}
\label{def:finer}
    Let $\chi$ and $\tilde \chi$ be two color mappings, with $\chi_G(u,v)$ and $\tilde \chi_G(u,v)$ representing the color of vertex pair $(u,v)$ in graph $G$. We say:
    \begin{itemize}[topsep=0pt]
    \setlength{\itemsep}{0pt}
        \item $\tilde \chi$ is \emph{finer} than $\chi$, denoted as $\tilde \chi \preceq \chi$, if for any two graphs $G=(\gV_G,\gE_G)$, $H=(\gV_H,\gE_H)$ and any vertices $u,v\in\gV_G$, $x,y\in\gV_H$, we have $\tilde \chi_G(u,v)=\tilde\chi_H(x,y)\implies\chi_G(u,v)=\chi_H(x,y)$.
        \item $\tilde \chi$ and $\chi$ are equivalent, denoted as $\tilde\chi\simeq\chi$, if $\tilde \chi \preceq \chi$ and $\chi \preceq \tilde \chi$.
        \item $\tilde \chi$ is \emph{strictly finer} than $\chi$, denoted as $\tilde \chi \prec \chi$, if $\tilde \chi \preceq \chi$ and $\tilde\chi\not\simeq\chi$.
    \end{itemize}
\end{definition}

\begin{remark}
\label{remark:finer}
    Several simple facts regarding this definition are as follows.
    \begin{enumerate}[label=(\alph*),topsep=0pt]
        \setlength{\itemsep}{0pt}
        \item For any color refinement algorithm, let $\{\chi^{(t)}\}_{t=0}^\infty$ be the sequence of color mappings generated at each iteration $t$, then $\chi^{(t+1)}\preceq \chi^{(t)}$ for any $t$. This is exactly why we call the algorithm ``color refinement''. As a result, the stable color mapping $\chi$ is finer than any intermediate color mapping $\chi^{(t)}$.
        \item \cref{def:finer} is closely related to the power of WL algorithms. Indeed, let $\chi^\mathsf{A}$ and $\chi^\mathsf{B}$ be two stable color mappings generated by algorithms $\mathsf A$ and $\mathsf B$, respectively. If both algorithms use the same pooling paradigm, then $\chi^\mathsf{B}\preceq \chi^\mathsf{A}$ implies that $\mathsf B$ is more powerful than $\mathsf A$, i.e. $\mathsf A\preceq\mathsf B$.
        \item Consider two SWL algorithms $\mathsf{A}$ and $\mathsf{B}$ with the same graph generation policy, but with different aggregation schemes $\gA$ and $\gB$. Denote $\chi^\mathsf{A}$ and $\chi^\mathsf{B}$ as the corresponding stable color mappings. Define a new color mapping $\tilde\chi=\mathscr{T}(\gA,\chi^\mathsf{B})$ that ``refines'' $\chi^\mathsf{B}$ using aggregation scheme $\gA\cup\{\agg^\mathsf{P}_\mathsf{uv}\}$:
        \begin{equation}
        \label{eq:mapping_transform}
        \begin{aligned}
            [\mathscr{T}(\gA,\chi)]_G(u,v) = \hash(\agg_1(u, v, G,\chi_G),\cdots,\agg_r(u, v, G,\chi_G)),
        \end{aligned}
        \end{equation}
        where $\gA\cup\{\agg^\mathsf{P}_\mathsf{uv}\}=\{\agg_i:i\in[r]\}$. Then we can prove that $\chi^\mathsf{B}\preceq\tilde\chi\implies \chi^\mathsf{B}\preceq \chi^\mathsf{A}$. If the two algorithms further share the same pooling paradigm, \cref{remark:finer}(b) yields $\mathsf{A}\preceq\mathsf{B}$. This gives a simple way to compare the expressiveness of different algorithms.
    \end{enumerate}
\end{remark}
\begin{proof}[Proof of \cref{remark:finer}(c)]
    Define a sequence of color mappings $\{\tilde\chi^{(t)}\}_{t=0}^\infty$ recursively, such that $\tilde\chi^{(0)}=\chi^\mathsf{B}$ and
    $$\tilde\chi_G^{(t+1)}(u, v) = \hash(\agg_1(u, v,  G,\tilde\chi_G^{(t)}),\cdots,\agg_r(u, v, G,\tilde\chi_G^{(t)}))$$
    for any $u,v$ in graph $G$. Clearly, $\tilde\chi^{(1)}$ is just $\tilde\chi$ in \cref{remark:finer}(c). Since we have both $\tilde\chi^{(0)}\preceq \tilde\chi^{(1)}$ (by the assumption of $\chi^\mathsf{B}\preceq\tilde\chi$) and $\tilde\chi^{(1)}\preceq \tilde\chi^{(0)}$ (by \cref{remark:finer}(a)), $\tilde\chi^{(1)}\simeq \tilde\chi^{(0)}$. Therefore, $\tilde\chi^{(t)}\simeq\chi^\mathsf{B}$ holds for all $t\in\mathbb N$. On the other hand, a simple induction over $t$ yields $\tilde\chi^{(t)}\preceq \chi^{\mathsf{A},(t)}$ (where $\chi^{\mathsf{A},(t)}$ is the color mapping at iteration $t$ for algorithm $\mathsf{A}$), since they are both refined by the same aggregation scheme $\gA$ (for the base case, $\tilde\chi^{(0)}=\chi^\mathsf{B}\preceq\chi^{\mathsf{B},(0)}=\chi^{\mathsf{A},(0)}$). By taking $t\to\infty$, this finally yields $\tilde\chi\preceq \chi^{\mathsf{A}}$, namely, $\chi^\mathsf{B}\preceq \chi^\mathsf{A}$, as desired.
\end{proof}

\subsection{Discussions on graph generation policies}
\label{sec:policy}

In this subsection, we make a detailed discussion regarding graph generalization policies and prove that the canonical node marking policy already achieves the best expressiveness among all of these policies (\cref{thm:node_marking}).

Depending on the choice of $(G^u,h_G^u)$, there are a total of 7 non-trivial combinations. For ease of presentation, we first give symbols to each of them:
\begin{itemize}[topsep=0pt]
    \setlength{\itemsep}{0pt}
    \item $\mathsf{NM}$: the node marking policy on the original graph;
    \item $\mathsf{DE}$: the distance encoding policy on the original graph;
    \item ${\mathsf{EGO}(k)}$: the $k$-hop ego network policy with constant node features;
    \item ${\mathsf{EGO}(k)\text{+}\mathsf{NM}}$: the policy with both node marking and $k$-hop ego network;
    \item ${\mathsf{EGO}(k)\text{+}\mathsf{DE}}$: the policy with both distance encoding and $k$-hop ego network;
    \item $\mathsf{ND}$: the node deletion policy with constant node features;
    \item $\mathsf{NDM}$: the policy with both node deletion and marking.
\end{itemize}

\begin{proposition}
\label{thm:policy_expand}
    Consider any fixed aggregation scheme $\gA$ that contains the two basic aggregations $\agg^\mathsf{P}_\mathsf{uv}$ and $\agg^\mathsf{L}_\mathsf{u}$ in \cref{def:swl}, and consider any fixed pooling paradigm $\pool$ defined in \cref{sec:swl}. We have the following result:
    \begin{itemize}[topsep=0pt]
    \setlength{\itemsep}{0pt}
        \item $\mathsf{NM}$ is as powerful as $\mathsf{DE}$;
        \item ${\mathsf{EGO}(k)\text{+}\mathsf{NM}}$ is as powerful as ${\mathsf{EGO}(k)\text{+}\mathsf{DE}}$;
        \item $\mathsf{NM}$ is more powerful than ${\mathsf{EGO}(k)\text{+}\mathsf{NM}}$;
        \item $\mathsf{NM}$ is more powerful than $\mathsf{NDM}$;
        \item ${\mathsf{EGO}(k)\text{+}\mathsf{NM}}$ is more powerful than ${\mathsf{EGO}(k)}$;
        \item $\mathsf{NDM}$ is more powerful than $\mathsf{ND}$.
    \end{itemize}
\end{proposition}
\begin{proof}
    Let $\chi^{\mathsf{Policy},(t)}$ be the color mapping of the SWL algorithm with graph generation policy $\mathsf{Pocily}$, aggregation scheme $\gA$, and pooling paradigm $\pool$ at iteration $t$, and let $\chi^{\mathsf{Policy}}$ be the corresponding stable color mapping. Here, $\mathsf{Pocily}\in\{\mathsf{NM}, \mathsf{DE}, \mathsf{EGO}(k), \mathsf{EGO}(k)\text{+}\mathsf{NM}, \mathsf{EGO}(k)\text{+}\mathsf{DE}, \mathsf{ND}, \mathsf{NDM}\}$. 

    We first consider the case $\mathsf{NM}$ vs. $\mathsf{DE}$. By definition, $\chi^{\mathsf{DE},(0)}$ is finer than $\chi^{\mathsf{NM},(0)}$. Since the subgraphs $\ldblbrace G^u:u\in\gV_G\rdblbrace$ are the same for the two policies and the aggregation scheme $\gA$ is fixed, a simple induction over $t$ then yields $\chi^{\mathsf{DE},(t)}\preceq\chi^{\mathsf{NM},(t)}$ for any $t\in\mathbb N$, namely, $\chi^{\mathsf{DE}}\preceq\chi^{\mathsf{NM}}$. It follows that $\mathsf{DE}$ is more powerful than $\mathsf{NM}$ (by \cref{remark:finer}(b)).

    To prove the converse direction, we leverage \cref{thm:node_marking_implies_distance} (which will be proved later). \cref{thm:node_marking_implies_distance} implies that $\chi^{\mathsf{NM},(D)}\preceq \chi^{\mathsf{DE},(0)}$ when the input graphs have bounded diameter $D$. Then using the same analysis as above, we have $\chi^{\mathsf{NM},(D+t)}\preceq \chi^{\mathsf{DE},(t)}$ for any $t\in\mathbb N$. By taking $t\to\infty$, this implies that $\chi^{\mathsf{NM}}\preceq\chi^{\mathsf{DE}}$, namely, $\mathsf{NM}$ is more powerful than $\mathsf{DE}$ (by \cref{remark:finer}(b)). Combining the two directions concludes the proof of the first bullet.

    The proof for the case $\mathsf{EGO}(k)\text{+}\mathsf{NM}$ vs. $\mathsf{EGO}(k)\text{+}\mathsf{DE}$ is almost the same, so we omit it for clarity.

    We next turn to the case $\mathsf{NM}$ vs. $\mathsf{EGO}(k)\text{+}\mathsf{NM}$. Initially, by definition we have $\chi^{\mathsf{NM},(0)}=\chi^{\mathsf{EGO}(k)\text{+}\mathsf{NM},(0)}$. Therefore, $\chi^{\mathsf{NM},(D)}\preceq\chi^{\mathsf{EGO}(k)\text{+}\mathsf{NM},(0)}$ (by \cref{remark:finer}(a)), where we assume the input graphs have bounded diameter $D$. Below, we aim to prove that $\chi^{\mathsf{NM},(t)}\preceq\chi^{\mathsf{EGO}(k)\text{+}\mathsf{NM},(t-D)}$ for any integer $t\ge D$. We prove it by induction. 

    The base case of $t=D$ already holds. Assume the above result holds for $t=T$ and consider $t=T+1$. Let $G=(\gV_G,\gE_G)$ and $H=(\gV_H,\gE_H)$ be two graphs with diameter no more than $D$. Consider any vertices $u,v\in\gV_G$ and $x,y\in\gV_H$ satisfying $\chi_G^{\mathsf{NM},(T+1)}(u,v)=\chi_H^{\mathsf{NM},(T+1)}(x,y)$, and we want to prove that $\chi_G^{\mathsf{EGO}(k)\text{+}\mathsf{NM},(T+1-D)}(u,v)=\chi_H^{\mathsf{EGO}(k)\text{+}\mathsf{NM},(T+1-D)}(x,y)$. By \cref{def:swl},
    $$\agg_i(u,v,G,\chi_G^{\mathsf{NM},(T)})=\agg_i(x,y,H,\chi_H^{\mathsf{NM},(T)})$$
    holds for all $i\in[r]$. If $\agg_i$ is any single-point aggregation or global aggregation, by induction we clearly have    $\agg_i(u,v,G,\chi_G^{\mathsf{EGO}(k)\text{+}\mathsf{NM},(T-D)})=\agg_i(x,y,H,\chi_H^{\mathsf{EGO}(k)\text{+}\mathsf{NM},(T-D)})$. If $\agg_i$ is any local aggregation, e.g., $\agg^\mathsf{L}_\mathsf{u}$, we have
    \begin{equation}
    \label{eq:proof_policy_expand_1}
    \begin{aligned}
        \ldblbrace \chi_G^{\mathsf{NM},(T)}(u,w): w\in\gN_{G}(v)\rdblbrace=\ldblbrace \chi_H^{\mathsf{NM},(T)}(x,z): z\in\gN_{H}(y)\rdblbrace
    \end{aligned}
    \end{equation}
    for policy $\mathsf{NM}$. We additionally need to prove that
    \begin{equation}
    \label{eq:proof_policy_expand_0}
    \begin{aligned}
        \ldblbrace \chi_G^{\mathsf{NM},(T)}(u,w): w\in\gN_{G^u}(v)\rdblbrace =\ldblbrace \chi_H^{\mathsf{NM},(T)}(x,z): z\in\gN_{H^x}(y)\rdblbrace,
    \end{aligned}
    \end{equation}
    where $G^u$ and $H^x$ are generated by policy $\mathsf{EGO}(k)\text{+}\mathsf{NM}$. This is due to the following observations: if $w \in\gN_G^k(u)$ and $z\notin\gN_H^k(x)$, then $\dis_G(u,w)\neq \dis_H(x,z)$. Therefore, by \cref{thm:node_marking_implies_distance} we have $\chi_G^{\mathsf{NM},(T)}(u,w)\neq \chi_H^{\mathsf{NM},(T)}(x,z)$. Similarly, if $w\notin\gN_G^k(u)$ and $z\in\gN_H^k(x)$, then $\chi_G^{\mathsf{NM},(T)}(u,w)\neq \chi_H^{\mathsf{NM},(T)}(x,z)$. This yields (\ref{eq:proof_policy_expand_0}). By induction,
    \begin{equation*}
    \begin{aligned}
        \ldblbrace \chi_G^{\mathsf{EGO}(k)\text{+}\mathsf{NM},(T-D)}(u,w): w\in\gN_{G^u}(v)\rdblbrace =\ldblbrace \chi_H^{\mathsf{EGO}(k)\text{+}\mathsf{NM},(T-D)}(x,z): z\in\gN_{H^x}(y)\rdblbrace.
    \end{aligned}
    \end{equation*}
    
    Therefore, in all cases we have
    \begin{equation*}
    \begin{aligned}
        \agg_i(u,v,G,\chi_G^{\mathsf{EGO}(k)\text{+}\mathsf{NM},(T-D)})=\agg_i(x,y,H,\chi_H^{\mathsf{EGO}(k)\text{+}\mathsf{NM},(T-D)}).
    \end{aligned}
    \end{equation*}
    This concludes the induction step. We finally obtain that the stable mappings satisfy $\chi^{\mathsf{NM}}\preceq \chi^{\mathsf{EGO}(k)\text{+}\mathsf{NM}}$ and thus $\mathsf{NM}$ is more powerful than ${\mathsf{EGO}(k)\text{+}\mathsf{NM}}$ (by \cref{remark:finer}(b)).

    We next turn to the case $\mathsf{NM}$ vs. $\mathsf{NDM}$. This case is similar to the above one. Initially, by definition we have $\chi^{\mathsf{NM},(0)}=\chi^{\mathsf{NDM},(0)}$. Therefore, $\chi^{\mathsf{NM},(D)}\preceq\chi^{\mathsf{NDM},(0)}$ where we assume the input graphs have bounded diameter $D$. We aim to prove that $\chi^{\mathsf{NM},(t)}\preceq\chi^{\mathsf{NDM},(t-D)}$ for any integer $t\ge D$. We prove it by induction. The base case of $t=D$ already holds.

    Assume the above result holds for $t=T$ and consider $t=T+1$. Let $\chi_G^{\mathsf{NM},(T+1)}(u,v)=\chi_H^{\mathsf{NM},(T+1)}(x,y)$. Then by \cref{def:swl},
    $$\agg_i(u,v,G,\chi_G^{\mathsf{NM},(T)})=\agg_i(x,y,H,\chi_H^{\mathsf{NM},(T)})$$
    holds for all $i\in[r]$. If $\agg_i$ is any single-point aggregation or global aggregation, by induction we have    $\agg_i(u,v,G,\chi_G^{\mathsf{NDM},(T-D)})=\agg_i(x,y,H,\chi_H^{\mathsf{NDM},(T-D)})$. If $\agg_i$ is any local aggregation, e.g., $\agg^\mathsf{L}_\mathsf{u}$, we have (\ref{eq:proof_policy_expand_1}) for policy $\mathsf{NM}$, and we additionally need to prove (\ref{eq:proof_policy_expand_0}), where $G^u$ and $H^x$ are generated by policy $\mathsf{NDM}$. Note that we have $\dis_G(u,v)=\dis_H(x,y)$ due to the assumption $\chi_G^{\mathsf{NM},(T+1)}(u,v)=\chi_H^{\mathsf{NM},(T+1)}(x,y)$ and \cref{thm:node_marking_implies_distance}. Consider the following three cases:
    \begin{itemize}[topsep=0pt]
    \setlength{\itemsep}{0pt}
        \item If $\dis_G(u,v)=\dis_H(x,y)\ge 2$, then $\gN_{G^u}(v)=\gN_{G}(v)$ and $\gN_{H^x}(y)=\gN_{H}(y)$.
        \item If $\dis_G(u,v)=\dis_H(x,y)=1$, then $\gN_{G^u}(v)=\gN_{G}(v)\backslash\{u\}$ and $\gN_{H^x}(y)=\gN_{H}(y)\backslash\{x\}$. We also have $\chi_G^{\mathsf{NM},(T)}(u,u)=\chi_H^{\mathsf{NM},(T)}(x,x)$, because by (\ref{eq:proof_policy_expand_1}) there exists a vertex $z\in\gV_H$ such that $\chi_G^{\mathsf{NM},(T)}(u,u)=\chi_H^{\mathsf{NM},(T)}(x,z)$, implying $0=\dis_G(u,u)=\dis_H(x,z)$ by \cref{thm:node_marking_implies_distance}.
        \item If $\dis_G(u,v)=\dis_H(x,y)=0$, then $\gN_{G^u}(v)=\emptyset$ and $\gN_{H^x}(y)=\emptyset$. 
    \end{itemize}
    In all cases (\ref{eq:proof_policy_expand_0}) holds, which concludes the induction step. We finally obtain that the stable mappings satisfy $\chi^{\mathsf{NM}}\preceq \chi^{\mathsf{NDM}}$ and thus $\mathsf{NM}$ is more powerful than ${\mathsf{NDM}}$ (by \cref{remark:finer}(b)).

    We next turn to the case ${\mathsf{EGO}(k)\text{+}\mathsf{NM}}$ vs. ${\mathsf{EGO}(k)}$. This case follows by a simple induction that $\chi^{\mathsf{EGO}(k)\text{+}\mathsf{NM},(t)}\preceq\chi^{\mathsf{EGO}(k),(t)}$ for all $t\in\mathbb N$.

    We finally turn to the case $\mathsf{NDM}$ vs. $\mathsf{ND}$. This case also follows by a simple induction that $\chi^{\mathsf{NDM},(t)}\preceq\chi^{\mathsf{ND},(t)}$ for all $t\in\mathbb N$.
\end{proof}

It remains to prove the following key lemma:

\begin{lemma}
\label{thm:node_marking_implies_distance}
    Consider an SWL algorithm $\mathsf{A}$ such that the aggregation scheme $\gA$ contains the two basic aggregations $\agg^\mathsf{P}_\mathsf{uv}$ and $\agg^\mathsf{L}_\mathsf{u}$ in \cref{def:swl}, and the node marking policy is used (possibly along with an ego network policy). Denote $\chi^{(t)}$ as the color mapping of $\mathsf{A}$ at iteration $t$. For any graphs $G=(\gV_G,\gE_G)$, $H=(\gV_H,\gE_H)$ and vertices $u,v\in\gV_G$, $x,y\in\gV_H$, when $t\ge \min(D_G,D_H)$, $\chi_G^{(t)}(u,v)=\chi_H^{(t)}(x,y)\implies \dis_{G^u}(u,v)=\dis_{G^x}(x,y)$. Here, $D_G$ and $D_H$ are the diameter of graphs $G$ and $H$, respectively.
\end{lemma}
\begin{proof}
    It suffices to prove the following result: if $\dis_{G^u}(u,v)\neq\dis_{G^x}(x,y)$, then $\chi_G^{(t)}(u,v)\neq\chi_H^{(t)}(x,y)$ for $t=\min(\dis_{G^u}(u,v), \dis_{G^x}(x,y))$. We prove it by induction over $t$.

    For the base case of $t=0$, without loss of generality we can assume $\dis_{G^u}(u,v)=0$ and $\dis_{H^x}(x,y)>0$. For node marking policy, we clearly have $\chi_G^{(0)}(u,v)\neq\chi_H^{(0)}(x,y)$ since $v=u$ but $y\neq x$.

    Assume the result holds for $t\le T$ and consider $t=T+1$. Without loss of generality, we can assume $\dis_{G^u}(u,v)=T+1$ and $\dis_{H^x}(x,y)>T+1$ (remark: $\dis_{H^x}(x,y)$ can be $\infty$ for ego network policy). If $\chi_G^{(T+1)}(u,v)=\chi_H^{(T+1)}(x,y)$, by definition of $\agg^\mathsf{L}_\mathsf{u}$ we have
    $$\ldblbrace \chi_G^{(T)}(u,w): w\in\gN_{G^u}(v)\rdblbrace=\ldblbrace \chi_H^{(T)}(x,z): z\in\gN_{H^x}(y)\rdblbrace.$$
    Pick any vertex $w\in\gN_{G^u}(v)$ satisfying $\dis_{G^u}(u,w)+1=\dis_{G^u}(u,v)$. Then, there is a vertex $z\in\gN_{H^x}(y)$ such that $\chi_G^{(T)}(u,w)=\chi_H^{(T)}(x,z)$. By induction, $\dis_{G^u}(u,w)=\dis_{H^x}(x,z)$. This yields a contradiction, since
    \begin{align*}
        \dis_{H^x}(x,y)\le \dis_{H^x}(x,z)+1=\dis_{G^u}(u,w)+1=\dis_{G^u}(u,v)=T+1.
    \end{align*}
    This concludes the induction step.
\end{proof}

The above lemma directly leads to the following corollary, which is useful in subsequent analysis.
\begin{corollary}
\label{thm:node_marking_distance_corollary}
    Let $\chi$ be the stable color mapping of SWL algorithm $\mathsf{A}(\gA,\pool)$ defined in \cref{def:node_marking_swl}, satisfying $\agg^\mathsf{L}_\mathsf{u}\in\gA$. For any graphs $G=(\gV_G,\gE_G)$, $H=(\gV_H,\gE_H)$ and vertices $u,v\in\gV_G$, $x,y\in\gV_H$, if $\chi_G(u,v)=\chi_H(x,y)$, then
    \begin{itemize}[topsep=0pt]
    \setlength{\itemsep}{0pt}
        \item $\dis_G(u,v)= \dis_H(x,y)$;
        \item $\chi_G(u,u)= \chi_H(x,x)$.
    \end{itemize}
\end{corollary}
\begin{proof}
    The first bullet directly follows from \cref{thm:node_marking_implies_distance}. The second bullet can be proved by induction over the distance $\dis_G(u,v)$. For the base case of $\dis_G(u,v)=\dis_H(x,y)=0$, $u=v$, $x=y$, and the result already holds. For the induction step, the proof is similar to the above proof of \cref{thm:node_marking_implies_distance}, in that we can find $w\in\gN_G(v)$ and $z\in\gN_H(y)$ such that $\dis_G(u,w)+1=\dis_G(u,v)$ and $\dis_H(x,z)+1=\dis_H(x,y)$, and $\chi_G(u,w)=\chi_H(x,z)$ (we omit the detail here). This finishes the induction step and concludes the proof.
\end{proof}

\subsection{Hierarchy of different aggregation schemes}
\label{sec:proof_aggregation_schemes}
This subsection gives a complete analysis of different aggregation schemes in SWL, which is related to the proofs of \cref{thm:aggregation}. Note that we focus on the canonical node marking policy with a fixed pooling paradigm $\pool$ throughout this subsection. In all proofs, we denote $G=(\gV_G,\gE_G)$ and $H=(\gV_H,\gE_H)$ as any connected graphs.

\begin{lemma}
\label{thm:local_lemma}
    Let $\chi$ be the stable color mapping of SWL algorithm $\mathsf{A}(\gA,\pool)$ defined in \cref{def:node_marking_swl}, satisfying $\agg^\mathsf{L}_\mathsf{u}\in\gA$. For any vertices $u\in\gV_G$ and $x\in\gV_H$, if $\chi_G(u,u)=\chi_H(x,x)$, then $\ldblbrace \chi_G(u,v):v\in\gV_G\rdblbrace=\ldblbrace \chi_H(x,y):y\in\gV_H\rdblbrace$.
\end{lemma}
\begin{proof}
    Actually, we can prove a stronger result: for any $k\in\mathbb N$,
    \begin{equation}
    \label{eq:proof_local_lemma_0}
        \ldblbrace \chi_G(u,v): v\in\gN_{G}^k(u)\rdblbrace=\ldblbrace \chi_H(x,y): y\in\gN_{H}^k(x)\rdblbrace.
    \end{equation}
    This implies \cref{thm:local_lemma} because $G$ and $H$ are connected graphs.
    
    We prove it by induction over $k$. The base case of $k=0$ is trivial. Now assume (\ref{eq:proof_local_lemma_0}) holds for all $k\le K$, and we want to prove that (\ref{eq:proof_local_lemma_0}) holds for $k=K+1$. Using the condition $\agg^\mathsf{L}_\mathsf{u}\in\gA$, for any vertices $v\in\gV_G$ and $y\in\gV_H$ satisfying $\chi_G(u,v)=\chi_H(x,y)$, we have
    \begin{equation}
    \label{eq:proof_local_lemma_1}
        \ldblbrace \chi_G(u,w): w\in\gN_{G}(v)\rdblbrace=\ldblbrace \chi_H(x,z): z\in\gN_{H}(y)\rdblbrace.
    \end{equation}
    Combining (\ref{eq:proof_local_lemma_1}) with (\ref{eq:proof_local_lemma_0}), we obtain
    \begin{equation}
    \label{eq:proof_local_lemma_2}
    \begin{aligned}
        \bigcup_{v\in\gD_G^{K}(u)}\ldblbrace \chi_G(u,w):w\in\gN_G(v)\rdblbrace=\bigcup_{y\in\gD_H^{K}(x)}\ldblbrace \chi_H(x,z):z\in\gN_H(y)\rdblbrace,
    \end{aligned}
    \end{equation}
    where we define $$\gD_G^K(u):=\gN_G^{K}(u)\backslash\gN_G^{K-1}(u)=\{v\in\gV_G:\dis_G(u,v)=K\}.$$
    Here, each vertex $w$ in (\ref{eq:proof_local_lemma_2}) satisfies $K-1\le\dis_G(u,w)\le K+1$, and each vertex $z$ in (\ref{eq:proof_local_lemma_2}) satisfies $K-1\le\dis_H(x,z)\le K+1$. By \cref{thm:node_marking_distance_corollary}, for any vertices $w\in\gV_G$ and $z\in\gV_H$, $\dis_G(u,w)\neq\dis_H(x,z)$ implies $\chi_G(u,w)\neq\chi_H(x,z)$. Therefore,
    \begin{equation}
    \label{eq:proof_local_lemma_3}
    \begin{aligned}
        \bigcup_{v\in\gD_G^{K}(u)}\ldblbrace \chi_G(u,w):w\in\gN_G(v)\cap\gD_G^{K+1}(u) \rdblbrace=\bigcup_{y\in\gD_H^{K}(x)}\ldblbrace \chi_H(x,z):z\in\gN_H(y)\cap \gD_H^{K+1}(x)\rdblbrace.
    \end{aligned}
    \end{equation}
    Rearranging the terms in (\ref{eq:proof_local_lemma_3}) yields an equivalent formula:
    \begin{equation}
    \label{eq:proof_local_lemma_4}
    \begin{aligned}
        \bigcup_{w\in\gD_G^{K+1}(u)}\ldblbrace \chi_G(u,w)\rdblbrace\times |\gN_G(w)\cap\gD_G^{K}(u)|=\bigcup_{z\in\gD_H^{K+1}(x)}\ldblbrace \chi_H(x,z)\rdblbrace\times |\gN_H(z)\cap\gD_H^{K}(x)|,
    \end{aligned}
    \end{equation}
    where we denote $\ldblbrace c\rdblbrace\times M$ as a multiset containing $M$ repeated elements $c$. Next, note that if $\chi_G(u,w)=\chi_H(x,z)$ for some $w\in\gV_G$ and $z\in\gV_H$, then by (\ref{eq:proof_local_lemma_1}) and \cref{thm:node_marking_distance_corollary}, we have $|\gN_G(w)\cap\gD_G^{K}(u)|=|\gN_H(z)\cap\gD_H^{K}(x)|$. This proves
    $$\ldblbrace \chi_G(u,w): w\in\gN_{G}^{K+1}(u)\rdblbrace=\ldblbrace \chi_H(x,z): z\in\gN_{H}^{K+1}(x)\rdblbrace$$
    and finishes the induction step. 
\end{proof}

\begin{corollary}
\label{thm:local_gt_global}
    Let $\chi^\mathsf{L}$, $\chi^\mathsf{G}$, and $\chi^\mathsf{LG}$ be the stable color mappings of SWL algorithms $\mathsf{A}(\gA\cup\{\agg^\mathsf{L}_{\mathsf{u}}\},\pool)$, $\mathsf{A}(\gA\cup\{\agg^\mathsf{G}_{\mathsf{u}}\},\pool)$ and $\mathsf{A}(\gA\cup\{\agg^\mathsf{L}_{\mathsf{u}},\agg^\mathsf{G}_{\mathsf{u}}\},\pool)$, respectively. Then, $\chi^\mathsf{LG}\simeq\chi^\mathsf{L}\preceq\chi^\mathsf{G}$.
\end{corollary}
\begin{proof}
    The proof is based on \cref{remark:finer}(c). We first prove that $\chi^\mathsf{L}\preceq\chi^\mathsf{G}$. Define an auxiliary color mapping $\tilde\chi=\mathscr{T}(\gA\cup\{\agg^\mathsf{G}_{\mathsf{u}}\},\chi^\mathsf{L})$ where $\mathscr{T}$ is defined in (\ref{eq:mapping_transform}). It suffices to prove that $\chi^\mathsf{L}\preceq\tilde\chi$.

    Consider any vertices $u,v\in\gV_G$ and $x,y\in\gV_H$ satisfying $\chi_G^\mathsf{L}(u,v)=\chi_H^\mathsf{L}(x,y)$. Since the mapping $\chi^\mathsf{L}$ is already stable, for any $\agg\in\gA$, we have
    $$\agg(u,v,G,\chi_G^\mathsf{L})=\agg(x,y,H,\chi_H^\mathsf{L}).$$
    Moreover, due to the use of $\agg^\mathsf{L}_\mathsf{u}$, by \cref{thm:node_marking_distance_corollary} we have $\chi_G^\mathsf{L}(u,u)=\chi_H^\mathsf{L}(x,x)$. Using \cref{thm:local_lemma}, we further obtain
    \begin{equation*}
        \ldblbrace \chi_G^\mathsf{L}(u,w): w\in\gV_G\rdblbrace=\ldblbrace \chi_H^\mathsf{L}(x,z): z\in\gV_H\rdblbrace.
    \end{equation*}
    Namely, 
    $$\agg^\mathsf{G}_\mathsf{u}(u,v,G,\chi_G^\mathsf{L})=\agg^\mathsf{G}_\mathsf{u}(x,y,H,\chi_H^\mathsf{L}).$$
    Therefore, $\tilde\chi(u,v)=\tilde\chi(x,y)$. We have proved $\chi^\mathsf{L}\preceq\tilde\chi$.

    We next turn to $\chi^\mathsf{L}\simeq\chi^\mathsf{LG}$, for which it suffices to prove $\chi^\mathsf{L}\preceq\chi^\mathsf{LG}$. The process is exactly the same as above.
\end{proof}

\begin{lemma}
\label{thm:global_lemma}
    Let $\chi$ be the stable color mapping of SWL algorithm $\mathsf{A}(\gA,\pool)$ defined in \cref{def:node_marking_swl}, satisfying $\agg^\mathsf{G}_\mathsf{u}\in\gA$. For any vertices $u,v\in\gV_G$ and $x,y\in\gV_H$, if $\chi_G(u,v)=\chi_H(x,y)$, then $\chi_G(u,u)=\chi_H(x,x)$.
\end{lemma}
\begin{proof}
    Since $\agg^\mathsf{G}_\mathsf{u}\in\gA$, we have
    \begin{equation*}
        \ldblbrace \chi_G(u,w): w\in\gV_G\rdblbrace=\ldblbrace \chi_H(x,z): z\in\gV_H\rdblbrace.
    \end{equation*}
    Therefore, there is a vertex $z\in\gV_H$ such that $\chi_G(u,u)=\chi_H(x,z)$. By definition of node marking policy, we must have $x=z$ (otherwise, the initial color satisfies $\chi_G^{(0)}(u,u)\neq\chi_H^{(0)}(x,z)$, a contradiction). This already proves that $\chi_G(u,u)=\chi_H(x,x)$.
\end{proof}

\begin{corollary}
\label{thm:global_gt_single_point}
    Let $\chi^\mathsf{G}$, $\chi^\mathsf{P}$, and $\chi^\mathsf{GP}$ be the stable color mappings of SWL algorithms $\mathsf{A}(\gA\cup\{\agg^\mathsf{G}_{\mathsf{u}}\},\pool)$, $\mathsf{A}(\gA\cup\{\agg^\mathsf{P}_{\mathsf{uu}}\},\pool)$, and $\mathsf{A}(\gA\cup\{\agg^\mathsf{G}_{\mathsf{u}},\agg^\mathsf{P}_{\mathsf{uu}}\},\pool)$, respectively. Then, $\chi^\mathsf{GP}\simeq\chi^\mathsf{G}\preceq\chi^\mathsf{P}$.
\end{corollary}
\begin{proof}
    Similar to \cref{thm:local_gt_global}, the proof is based on \cref{remark:finer}(c). We only prove $\chi^\mathsf{G}\preceq\chi^\mathsf{P}$, and the proof of $\chi^\mathsf{G}\preceq\chi^\mathsf{GP}$ is exactly the same. Define an auxiliary color mapping $\tilde\chi=\mathscr{T}(\gA\cup\{\agg^\mathsf{P}_{\mathsf{uu}}\},\chi^\mathsf{G})$ where $\mathscr{T}$ is defined in (\ref{eq:mapping_transform}). It suffices to prove that $\chi^\mathsf{G}\preceq\tilde\chi$.

    Consider any vertices $u,v\in\gV_G$ and $x,y\in\gV_H$ satisfying $\chi_G^\mathsf{G}(u,v)=\chi_H^\mathsf{G}(x,y)$. Due to the presence of $\agg^\mathsf{G}_{\mathsf{u}}$, by \cref{thm:global_lemma} we have $\chi_G^\mathsf{G}(u,u)=\chi_H^\mathsf{G}(x,x)$. This already implies 
    $$\agg^\mathsf{P}_\mathsf{uu}(u,v,G,\chi_G^\mathsf{G})=\agg^\mathsf{P}_\mathsf{uu}(x,y,H,\chi_H^\mathsf{G}).$$
    For any $\agg\in\gA$, we also have
    $$\agg(u,v,G,\chi_G^\mathsf{G})=\agg(x,y,H,\chi_H^\mathsf{G}).$$
    Therefore, $\tilde\chi(u,v)=\tilde\chi(x,y)$, namely, $\chi^\mathsf{G}\preceq\tilde\chi$.
\end{proof}

\begin{lemma}
\label{thm:local_local_lemma}
    Let $\chi$ be the stable color mapping of SWL algorithm $\mathsf{A}(\{\agg^\mathsf{L}_\mathsf{u},\agg^\mathsf{L}_\mathsf{v}\},\pool)$. Then for any vertices $u,v\in\gV_G$ and $x,y\in\gV_H$, if $\chi_G(u,v)=\chi_H(x,y)$, then $\chi_G(v,u)=\chi_H(y,x)$.
\end{lemma}
\begin{proof}
    We will prove a stronger result: let $\chi^{(t)}$ be the color mapping at iteration $t$, then for any $t\in\mathbb N$, $\chi_G^{(t)}(u,v)=\chi_H^{(t)}(x,y)\iff\chi_G^{(t)}(v,u)=\chi_H^{(t)}(y,x)$. We prove it by induction over $t$.

    The base case of $t=0$ trivially holds by definition of the node marking. Assume the above result holds for $t=T$, and consider $t=T+1$. Let $\chi_G^{(T+1)}(u,v)=\chi_H^{(T+1)}(x,y)$. By definition of the aggregation scheme $\{\agg^\mathsf{L}_\mathsf{u},\agg^\mathsf{L}_\mathsf{v}\}$, we have
    \begin{align*}
        \ldblbrace \chi_G^{(T)}(u,w): w\in\gN_{G}(v)\rdblbrace=\ldblbrace \chi_H^{(T)}(x,z): z\in\gN_{H}(y)\rdblbrace,\\
        \ldblbrace \chi_G^{(T)}(w,v): w\in\gN_{G}(u)\rdblbrace=\ldblbrace \chi_H^{(T)}(z,y): z\in\gN_{H}(x)\rdblbrace.
    \end{align*}
    Using induction we obtain
    \begin{align*}
        \ldblbrace \chi_G^{(T)}(w,u): w\in\gN_{G}(v)\rdblbrace=\ldblbrace \chi_H^{(T)}(z,x): z\in\gN_{H}(y)\rdblbrace,\\
        \ldblbrace \chi_G^{(T)}(v,w): w\in\gN_{G}(u)\rdblbrace=\ldblbrace \chi_H^{(T)}(y,z): z\in\gN_{H}(x)\rdblbrace.
    \end{align*}
    Therefore, $\chi_G^{(T+1)}(v,u)=\chi_H^{(T+1)}(y,x)$, which finishes the induction step.
\end{proof}

\begin{corollary}
\label{thm:transpose_eq_local}
    Let $\chi^\mathsf{LL}$, $\chi^\mathsf{LP}$, and $\chi^\mathsf{LLP}$ be the stable color mappings of SWL algorithms $\mathsf{A}(\{\agg^\mathsf{L}_{\mathsf{u}},\agg^\mathsf{L}_{\mathsf{v}}\},\pool)$, $\mathsf{A}(\{\agg^\mathsf{L}_{\mathsf{u}},\agg^\mathsf{P}_{\mathsf{vu}}\},\pool)$, and $\mathsf{A}(\{\agg^\mathsf{L}_{\mathsf{u}},\agg^\mathsf{L}_{\mathsf{v}},\agg^\mathsf{P}_{\mathsf{vu}}\},\pool)$, respectively. Then, $\chi^\mathsf{LL}\simeq\chi^\mathsf{LP}\simeq\chi^\mathsf{LLP}$.
\end{corollary}
\begin{proof}
    Similar to \cref{thm:local_gt_global}, the proof is based on \cref{remark:finer}(c). We only prove $\chi^\mathsf{LL}\simeq\chi^\mathsf{LP}$, and the proof of $\chi^\mathsf{LL}\simeq\chi^\mathsf{LLP}$ is exactly the same. 

    We first prove $\chi^\mathsf{LP}\preceq\chi^\mathsf{LL}$. Define an auxiliary color mapping $\tilde\chi=\mathscr{T}(\{\agg^\mathsf{L}_{\mathsf{u}},\agg^\mathsf{L}_{\mathsf{v}}\},\chi^\mathsf{LP})$ where $\mathscr{T}$ is defined in (\ref{eq:mapping_transform}). It suffices to prove that $\chi^\mathsf{LP}\preceq\tilde\chi$. Consider any vertices $u,v\in\gV_G$ and $x,y\in\gV_H$ satisfying $\chi_G^\mathsf{LP}(u,v)=\chi_H^\mathsf{LP}(x,y)$. Since the mapping $\chi^\mathsf{LP}$ is already stable, we have $\chi_G^\mathsf{LP}(v,u)=\chi_H^\mathsf{LP}(y,x)$ and
    \begin{equation}
    \label{eq:proof_transpose_eq_local_0}
        \ldblbrace \chi_G^\mathsf{LP}(u,w): w\in\gN_{G}(v)\rdblbrace=\ldblbrace \chi_H^\mathsf{LP}(x,z): z\in\gN_{H}(y)\rdblbrace.
    \end{equation}
    Since $\chi_G^\mathsf{LP}(v,u)=\chi_H^\mathsf{LP}(y,x)$, we also have
    \begin{equation*}
        \ldblbrace \chi_G^\mathsf{LP}(v,w): w\in\gN_{G}(u)\rdblbrace=\ldblbrace \chi_H^\mathsf{LP}(y,z): z\in\gN_{H}(x)\rdblbrace.
    \end{equation*}
    This further implies
    \begin{equation}
    \label{eq:proof_transpose_eq_local_1}
        \ldblbrace \chi_G^\mathsf{LP}(w,v): w\in\gN_{G}(u)\rdblbrace=\ldblbrace \chi_H^\mathsf{LP}(z,y): z\in\gN_{H}(x)\rdblbrace.
    \end{equation}
    Combining (\ref{eq:proof_transpose_eq_local_0}) and (\ref{eq:proof_transpose_eq_local_1}) we obtain $\tilde\chi_G(u,v)=\tilde\chi_H(x,y)$, as desired.

    We next prove $\chi^\mathsf{LL}\preceq\chi^\mathsf{LP}$. Define an auxiliary color mapping $\tilde\chi=\mathscr{T}(\{\agg^\mathsf{L}_{\mathsf{u}},\agg^\mathsf{P}_{\mathsf{vu}}\},\chi^\mathsf{LL})$ where $\mathscr{T}$ is defined in (\ref{eq:mapping_transform}). It suffices to prove that $\chi^\mathsf{LL}\preceq\tilde\chi$. This simply follows by the fact that the stable color mapping $\chi^\mathsf{LL}$ cannot be refined using $\agg^\mathsf{L}_{\mathsf{u}}$ (by definition) or using $\agg^\mathsf{P}_{\mathsf{vu}}$ (by \cref{thm:local_local_lemma}).
\end{proof}

\subsection{Analyzing the pooling paradigm}
\label{sec:proof_pooling_paradigm}

This subsection discusses how the pooling paradigm can influence the expressive power of the SWL algorithm, which is related to the proofs of \cref{thm:pooling}. In all proofs, we denote $G=(\gV_G,\gE_G)$ and $H=(\gV_H,\gE_H)$ as any graphs.

\begin{lemma}
\label{thm:pooling_vs_lt_sv}
    Let $\gA$ be defined in \cref{def:node_marking_swl} with $\{\agg^\mathsf{G}_{\mathsf{u}},\agg^\mathsf{L}_{\mathsf{u}}\}\cap\gA\neq\emptyset$. Then, $\mathsf{A}(\gA,\mathsf{VS})\preceq\mathsf{A}(\gA,\mathsf{SV})$.
\end{lemma}
\begin{proof}
    Let $c^\mathsf{VS}(G)$ and $c^\mathsf{SV}(G)$ be the graph representations computed by algorithms $\mathsf{A}(\gA,\mathsf{VS})$ and $\mathsf{A}(\gA,\mathsf{SV})$, respectively. Since both algorithms use the same aggregation scheme, we denote the stable color mapping as $\chi$. We aim to prove that if $c^\mathsf{SV}(G)=c^\mathsf{SV}(H)$, then $c^\mathsf{VS}(G)=c^\mathsf{VS}(H)$.

    Let $c^\mathsf{SV}(G)=c^\mathsf{SV}(H)$, then by definition of $\mathsf{SV}$ pooling
    \begin{equation*}
        \ldblbrace r^\mathsf{V}_G(v):v\in\gV_G\rdblbrace=\ldblbrace r^\mathsf{V}_H(y):y\in\gV_H\rdblbrace,
    \end{equation*}
    where we denote $r^\mathsf{V}_G(v)=\ldblbrace\chi_G(u,v):u\in\gV_G\rdblbrace$. Consider any vertices $v\in\gV_G$ and $y\in\gV_H$ satisfying $r^\mathsf{V}_G(v)=r^\mathsf{V}_H(y)$. Then, there exists vertex $w\in\gV_H$ such that $\chi_G(v,v)=\chi_H(w,y)$. Due to the definition of node marking, we must have $w=y$. This implies that $\chi_G(v,v)=\chi_H(y,y)$. Now separately consider two cases:
    \begin{itemize}[topsep=0pt]
    \setlength{\itemsep}{0pt}
        \item If $\agg^\mathsf{G}_{\mathsf{u}}\in\gA$, then by definition of stable color mapping we have $\ldblbrace\chi_G(v,u):u\in\gV_G\rdblbrace=\ldblbrace\chi_H(y,x):x\in\gV_H\rdblbrace$;
        \item If $\agg^\mathsf{L}_{\mathsf{u}}\in\gA$, then by \cref{thm:local_lemma} we have $\ldblbrace\chi_G(v,u):u\in\gV_G\rdblbrace=\ldblbrace\chi_H(y,x):x\in\gV_H\rdblbrace$.
    \end{itemize}
    In both cases, we have $r^\mathsf{S}_G(v)=r^\mathsf{S}_H(y)$ where we denote $r^\mathsf{S}_G(v)=\ldblbrace\chi_G(v,u):u\in\gV_G\rdblbrace$.

    Therefore, we have proved that $r^\mathsf{V}_G(v)=r^\mathsf{V}_H(y)\implies r^\mathsf{S}_G(v)=r^\mathsf{S}_H(y)$. This finally yields
    \begin{equation*}
        \ldblbrace r^\mathsf{S}_G(v):v\in\gV_G\rdblbrace=\ldblbrace r^\mathsf{S}_H(y):y\in\gV_H\rdblbrace,
    \end{equation*}
    namely, $c^\mathsf{VS}(G)=c^\mathsf{VS}(H)$, as desired.
\end{proof}

\subsection{Proof of theorems in \cref{sec:swl_hierarchy}}
\label{sec:proof_hierarchy_main}

We are now ready to prove all the main results in \cref{sec:swl_hierarchy}, which we restate below.

\textbf{\cref{thm:aggregation}.} \emph{Under the notation of \cref{def:node_marking_swl}, the following hold:
    \begin{itemize}[topsep=0pt]
    \raggedright
    \setlength{\itemsep}{0pt}
        \item $\mathsf{A}(\gA\cup\{\agg^\mathsf{G}_{\mathsf{u}}\},\pool)\preceq\mathsf{A}(\gA\cup\{\agg^\mathsf{L}_{\mathsf{u}}\},\pool)$ and $\mathsf{A}(\gA\cup\{\agg^\mathsf{L}_{\mathsf{u}}\},\pool)\simeq \mathsf{A}(\gA\cup\{\agg^\mathsf{L}_{\mathsf{u}},\agg^\mathsf{G}_{\mathsf{u}}\},\pool)$;
        \item $\mathsf{A}(\gA\cup\{\agg^\mathsf{P}_{\mathsf{uu}}\},\pool)\preceq\mathsf{A}(\gA\cup\{\agg^\mathsf{G}_{\mathsf{u}}\},\pool)$ and $\mathsf{A}(\gA\cup\{\agg^\mathsf{G}_{\mathsf{u}}\},\pool)\simeq\mathsf{A}(\gA\cup\{\agg^\mathsf{G}_{\mathsf{u}},\agg^\mathsf{P}_{\mathsf{uu}}\}\},\pool)$;
        \item $\mathsf{A}(\{\agg^\mathsf{L}_{\mathsf{u}},\agg^\mathsf{P}_{\mathsf{vu}}\},\pool)\simeq\mathsf{A}(\{\agg^\mathsf{L}_{\mathsf{u}},\agg^\mathsf{L}_{\mathsf{v}}\},\pool)\simeq\mathsf{A}(\{\agg^\mathsf{L}_{\mathsf{u}},\agg^\mathsf{L}_{\mathsf{v}},\agg^\mathsf{P}_{\mathsf{vu}}\},\pool)$.
    \end{itemize}}
\begin{proof}
    Based on \cref{remark:finer}(b), we only need to focus on the stable color mappings of these algorithms. The proof readily follows by using \cref{thm:local_gt_global,thm:global_gt_single_point,thm:transpose_eq_local}.
\end{proof}

\textbf{\cref{thm:symmetry}.} \emph{Let $\gA$ be any aggregation scheme defined in \cref{def:node_marking_swl}. Denote $\gA^{\mathsf{u}\leftrightarrow\mathsf{v}}$ as the aggregation scheme obtained from $\gA$ by exchanging the element $\agg^\mathsf{P}_{\mathsf{uu}}$ with $\agg^\mathsf{P}_{\mathsf{vv}}$, exchanging $\agg^\mathsf{L}_{\mathsf{u}}$ with $\agg^\mathsf{L}_{\mathsf{v}}$, and exchanging $\agg^\mathsf{G}_{\mathsf{u}}$ with $\agg^\mathsf{G}_{\mathsf{v}}$. Then, $\mathsf{A}(\gA, \mathsf{VS})\simeq\mathsf{A}(\gA^{\mathsf{u}\leftrightarrow\mathsf{v}}, \mathsf{SV})$.}
\begin{proof}
    The proof is almost trivial by symmetry. It is easy to see that for any vertices $u,v\in\gV_G$ and $x,y\in\gV_H$, $\chi_G(u,v)=\chi_H(x,y)\iff\chi^{\mathsf{u}\leftrightarrow\mathsf{v}}_G(v,u)=\chi^{\mathsf{u}\leftrightarrow\mathsf{v}}_H(y,x)$, where $\chi$ and $\chi^{\mathsf{u}\leftrightarrow\mathsf{v}}$ are the stable color mapping of SWL algorithms $\mathsf{A}(\gA, \mathsf{VS})$ and $\mathsf{A}(\gA^{\mathsf{u}\leftrightarrow\mathsf{v}}, \mathsf{SV})$, respectively.
\end{proof}

\textbf{\cref{thm:pooling}.} \emph{Let $\gA$ be defined in \cref{def:node_marking_swl} with $\agg^\mathsf{L}_{\mathsf{u}}\in\gA$. Then the following hold:
    \begin{itemize}[topsep=0pt]
    \setlength{\itemsep}{0pt}
        \item $\mathsf{A}(\gA,\mathsf{VS})\preceq\mathsf{A}(\gA,\mathsf{SV})$;
        \item If $\{\agg^\mathsf{G}_{\mathsf{v}},\agg^\mathsf{L}_{\mathsf{v}}\}\cap\gA\neq\emptyset$, then $\mathsf{A}(\gA,\mathsf{VS})\simeq\mathsf{A}(\gA,\mathsf{SV})$.
    \end{itemize}}
\begin{proof}
    The first bullet is a direct consequence of \cref{thm:pooling_vs_lt_sv}. The second bullet is a direct consequence of the first bullet and \cref{thm:symmetry}, since we have both $\mathsf{A}(\gA,\mathsf{VS})\preceq\mathsf{A}(\gA,\mathsf{SV})$ and $\mathsf{A}(\gA,\mathsf{SV})\simeq \mathsf{A}(\gA^{\mathsf{u}\leftrightarrow\mathsf{v}},\mathsf{VS})\preceq\mathsf{A}(\gA^{\mathsf{u}\leftrightarrow\mathsf{v}},\mathsf{SV})\simeq\mathsf{A}(\gA,\mathsf{VS})$.
\end{proof}

\textbf{\cref{thm:swl_hierarchy}.} \emph{Let $\mathsf{A}(\gA,\pool)$ be any SWL algorithm defined in \cref{def:node_marking_swl} with at least one local aggregation, i.e. $\{\agg^\mathsf{L}_{\mathsf{u}},\agg^\mathsf{L}_{\mathsf{v}}\}\cap\gA\neq \emptyset$. Then, $\mathsf{A}(\gA,\pool)$ must be as expressive as one of the 6 SWL algorithms defined below: {\normalfont
    \begin{itemize}[topsep=0pt]
    \raggedright 
    \setlength{\itemsep}{0pt}
        \item (Vanilla SWL) $\mathsf{SWL(VS)}:=\mathsf{A}(\{\agg^\mathsf{L}_{\mathsf{u}}\},\mathsf{VS})$, $\mathsf{SWL(SV)}:=\mathsf{A}(\{\agg^\mathsf{L}_{\mathsf{u}}\},\mathsf{SV})$;
        \item (SWL with additional single-point aggregation) $\mathsf{PSWL(VS)}:=\mathsf{A}(\{\agg^\mathsf{L}_{\mathsf{u}},\agg^\mathsf{P}_{\mathsf{vv}}\},\mathsf{VS})$, $\mathsf{PSWL(SV)}:=\mathsf{A}(\{\agg^\mathsf{L}_{\mathsf{u}},\agg^\mathsf{P}_{\mathsf{vv}}\},\mathsf{SV})$;
        \item (SWL with additional global aggregation) $\mathsf{GSWL}:=\mathsf{A}(\{\agg^\mathsf{L}_{\mathsf{u}},\agg^\mathsf{G}_{\mathsf{v}}\},\mathsf{VS})$;
        \item (Symmetrized SWL) $\mathsf{SSWL}:=\mathsf{A}(\{\agg^\mathsf{L}_{\mathsf{u}},\agg^\mathsf{L}_{\mathsf{v}}\},\mathsf{VS})$.
    \end{itemize}}
    Moreover, we have
    \begin{gather*}
        \mathsf{SWL(VS)}\preceq\mathsf{SWL(SV)}\ \text{and}\ \mathsf{PSWL(VS)}\preceq\mathsf{PSWL(SV)},\\
        \mathsf{SWL(VS)}\preceq\mathsf{PSWL(VS)}\ \text{and}\ \mathsf{SWL(SV)}\preceq\mathsf{PSWL(SV)},\\
        \mathsf{PSWL(SV)}\preceq\mathsf{GSWL}\preceq\mathsf{SSWL}.
    \end{gather*}}
\begin{proof}
    Due to \cref{thm:symmetry}, we can assume $\agg^\mathsf{L}_\mathsf{u}\in\gA$ without loss of generality. We separately consider several cases:
    \begin{itemize}[topsep=0pt]
    \raggedright 
    \setlength{\itemsep}{0pt}
        \item Case 1: $\{\agg^\mathsf{L}_\mathsf{v},\agg^\mathsf{G}_\mathsf{v},\agg^\mathsf{P}_\mathsf{vv},\agg^\mathsf{P}_\mathsf{vu}\}\cap\gA=\emptyset$. In this case, we have
        \begin{align*}
            \mathsf{A}(\gA,\pool)\preceq\mathsf{A}(\{\agg^\mathsf{L}_\mathsf{u},\agg^\mathsf{G}_\mathsf{u},\agg^\mathsf{P}_\mathsf{uu}\},\pool)\simeq\mathsf{A}(\{\agg^\mathsf{L}_\mathsf{u},\agg^\mathsf{G}_\mathsf{u}\},\pool)\simeq\mathsf{A}(\{\agg^\mathsf{L}_\mathsf{u}\},\pool)
        \end{align*}
        by \cref{thm:aggregation}. On the other hand, clearly $\mathsf{A}(\{\agg^\mathsf{L}_\mathsf{u}\},\pool)\preceq\mathsf{A}(\gA,\pool)$. We thus have $\mathsf{A}(\gA,\pool)\simeq \mathsf{A}(\{\agg^\mathsf{L}_\mathsf{u}\},\pool)$, namely, $\mathsf{A}(\gA,\pool)\simeq \mathsf{SWL(VS)}$ or $\mathsf{A}(\gA,\pool)\simeq \mathsf{SWL(SV)}$.
        \item Case 2: $\agg^\mathsf{P}_\mathsf{vv}\in\gA$ and $\{\agg^\mathsf{L}_\mathsf{v},\agg^\mathsf{G}_\mathsf{v},\agg^\mathsf{P}_\mathsf{vu}\}\cap\gA=\emptyset$. In this case, we have
        \begin{align*}
            \mathsf{A}(\gA,\pool)&\preceq\mathsf{A}(\{\agg^\mathsf{L}_\mathsf{u},\agg^\mathsf{G}_\mathsf{u},\agg^\mathsf{P}_\mathsf{uu},\agg^\mathsf{P}_\mathsf{vv}\},\pool)\\
            &\simeq\mathsf{A}(\{\agg^\mathsf{L}_\mathsf{u},\agg^\mathsf{G}_\mathsf{u},\agg^\mathsf{P}_\mathsf{vv}\},\pool)\\
            &\simeq\mathsf{A}(\{\agg^\mathsf{L}_\mathsf{u},\agg^\mathsf{P}_\mathsf{vv}\},\pool)
        \end{align*}
        by \cref{thm:aggregation}. On the other hand, clearly $\mathsf{A}(\{\agg^\mathsf{L}_\mathsf{u},\agg^\mathsf{P}_\mathsf{vv}\},\pool)\preceq\mathsf{A}(\gA,\pool)$. We thus have $\mathsf{A}(\gA,\pool)\simeq \mathsf{A}(\{\agg^\mathsf{L}_\mathsf{u},\agg^\mathsf{P}_\mathsf{vv}\},\pool)$, namely, $\mathsf{A}(\gA,\pool)\simeq \mathsf{PSWL(VS)}$ or $\mathsf{A}(\gA,\pool)\simeq \mathsf{PSWL(SV)}$.
        \item Case 3: $\agg^\mathsf{G}_\mathsf{v}\in\gA$ and $\{\agg^\mathsf{L}_\mathsf{v},\agg^\mathsf{P}_\mathsf{vu}\}\cap\gA=\emptyset$. In this case, we have
        \begin{align*}
            \mathsf{A}(\gA,\pool)&\preceq\mathsf{A}(\{\agg^\mathsf{L}_\mathsf{u},\agg^\mathsf{G}_\mathsf{u},\agg^\mathsf{P}_\mathsf{uu},\agg^\mathsf{G}_\mathsf{v},\agg^\mathsf{P}_\mathsf{vv}\},\pool)\\
            &\simeq\mathsf{A}(\{\agg^\mathsf{L}_\mathsf{u},\agg^\mathsf{G}_\mathsf{u},\agg^\mathsf{G}_\mathsf{v}\},\pool)\\
            &\simeq\mathsf{A}(\{\agg^\mathsf{L}_\mathsf{u},\agg^\mathsf{G}_\mathsf{v}\},\pool)
        \end{align*}
        by \cref{thm:aggregation}. On the other hand, clearly $\mathsf{A}(\{\agg^\mathsf{L}_\mathsf{u},\agg^\mathsf{G}_\mathsf{v}\},\pool)\preceq\mathsf{A}(\gA,\pool)$. We thus have $\mathsf{A}(\gA,\pool)\simeq \mathsf{A}(\{\agg^\mathsf{L}_\mathsf{u},\agg^\mathsf{G}_\mathsf{v}\},\pool)$. Moreover, by \cref{thm:pooling} we have $\mathsf{A}(\{\agg^\mathsf{L}_\mathsf{u},\agg^\mathsf{G}_\mathsf{v}\},\mathsf{VS})=\mathsf{A}(\{\agg^\mathsf{L}_\mathsf{u},\agg^\mathsf{G}_\mathsf{v}\},\mathsf{SV})$. Therefore, $\mathsf{A}(\gA,\pool)\simeq \mathsf{GSWL}$.
        \item Case 4: $\agg^\mathsf{L}_\mathsf{v}\in\gA$ or $\agg^\mathsf{P}_\mathsf{vu}\in\gA$. In this cases, a similar analysis yields
        \begin{align*}
            \mathsf{A}(\gA,\pool)&\preceq\mathsf{A}(\{\agg^\mathsf{L}_\mathsf{u},\agg^\mathsf{L}_\mathsf{v},\agg^\mathsf{P}_\mathsf{vu}\},\pool)\\
            &\simeq \mathsf{A}(\{\agg^\mathsf{L}_\mathsf{u},\agg^\mathsf{L}_\mathsf{v}\},\pool)\\
            &\simeq \mathsf{A}(\{\agg^\mathsf{L}_\mathsf{u},\agg^\mathsf{P}_\mathsf{vu}\},\pool)
        \end{align*}
        by \cref{thm:aggregation}. On the other hand, $\mathsf{A}(\{\agg^\mathsf{L}_\mathsf{u},\agg^\mathsf{L}_\mathsf{v}\},\pool)\simeq\mathsf{A}(\{\agg^\mathsf{L}_\mathsf{u},\agg^\mathsf{P}_\mathsf{vu}\},\pool)\preceq\mathsf{A}(\gA,\pool)$. We thus have $\mathsf{A}(\gA,\pool)\simeq \mathsf{A}(\{\agg^\mathsf{L}_\mathsf{u},\agg^\mathsf{L}_\mathsf{v}\},\pool)$. Moreover, by \cref{thm:pooling} we have $\mathsf{A}(\{\agg^\mathsf{L}_\mathsf{u},\agg^\mathsf{L}_\mathsf{v}\},\mathsf{VS})=\mathsf{A}(\{\agg^\mathsf{L}_\mathsf{u},\agg^\mathsf{L}_\mathsf{v}\},\mathsf{SV})$. Therefore, $\mathsf{A}(\gA,\pool)\simeq \mathsf{SSWL}$.
    \end{itemize}
    Combining the four cases concludes the proof.
\end{proof}

\section{Discussions on Subgraph GNNs beyond the Framework of \cref{def:layer}}
\label{sec:other_aggregation}

There have been several prior works that design subgraph GNNs beyond the aggregation schemes of \cref{def:layer}. In this section, we will investigate them and compare the expressive power with our framework. We focus on the WL algorithm corresponding to each subgraph GNN because it has the same expressive power as the GNN model in distinguishing non-isomorphic graphs (which can be easily proved following \cref{sec:swl_eq_subgraph_gnn}). Throughout this section, we assume that node marking policy is used since it achieves the strongest expressive power according to \cref{thm:node_marking}. 

Below, we discuss the following works:
\begin{itemize}[topsep=0pt]
    \setlength{\itemsep}{0pt}
    \item $\mathsf{GNN}\text{-}\mathsf{AK}$ \citep{zhao2022stars};
    \item $\mathsf{GNN}\text{-}\mathsf{AK}\text{-}\mathsf{ctx}$ \citep{zhao2022stars};
    \item $\mathsf{DSS}\text{-}\mathsf{WL}$ \citep{bevilacqua2022equivariant};
    \item $\mathsf{SUN}$ \citep{frasca2022understanding};
    \item $\mathsf{ReIGN(2)}$ \citep{frasca2022understanding}.
\end{itemize}

$\mathsf{GNN}\text{-}\mathsf{AK}$ \citep{zhao2022stars}. The GNN aggregation scheme can be written as
\begin{align*}
    \chi_G^{(t+1)}(u, v) =\left\{\begin{array}{ll}
        \begin{aligned}
            \hash(&\chi_G^{(t)}(u,v),\\
                &\chi_G^{(t)}(v,v),\\
                &\ldblbrace\chi_G^{(t)}(u, w):w\in\gN_G(v)\rdblbrace)
        \end{aligned}
         & \text{if }u\neq v, \\
        \begin{aligned}
            \hash(&\chi_G^{(t)}(v,v),\\
                &\ldblbrace\chi_G^{(t)}(u, w):w\in\gN_G(v)\rdblbrace,\\
                &\ldblbrace\chi_G^{(t)}(u, w):w\in\gV_G\rdblbrace)
        \end{aligned}
         & \text{if }u= v.
    \end{array}\right. 
\end{align*}
$\mathsf{GNN}\text{-}\mathsf{AK}$ uses the vertex-subgraph pooling. As can be seen, there is an additional global aggregation $\agg^\mathsf{G}_\mathsf{u}$ when $u=v$, which differs from the case of $u\neq v$. Therefore, it goes beyond the framework of \cref{def:swl}.
\begin{proposition}
    $\mathsf{GNN}\text{-}\mathsf{AK}$ is as powerful as $\mathsf{PSWL(VS)}$.
\end{proposition}
\begin{proof}
    Consider the following two SWL algorithms defined in \cref{def:node_marking_swl}: $\mathrm{(i)}$ $\mathsf{A}(\{\agg^\mathsf{L}_\mathsf{u},\agg^\mathsf{P}_\mathsf{vv}\},\mathsf{VS})$, and $\mathrm{(ii)}$ $\mathsf{A}(\{\agg^\mathsf{L}_\mathsf{u},\agg^\mathsf{P}_\mathsf{vv},\agg^\mathsf{G}_\mathsf{u}\},\mathsf{VS})$. It is clear that the stable color mapping of $\mathsf{GNN}\text{-}\mathsf{AK}$ is finer than that of $\mathsf{A}(\{\agg^\mathsf{L}_\mathsf{u},\agg^\mathsf{P}_\mathsf{vv}\},\mathsf{VS})$, but the stable color mapping of $\mathsf{A}(\{\agg^\mathsf{L}_\mathsf{u},\agg^\mathsf{P}_\mathsf{vv},\agg^\mathsf{G}_\mathsf{u}\},\mathsf{VS})$ is finer than $\mathsf{GNN}\text{-}\mathsf{AK}$. However, both algorithms are equivalent to $\mathsf{PSWL(VS)}$ as shown in \cref{thm:swl_hierarchy}. Therefore, by \cref{remark:finer}(b) $\mathsf{GNN}\text{-}\mathsf{AK}$ is as powerful as $\mathsf{PSWL(VS)}$.
\end{proof}

$\mathsf{GNN}\text{-}\mathsf{AK}\text{-}\mathsf{ctx}$ \citep{zhao2022stars}. The GNN aggregation scheme can be written as
\begin{align*}
    \chi_G^{(t+1)}(u, v) =\left\{\begin{array}{ll}
        \begin{aligned}
            \hash(&\chi_G^{(t)}(u,v),\\
                &\chi_G^{(t)}(v,v),\\
                &\ldblbrace\chi_G^{(t)}(u, w):w\in\gN_G(v)\rdblbrace)
        \end{aligned}
         & \text{if }u\neq v, \\
        \begin{aligned}
            \hash(&\chi_G^{(t)}(v,v),\\
                &\ldblbrace\chi_G^{(t)}(u, w):w\in\gN_G(v)\rdblbrace,\\
                &\ldblbrace\chi_G^{(t)}(u, w):w\in\gV_G\rdblbrace,\\
                &\ldblbrace\chi_G^{(t)}(w, v):w\in\gV_G\rdblbrace)
        \end{aligned}
         & \text{if }u= v.
    \end{array}\right. 
\end{align*}
$\mathsf{GNN}\text{-}\mathsf{AK}$ also uses the vertex-subgraph pooling. Compared with $\mathsf{GNN}\text{-}\mathsf{AK}$, $\mathsf{GNN}\text{-}\mathsf{AK}\text{-}\mathsf{ctx}$ further introduces the cross-graph global aggregation $\agg^\mathsf{G}_\mathsf{v}$ when $u=v$ (which they called the \emph{contextual encoding}).
\begin{proposition}
\label{thm:gnn-ak-ctx}
    $\mathsf{GNN}\text{-}\mathsf{AK}\text{-}\mathsf{ctx}$ is as powerful as $\mathsf{GSWL}$.
\end{proposition}
\begin{proof}
    Similar to the above proof, by using the result that $\mathsf{GSWL}$ is as powerful as $\mathsf{A}(\{\agg^\mathsf{L}_\mathsf{u},\agg^\mathsf{P}_\mathsf{vv},\agg^\mathsf{G}_\mathsf{v}\},\mathsf{VS})$ (\cref{thm:swl_hierarchy}), it is clear that $\mathsf{GSWL}$ is more powerful than $\mathsf{GNN}\text{-}\mathsf{AK}\text{-}\mathsf{ctx}$. It remains to prove that $\mathsf{GNN}\text{-}\mathsf{AK}\text{-}\mathsf{ctx}$ is more powerful than $\mathsf{GSWL}$.

    The proof is based on \cref{remark:finer}(c). Let $\chi$ be the stable color mapping of $\mathsf{GNN}\text{-}\mathsf{AK}\text{-}\mathsf{ctx}$. Define an auxiliary color mapping $\tilde\chi=\mathscr{T}(\{\agg^\mathsf{L}_{\mathsf{u}},\agg^\mathsf{G}_{\mathsf{v}}\},\chi)$ where $\mathscr{T}$ is defined in (\ref{eq:mapping_transform}). It suffices to prove that $\chi\preceq\tilde\chi$.
    
    Consider any vertices $u,v\in\gV_G$ and $x,y\in\gV_H$ satisfying $\chi_G(u,v)=\chi_H(x,y)$. Since the mapping $\chi$ is already stable, by definition of $\mathsf{GNN}\text{-}\mathsf{AK}\text{-}\mathsf{ctx}$ we have 
    \begin{align}
        \label{eq:proof_gnnakctx_0}
        \chi_G(v,v)&=\chi_H(y,y),\\
        \label{eq:proof_gnnakctx_1}
        \ldblbrace\chi_G(u, w):w\in\gN_G(v)\rdblbrace&=\ldblbrace\chi_H(x, z):z\in\gN_H(y)\rdblbrace.
    \end{align}
    Again by definition of the stable color mapping, (\ref{eq:proof_gnnakctx_0}) implies that
    \begin{align}
    \label{eq:proof_gnnakctx_2}
        \ldblbrace\chi_G(w, v):w\in\gV_G\rdblbrace&=\ldblbrace\chi_H(z, y):z\in\gV_H\rdblbrace.
    \end{align}
    Combining with (\ref{eq:proof_gnnakctx_1}) and (\ref{eq:proof_gnnakctx_2}), we obtain that $\tilde \chi_G(u,v)=\tilde \chi_H(x,y)$, concluding the proof.
\end{proof}

$\mathsf{DSS}\text{-}\mathsf{WL}$ \citep{bevilacqua2022equivariant}. The aggregation scheme of $\mathsf{DSS}\text{-}\mathsf{WL}$ can be written as
\begin{align*}
    \chi_G^{(t+1)}(u,v) = \hash(&\chi_G^{(t)}(u,v),\\
    &\ldblbrace\chi_G^{(t)}(u,w):w\in\gN_G(v)\rdblbrace,\\
    &\ldblbrace\chi_G^{(t)}(w,v):w\in\gV_G\rdblbrace,\\
    &\ldblbrace\chi_G^{(t)}(w,w'):w\in\gV_G,w'\in\gN_G(v)\rdblbrace).
\end{align*}
Here, the last aggregation does not belong to \cref{def:swl}. $\mathsf{DSS}\text{-}\mathsf{WL}$ also uses the vertex-subgraph pooling. 
\begin{proposition}
    $\mathsf{DSS}\text{-}\mathsf{WL}$ is as powerful as $\mathsf{GSWL}$.
\end{proposition}
\begin{proof}
    Clearly, $\mathsf{DSS}\text{-}\mathsf{WL}$ is more powerful than the SWL algorithm $\mathsf{A}(\{\agg^\mathsf{L}_\mathsf{u},\agg^\mathsf{G}_\mathsf{v}\},\mathsf{VS})$, which is precisely $\mathsf{GSWL}$. It remains to prove that $\mathsf{GSWL}$ is more powerful than $\mathsf{DSS}\text{-}\mathsf{WL}$.

    Similar to \cref{thm:gnn-ak-ctx}, the proof is based on \cref{remark:finer}(c). Let $\chi$ be the stable color mapping of $\mathsf{GSWL}$. For any vertices $u,v\in\gV_G$ and $x,y\in\gV_H$, if $\chi_G(u,v)=\chi_H(x,y)$, then we have 
    \begin{align}
    \label{eq:proof_dsswl_0}
        \ldblbrace\chi_G(u, w):w\in\gN_G(v)\rdblbrace&=\ldblbrace\chi_H(x, z):z\in\gN_H(y)\rdblbrace,\\
    \label{eq:proof_dsswl_1}
        \ldblbrace\chi_G(w, v):w\in\gV_G\rdblbrace&=\ldblbrace\chi_H(z, y):z\in\gV_H\rdblbrace.
    \end{align}
    Plugging (\ref{eq:proof_dsswl_1}) into (\ref{eq:proof_dsswl_0}) yields
    \begin{align*}
        \ldblbrace\ldblbrace\chi_G(w, w'):w\in\gV_G\rdblbrace :w'\in\gN_G(v)\rdblbrace=\ldblbrace\ldblbrace\chi_H(z, z'):z\in\gV_H\rdblbrace:z'\in\gN_H(y)\rdblbrace.
    \end{align*}
    Therefore,
    \begin{equation}
    \label{eq:proof_dsswl_2}
    \begin{aligned}
        \ldblbrace\chi_G(w, w'):w\in\gN_G(v),w'\in\gV_G\rdblbrace=\ldblbrace\chi_H(z, z'):w\in\gN_H(y),z'\in\gV_H\rdblbrace.
    \end{aligned}
    \end{equation}
    Combining with (\ref{eq:proof_dsswl_0}), (\ref{eq:proof_dsswl_1}), and (\ref{eq:proof_dsswl_2}), it implies that $\mathsf{DSS}\text{-}\mathsf{WL}$ cannot further refine the stable color mapping $\chi$, which concludes the proof.
\end{proof}

$\mathsf{SUN}$ \citep{frasca2022understanding}. The WL aggregation scheme can be written as
\begin{align*}
    \chi_G^{(t+1)}(u, v) =\hash(&\chi_G^{(t)}(u,v),\chi_G^{(t)}(u,u),\chi_G^{(t)}(v,v),\\
                &\ldblbrace\chi_G^{(t)}(u,w):w\in\gN_G(v)\rdblbrace,\\
                &\ldblbrace\chi_G^{(t)}(u,w):w\in\gV_G\rdblbrace,\\
                &\ldblbrace\chi_G^{(t)}(w,v):w\in\gV_G\rdblbrace,\\
                &\ldblbrace\chi_G^{(t)}(w,w'):w\in\gV_G,w'\in\gN_G(v)\rdblbrace).
\end{align*}
We note that the formulation of \citet{frasca2022understanding} slightly differs from the above WL formula, in that $\mathsf{SUN}$ introduces different model parameters separately for the cases of $u=v$ and $u\neq v$, respectively. However, when using a node marking policy, introducing two sets of parameters for the two cases does not theoretically increase the expressivity (but it may benefit practical performance in real-world tasks).
\begin{proposition}
    $\mathsf{SUN}$ is as powerful as $\mathsf{GSWL}$.
\end{proposition}
\begin{proof}
    The proof is almost the same as the above one, by using the result that $\mathsf{GSWL}$ is as powerful as $\mathsf{A}(\{\agg^\mathsf{L}_\mathsf{u},\agg^\mathsf{G}_\mathsf{u},\agg^\mathsf{G}_\mathsf{v},\agg^\mathsf{P}_\mathsf{uu},\agg^\mathsf{P}_\mathsf{vv}\},\mathsf{VS})$ (\cref{thm:swl_hierarchy}). We omit the details here.
\end{proof}

$\mathsf{ReIGN(2)}$ \citep{frasca2022understanding}. This GNN architecture is motivated by 2-IGN \citep{maron2019invariant,maron2019universality} by extending each basic equivariant linear operator into various types of local/global aggregations. Each atomic aggregation operation in $\mathsf{ReIGN(2)}$ can be symbolized as $\agg^{\mathsf{op}_1,\mathsf{op}_2}$, where $\mathsf{op}_1$ and $\mathsf{op}_2$ can take one of the following symbols: $\mathsf{Pu}$, $\mathsf{Pv}$, $\mathsf{G}$, $\mathsf{Lu}$, $\mathsf{Lv}$, and $\mathsf{D}$. The semantic of $\agg^{\mathsf{op}_1,\mathsf{op}_2}$ is defined as follows:
\begin{align*}
    \agg^{\mathsf{op}_1,\mathsf{op}_2}(u,v,G,\chi)=\ldblbrace \chi(w,w'):w\in \bigcirc,w'\in\bigcirc\rdblbrace,
\end{align*}
where the first/second $\bigcirc$ is filled by one of the following expression depending on $\mathsf{op}_1$/$\mathsf{op}_2$, respectively:
\begin{itemize}[topsep=0pt]
\setlength{\itemsep}{0pt}
    \item For symbol $\mathsf{Pu}$: $\bigcirc$ is filled by $\{u\}$;
    \item For symbol $\mathsf{Pv}$: $\bigcirc$ is filled by $\{v\}$;
    \item For symbol $\mathsf{Lu}$: $\bigcirc$ is filled by $\gN_G(u)$;
    \item For symbol $\mathsf{Lv}$: $\bigcirc$ is filled by $\gN_G(v)$;
    \item For symbol $\mathsf{G}$: $\bigcirc$ is filled by $\gV_G$;
    \item For symbol $\mathsf{D}$ : $\bigcirc$ is filled by $\{w\}$. This symbol corresponds to diagonal aggregation and can only be used by $\mathsf{op}_2$.
\end{itemize}
Based on the choice of $\mathsf{op}_1$ and $\mathsf{op}_2$, there are a total of $5\times 6-2=28$ nonequivalent aggregation operations. Note that $\agg^{\mathsf{Pu},\mathsf{D}}$ is equivalent to $\agg^{\mathsf{Pu},\mathsf{Pu}}$ and $\agg^{\mathsf{Pv},\mathsf{D}}$ is equivalent to $\agg^{\mathsf{Pv},\mathsf{Pv}}$. As a result, $\mathsf{ReIGN(2)}$ incorporates all these 28 aggregation operations into the WL iteration. Similar to $\mathsf{SUN}$, $\mathsf{ReIGN(2)}$ also introduces
different model parameters separately for the cases of $u = v$
and $u\neq v$, respectively. It can be calculated that the total number of linear equivariant transformations is $28+11=39$.

\begin{proposition}
\label{thm:reign}
    $\mathsf{ReIGN(2)}$ is as powerful as $\mathsf{SSWL}$.
\end{proposition}
\begin{proof}
    First, it is obvious that $\mathsf{ReIGN(2)}$ is more powerful than $\mathsf{SSWL}$. Therefore, it remains to prove that $\mathsf{SSWL}$ is more powerful than $\mathsf{ReIGN(2)}$. Similar to the previous propositions, the proof is based on \cref{remark:finer}(c). Let $\chi$ be the stable color mapping of $\mathsf{SSWL}$. Consider any vertices $u,v\in\gV_G$ and $x,y\in\gV_H$ satisfying $\chi_G(u,v)=\chi_H(x,y)$. Since $\mathsf{SSWL}$ is as powerful as $\mathsf{A}(\{\agg^\mathsf{L}_\mathsf{u},\agg^\mathsf{L}_\mathsf{x},\agg^\mathsf{G}_\mathsf{u},\agg^\mathsf{G}_\mathsf{v},\agg^\mathsf{P}_\mathsf{uu},\agg^\mathsf{P}_\mathsf{vv},\agg^\mathsf{P}_\mathsf{vu}\},\mathsf{VS})$, we have
    \begin{align*}
        \chi_G(u,u)&=\chi_H(x,x),\\
        \chi_G(v,v)&=\chi_H(y,y),\\
        \chi_G(v,u)&=\chi_H(y,x),\\
        \ldblbrace\chi_G(u,w):w\in\gV_G\rdblbrace&=\ldblbrace\chi_H(x,z):z\in\gV_H\rdblbrace,\\
        \ldblbrace\chi_G(w,v):w\in\gV_G\rdblbrace&=\ldblbrace\chi_H(z,y):z\in\gV_H\rdblbrace,\\
        \ldblbrace\chi_G(u,w):w\in\gN_G(v)\rdblbrace&=\ldblbrace\chi_H(x,z):z\in\gN_H(y)\rdblbrace,\\
        \ldblbrace\chi_G(w,v):w\in\gN_G(u)\rdblbrace&=\ldblbrace\chi_H(z,y):z\in\gN_H(F)\rdblbrace.
    \end{align*}
    Using a similar proof technique as previous propositions, we can show that $\chi$ cannot be refined by all $\agg^{\mathsf{op}_1,\mathsf{op}_2}$. We list one representative example below.

    The diagonal aggregation $\agg^{\mathsf{G},\mathsf{D}}$. Combining the fourth equation and the second equation above, we obtain
    $$\ldblbrace\chi_G(w,w):w\in\gV_G\rdblbrace=\ldblbrace\chi_H(z,z):z\in\gV_H\rdblbrace,$$
    as desired.
\end{proof}

\section{Proof of Theorems in \cref{sec:lfwl}}
\label{sec:proof_lfwl}

This section aims to prove \cref{thm:2lfwl}. Throughout this section, we denote $G=(\gV_G,\gE_G)$ and $H=(\gV_H,\gE_H)$ as any graphs. Denote $\chi^\mathsf{PSWL}$, $\chi^\mathsf{SSWL}$, $\chi^\mathsf{FWL}$, $\chi^\mathsf{LFWL}$, and $\chi^\mathsf{SLFWL}$ as the stable color mappings of $\mathsf{PSWL(VS)}$, $\mathsf{SSWL}$, $\mathsf{FWL(2)}$, $\mathsf{LFWL(2)}$, and $\mathsf{SLFWL(2)}$, respectively. 

We begin with the following simple fact, which holds by definition of the isomorphism type.

\begin{fact}
\label{thm:chi_fwl_isomorphism_type}
    Let $\chi\in\{\chi^\mathsf{FWL},\chi^\mathsf{LFWL},\chi^\mathsf{SLFWL}\}$. For any vertices $u,v\in\gV_G$ and $x,y\in\gV_H$, if $\chi_G(u,v)=\chi_H(x,y)$, then:
    \begin{itemize}[topsep=0pt]
    \setlength{\itemsep}{0pt}
        \item $u=v\iff x=y$;
        \item $\{u,v\}\in\gE_G\iff \{x,y\}\in\gE_H$.
    \end{itemize}
\end{fact}

\begin{lemma}
\label{thm:chi_fwl_basic}
    The following relations hold:
    \begin{itemize}[topsep=0pt]
    \setlength{\itemsep}{0pt}
        \item $\chi^\mathsf{LFWL}\preceq\chi^\mathsf{PSWL}$;
        \item $\chi^\mathsf{SLFWL}\preceq\chi^\mathsf{SSWL}$;
        \item $\chi^\mathsf{SLFWL}\preceq\chi^\mathsf{LFWL}$;
        \item $\chi^\mathsf{FWL}\preceq\chi^\mathsf{SLFWL}$.
    \end{itemize}
\end{lemma}
\begin{proof}
    Note that all the FWL-type algorithms considered in \cref{thm:chi_fwl_basic} use isomorphism type as initial colors, which is finer than node marking in SWL algorithms. In this case, it is straightforward to see that \cref{remark:finer}(c) still applies. Namely, it suffices to prove that the stable color mapping of each stronger algorithm cannot get refined using the aggregation scheme of the weaker algorithm.

    We first prove $\chi^\mathsf{LFWL}\preceq\mathscr{T}(\{\agg^\mathsf{L}_\mathsf{u},\agg^\mathsf{P}_\mathsf{vv}\},\chi^\mathsf{LFWL}):=\tilde\chi$, where $\mathscr{T}$ is defined in (\ref{eq:mapping_transform}). Consider any vertices $u,v\in\gV_G$ and $x,y\in\gV_H$ satisfying $\chi_G^\mathsf{LFWL}(u,v)=\chi_H^\mathsf{LFWL}(x,y)$. Then by definition,
    \begin{align*}
        \ldblbrace( \chi^\mathsf{LFWL}_G(u,w),\chi^\mathsf{LFWL}_G(w,v)):w\in\gN_G^1(v)\rdblbrace=\ldblbrace( \chi^\mathsf{LFWL}_H(x,z),\chi^\mathsf{LFWL}_H(z,y)):z\in\gN_H^1(y)\rdblbrace.
    \end{align*}
    It must be the case that
    \begin{align*}
        (\chi^\mathsf{LFWL}_G(u,v),\chi^\mathsf{LFWL}_G(v,v))&=(\chi^\mathsf{LFWL}_H(x,y),\chi^\mathsf{LFWL}_H(y,y)),\\
        \ldblbrace( \chi^\mathsf{LFWL}_G(u,w),\chi^\mathsf{LFWL}_G(w,v)):w\in\gN_G(v)\rdblbrace&=\ldblbrace( \chi^\mathsf{LFWL}_H(x,z),\chi^\mathsf{LFWL}_H(z,y)):z\in\gN_H(y)\rdblbrace,
    \end{align*}
    due to \cref{thm:chi_fwl_isomorphism_type}. Therefore,
    \begin{align*}
        \chi^\mathsf{LFWL}_G(v,v)=\chi^\mathsf{LFWL}_H(y,y)
    \end{align*}
    and
    \begin{align*}
        \ldblbrace \chi^\mathsf{LFWL}_G(u,w): w\in\gN_G(v)\rdblbrace
        =\ldblbrace \chi^\mathsf{LFWL}_H(x,z): z\in\gN_H(y)\rdblbrace.
    \end{align*}
    Namely, $\tilde\chi_G(u,v)=\tilde\chi_H(x,y)$. This proves $\chi^\mathsf{LFWL}\preceq\chi^\mathsf{PSWL}$.

    We next prove $\chi^\mathsf{SLFWL}\preceq\mathscr{T}(\{\agg^\mathsf{L}_\mathsf{u},\agg^\mathsf{L}_\mathsf{v}\},\chi^\mathsf{SLFWL}):=\tilde\chi$. Consider any vertices $u,v\in\gV_G$ and $x,y\in\gV_H$ satisfying $\chi_G^\mathsf{SLFWL}(u,v)=\chi_H^\mathsf{SLFWL}(x,y)$. Then by definition,
    \begin{align*}
        &\ldblbrace( \chi_G^\mathsf{SLFWL}(u,w),\chi_G^\mathsf{SLFWL}(w,v)):w\in\gN_G^1(u)\cup\gN_G^1(v)\rdblbrace\\
        =&\ldblbrace( \chi_H^\mathsf{SLFWL}(x,z),\chi_H^\mathsf{SLFWL}(z,y)):z\in\gN_H^1(x)\cup\gN_H^1(y)\rdblbrace.
    \end{align*}
    Using \cref{thm:chi_fwl_isomorphism_type} we have
    \begin{align*}
        \ldblbrace( \chi_G^\mathsf{SLFWL}(u,w),\chi_G^\mathsf{SLFWL}(w,v)):w\in\gN_G(u)\rdblbrace&=\ldblbrace( \chi_H^\mathsf{SLFWL}(x,z),\chi_H^\mathsf{SLFWL}(z,y)):z\in\gN_H(F)\rdblbrace,\\
        \ldblbrace( \chi_G^\mathsf{SLFWL}(u,w),\chi_G^\mathsf{SLFWL}(w,v)):w\in\gN_G(v)\rdblbrace&=\ldblbrace( \chi_H^\mathsf{SLFWL}(x,z),\chi_H^\mathsf{SLFWL}(z,y)):z\in\gN_H(y)\rdblbrace.
    \end{align*}
    Therefore,
    \begin{align*}
        \ldblbrace \chi_G^\mathsf{SLFWL}(u,w): w\in\gN_G(v)\rdblbrace
        &=\ldblbrace \chi_H^\mathsf{SLFWL}(x,z): z\in\gN_H(y)\rdblbrace,\\
        \ldblbrace\chi_G^\mathsf{SLFWL}(w,v): w\in\gN_G(u)\rdblbrace
        &=\ldblbrace\chi_H^\mathsf{SLFWL}(z,y): z\in\gN_H(F)\rdblbrace.
    \end{align*}
    Namely, $\tilde\chi_G(u,v)=\tilde\chi_H(x,y)$. This proves $\chi^\mathsf{SLFWL}\preceq\chi^\mathsf{SSWL}$.

    The third and fourth bullets follow exactly the same procedure, so we omit the proof for clarity.
\end{proof}

\begin{lemma}
\label{thm:fwl_pooling}
    Let $\chi\in\{\chi^\mathsf{FWL},\chi^\mathsf{LFWL},\chi^\mathsf{SLFWL}\}$. If 
    $$\ldblbrace\chi_G(u,v):u,v\in\gV_G\rdblbrace=\ldblbrace\chi_H(x,y):x,y\in\gV_H\rdblbrace,$$
    then 
    \begin{align*}
        \ldblbrace\ldblbrace\chi_G(u,v):v\in\gV_G\rdblbrace:u\in\gV_G\rdblbrace=\ldblbrace\ldblbrace\chi_H(x,y):y\in\gV_H\rdblbrace:x\in\gV_H\rdblbrace.
    \end{align*}
\end{lemma}
\begin{proof}
    Based on the assumption of \cref{thm:fwl_pooling} and \cref{thm:chi_fwl_isomorphism_type}, we have
    $$\ldblbrace\chi_G(u,u):u\in\gV_G\rdblbrace=\ldblbrace\chi_H(x,x):x\in\gV_H\rdblbrace.$$
    Therefore, it suffices to prove that for any vertices $u\in\gV_G$ and $x\in\gV_H$, if $\chi_G(u,u)=\chi_H(x,x)$, then
    \begin{equation}
    \label{eq:proof_fwl_pooling}
        \ldblbrace\chi_G(u,v):v\in\gV_G\rdblbrace=\ldblbrace\chi_H(x,y):y\in\gV_H\rdblbrace.
    \end{equation}
    Based on \cref{thm:chi_fwl_basic}, we have $\chi\preceq \chi^\mathsf{PSWL}$. Note that $\chi^\mathsf{PSWL}$ incorporates the aggregation $\agg^\mathsf{L}_\mathsf{u}$. Therefore, \cref{thm:local_lemma} applies. We can follow the same proof technique of \cref{thm:local_lemma} to obtain (\ref{eq:proof_fwl_pooling}).
\end{proof}

We are now ready to prove \cref{thm:2lfwl}, which we restate below:

\textbf{\cref{thm:2lfwl}.} \emph{The following relations hold:
    \begin{itemize}[topsep=0pt]
    \setlength{\itemsep}{0pt}
        \item $\mathsf{LFWL(2)}\preceq\mathsf{SLFWL(2)}\preceq\mathsf{FWL(2)}$;
        \item $\mathsf{PSWL}(\mathsf{VS})\preceq\mathsf{LFWL(2)}$;
        \item $\mathsf{SSWL}\preceq\mathsf{SLFWL(2)}$.
    \end{itemize}}
\begin{proof}
    The first bullet readily follows from \cref{thm:chi_fwl_basic} and \cref{remark:finer}(b). For the other two bullets, although these algorithms have different pooling paradigms, we have proved that the pooling paradigm of FWL-type algorithms is as powerful as the pooling paradigm $\mathsf{VS}$ (\cref{thm:fwl_pooling}). Therefore, the results hold by \cref{thm:chi_fwl_basic}.
\end{proof}

\section{Proof of Theorems in \cref{sec:pebbling_game}}
\label{sec:proof_pebbling_game}

This section proves the equivalence between SWL/FWL-type algorithms and pebbling games. For ease of presentation, we first define several notations.

Let $G=(\gV_G,\gE_G)$ and $H=(\gV_H,\gE_H)$ be two graphs, and let $u,v$ be two types of pebbles. For each type of pebbles $u$, the placement information can be represented by a vertex pair $(u_G,u_H)$ where $u_G\in\gV_G$ and $u_H\in\gV_H$ are the corresponding vertices that hold pebble $u$. Without abuse of notation, we also use the symbol $u$ to represent the placement information of pebble $u$, i.e. $u=(u_G,u_H)$.

We next define a game modified from \cref{sec:pebbling_game}, called the $L$-round $(u,v)$-pebbling game.

\begin{definition}
    Given aggregation scheme $\gA$ and an integer $L\in\mathbb N$, define the $L$-round $(u,v)$-pebbling game $\mathsf{G}^{\gA,L}(u;v)$ as follows. Initially, pebbles $u$ and $v$ are already placed on graphs $G$ and $H$ according to specified locations $u=(u_G,u_H)$, $v=(v_G,v_H)$. The game has $L$ rounds. In each round, Spoiler and Duplicator can change the position of $u$ and $v$ according to the game rules of $\gA$ defined in \cref{sec:pebbling_game}. Spoiler wins if after certain round $0\le l\le L$, the isomorphism type of vertex pair $(u_G,v_G)$ in graph $G$ differs from the isomorphism type of vertex pair $(u_H,v_H)$ in graph $H$. Duplicator wins the game if Spoiler does not win after playing $L$ rounds.
\end{definition}

We are ready to establish the connection between SWL and the $(u,v)$-pebbling game. Below, denote $\chi^{\gA,(t)}$ as the color mapping of SWL algorithm $\mathsf{A}(\gA,\pool)$ at iteration $t$.

\begin{lemma}
\label{thm:pebble_lemma1}
    Let $l\in\mathbb N$ be any integer. For any vertices $u_G,v_G\in\gV_G$ and $u_H,v_H\in\gV_H$, if $\chi^{\gA,(l)}_G(u_G,v_G)\neq \chi^{\gA,(l)}_H(u_H,v_H)$, then Spoiler can win the $l$-round $(u,v)$-pebbling game $\mathsf{G}^{\gA,l}(u;v)$ with $u=(u_G,u_H)$, $v=(v_G,v_H)$.
\end{lemma}
\begin{proof}
    The proof is based on induction over $l$. First consider the base case of $l=0$. If $\chi^{\gA,(0)}_G(u_G,v_G)\neq \chi^{\gA,(0)}_H(u_H,v_H)$, by definition of node marking policy we have either ($u_G=v_G$, $u_H\neq v_H$) or ($u_G\neq v_G$, $u_H=v_H$). Clearly, $(u_G,v_G)$ and $(u_H,v_H)$ have different isomorphism types and thus Spoiler wins.

    Now assume that \cref{thm:pebble_lemma1} holds for $l\le L$, and consider $l=L+1$. Let $\chi^{\gA,(L+1)}_G(u_G,v_G)\neq \chi^{\gA,(L+1)}_H(u_H,v_H)$. If $\chi^{\gA,(L)}_G(u_G,v_G)\neq \chi^{\gA,(L)}_H(u_H,v_H)$, then by induction Spoiler wins. Otherwise, there exists an aggregation operation $\agg\in\gA$ such that
    $$\agg(u_G,v_G,G,\chi^{\gA,(L)}_G)\neq \agg(u_H,v_H,H,\chi^{\gA,(L)}_H).$$
    We separately consider which type of atomic aggregation operation $\agg$ is:
    \begin{itemize}[topsep=0pt]
    \setlength{\itemsep}{0pt}
        \item Single-point aggregation $\agg^\mathsf{P}_\mathsf{vu}$. In this case, we have $\chi_G^{\gA,(L)}(v_G,u_G) \neq \chi_H^{\gA,(L)}(v_H,u_H)$. In the first round, Spoiler can choose to swap pebbles $u$ and $v$. The remaining game will then be equivalent to $\mathsf{G}^{\gA,L}(u,v)$ with $u=(v_G,v_H)$, $v=(u_G,u_H)$. By induction, Spoiler wins the game.
        \item Single-point aggregation $\agg^\mathsf{P}_\mathsf{uu}$. In this case, we have $\chi_G^{\gA,(L)}(u_G,u_G) \neq \chi_H^{\gA,(L)}(u_H,u_H)$. In the first round, Spoiler can choose to move pebbles $v$ to the position of $u$. The remaining game will then be equivalent to $\mathsf{G}^{\gA,L}(u,v)$ with $u=(u_G,u_H)$, $v=(u_G,u_H)$. By induction, Spoiler wins the game.
        \item Global aggregation $\agg^\mathsf{G}_\mathsf{u}$. In this case, we have
        \begin{align*}
            \ldblbrace\chi_G^{\gA,(L)}(u_G,w_G):w_G\in\gV_G\rdblbrace
            \neq \ldblbrace\chi_H^{\gA,(L)}(u_H,w_H):w_H\in\gV_H\rdblbrace.
        \end{align*}
        Therefore, there exists a color $c$ such that $|\gC_G(u_G,c)|\neq |\gC_H(u_H,c)|$, where we denote $$\gC_G(u_G,c)=|\{w_G\in\gV_G:\chi_G^{\gA,(L)}(u_G,w_G)=c\}|.$$
        If $|\gC_G(u_G,c)|> |\gC_H(u_H,c)|$, Spoiler can select the vertex subset $\gS^\mathsf{S}=\gC_G(u_G,c)\subset\gV_G$. It can be seen that no matter how Duplicator responds with $\gS^\mathsf{D}\subset\gV_H$, there exists $w_H\in\gS^\mathsf{D}$ such that $\chi_H^{\gA,(L)}(u_H,w_H)\neq c$. Spoiler thus select this vertex $x^\mathsf{S}=w_H$, and no matter how Duplicator responds with $x^\mathsf{D}=w_G\in\gS^\mathsf{S}$, we have $\chi_G^{\gA,(L)}(u_G,w_G)\neq \chi_H^{\gA,(L)}(u_H,w_H)$. The remaining game will then be equivalent to $\mathsf{G}^{\gA,L}(u,v)$ with $u=(u_G,u_H)$, $v=(w_G,w_H)$. By induction, Spoiler wins the game.

        If $|\gC_G(u_G,c)|< |\gC_H(u_H,c)|$, Spoiler can select the vertex subset $\gS^\mathsf{S}=\gC_H(u_H,c)\subset\gV_H$, and the conclusion is the same.
        \item Local aggregation $\agg^\mathsf{L}_\mathsf{u}$. In this case, we have
        \begin{align*}
            \ldblbrace\chi_G^{\gA,(L)}(u_G,w_G):w_G\in\gN_G(v_G)\rdblbrace
            \neq \ldblbrace\chi_H^{\gA,(L)}(u_H,w_H):w_H\in\gN_H(v_H)\rdblbrace.
        \end{align*}
        Therefore, there exists a color $c$ such that $|\gC_G(u_G,v_G,c)|\neq |\gC_H(u_H,v_H,c)|$, where we denote $$\gC_G(u_G,v_G,c)=|\{w_G\in\gN_G(v_G):\chi_G^{\gA,(L)}(u_G,w_G)= c\}|.$$
        If $|\gC_G(u_G,v_G,c)|> |\gC_H(u_H,v_H,c)|$, Spoiler can select the vertex subset $\gS^\mathsf{S}=\gC_G(u_G,v_G,c)\subset\gN_G(v_G)$. If $|\gC_G(u_G,v_G,c)|< |\gC_H(u_H,v_H,c)|$, Spoiler can select the vertex subset $\gS^\mathsf{S}=\gC_H(u_H,v_H,c)\subset\gN_H(v_H)$. Using a similar analysis as the above case, we can conclude that Spoiler wins the game.
    \end{itemize}
    The cases of $\agg^\mathsf{L}_\mathsf{v}$, $\agg^\mathsf{G}_\mathsf{v}$, $\agg^\mathsf{P}_\mathsf{vv}$ are similar (symmetric) to $\agg^\mathsf{L}_\mathsf{u}$, $\agg^\mathsf{G}_\mathsf{u}$, $\agg^\mathsf{P}_\mathsf{uu}$, so we omit them for clarity. We have concluded the induction step.
\end{proof}

\begin{lemma}
\label{thm:pebble_lemma2}
    Let $l\in\mathbb N$ be any integer. Assume $\{\agg^\mathsf{L}_\mathsf{u},\agg^\mathsf{L}_\mathsf{v}\}\cap\gA\neq\emptyset$. For any vertices $u_G,v_G\in\gV_G$ and $u_H,v_H\in\gV_H$, if $\chi^{\gA,(l+1)}_G(u_G,v_G)=\chi^{\gA,(l+1)}_H(u_H,v_H)$, then Duplicator can win the $l$-round $(u,v)$-pebbling game $\mathsf{G}^{\gA,l}(u;v)$ with $u=(u_G,u_H)$, $v=(v_G,v_H)$.
\end{lemma}
\begin{proof}
    The proof is based on induction over $l$. First, consider the base case of $l=0$. Let $\chi^{\gA,(1)}_G(u_G,v_G)=\chi^{\gA,(1)}_H(u_H,v_H)$. If $u_G=v_G$, then $u_H=v_H$ (due to the node marking policy). If $\{u_G,v_G\}\in\gE_G$, then $\{u_H,v_H\}\in\gE_H$ (which follows by applying the local aggregation, similar to the proof of \cref{thm:node_marking_implies_distance}). Therefore, $(u_G,v_G)$ and $(u_H,v_H)$ have the same isomorphism type.

    Now assume that \cref{thm:pebble_lemma2} holds for $l\le L$, and consider $l=L+1$. Let $\chi^{\gA,(L+2)}_G(u_G,v_G)= \chi^{\gA,(L+2)}_H(u_H,v_H)$. Then,
    \begin{equation}
    \label{eq:proof_pebble_1}
        \agg(u_G,v_G,G,\chi^{\gA,(L+1)}_G)= \agg(u_H,v_H,H,\chi^{\gA,(L+1)}_H)
    \end{equation}
    holds for all $\agg\in\gA$. Separately consider various possible strategies for Spoiler:
    \begin{itemize}[topsep=0pt]
    \setlength{\itemsep}{0pt}
        \item If $\agg^\mathsf{P}_\mathsf{vu}\in\gA$ and Spoiler chooses to swap pebbles $u$ and $v$. Setting $\agg=\agg^\mathsf{P}_\mathsf{vu}$ in (\ref{eq:proof_pebble_1}) yields $\chi^{\gA,(L+1)}_G(v_G,u_G)=\chi^{\gA,(L+1)}_H(v_H,u_H)$. The remaining game is equivalent to $\mathsf{G}^{\gA,L}(u,v)$ with $u=(v_G,v_H)$, $v=(u_G,u_H)$. By induction, Duplicator wins the game.
        \item If $\agg^\mathsf{P}_\mathsf{uu}\in\gA$ and Spoiler chooses to move pebbles $v$ to the position of pebble $u$. This case is similar to the above one, and we have $\chi^{\gA,(L+1)}_G(u_G,u_G)=\chi^{\gA,(L+1)}_H(u_H,u_H)$. The remaining game is equivalent to $\mathsf{G}^{\gA,L}(u,v)$ with $u=(u_G,u_H)$, $v=(u_G,u_H)$. By induction, Duplicator wins the game.
        \item If $\agg^\mathsf{G}_\mathsf{u}\in\gA$, and Spoiler chooses a subset $\gS^\mathsf{S}$. Setting $\agg=\agg^\mathsf{G}_\mathsf{u}$ in (\ref{eq:proof_pebble_1}) yields 
        \begin{align*}
            \ldblbrace\chi_G^{\gA,(L+1)}(u_G,w_G):w_G\in\gV_G\rdblbrace
            = \ldblbrace\chi_H^{\gA,(L+1)}(u_H,w_H):w_H\in\gV_H\rdblbrace.
        \end{align*}
        If $\gS^\mathsf{S}\subset\gV_G$, then Duplicator can respond with a subset $\gS^\mathsf{D}\subset\gV_H$ such that
        \begin{align*}
            \ldblbrace\chi_G^{\gA,(L+1)}(u_G,w_G):w_G\in\gS^\mathsf{S}\rdblbrace=\ldblbrace\chi_H^{\gA,(L+1)}(u_H,w_H):w_H\in\gS^\mathsf{D}\rdblbrace.
        \end{align*}
        If $\gS^\mathsf{S}\subset\gV_H$, then Duplicator can respond with a subset $\gS^\mathsf{D}\subset\gV_G$ such that
        \begin{align*}
            \ldblbrace\chi_G^{\gA,(L+1)}(u_G,w_G):w_G\in\gS^\mathsf{D}\rdblbrace= \ldblbrace\chi_H^{\gA,(L+1)}(u_H,w_H):w_H\in\gS^\mathsf{S}\rdblbrace.
        \end{align*}
        In both cases, we clearly have $|\gS^\mathsf{S}|=|\gS^\mathsf{D}|$. Next, no matter how Spoiler moves pebble $v$ to a vertex $x^\mathsf{S}\in\gS^\mathsf{D}$, Duplicator can always respond by moving the other pebble $v$ to a vertex $x^\mathsf{D}\in\gS^\mathsf{S}$, such that $\chi_G^{\gA,(L+1)}(u_G,\tilde v_G)=\chi_H^{\gA,(L+1)}(u_H,\tilde v_H)$, where $(\tilde v_G,\tilde v_H)$ is the new position of pebbles $v$. The remaining game is equivalent to $\mathsf{G}^{\gA,L}(u,v)$ with $u=(u_G,u_H)$, $v=(\tilde v_G,\tilde v_H)$. By induction, Duplicator wins the game.
        \item If $\agg^\mathsf{L}_\mathsf{u}\in\gA$, then all the procedure is similar to the above one except that the subsets $\gS^\mathsf{S}$ and $\gS^\mathsf{D}$ contain only the neighboring vertices adjacent to pebbles $v$.
    \end{itemize}
    The cases of $\agg^\mathsf{L}_\mathsf{v}$, $\agg^\mathsf{G}_\mathsf{v}$, $\agg^\mathsf{P}_\mathsf{vv}$ are similar (symmetric) to $\agg^\mathsf{L}_\mathsf{u}$, $\agg^\mathsf{G}_\mathsf{u}$, $\agg^\mathsf{P}_\mathsf{uu}$, so we omit them for clarity. We have concluded the induction step.
\end{proof}

Combining \cref{thm:pebble_lemma1,thm:pebble_lemma2} immediately yields the following theorem:

\begin{theorem}
\label{thm:swl_and_uvpebble}
    Let $\chi^{\gA}$ be the stable color mapping of SWL algorithm $\mathsf{A}(\gA,\pool)$, satisfying $\{\agg^\mathsf{L}_\mathsf{u},\agg^\mathsf{L}_\mathsf{v}\}\cap\gA\neq\emptyset$. For any vertices $u_G,v_G\in\gV_G$ and $u_H,v_H\in\gV_H$, $\chi^{\gA}_G(u_G,v_G)= \chi^{\gA}_H(u_H,v_H)$ if and only if Duplicator can win the $l$-round $(u,v)$-pebbling game $\mathsf{G}^{\gA,l}(u;v)$ for any $l\in\mathbb N$ with $u=(u_G,u_H)$, $v=(v_G,v_H)$.
\end{theorem}

We next turn to FWL-type algorithms. We can similarly define the $L$-round $(u,v)$-pebbling game $\mathsf{G}^{L}$ for $\mathsf{FWL(2)}$, $\mathsf{LFWL(2)}$, and $\mathsf{SLFWL(2)}$. We have the following theorem parallel to \cref{thm:swl_and_uvpebble}.

\begin{theorem}
\label{thm:fwl_and_uvpebble}
    Let $\chi$ be the stable color mapping of any FWL-type algorithm, e.g., $\mathsf{FWL(2)}$, $\mathsf{LFWL(2)}$, and $\mathsf{SLFWL(2)}$. For any vertices $u_G,v_G\in\gV_G$ and $u_H,v_H\in\gV_H$, $\chi_G(u_G,v_G)= \chi_H(u_H,v_H)$ if and only if Duplicator can win the corresponding $l$-round $(u,v)$-pebbling game $\mathsf{G}^{l}(u;v)$ for any $l\in\mathbb N$ with $u=(u_G,u_H)$, $v=(v_G,v_H)$.
\end{theorem}
\begin{proof}
    The proof is highly similar to the proof of \cref{thm:pebble_lemma1,thm:pebble_lemma2}. For clarity, we only take $\mathsf{LFWL(2)}$ as an example. We use induction over $l$ to prove that for any $l\in\mathbb N$, $\chi_G^{(l)}(u_G,v_G)= \chi_H^{(l)}(u_H,v_H)$ if and only if Duplicator can win the $l$-round $(u,v)$-pebbling game $\mathsf{G}^{l}(u;v)$ with $u=(u_G,u_H)$, $v=(v_G,v_H)$. The base case of $l=0$ is trivial.

    For the induction step, suppose the result holds for $l\le L$ and consider $l=L+1$. Let $\chi_G^{(L+1)}(u_G,v_G)\neq \chi_H^{(L+1)}(u_H,v_H)$. If $\chi_G^{(L)}(u_G,v_G)\neq \chi_H^{(L)}(u_H,v_H)$, Spoiler wins by induction. Otherwise, by definition of $\mathsf{LFWL(2)}$ we have
    \begin{align*}
        \ldblbrace( \chi^{(L)}_G(u_G,w_G),\chi^{(L)}_G(w_G,v_G)):w_G\in\gN_G^1(v_G)\rdblbrace
        \neq \ldblbrace( \chi^{(L)}_H(u_H,w_H),\chi^{(L)}_H(w_H,v_H)):w_H\in\gN_H^1(v_H)\rdblbrace.
    \end{align*}
    Therefore, there exists a color $c$ such that $|\gC_G(u_G,v_G,c)|\neq |\gC_H(u_H,v_H,c)|$, where we denote $$\gC_G(u,v,c)=|\{w\in\gN_G^1(v):(\chi_G^{(L)}(u,w),\chi_G^{(L)}(w,v))=c\}|.$$
    Assume $|\gC_G(u_G,v_G,c)|> |\gC_H(u_H,v_H,c)|$ without loss of generality, then Spoiler can select the vertex subset $\gS^\mathsf{S}=\gC_G(u_G,v_G,c)\subset\gN_G^1(v_G)$. No matter how Duplicator responds with $\gS^\mathsf{D}\subset\gN_H^1(v_H)$, there exists $w_H\in\gS^\mathsf{D}$ such that $(\chi_H^{(L)}(u_H,w_H),\chi_H^{(L)}(w_H,v_H))\neq c$. Spoiler thus select this vertex $x^\mathsf{S}=w_H$, and no matter how Duplicator responds with $x^\mathsf{D}=w_G\in\gS^\mathsf{S}$, we have either $\chi_G^{(L)}(u_G,w_G)\neq \chi_H^{(L)}(u_H,w_H)$ or $\chi_G^{(L)}(w_G,v_G)\neq \chi_H^{(L)}(w_H,v_H)$. Spoiler chooses to move pebbles $v$ or $u$ depending on which relation does not hold. The remaining game will then be equivalent to $\mathsf{G}^{L}(u,v)$ with $u=(\tilde u_G,\tilde u_H)$, $v=(\tilde v_G,\tilde v_H)$ such that $\chi_G^{(L)}(\tilde u_G,\tilde v_G)\neq \chi_H^{(L)}(\tilde u_H,\tilde v_H)$. By induction, Spoiler wins the game.

    For the converse direction, the proof is similar and we omit it for clarity.
\end{proof}

Finally, we complete the analysis by incorporating different pooling paradigms into pebbling games. We will prove the following general result:

\begin{lemma}
\label{thm:pooling_pebble}
    Let $\chi$ be the stable color mapping of any WL algorithm and let $\mathsf{G}$ be the corresponding pebbling game, such that $\chi_G(u_G,v_G)=\chi_H(u_H,v_H)$ if and only if Duplicator can win the $l$-round $(u,v)$-pebbling game $\mathsf{G}^{l}(u;v)$ for all $l\in\mathbb N$ with $u=(u_G,u_H)$, $v=(v_G,v_H)$. Then,
    \begin{itemize}[topsep=0pt]
    \raggedright
    \setlength{\itemsep}{0pt}
        \item $\ldblbrace \chi_G(u_G,v_G):u_G,v_G\in\gV_G\rdblbrace=\ldblbrace \chi_H(u_H,v_H):u_H,v_H\in\gV_H\rdblbrace$ if and only if Duplicator can win the pebbling game when $u=(u_G,u_H)$, $v=(v_G,v_H)$ are selected according to the game rule of FWL-type algorithms defined in \cref{sec:pebbling_game};
        \item $\ldblbrace \ldblbrace\chi_G(u_G,v_G):v_G\in\gV_G\rdblbrace:u_G\in\gV_G\rdblbrace=\ldblbrace \ldblbrace\chi_H(u_H,v_H):v_H\in\gV_H\rdblbrace:u_H\in\gV_H\rdblbrace$ if and only if Duplicator can win the pebbling game when $u=(u_G,u_H)$, $v=(v_G,v_H)$ are selected according to the game rule of $\mathsf{VS}$ pooling defined in \cref{sec:pebbling_game};
        \item $\ldblbrace \ldblbrace\chi_G(u_G,v_G):u_G\in\gV_G\rdblbrace:v_G\in\gV_G\rdblbrace=\ldblbrace \ldblbrace\chi_H(u_H,v_H):u_H\in\gV_H\rdblbrace:v_H\in\gV_H\rdblbrace$ if and only if Duplicator can win the pebbling game when $u=(u_G,u_H)$, $v=(v_G,v_H)$ are selected according to the game rule of $\mathsf{SV}$ pooling defined in \cref{sec:pebbling_game}.
    \end{itemize}
\end{lemma}
\begin{proof}
    We only prove the second bullet and other cases are similar. First assume $\ldblbrace \ldblbrace\chi_G(u_G,v_G):v_G\in\gV_G\rdblbrace:u_G\in\gV_G\rdblbrace=\ldblbrace \ldblbrace\chi_H(u_H,v_H):v_H\in\gV_H\rdblbrace:u_H\in\gV_H\rdblbrace$. According to the game rule, both players first place pebbles $u$ based on a vertex selection procedure. Without loss of generality, suppose Spoiler chooses a subset $\gS^\mathsf{S}\subset\gV_G$. Then Duplicator can respond with a subset $\gS^\mathsf{D}\subset\gV_H$ such that
    \begin{align*}
        \ldblbrace \ldblbrace\chi_G(u_G,v_G):v_G\in\gV_G\rdblbrace:u_G\in\gS^\mathsf{S}\rdblbrace=\ldblbrace \ldblbrace\chi_H(u_H,v_H):v_H\in\gV_H\rdblbrace:u_H\in\gS^\mathsf{D}\rdblbrace.
    \end{align*}
    Then no matter how Spoiler selects a vertex $x^\mathsf{S}=u_H\in\gS^\mathsf{D}$, Duplicator can always select $x^\mathsf{D}=u_G\in\gS^\mathsf{S}$, such that
    \begin{align*}
        \ldblbrace\chi_G(u_G,v_G):v_G\in\gV_G\rdblbrace= \ldblbrace\chi_H(u_H,v_H):v_H\in\gV_H\rdblbrace.
    \end{align*}
    Similarly, after selecting the position of pebbles $v$, Duplicator always has a strategy to ensure that $\chi_G(u_G,v_G)=\chi_H(u_H,v_H)$. For the remaining game, Duplicator can win due to the assumption of \cref{thm:pooling_pebble}.

    For the converse direction, assume $\ldblbrace \ldblbrace\chi_G(u_G,v_G):v_G\in\gV_G\rdblbrace:u_G\in\gV_G\rdblbrace\neq\ldblbrace \ldblbrace\chi_H(u_H,v_H):v_H\in\gV_H\rdblbrace:u_H\in\gV_H\rdblbrace$. Similar to the proof of global aggregation in \cref{thm:pebble_lemma1}, Spoiler has a strategy to ensure that 
    \begin{align*}
        \ldblbrace\chi_G(u_G,v_G):v_G\in\gV_G\rdblbrace\neq\ldblbrace\chi_H(u_H,v_H):v_H\in\gV_H\rdblbrace
    \end{align*}
    after placing pebble $u$ to position $(u_G,u_H)$. Again, after placing pebble $v$ to position $(v_G,v_H)$, Spoiler has a strategy to ensure that $\chi_G(u_G,v_G)\neq\chi_H(u_H,v_H)$. For the remaining game, Spoiler can win due to the assumption of \cref{thm:pooling_pebble}.
\end{proof}

Consequently, \cref{thm:swl_pebble_game} and \cref{thm:fwl_pebble_game} holds by \cref{thm:swl_and_uvpebble,thm:fwl_and_uvpebble,thm:pooling_pebble}.

\section{Proof of Separation Results (\cref{thm:separation})}
\label{sec:proof_separation}

This section contains the proof of the main result in this paper (\cref{thm:separation}). The proof is quite involved and is divided into three parts. First, we introduce a novel construction of counterexample graphs that are based on (and greatly extend) the work of \citet{furer2001weisfeiler}. We provide an in-depth analysis of the isomorphism properties of these counterexample graphs through a set of key theorems. Then, in light of the special properties, we simplify the pebbling game developed in \cref{sec:pebbling_game} for each type of SWL/FWL algorithm, which specifically targets these counterexample graphs. Finally, we prove all separation results in \cref{sec:separation_results} using the pebbling game viewpoint and give concrete counterexample graphs for each pair of algorithms.

\subsection{Generalized F{\"u}rer graphs and their properties}
\label{sec:proof_separation_part1}

\begin{figure*}[t]
    \small
    \setlength{\tabcolsep}{0pt}
    \begin{tabular}{ccc}
        \includegraphics[width=0.16\textwidth]{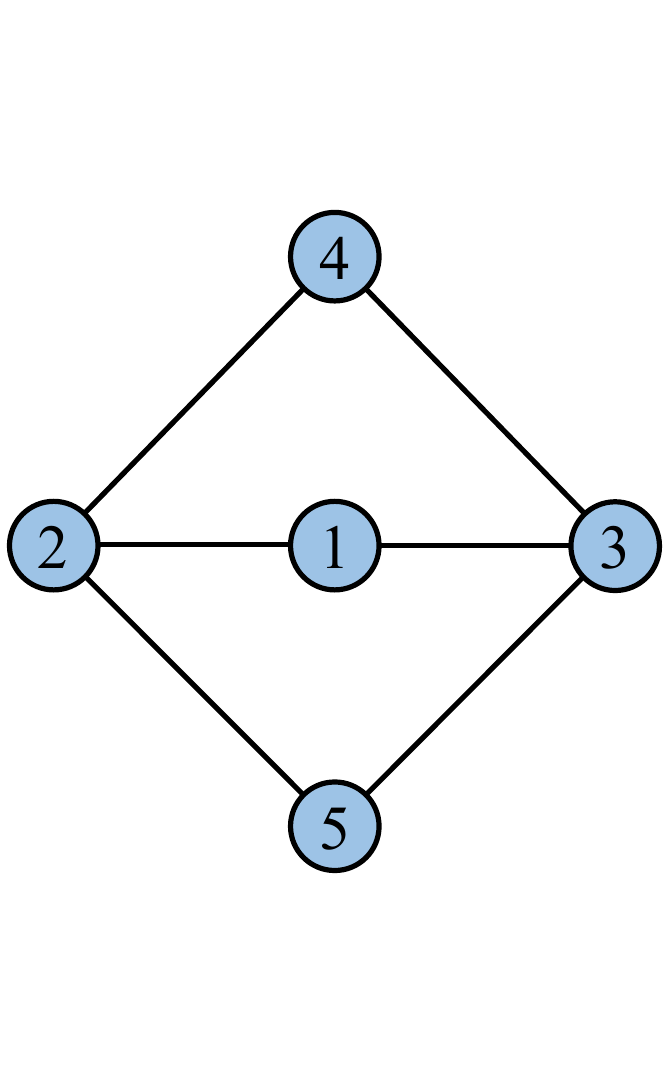} & \includegraphics[width=0.39\textwidth]{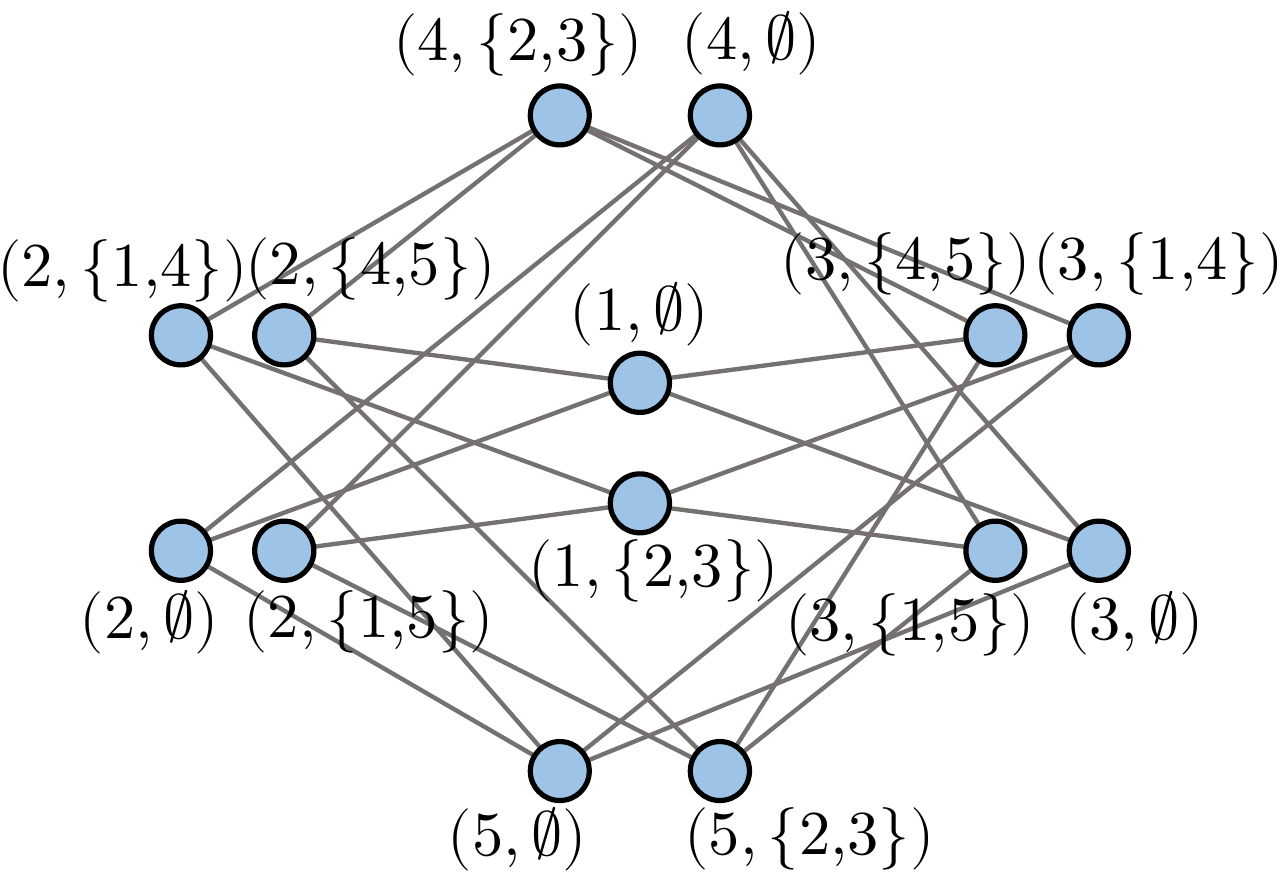} & \includegraphics[width=0.39\textwidth]{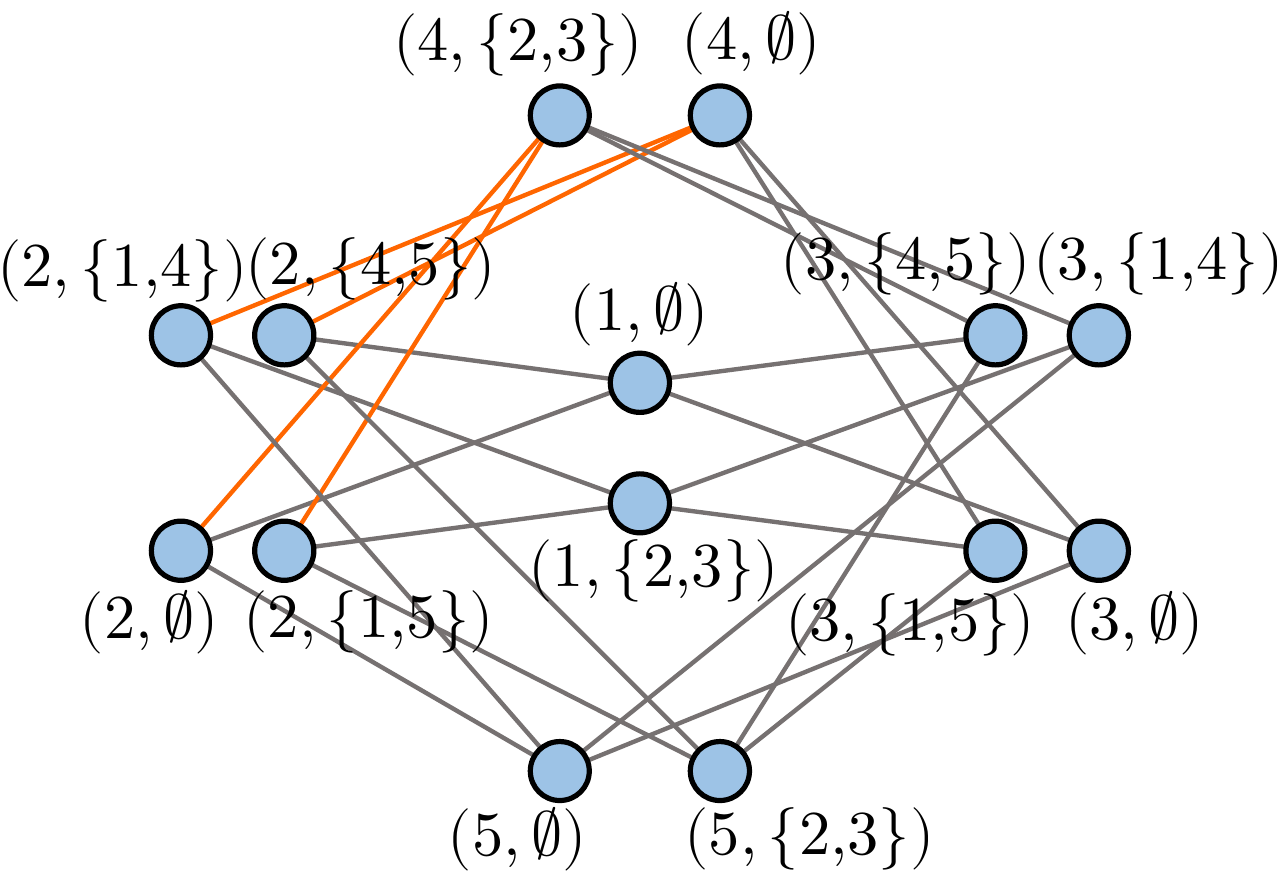}\\
        (a) Base graph $F$ & (b) F{\"u}rer graph $G(F)$ & (c) Twisted F{\"u}rer graph $H(F)$ for edge $\{2,4\}$
    \end{tabular}
    \caption{Illustration of the construction of F{\"u}rer graph and twisted F{\"u}rer graph.}
    \label{fig:furer}
\end{figure*}

We first introduce a class of graphs which we call the F{\"u}rer graphs \citep{furer2001weisfeiler}.

\begin{definition}[F{\"u}rer graphs]
    Given any connected graph $F=(\gV_F,\gE_F)$, the F{\"u}rer graph $G(F)=(\gV_G,\gE_G)$ is constructed as follows:
    \begin{align*}
        \gV_G&=\{(x,\gX):x\in\gV_F,\gX\subset\gN_F(x),|\gX|\bmod 2 = 0\},\\
        \gE_G&=\{\{(x,\gX),(y,\gY)\}\subset\gV_G:\{x,y\}\in\gE_F,(x\in\gY\leftrightarrow y\in\gX)\}.
    \end{align*}
    Here, $x\in\gY\leftrightarrow y\in\gX$ means that either ($x\in\gY$ and $y\in\gX$) or ($x\notin\gY$ and $y\notin\gX$). For each vertex $x\in\gV_F$, denote the set
    \begin{equation}
        \meta_F(x):=\{(x,\gX):\gX\subset\gN_F(x),|\gX|\bmod 2 = 0\},
    \end{equation}
    which is called the meta vertices of $G(F)$ associated to vertex $F$. Clearly, $\gV_G=\bigcup_{x\in\gV_F}\meta_F(x)$.
\end{definition}

We next define an operation called ``twist'':
\begin{definition}[Twist]
    Let $G(F)=(\gV_G,\gE_G)$ be the F{\"u}rer graph of $F=(\gV_F,\gE_F)$, and let $\{x,y\}\in\gE_F$ be an edge of $F$. The \emph{twisted} F{\"u}rer graph for edge $\{x,y\}$ is constructed as follows: $\twist(G(F),\{x,y\}):=(\gV_G,\gE_H)$, where
    \begin{align*}
        \gE_H:=\gE_G\triangle\{\{\xi,\eta\}:\xi\in\meta_F(x),\eta\in\meta_F(y)\}.
    \end{align*}
    Here, $\triangle$ is the symmetric difference operation, i.e., $\gA\triangle\gB=(\gA\backslash\gB)\cup(\gB\backslash\gA)$.
\end{definition}

In other words, the twisted F{\"u}rer graph $\twist(G(F),\{x,y\})$ is the graph modified from $G(F)$ by deleting all edges of the form  $\{(x,\gX),(y,\gY)\}\in\gE_G$ and adding the following set of edges
$$\{\{(x,\gX),(y,\gY)\}\subset\gV_G:(x\in\gY\leftrightarrow y\notin\gX)\}.$$
We give an illustration of the construction of F{\"u}rer graph and twisted F{\"u}rer graph for a simple graph $F$ in \cref{fig:furer}.

The twist operation can be further generalized into twisting a set of edges. We adopt the following notations:
\begin{equation}
\label{eq:twist}
    \twist(G(F), \gE):=\twist(\cdots\twist(G(F),e_1)\cdots,e_k)
\end{equation}
given an edge set $\gE=\{e_1,\cdots,e_k\}\subset \gE_F$. Note that the resulting graph $\twist(G(F), \gE)$ does not depend on the order of edges $e_1,\cdots,e_k$ for twisting, so (\ref{eq:twist}) is well-defined.

The first key result below shows that if we twist any two edges of a F{\"u}rer graph, the resulting graph is isomorphic to the original graph.

\begin{lemma}
\label{thm:furer_double}
    Let $G(F)=(\gV_G,\gE_G)$ be the F{\"u}rer graph of $F=(\gV_F,\gE_F)$. Then, for any two different edges $\{x,y\},\{u,v\}\in\gE_F$,
    \begin{align*}
        \twist(G(F),\{\{x,y\},\{u,v\}\})\simeq G(F).
    \end{align*}
    Moreover, there exists an isomorphism $f:\gV_G\to\gV_G$ from $G(F)$ to $\twist(G(F),\{\{x,y\},\{u,v\}\})$ that maps each meta vertex set $\meta_F(x)$ to itself for all $x\in\gV_F$.
\end{lemma}
\begin{proof}
    Denote $\widehat G(F):=\twist (G(F),\{\{x,y\},\{u,v\}\})$. Since $F$ is connected, one can always find a simple path $(w_0,w_1,\cdots,w_k)$, $k\ge 1$ with $\{w_0,w_1\}=\{x,y\}$ and $\{w_{k-1},w_k\}=\{u,v\}$. Denote $\gP=\{w_1,\cdots,w_{k-1}\}$. Construct a mapping $f:\gV_G\to\gV_G$ as follows:
    \begin{equation}
    \label{eq:proof_furer_double}
        f(z,\gZ)=\left\{\begin{array}{ll}
            (z,\gZ\triangle\{w_{i-1},w_{i+1}\}) & \text{if }z= w_i,i\in[k-1], \\
            (z,\gZ) & \text{if }z\notin\gP.
        \end{array}\right.
    \end{equation}
    We will prove that $f$ is an isomorphism from $G(F)$ to $\widehat G(F)$. First, since $|\gZ|\bmod 2 =0$ implies that $|\gZ\triangle\{w_{i-1},w_{i+1}\}|\bmod 2 =0$, $f$ is indeed a valid mapping from $\gV_G$ to $\gV_G$. Also, it is straightforward to see that $f$ is bijective. It remains to verify that for any edge $\{(z,\gZ),(z',\gZ')\}\in\gE_G$, $\{f(z,\gZ),f(z',\gZ')\}$ is an edge of $\widehat G(F)$. Separately consider the following cases:
    \begin{itemize}[topsep=0pt]
    \setlength{\itemsep}{0pt}
        \item If $z,z'\notin\gP$, then $\{f(z,\gZ),f(z',\gZ')\}=\{(z,\gZ),(z',\gZ')\}$ is clearly an edge of $\widehat G(F)$.
        \item If $z,z'\in\gP$, denote $z=w_i$ and $z'=w_j$. Then it is straightforward to see that $w_i\in\gZ'\leftrightarrow w_j\in\gZ$ if and only if $w_i\in\gZ'\triangle \{w_{j-1},w_{j+1}\}\leftrightarrow w_j\in\gZ\triangle \{w_{i-1},w_{i+1}\}$. Therefore, $\{f(z,\gZ),f(z',\gZ')\}=\{(z,\gZ\triangle \{w_{i-1},w_{i+1}\}),(z',\gZ'\triangle \{w_{j-1},w_{j+1}\})\}$ is an edge of $\widehat G(F)$.
        \item If $z=w_i\in\gP$, $z'\notin\gP$, $\{z,z'\}\neq\{x,y\}$, and $\{z,z'\}\neq\{u,v\}$, then $z'\neq w_{i-1}$ and $z'\neq w_{i+1}$. Therefore, $z\in\gZ'\leftrightarrow z'\in\gZ$ if and only if $z\in\gZ'\leftrightarrow z'\in\gZ\triangle \{w_{i-1},w_{i+1}\}$. This implies that $\{f(z,\gZ),f(z',\gZ')\}=\{(z,\gZ\triangle \{w_{i-1},w_{i+1}\}),(z',\gZ')\}$ is an edge of $\widehat G(F)$.
        \item If $\{z,z'\}=\{x,y\}$, we can denote $z=w_0$ and $z'=w_1$. We have $z\in\gZ'\leftrightarrow z'\in\gZ$ if and only if $z\notin\gZ'\triangle\{w_0,w_2\}\leftrightarrow z'\in\gZ$. Note that there is a twist in $\widehat G(F)$ for edge $\{x,y\}$, so we still obtain that $\{f(z,\gZ),f(z',\gZ')\}=\{(z,\gZ),(z',\gZ'\triangle \{w_{0},w_{2}\})\}$ is an edge of $\widehat G(F)$.
        \item Finally, if $\{z,z'\}=\{u,v\}$, the analysis is the same as the above one and $\{f(z,\gZ),f(z',\gZ')\}$ is an edge of $\widehat G(F)$.
    \end{itemize}
    In all cases, $\{f(z,\gZ),f(z',\gZ')\}$ is an edge of $\widehat G(F)$. Moreover, it is clear that $f$ maps each meta vertex set $\meta_F(x)$ to itself for all $x\in\gV_F$, which concludes the proof.
\end{proof}

Based on \cref{thm:furer_double}, it is convenient to define a notion called \emph{proper} isomorphism:

\begin{definition}[Proper isomorphism]
    Let $G(F)=(\gV_G,\gE_G)$ be the F{\"u}rer graph of $F=(\gV_F,\gE_F)$ and $\widehat G(F)=\twist(G(F),\gE)$ for some $\gE\subset\gE_F$. We say $f$ is a proper isomorphism from $G(F)$ to $\widehat G(F)$, if $f$ is an isomorphism from $G(F)$ to $\widehat G(F)$ that maps each meta vertex set $\meta_F(x)$ to itself for all $x\in\gV_F$.
\end{definition}

\cref{thm:furer_double} can be generalized into the following corollary:

\begin{corollary}
\label{thm:furer_isomorphism}
    Let $G(F)=(\gV_G,\gE_G)$ be the F{\"u}rer graph of $F=(\gV_F,\gE_F)$. Then, for any edge set $\gE\subset\gE_F$ and any two different edges $\{x,y\},\{u,v\}\in\gE_F$,
    \begin{align*}
        \twist(G(F),\gE\triangle\{\{x,y\},\{u,v\}\})\simeq\twist(G(F),\gE).
    \end{align*}
    Moreover, any proper isomorphism $f:\gV_G\to\gV_G$ from $G(F)$ to $\twist(G(F),\{\{x,y\},\{u,v\}\})$ is also a proper isomorphism from $\twist(G(F),\gE)$ to $\twist(G(F),\gE\triangle\{\{x,y\},\{u,v\}\})$.
\end{corollary}
\begin{proof}
    Denote $\widehat G(F):=\twist (G(F),\{\{u,v\},\{x,y\}\})$, $H(F):=\twist(G(F),\gE)$, and $\widehat H(F):=\twist(\widehat G(F),\gE)$. Note that by definition of the twist operation, we equivalently have $\widehat H(F)=\twist(G(F),\gE\triangle\{\{x,y\},\{u,v\}\})$.
    Due to \cref{thm:furer_double}, we have $\widehat G(F)\simeq G(F)$. Let $f$ be a proper isomorphism from $G(F)$ to $\widehat G(F)$ (according to \cref{thm:furer_double}). It suffices to prove that $f$ is also an isomorphism from $H(F)$ to $\widehat H(F)$.
    
    For any edge $\{(w,\gW),(z,\gZ)\}$ in $H(F)$:
    \begin{itemize}[topsep=0pt]
    \setlength{\itemsep}{0pt}
        \item If $\{w,z\}\in\gE$, then $\{(w,\gW),(z,\gZ)\}$ is not an edge in $G(F)$. Therefore, $\{f(w,\gW),f(z,\gZ)\}$ is not an edge in $\widehat G(F)$. Since $f$ maps $\meta_F(w)$ to $\meta_F(w)$ and maps $\meta_F(z)$ to $\meta_F(z)$, we obtain that $\{f(w,\gW),f(z,\gZ)\}$ is an edge in $\widehat H(F)$.
        \item If $\{w,z\}\notin\gE$, then $\{(w,\gW),(z,\gZ)\}$ is an edge in $G(F)$. Therefore, $\{f(w,\gW),f(z,\gZ)\}$ is an edge in $\widehat G(F)$. Since $f$ maps $\meta_F(w)$ to $\meta_F(w)$ and maps $\meta_F(z)$ to $\meta_F(z)$, we also obtain that $\{f(w,\gW),f(z,\gZ)\}$ is an edge in $\widehat H(F)$.
    \end{itemize}
    In both cases, $\{f(w,\gW),f(z,\gZ)\}$ is an edge in $\widehat H(F)$. Since \cref{thm:furer_double} has proved that $f$ is bijective, $f$ is an isomorphism from $H(F)$ to $\widehat H(F)$ and thus $H(F)\simeq\widehat H(F)$.
\end{proof}

As a special case, \cref{thm:furer_isomorphism} leads to the following important fact:

\begin{corollary}
\label{thm:furer_single}
    Let $G(F)$ be the F{\"u}rer graph of $F=(\gV_F,\gE_F)$. Then, for any two edges $\{x,y\},\{u,v\}\in\gE_F$,
    \begin{align*}
        \twist(G(F),\{x,y\})\simeq\twist(G(F),\{u,v\}).
    \end{align*}
\end{corollary}
\begin{proof}
    Setting $\gE=\{u,v\}$ in \cref{thm:furer_isomorphism} readily concludes the proof.
\end{proof}

\cref{thm:furer_single} shows that the structure of a twisted F{\"u}rer graph does not depend on which edge is twisted. Therefore, we can simply denote $H(F)$ as the twisted F{\"u}rer graph of $F$ without specifying the twisted edge $\{x,y\}$. Moreover, recursively applying \cref{thm:furer_isomorphism} obtains that, if we twist $k$ edges of a F{\"u}rer graph $G(F)$, the resulting graph is isomorphic to either $G(F)$ or $H(F)$, depending on whether $k$ is even or odd. To complete the result, we show $G(F)$ and $H(F)$ are actually non-isomorphic under certain conditions:

\begin{lemma}
\label{thm:furer_single_nonisomorphic}
    Let $G(F)=(\gV_G,\gE_G)$ be the F{\"u}rer graph of $F=(\gV_F,\gE_F)$, and let $H(F)=(\gV_G,\gE_H)$ be the twisted F{\"u}rer graph. Then there does not exist a proper isomorphism $f:\gV_G\to\gV_G$ from $G(F)$ to $H(F)$.
\end{lemma}
\begin{proof}
    Assume $H(F)=\twist(G(F),\{u,v\})$ for some $\{u,v\}\in\gE_F$ and $f:\gV_G\to\gV_G$ is a proper isomorphism. Then we can write $f(x,\emptyset)=(x,\gT_x)$ for all $x\in\gV_F$. Note that for any $\{x,y\}\in\gE_F$, by definition of the F{\"u}rer graph there is an edge between vertices $(x,\emptyset)$ and $(y,\emptyset)$ in $G(F)$. However, we will prove that this is not the case for $H(F)$: there must exist an odd number of edges $\{x,y\}\in\gE_F$ such that $\{(x,\gT_x),(y,\gT_y)\}\notin\gE_H$. This will lead to a contradiction and finish the proof.

    Formally, let $\gE=\{\{(x,\gT_x),(y,\gT_y)\}:\{x,y\}\in\gE_F\}$, and our goal is prove that $|\gE\backslash\gE_H|\bmod 2=1$. The proof is based on induction. First, consider the base case when $\gT_x=\emptyset$ for all $x\in\gV_F$. Clearly, there is exactly one element $\{(u,\gT_u),(v,\gT_v)\}\notin\gE_H$ since $H(F)$ is obtained from $G(F)$ by twisting edge $\{u,v\}$. Next, for the induction step, we show if $|\gE\backslash\gE_H|\bmod 2=1$, then $|\tilde\gE\backslash\gE_H|\bmod 2=1$ holds for any $\tilde\gE$ that is modified from $\gE$ by changing a given $\gS_z$ to another feasible $\tilde \gS_z$ for some $z\in\gV_F$, namely,
    \begin{align*}
        \tilde\gE=\{\{(x,\gT_x),(y,\gT_y)\}:\{x,y\}\in\gE_F,x,y\neq z\}\cup\{\{(x,\gT_x),(z,\tilde\gT_z)\}:\{x,z\}\in\gE_F\}.
    \end{align*}
    This is because 
    \begin{align*}
        |\tilde\gE\backslash\gE_H|-|\gE\backslash\gE_H|&\equiv |(\tilde\gE\triangle\gE)\backslash\gE_H|\\
        &\equiv |\{x\in\gN_F(z):x\in\gT_z\leftrightarrow x\notin\tilde\gT_z\}|\\
        &\equiv|\gT_z\triangle\tilde\gT_z|\equiv 0 \pmod 2,
    \end{align*}
    where the last equation holds because $|\gT_z|\equiv|\tilde\gT_z|\equiv 0\pmod 2$. This concludes the induction step.
    
    Since any set $\gE$ can be obtained from the initial set $\{\{(x,\emptyset),(y,\emptyset)\}:\{x,y\}\in\gE_F\}$ by modifying $\emptyset$ to $\gT_z$ for each $z\in\gV_F$ and the parity of $|\gE\backslash\gE_H|$ does not change throughout the process, we have concluded the proof.
\end{proof}

Below, we proceed to perform an in-depth analysis of the properties regarding the isomorphisms of (twisted) F{\"u}rer graphs. We need several definitions.

\begin{definition}[Connected components]
\label{def:connected_component}
    Let $F=(\gV_F,\gE_F)$ be a connected graph and let $\gS\subset \gV_F$ be a vertex set, called separation vertices. We say two edges $\{u,v\},\{x,y\}\in\gE_F$ are in the same connected component if there is a path $(y_0,y_1,\cdots,y_k)$ satisfying that $\{y_0,y_1\}=\{u,v\}$, $\{y_{k-1},y_k\}=\{x,y\}$ and $y_i\notin\gS$ for all $i\in[k-1]$. It is easy to see that the above relationship between edges forms an \emph{equivalence relation}. Therefore, we can define a partition over the edge set:
    $\mathsf{CC}_\gS(F)=\{\gP_i:i\in[M]\}$, where each $\gP_i\subset\gE_F$ is called a connected component.
\end{definition}

We are ready to state the central theorem:

\begin{theorem}
\label{thm:furer_main}
    Let $G(F)=(\gV_G,\gE_G)$ be either the original or twisted F{\"u}rer graph of $F=(\gV_F,\gE_F)$, and let $\gS\subset \gV_F$ be any set. 
    For each $u\in\gS$, let $(u,\gT_u),(u,\gU_u)\in\meta_F(u)$ be any given vertex sets. Then, there exists a proper isomorphism $f$ from graph $G(F)$ to graph $\twist(G(F),\gE)$ for some $\gE\subset\gE_F$ with $|\gE|\bmod 2 = 0$, such that $f(u,\gT_u)=(u,\gU_u)$ for all $u\in\gS$. Moreover, for any $\tilde\gE\subset\gE_F$, there exists a proper isomorphism $\tilde f$ from $G(F)$ to $\twist(G(F),\tilde \gE)$ such that $\tilde f(u,\gT_u)=(u,\gU_u)$ for all $u\in\gS$ if and only if $|\gP\cap\tilde \gE|\equiv|\gP\cap\gE|\pmod 2$ for all $\gP\in\mathsf{CC}_\gS(F)$.
\end{theorem}
\begin{proof}
    We only consider the case when $G(F)$ is a F{\"u}rer graph, and the case of twisted F{\"u}rer graph is similar. We prove the theorem by induction over the size of $\gS$. For the base case of $|\gS|=0$, there is clearly a trivial isomorphism (identity map) from $G(F)$ to $\twist(G(F),\emptyset)=G(F)$.

    Now assume that the result holds for set $\gS$, and there exists a proper isomorphism $f$ from $G(F)$ to $\twist(G(F),\gE)$ for some $\gE$ with $|\gE|\bmod 2 = 0$, such that $f(u,\gT_u)=(u,\gU_u)$ for all $u\in\gS$. Consider adding a vertex $v\in\gV_F$ and two given sets $(v,\gT_v),(v,\gU_v)\in\meta_F(v)$. We will construct a new proper isomorphism $f^\mathsf{new}$ from $G(F)$ to $\twist(G(F),\gE^\mathsf{new})$ for some $\gE^\mathsf{new}\subset\gE_F$ with $|\gE^\mathsf{new}|\bmod 2 = 0$, such that $f^\mathsf{new}(u,\gT_u)=(u,\gU_u)$ for all $u\in\gS$ and $f^\mathsf{new}(v,\gT_v)=(v,\gU_v)$. Denote $\tilde \gU_v:=f(v,\gT_v)$ and denote $\gD_v=\gU_v\triangle\tilde\gU_v$. Note that the size of $\gD_v$ is even.

    Denote $\gD_v=\{x_1,\cdots,x_{2k}\}$. We define $k$ mappings $f_i:\gV_G\to\gV_G$, $i\in[k]$ as follows:
    \begin{align}
    \label{eq:proof_furer_main}
        f_i(w,\gW)=\left\{\begin{array}{ll}
            (w,\gW\triangle\{x_{2i-1},x_{2i}\}) & \text{if }w=v, \\
            (w,\gW) & \text{otherwise}.
        \end{array}\right.
    \end{align}
    We set $f^\mathsf{new}$ to be the composition of a series of mappings $f^\mathsf{new}:=f_k\circ\cdots\circ f_1\circ f$. By definition, we have
    $$f^\mathsf{new}(v,\gT_v)=(f_k\circ\cdots\circ f_1)(v,\tilde\gU_v)=(v,\tilde\gU_v\triangle\gD_v)=(v,\gU_v),$$
    and for all $u\in\gS$,
    $$f^\mathsf{new}(u,\gT_u)=(f_k\circ\cdots\circ f_1)(u,\gU_u)=(u,\gU_u).$$
    
    It remains to verify that $f^\mathsf{new}$ is an isomorphism from $G(F)$ to $\twist(G(F),\gE^\mathsf{new})$ for some $\gE^\mathsf{new}$ with $|\gE^\mathsf{new}|\bmod 2=0$. Denote $\gE^{(i)}=\bigcup_{j=1}^{2i} \{\{v,u_{j}\}\}$ for $i\in\{0,1,\cdots,k\}$. Based on the proof of \cref{thm:furer_double} (i.e., the construction in (\ref{eq:proof_furer_double})), $f_i$ is an isomorphism from $G(F)$ to $\twist(G(F),\{\{v,u_{2i-1}\},\{v,u_{2i}\}\})$. Further using \cref{thm:furer_isomorphism}, $f_i$ is also an isomorphism from $\twist(G(F),\gE\triangle\gE^{(i-1)})$ to $\twist(G(F),\gE\triangle\gE^{(i)})$. Thus by composition, $f^\mathsf{new}$ is an isomorphism from $G(F)$ to $\twist(G(F),\gE\triangle\gE^{(k)})$, namely, $\gE^\mathsf{new}=\gE\triangle\gE^{(k)}$. Also, since $|\gE^{(k)}|=2k$ and $|\gE|\bmod 2=0$, we have $|\gE^\mathsf{new}|\bmod 2=0$, as desired. This finishes the induction step and concludes the proof of the first part.

    For the second part, let us first consider how to find a proper isomorphism $\tilde f$ from $G(F)$ to $\twist(G(F),\tilde \gE)$ for other $\tilde \gE$ satisfying $|\gP\cap\tilde \gE|\equiv|\gP\cap\gE|\pmod 2$ for all $\gP\in\mathsf{CC}_\gS(F)$, such that $\tilde f(u,\gT_u)=(u,\gU_u)$ for all $u\in\gS$. For each $\gP\in\mathsf{CC}_\gS(F)$, denote $\gD_\gP:=(\gE\triangle\tilde\gE)\cap\gP$. Clearly, $\bigcup_{\gP\in\mathsf{CC}_\gS(F)} \gD_\gP=\gE\triangle\tilde\gE$. By assumption we have 
    \begin{align*}
        |\gD_\gP|&=|(\gE\triangle\tilde\gE)\cap\gP|=|(\gE\cap\gP)\triangle(\tilde\gE\cap\gP)|\\
        &\equiv |\gE\cap\gP|+|\tilde\gE\cap\gP|\equiv 0\pmod 2.
    \end{align*}
    Therefore, there is a proper isomorphism $f^\gP$ from $G(F)$ to $\twist(G(F),\gD_\gP)$ based on \cref{thm:furer_double,thm:furer_isomorphism}. Concretely, denote $\gD_\gP=\{e_1,\cdots,e_{2k}\}$, then $f^\gP$ can be constructed as a composition $f^\gP=f^\gP_k\circ\cdots\circ f^\gP_1$ where each $f^\gP_i$ is a proper isomorphism from $\twist(G(F),\bigcup_{j=1}^{2(i-1)} \{e_j\})$ to $\twist(G(F),\bigcup_{j=1}^{2i} \{e_j\})$. Since all edges $e_j\in \gD_\gP$ are in the same connected component, there exists a path containing edges $e_{2i-1}$ and $e_{2i}$ and it does not go through vertices in $\gS$. Therefore, by construction of (\ref{eq:proof_furer_double}) in \cref{thm:furer_double}, all the mappings $f^\gP_i$ does not change the value for inputs $\meta_F(u)$ for all $u\in\gS$. Namely, $f^\gP(u,\gU_u)=(u,\gU_u)$ for all $u\in\gS$. Finally, we set $\tilde f=(\circ_{\gP\in\mathsf{CC}_\gS(F)} f^\gP)\circ f$ to be the composition of $f$ and all $f^\gP$. We have $\tilde f(u,\gT_u)=(\circ_{\gP\in\mathsf{CC}_\gS(F)} f^\gP)(u,\gU_u)=(u,\gU_u)$, as desired. Moreover, $\tilde f$ is indeed a proper isomorphism from $G(F)$ to $\twist(G(F), \gE\triangle(\bigcup_{\gP\in\mathsf{CC}_\gS(F)}\gD_\gP))=\twist(G(F), \tilde\gE)$.

    Conversely, suppose $\tilde\gE$ satisfies that $|\gP\cap\tilde \gE|\not\equiv|\gP\cap\gE|\pmod 2$ for some $\gP\in\mathsf{CC}_\gS(F)$. We will prove that any proper isomorphism $\tilde f$ from $G(X)$ to $\twist(G(X),\tilde\gE)$ cannot satisfy $\tilde f(u,\gT_u)=(u,\gU_u)$ for all $u\in\gS$. To prove the result, it suffices to prove that any proper isomorphism $\hat f$ from $\twist(G(X),\gE)$ to $\twist(G(X),\tilde\gE)$ cannot satisfy $\hat f(u,\gU_u)=(u,\gU_u)$ for all $u\in\gS$. Let $\gV^\gP_F:=\bigcup_{\{x,y\}\in\gP}\{x,y\}\subset \gV_F$ be the set of vertices associated to the connected component $\gP$, and let $\gV^\gP_G:=\bigcup_{x\in\gV^\gP_F}\meta_F(x)$. It thus suffices to prove that any proper isomorphism $f^\gP:\gV^\gP_G\to\gV^\gP_G$ from the induced subgraph $G^\gP:=\twist(G(X),\gE)[\gV^\gP_G]$ to the induced subgraph $H^\gP:=\twist(G(X),\tilde\gE)[\gV^\gP_G]$ cannot satisfy $f^\gP(u,\gU_u)=(u,\gU_u)$ for all $u\in\gS\cap\gV^\gP_F$.

    The proof follows the same technique as \cref{thm:furer_single_nonisomorphic}. For each $x\in\gV^\gP_F\backslash\gS$, pick an arbitrary meta vertex $(x,\gU_x)\in\meta_F(x)$. Combined with all $(u,\gU_u)$ for $u\in\gS\cap\gV^\gP_F$, now each vertex $x\in\gV^\gP_F$ is associated with a set $\gU_x$. First consider the base case when $f^\gP(x,\gU_x)=(x,\gU_x)$ for all $x\in\gV^\gP_F$. It can be proved that the set $\{\{(x,\gU_x),(y,\gU_y)\}:\{x,y\}\in\gP\}$ contains an odd/even number of edges in $G^\gP$ but an even/odd number of edges in $H^\gP$, i.e., their parity differs. This is because $H^\gP$ can be obtained from $G^\gP$ by twisting the edge set $(\gE\triangle\tilde\gE)\cap\gP$, which contains an \emph{odd} number of edges. Therefore, $f^\gP$ is not a proper isomorphism from $G^\gP$ to $H^\gP$ in this case. For the induction step, consider gradually changing the output $f^\gP(w,\gU_w)=(w,\gU_w)$ to $f^\gP(w,\gU_w)=(w,\tilde\gU_w)$ for each $w\in\gV^\gP_F\backslash\gS$ where $(w,\tilde\gU_w)\in\meta_F(w)$ can be an arbitrary meta vertex. It can be proved that the parity defined above does not change throughout the process (following the proof of \cref{thm:furer_single_nonisomorphic}). We have thus proved the induction step, and in all cases there does not exist a proper isomorphism $f^\gP:\gV^\gP_G\to\gV^\gP_G$ from $G^\gP$ to $H^\gP$ satisfying $f^\gP(u,\gU_u)=(u,\gU_u)$ for all $u\in\gS\cap\gV^\gP_F$.
\end{proof}

\cref{thm:furer_main} partially answers the question of how to construct a proper isomorphism $f$ from a F{\"u}rer graph $G(F)$ to another F{\"u}rer graph $\twist(G(F),\gE)$ when the mapped outputs $f(\xi)$ are specified for several given inputs $\xi\in\gV_G$. However, it does not fully address the problem, because it does not consider the case when two or more inputs $\xi$ are from the same set $\meta_F(u)$ for some $u\in\gV_F$. In the following, we will focus on this general setting. We first consider the special case when all $\xi$ are from the same meta vertex set $\meta_F(u)$. We have the following result:

\begin{lemma}
\label{thm:furer_multiple_s_in_u}
    Let $G(F)=(\gV_G,\gE_G)$ be either the original or twisted F{\"u}rer graph of $F=(\gV_F,\gE_F)$. Let $u\in\gV_F$ and $(u,\gT^1),(u,\gU^1),\cdots,(u,\gT^k),(u,\gU^k)\in\meta_F(u)$ be vertices in $\gV_G$. Then, there exists a set $\gD\subset\gN_F(u)$ such that $\gT^i\triangle\gU^i=\gD$ for all $i\in[k]$, if and only if there exists a proper isomorphism $f$ from graph $G(F)$ to graph $\twist(G(F),\gE)$ for some $\gE\subset\gE_F$ with $|\gE|\bmod 2 = 0$, such that $f(u,\gT^i)=(u,\gU^i)$ for all $i\in[k]$.
\end{lemma}
\begin{proof}
    ``$\Rightarrow$''. Let $\gD\subset\gN_F(u)$ satisfy that $\gT^i\triangle\gU^i=\gD$ for all $i\in[k]$. Clearly, $|\gD|\bmod 2=0$. Denote $\gD=\{v_1,\cdots,v_{2l}\}$. Similar to the construction in (\ref{eq:proof_furer_main}), construct a mapping $f:\gV_G\to\gV_G$ to be $f=f_l\circ\cdots\circ f_1$, where
    \begin{align*}
        f_j(w,\gW)=\left\{\begin{array}{ll}
            (w,\gW\triangle\{v_{2j-1},v_{2j}\}) & \text{if }w=u, \\
            (w,\gW) & \text{otherwise}.
        \end{array}\right.
    \end{align*}
    Using a similar analysis, we obtain that $f$ is a proper isomorphism from $G(F)$ to $\twist(G(F),\bigcup_{j=1}^{2l}\{\{u,v_j\}\})$ and $f(u,\gT^i)=(u,\gT^i\triangle\gD)=(u,\gU^i)$ for all $i\in[k]$.

    ``$\Leftarrow$''. Assume there does not exist $\gD\subset\gN_F(u)$ satisfying $\gT^i\triangle\gU^i=\gD$ for all $i\in[k]$. Then, there must exist two indices $i,j$ and a vertex $v\in\gN_F(u)$ such that $v\in\gT^i\triangle\gU^i$ but $v\notin\gT^j\triangle\gU^j$. We show any proper isomorphism $f$ from $G(F)$ to $\twist(G(F),\gE)$ cannot satisfy both $f(u,\gT^i)=(u,\gU^i)$ and $f(u,\gT^j)=(u,\gU^j)$. Let $f(v,\emptyset)=(v,\gV)$. This is simply due to the following fact:
    \begin{itemize}[topsep=0pt]
    \setlength{\itemsep}{0pt}
        \item If $f(u,\gT^i)=(u,\gU^i)$, then by definition of isomorphism we have $u\in\emptyset\leftrightarrow v\in\gT^i\iff u\in\gV\leftrightarrow v\in\gU^i$. Since $v\in\gT^i\triangle\gU^i$, we obtain $u\in\gV$;
        \item If $f(u,\gT^j)=(u,\gU^j)$, then by definition of isomorphism we have $u\in\emptyset\leftrightarrow v\in\gT^j\iff u\in\gV\leftrightarrow v\in\gU^j$. Since $v\notin\gT^j\triangle\gU^j$, we obtain $u\notin\gV$.
    \end{itemize}
    This yields a contradiction and concludes the proof.
\end{proof}

We finally consider the most general setting. Our result is present as follows:

\begin{corollary}
\label{thm:furer_main_corollary}
    Let $G(F)=(\gV_G,\gE_G)$ be either the original or twisted F{\"u}rer graph of $F=(\gV_F,\gE_F)$. Let $\{(u_i,\gT_i)\}_{i=1}^k\subset\gV_G$ and $\{(u_i,\gU_i)\}_{i=1}^k\subset\gV_G$ be two vertex sets of $G(F)$. Define $\gR(u):=\{(\gT_i,\gU_i):u_i=u\}$ and let $\gS=\{u_i:i\in[k]\}$. The following two items are equivalent:
    \begin{itemize}[topsep=0pt]
    \setlength{\itemsep}{0pt}
        \item There exists a proper isomorphism $f$ from graph $G(F)$ to graph $\twist(G(F),\gE)$ for some $\gE\subset \gE_F$ with $|\gE|\bmod 2 = 0$, such that $f(u_i,\gT_i)=(u_i,\gU_i)$ for all $i\in[k]$.
        \item For all $v\in\gV_F$, there exists $\gD_v\subset\gN_F(v)$ such that $\gT_i\triangle\gU_i=\gD_v$ holds for all $(\gT_i,\gU_i)\in\gR(v)$.
    \end{itemize}
    Moreover, if the first item holds, then for any $\tilde\gE\subset\gE_F$, there exists a proper isomorphism $\tilde f$ from $G(F)$ to $\twist(G(F),\tilde \gE)$ such that $f(u_i,\gT_i)=(u_i,\gU_i)$ for all $i\in[k]$, if and only if $|\gP\cap\tilde \gE|\equiv|\gP\cap\gE|\pmod 2$ for all $\gP\in\mathsf{CC}_\gS(F)$.
\end{corollary}
\begin{proof}
    First assume the second item does not hold for some $v\in\gV_F$. Then by \cref{thm:furer_multiple_s_in_u}, there does not exist a proper isomorphism $f$ from graph $G(F)$ to some $\twist(G(F),\gE)$ such that $f(v,\gT_i)=(v,\gU_i)$ for all $(\gT_i,\gU_i)\in\gR(v)$. Clearly, the first item of \cref{thm:furer_main_corollary} does not hold either.

    Now assume the second item holds. For each $v\in\gS$, pick an arbitrary element in $\gR(v)$, denoted as $(\gT^v,\gU^v)$. We can then invoke \cref{thm:furer_main} with $\{\gT^v\}_{v\in\gS}$ and $\{\gU^v\}_{v\in\gS}$. Denote $f$ as the proper isomorphism from $G(F)$ to $\twist(G(F),\gE)$ for some $\gE$ with $|\gE|\bmod 2 = 0$ returned by \cref{thm:furer_main}, such that $f(v,\gT^v)=(v,\gU^v)$ for all $v\in\gS$. It remains to prove that for all other elements $(\gT_i,\gU_i)\notin \{(\gT^v,\gU^v):v\in\gS\}$, we still have $f(u_i,\gT_i)=(u_i,\gU_i)$.

    Observe that the construction of $f$ in \cref{thm:furer_main} has the form $f(w,\gW)=(w,\gW\triangle\gD_w)$ for all $(w,\gW)\in\gV_G$ where $\gD_w$ is a fixed set for each $w\in\gV_F$ (which can be seen from the proof of \cref{thm:furer_main}). Under the notation above, we have $\gT^v\triangle\gD_v=\gU^v$ for all $v\in\gS$. If $f(u_i,\gT_i)\neq(u_i,\gU_i)$ for some $i$, then $\gT_i\triangle\gW_{u_i}\neq\gU_i$. This implies that $\gT^{u_i}\triangle\gU^{u_i}\neq \gT_i\triangle \gU_i$, which contradicts the second item of \cref{thm:furer_main_corollary} and concludes the proof.
\end{proof}

Before closing this subsection, we finally introduce a notion called \emph{proper F{\"u}rer graphs}, which will be widely used in subsequent analysis.

\begin{definition}[Proper F{\"u}rer graphs]
    A F{\"u}rer graph $G(F)$ is called proper, if the base graph $F=(\gV_F,\gE_F)$ has the following properties:
    \begin{itemize}[topsep=0pt]
    \setlength{\itemsep}{0pt}
        \item $F$ is a connected graph and the degree of any vertex $u\in\gV_F$ is at least two;
        \item There is at least one vertex $u\in\gV_F$ with a degree of at least three.
    \end{itemize}
\end{definition}
\begin{proposition}
\label{thm:furer_connectivity}
    Let $G(F)$ and $H(F)$ be any proper F{\"u}rer graph and its twisted graph, respectively. Then both $G(F)$ and $H(F)$ are connected, and the degree of any vertex in both $G(F)$ and $H(F)$ is at least two.
\end{proposition}
\begin{proof}
    By definition of (twisted) F{\"u}rer graphs, the degree of a vertex $(u,\gU)$ in $G(F)$ is $\sum_{v\in\gN_F(u)}|\{(v,\gV)\in\meta_F(v):u\in\gV\leftrightarrow v\in\gU\}|$. By the assumption that $|\gN_F(v)|\ge 2$ for all $v\in\gV_F$, the degree of a vertex $(u,\gU)$ is always $\sum_{v\in\gN_F(u)} 2^{|\gN_F(v)|-2}\ge \sum_{v\in\gN_F(u)}1\ge 2$. The case of $H(F)$ is similar. Thus all vertices in both $G(F)$ and $H(F)$ have a degree of at least two.

    We next investigate the connectivity of $G(F)$ and $H(F)$. Denote $u$ as any vertex in $G(F)$ or $H(F)$ with a degree of at least 3. We first show that any two vertices $(u,\gT),(u,\gU)\in\gV_G$ satisfying $|\gT\triangle\gU|=2$ are in the same connected component. To see this, pick any $w\in\gN_F(u)\backslash(\gT\triangle\gU)$ (which exists since $|\gN_F(u)|\ge 3$), and consider the two vertices $(w,\emptyset)$ and $(w,\gW)$ satisfying $u\in\gW$ (the existence of $\gW$ is due to $w\in\gN_F(u)$ and $|\gN_F(w)|\ge 2$). Then,
    \begin{itemize}[topsep=0pt]
    \setlength{\itemsep}{0pt}
        \item If $\{(w,\emptyset),(u,\gT)\}$ is an edge of $G(F)$/$H(F)$, then $\{(w,\emptyset),(u,\gU)\}$ is an edge of $G(F)$/$H(F)$ (because $w\in\gT$ if and only if $w\in\gU$);
        \item If $\{(w,\emptyset),(u,\gT)\}$ is not an edge of $G(F)$/$H(F)$, then $\{(w,\gW),(u,\gT)\}$ is an edge of $G(F)$/$H(F)$ (because $u\notin\emptyset$ but $u\in\gW$). Therefore, $\{(w,\gW),(u,\gU)\}$ is an edge of $G(F)$/$H(F)$ (because $w\in\gT$ if and only if $w\in\gU$).
    \end{itemize}
    In both cases, there is a path from $(u,\gT)$ to $(u,\gU)$. Next, we can simply remove the assumption $|\gT\triangle\gU|=2$: any $(u,\gT),(u,\gU)\in\gV_G$ are also in the same connected component. Finally, for any vertex $(x,\gX)$ in graph $G(F)$/$H(F)$, there is a path from $(x,\gX)$ to some vertex in $\meta_F(u)$ since $F$ is connected. Using $\meta_F(u)$ as a ``transit set'', we have proved that the graph $G(F)$/$H(F)$ is connected.
\end{proof}

\subsection{Simplified pebbling games for F{\"u}rer graphs}
\label{sec:pebbling_furer}

In \cref{sec:pebbling_game}, we developed a unified analyzing framework for all types of WL algorithms based on pebbling games. While the game viewpoint provides interesting and novel insights into the power of different algorithms, it is still quite challenging to directly applying such games for F{\"u}rer graphs due to their sophisticated structure. In this subsection, we propose a class of simplified pebbling game motivated from \citet{furer2001weisfeiler}, which makes our analysis much easier.

We begin by introducing the augmented F{\"u}rer graphs defined as follows:

\begin{definition}[Augmented F{\"u}rer graphs]
\label{def:aug_furer}
    Let $G(F)=(\gV_G,\gE_G)$ be either the original or twisted \emph{proper} F{\"u}rer graph of $F=(\gV_F,\gE_F)$ where $\gV_F=[n]$. Let $\tilde G(F)$ be the graph augmented from $G(F)$ in the following way:  for each $u\in\gV_F$, add a chain $C_u$ of length $u+1$ and link one endpoint of $C_u$ to all vertices in $\meta_F(u)$ (see \cref{fig:fuer_aug} for an illustration). The vertices on the chains are called \emph{auxiliary vertices}. For each vertex $\xi$ in $\tilde G(F)$, by construction it is associated with a vertex $u$ in the base graph. We call $u$ is the \emph{base vertex} of $\xi$ and denote $B(\xi)=u$.
\end{definition}

\begin{figure}[ht]
  \begin{center}
   \includegraphics[width=0.49\textwidth]{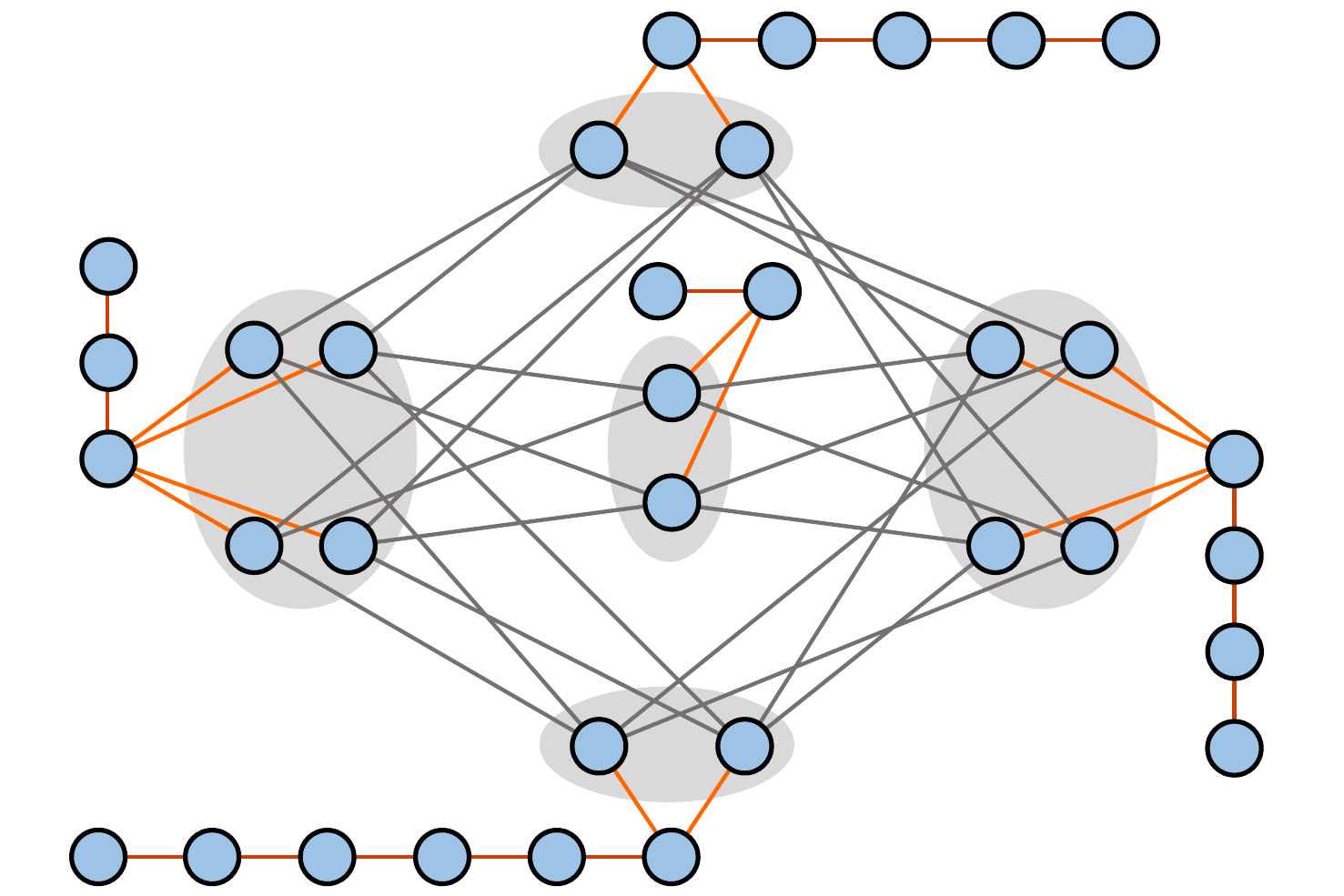}
  \end{center}
  \vspace{-10pt}
  \caption{Illustration of the augmented F{\"u}rer graph for the F{\"u}rer graph in \cref{fig:furer}. Here, the nodes in gray regions are the vertices of the original F{\"u}rer graph and other vertices are from the chains. We also use different colors to distinguish different types of edges.}
  \label{fig:fuer_aug}

\end{figure}

The main motivation of \cref{def:aug_furer} is that the added auxiliary vertices help distinguish the sets $\meta_F(u)$ for different $u$ since the lengths of chains $C_u$ are different. Indeed, let $\mathsf{A}$ be any SWL algorithm containing a local aggregation or any FWL-type algorithm, and consider playing the pebbling game for $\mathsf{A}$ on augmented F{\"u}rer graphs $\tilde G(F)$ and its twisted version $\tilde H(F)$. We have the following result, which shows that Duplicator's best strategy is to match the base vertices for each pair of pebbles. 

\begin{lemma}
\label{thm:furer_pebble_match1}
    Consider the pebbling game for any WL algorithm $\mathsf{A}$ played on graphs $\tilde G(F)$ and $\tilde H(F)$. Let $(\xi_G,\xi_H)$ be the position of any pebbles $u$/$v$ after any round. If $B(\xi_G)\neq B(\xi_H)$, then Spoiler can win the remaining game.
\end{lemma}
\begin{proof}
    It suffices to consider the (weakest) vanilla SWL algorithm $\mathsf{SWL(VS)}$, because Spoiler has more choices to win when considering more powerful WL algorithms. Also, if $B(\xi_G)\neq B(\xi_H)$ holds for pebbles $u$, Spoiler can just move pebble $v$ in $\tilde G(F)$ in subsequent rounds so that the position of $v$ eventually coincides with pebble $u$. This is feasible because $\tilde G(F)$ is connected (\cref{thm:furer_connectivity}). Now if the position of pebble $v$ in $\tilde H(F)$ does not coincide with pebble $u$, Spoiler already wins. Therefore, in the remaining proof we can assume that $B(\xi_G)\neq B(\xi_H)$ holds for pebbles $v$.

    Without loss of generality, assume $B(\xi_G):=v_G<v_H:=B(\xi_H)$ (note that $\gV_F$ is a number set). Spoiler's strategy is then to move pebble $v$ in $\tilde G(F)$ towards the endpoint of the chain $C_{v_G}$. Throughout the process, Duplicator has to keep the pebble $v$ in $\tilde H(F)$ located on the chain $C_{v_H}$ (otherwise, the vertices not from the chains must have a degree of at least three (\cref{thm:furer_connectivity}), and when the degrees of $v_G$ and $v_H$ do not match, Spoiler can win in the next round). When Spoiler finally places pebble $v$ in $\tilde G(F)$ to the endpoint of chain $C_{v_G}$, Duplicator cannot place the other pebble $v$ in $\tilde H(F)$ to a vertex of degree 1, so Spoiler can win in the next round.
\end{proof}

Similarly, Duplicator has to ensure that for any pebbles $u$/$v$ on the two graphs, either they are both placed on the auxiliary vertices, or neither of them is placed on the auxiliary vertices. When they are both placed on the auxiliary vertices, the distance to the corresponding chain endpoint must be the same. It is easy to see that Duplicator can always achieve this goal. Therefore, when Duplicator follows her best strategy, there is no reason for Spoiler to place pebbles on these auxiliary vertices.

We next consider the case when the positions of both pebbles $u,v$ in $\tilde G(F)$ correspond to the same base vertex. We have the following result:

\begin{lemma}
\label{thm:furer_pebble_match2}
    Consider the pebbling game for any WL algorithm $\mathsf{A}$ played on graphs $\tilde G(F)$ and $\tilde H(F)$. Let $u=(\xi_G,\xi_H)$ and $v=(\eta_G,\eta_H)$ be the positions of pebbles $u$ and $v$ after any round. Assume all pebbles are not placed on auxiliary vertices and correspond to the same base vertex $w\in\gV_F$. Denote $\xi_G=(w,\gT_G)$, $\xi_H=(w,\gT_H)$, $\eta_G=(w,\gU_G)$, $\eta_H=(w,\gU_H)$. If $\gT_G\triangle\gU_G\neq\gT_H\triangle\gU_H$, Spoiler can win the remaining game.
\end{lemma}
\begin{proof}
    The reason why Spoiler can win is essentially given in the proof of \cref{thm:furer_multiple_s_in_u}. Similar to the proof of \cref{thm:furer_pebble_match1}, we only consider the (weakest) vanilla SWL algorithm $\mathsf{SWL(VS)}$. Since $\gT_G\triangle\gU_G\neq\gT_H\triangle\gU_H$, there is a vertex $x\in \gT_G\triangle\gT_H\triangle\gU_G\triangle\gU_H$. Clearly, $x\in\gN_F(w)$. Spoiler's strategy is then to move pebble $v$ to any of its neighbors $(x,\tilde\gU_G)$ for some $\tilde\gU_G$, and Duplicator should move the other pebble $v$ to any of its neighbors $(x,\tilde\gU_H)$ for some $\tilde\gU_H$. By definition of F{\"u}rer graphs, since $\{(w,\gU_G),(x,\tilde\gU_G)\}$ is an edge of $\tilde G(F)$ and $\{(w,\gU_H),(x,\tilde\gU_H)\}$ is an edge of $\tilde H(F)$, we have
    \begin{align*}
        &(x\in\gU_G\leftrightarrow w\in\tilde\gU_G)=(x\in\gU_H\leftrightarrow w\in\tilde\gU_H)\\
        \iff & x\in\gU_G\triangle\gU_H\leftrightarrow w\in\tilde\gU_G\triangle\tilde\gU_H\\
        \iff & x\notin\gT_G\triangle\gT_H\leftrightarrow w\in\tilde\gU_G\triangle\tilde\gU_H\\
        \iff & (x\in\gT_G\leftrightarrow w\in\tilde\gU_G)\neq(x\in\gT_H\leftrightarrow w\in\tilde\gU_H)
    \end{align*}
    Therefore, the isomorphism type of the two vertex pairs $(\xi_G,\tilde\eta_G)$ and $(\xi_H,\tilde\eta_H)$ is not the same, where $\tilde\eta_G:=(x,\tilde\gU_G)$ and $\tilde\eta_H:=(x,\tilde\gU_H)$. Spoiler thus wins the game.
\end{proof}

\cref{thm:furer_pebble_match2} further limits Duplicator's strategy when two pebbles share the same base vertex. Moreover, when Duplicator follows her best strategy, it also implies that Spoiler cannot gain extra advantage when he places \emph{multiple} pebbles to positions that belong to the same base vertex. Actually, we will show below that the strategy for both Spoiler and Duplicator can be reduced to focusing only on the information of the base vertices placed by pebbles $u,v$. In particular, this results in a simplified pebbling game defined as follows.

\textbf{Simplified pebbling game for augmented F{\"u}rer graphs}. Let $F=(\gV_F,\gE_F)$ be the base graph of a proper F{\"u}rer graph. The simplified pebbling game is played on $F$. There are three pebbles $u,v,w$ of different types. Initially, all three pebbles are left outside the graph $F$. We first describe the game rule for Spoiler, which is similar to \cref{sec:pebbling_game} but is much simpler.

First consider SWL algorithms $\mathsf{A}(\gA,\pool)$. If $\pool=\mathsf{VS}$, Spoiler first places pebble $u$ on any vertex of $F$ and then places pebble $v$ on any vertex of $F$. If $\pool=\mathsf{SV}$, Spoiler first places pebble $v$ and then places pebble $u$.

The game then cyclically executes the following process. Depending on the aggregation scheme $\gA$, Spoiler can freely choose one of the following ways to play:
\begin{itemize}[topsep=0pt]
    \setlength{\itemsep}{0pt}
    \item Local aggregation $\agg^\mathsf{L}_\mathsf{u}\in\gA$. Spoiler first places pebble $w$ adjacent to the vertex placed by pebble $v$, then swaps pebbles $v$ and $w$, and finally places pebble $w$ outside the graph $F$.
    \item Global aggregation $\agg^\mathsf{G}_\mathsf{u}\in\gA$. Spoiler first places pebble $w$ on any vertex of $F$, then swaps pebbles $v$ and $w$, and finally places pebble $w$ outside $F$.
    \item Single-point aggregation $\agg^\mathsf{P}_\mathsf{uu}\in\gA$. Spoiler first places pebble $w$ to the position of pebble $u$, then swaps pebbles $v$ and $w$, and finally places pebble $w$ outside $F$.
    \item Single-point aggregation $\agg^\mathsf{P}_\mathsf{vu}\in\gA$. Spoiler swaps the position of pebbles $u$ and $v$.
\end{itemize}
The cases of $\agg^\mathsf{L}_\mathsf{v}$, $\agg^\mathsf{G}_\mathsf{v}$, $\agg^\mathsf{P}_\mathsf{vv}$ are similar (symmetric) to $\agg^\mathsf{L}_\mathsf{u}$, $\agg^\mathsf{G}_\mathsf{u}$, $\agg^\mathsf{P}_\mathsf{uu}$, so we omit them for clarity. 

Next consider FWL-type algorithms. Initially, Spoiler simultaneously places pebbles $u$ and $v$ on two vertices of $F$. The game then cyclically executes the following process. For $\mathsf{LFWL(2)}$, Spoiler first places pebble $w$ to some vertex in $\gN_F^1(v)$, then either swaps pebbles $v$, $w$ or swaps pebbles $u$, $w$, and finally places $w$ outside the graph $F$. The cases of $\mathsf{SLFWL(2)}$ and $\mathsf{FWL(2)}$ are similar, expect that $\gN_F^1(v)$ is replaced by $\gN_F^1(u)\cup\gN_F^1(v)$ and $\gV_F$, respectively.

We next describe the game rule for Duplicator, which is of a very
different kind. In brief, she maintains a subset $\gQ$ of connected components $\gQ\subset\mathsf{CC}_{\gS}(F)$ (\cref{def:connected_component}) where the set $\gS$ contains vertices of $F$ on which the pebbles $u,v,w$ are currently located. Initially, $\gQ:=\mathsf{CC}_{\emptyset}(F)=\{\gE_F\}$. Note that throughout the game, Spoiler only performs three types of basic operations: $(\mathrm{i})$ add a pebble/two pebbles and place it/them on vertices of $F$; $(\mathrm{ii})$ remove a pebble and leave it outside the graph $F$; $(\mathrm{iii})$ swap the positions of two pebbles. Once Spoiler performs an operation above, Duplicator will update $\gQ$ according to the following rules so that the parity of $|\gQ|$ is always odd throughout the game.
\begin{itemize}[topsep=0pt]
    \setlength{\itemsep}{0pt}
    \item When Spoiler places some pebble(s) on vertices of $F$, there are two cases. If $\mathsf{CC}_{\gS}(F)$ does not change, then Duplicator does nothing. Otherwise, the presence of new pebbles will split some connected components into a set of smaller regions. For each original connected component $\gP\subset\gE_F$ that is split into $\gP_1,\cdots,\gP_k$ with $\bigcup_{i=1}^k \gP_i=\gP$, Duplicator can replace $\gQ$ by $\tilde\gQ=(\gQ\backslash \gP)\cup \{\gP_{j_1},\cdots,\gP_{j_l}\}$ for some $j_1,\cdots,j_l\in[k]$, such that $|\tilde\gQ|\bmod 2=1$. In other words, Duplicator updates the set $\gQ$ by removing the old connected component $\gP$ (if in the set) and adding some new partitioned components, while ensuring that the parity of the size of $\gQ$ does not change.
    \item When Spoiler removes a pebble and leave it outside the graph $F$, there are also two cases. If $\mathsf{CC}_{\gS}(F)$ does not change, then Duplicator does nothing. Otherwise, the removal of a pebble will merge several connected components $\gP_1,\cdots,\gP_k$ into a larger one $\gP=\bigcup_{i=1}^k \gP_i$. Duplicator then replaces $\gQ$ by either $\tilde\gQ=\gQ\backslash \{\gP_{1},\cdots,\gP_{k}\}$ or $\tilde\gQ=(\gQ\backslash \{\gP_{1},\cdots,\gP_{k}\})\cup\gP$, depending on which one satisfies $|\tilde\gQ|\bmod 2=1$. In other words, Duplicator updates the set $\gQ$ by removing these small connected components and optionally adding the merged component to preserve parity.
    \item When Spoiler swaps the positions of two pebbles, $\mathsf{CC}_{\gS}(F)$ clearly does not change and thus Duplicator does nothing.
\end{itemize}
For the case of local aggregation $\agg^\mathsf{L}_\mathsf{u}$, there is an extra constraint for Duplicator: after Spoiler places pebble $w$ adjacent to $v$ and Duplicator updates $\gQ$, Duplicator should additionally ensure that $\{\{v,w\}\}\notin\gQ$. Similar game rule applies for local aggregation $\agg^\mathsf{L}_\mathsf{v}$.

After any round, Spoiler wins if pebble $u$ is adjacent to $v$ and $\{\{u,v\}\}\in\gQ$. In other words, Spoiler wins if there is a connected component in $\gQ$ with only one edge. Finally, Duplicator wins if Spoiler cannot win after any number of rounds.

Below, we will prove that the simplified pebbling game designed above is actually equivalent to the original pebbling game. Importantly, the simplified pebbling game is played on the base graph $F$ rather than the sophisticated (augmented) F{\"u}rer graphs and avoids the complicated vertex selection procedure (\cref{def:vertex_selection}), which greatly eases the analysis of players' strategies.

\begin{theorem}
    Let $\tilde G(F)$ and $\tilde H(F)$ be any augmented proper F{\"u}rer graph and its twisted version for some base graph $F$. For any WL algorithm $\mathsf{A}$ considered in this paper, Spoiler can win the corresponding pebbling game on graphs $\tilde G(F)$ and $\tilde H(F)$ if and only if he can win the simplified pebbling game on graph $F$.
\end{theorem}
\begin{proof}
    In the original pebbling game, let $\tilde H(F)=\twist(\tilde G(F),\gE)$ for some $\gE$ with $|\gE|\bmod 2= 1$. Based on \cref{thm:furer_pebble_match1}, after any round we can assume that the pebbles $u,v$ are placed on $u=(\xi_G,\xi_H)$, $v=(\eta_G,\eta_H)$ with matching base vertices, i.e., we can denote $\xi_G=(x,\gT_x)$, $\xi_H=(x,\gU_x)$, $\eta_G=(y,\gT_y)$, $\eta_H=(y,\gU_y)$. We also assume that the condition of \cref{thm:furer_pebble_match2} holds when $x=y$. The proof is divided into the following parts. 

    \textbf{Part 1} (understanding the relationship between the two types of pebbling games). Consider two game states with different pebble positions:
    \begin{itemize}[topsep=0pt]
    \raggedright
    \setlength{\itemsep}{0pt}
        \item State 1: the positions of pebbles are $u=((x,\gT_x),(x,\gU_x^{(1)}))$, $v=((y,\gT_y),(y,\gU_y^{(1)}))$;
        \item State 2: the positions of pebbles are $u=((x,\gT_x),(x,\gU_x^{(2)}))$, $v=((y,\gT_y),(y,\gU_y^{(2)}))$.
    \end{itemize}
    In other words, the positions of pebbles on graph $\tilde G(F)$ are the same for the two states, but the positions of pebbles on graph $\tilde H(F)$ differ. By \cref{thm:furer_main_corollary}, there is a proper isomorphism $f$ from $\tilde H(F)$ to $\twist(\tilde H(F), \tilde \gE)$ for some $\tilde\gE$ with $|\tilde\gE|\bmod 2=0$, such that $f(x,\gU_x^{(1)})=(x,\gU_x^{(2)})$ and $f(y,\gU_y^{(1)})=(y,\gU_y^{(2)})$. Note that the second bullet of \cref{thm:furer_main_corollary} is satisfied since we assume that Duplicator follows the strategy of \cref{thm:furer_pebble_match2} and thus $\gU_x^{(1)}\triangle\gU_y^{(1)}=\gT_x\triangle\gT_y=\gU_x^{(2)}\triangle\gU_y^{(2)}$ when $x=y$. Now using \cref{thm:furer_main_corollary} again, there is a proper \emph{automorphism} $\tilde f$ of $\tilde H(F)$ satisfying $\tilde f(x,\gU_x^{(1)})=(x,\gU_x^{(2)})$ and $\tilde f(y,\gU_y^{(1)})=(y,\gU_y^{(2)})$, if and only if $|\tilde\gE\cap\gP|\bmod 2=0$ for all $\gP\in\mathsf{CC}_{\{x,y\}}(F)$. In other words, if $|\tilde\gE\cap\gP|\bmod 2=0$ for all $\gP\in\mathsf{CC}_{\{x,y\}}(F)$, then the two states are equivalent.
    
    Similarly, since $\tilde H(F)=\twist(\tilde G(F),\gE)$, one can also find for each $i=1,2$ a proper isomorphism $f_i$ from $\tilde G(F)$ to $\twist(\tilde H(F), \tilde \gE_i)$, such that $f(x,\gT_x)=(x,\gU_x^{(i)})$ and $f(y,\gT_y)=(y,\gU_y^{(i)})$. Based on the above analysis, whether Spoiler can win the game at state $i$ will thus depend purely on the set
    \begin{align*}
        \gQ_i&=\{\gP\in\mathsf{CC}_{\{x,y\}}(F):|\gE_i\cap\gP|\bmod 2=0\}\\
        &=\{\gP\in\mathsf{CC}_{\{x,y\}}(F):|(\gE\triangle\gE_i)\cap\gP|\bmod 2=1\}.
    \end{align*}
    Namely, if $\gQ_1=\gQ_2$, then the two states are equivalent. This is why in the simplified pebbling game Duplicator only maintains the set $\gQ$, which has a similar meaning to $\gQ_i$.
    
    \textbf{Part 2} (regarding vertex selection). We show vertex selection in \cref{def:vertex_selection} can be simplified to satisfy $|\gS^\mathsf{S}|=|\gS^\mathsf{D}|=1$. First, when $\gS^\mathsf{S}$ contains multiple vertices that correspond to different base vertices in $F$, Duplicator must respond by matching each base vertex separately and merging them to obtain $\gS^\mathsf{D}$, otherwise Spoiler can win according to \cref{thm:furer_pebble_match1}. When Duplicator follows this strategy, there is no reason for Spoiler to choose multiple base vertices. Next, when $\gS^\mathsf{S}$ contains multiple vertices that correspond to the same base vertex in $F$, Duplicator must match each $(x,\gX)\in\gS^\mathsf{S}$ with $(x,\gX\triangle\gD)\in\gS^\mathsf{D}$ by selecting a set $\gD$ (according to \cref{thm:furer_pebble_match2}). In this case, Spoiler still does not gain an additional benefit by selecting $\gS^\mathsf{S}$ with multiple elements. Moreover, it does not make any difference whether Spoiler chooses to move pebbles on $\tilde G(F)$ or on $\tilde H(F)$. Therefore, the pebbling game can be simplified so that Spoiler directly moves a pebble in $\tilde G(F)$ and Duplicator responds by moving the corresponding pebble in $\tilde H(F)$. (Nevertheless, note that the vertex selection procedure is still necessary when dealing with auxiliary vertices as in the proof of \cref{thm:furer_pebble_match1}).

    \textbf{Part 3} (equivalence between updating pebble positions and updating $\gQ$ for Duplicator). Suppose that in a certain round Spoiler places a pebble $w$ to vertex $(z,\gT_z)$ in $\tilde G(F)$ . If the placement of $w$ does not increase the number of connected components, then no matter how Duplicator responds by replacing the other pebble $w$ to $(z,\gU_z)$ in $\tilde H(F)$, the game is equivalent due to Part 1 and the set $\gQ$ should not change, which coincides with the game rule. If the placement of $w$ increases the number of connected components, then how Duplicator chooses the position of the other pebble $w$ will matter. Suppose Duplicator places the other pebble $w$ on the vertex $(z,\gU_z)$, then the value of $\gT_z\triangle\gU_z$ determines the update of $\gQ$ by \cref{thm:furer_main_corollary}. Conversely, each possible game rule for updating $\gQ$ also corresponds to at least one feasible position $(z,\gU_z)$. 

    For the local aggregation $\agg^\mathsf{L}_\mathsf{u}$, there is an additional restriction that the pebble $w$ should be adjacent to pebble $v$. Clearly, the presence of $w$ will make a new connected component $\{\{v,w\}\}$. It is easy to see that $\{\{v,w\}\}\notin\gQ$, otherwise pebble $w$ is not adjacent to pebble $v$ in $\tilde H(F)$. For localized FWL aggregations, although pebble $w$ should also be placed in the neighborhood of some pebble (e.g., $w\in\gN_{\tilde G(F)}^1(v)$), we may not add this restriction for Duplicator, because if Duplicator does not obey the game rule, Spoiler can always win after this round by swapping a pair of pebbles (e.g., swapping $u$ and $w$) such that the isomorphism types of pebbles $u$ and $v$ differ between $\tilde G(F)$ and $\tilde H(F)$.

    Similarly, when Spoiler places a pebble $w$ outside the graph $\tilde G(F)$, the connected components may merge, and $\gQ$ should be updated accordingly while preserving the parity of its size. This matches the design of the simplified pebbling game. Finally, if Spoiler swaps a pair of pebbles, all connected components remains unchanged, so Duplicator does nothing in the simplified pebbling game.
\end{proof}

\subsection{Concrete constructions}
\label{sec:proof_separation_part3}

In this section, we give concrete constructions to prove all results of \cref{thm:separation}. We split the proof into a collection of lemmas. All the proofs are based on constructing base graphs $F$ and studying the simplified pebbling game developed in \cref{sec:pebbling_furer} on $F$.

\textbf{Illustration}. For clarity, we illustrate the proof of each lemma with a set of figures (\cref{fig:counterexample_pswlsv,fig:counterexample_swlvs,fig:counterexample_gswl,fig:counterexample_lfwl_sswl,fig:counterexample_slfwl,fig:counterexample_slfwl_2,fig:counterexample_fwl,fig:counterexample_pooling}). In each of these figures, the node in orange/green/purple responds to the vertex that holds pebble $u$/$v$/$w$, respectively. We use bold red edges to denote connected components in $\gQ$ chosen by Duplicator.

\begin{figure}[!t]
    \small
    \centering
    \begin{tabular}{c|ccc}
         & \includegraphics[width=0.225\textwidth]{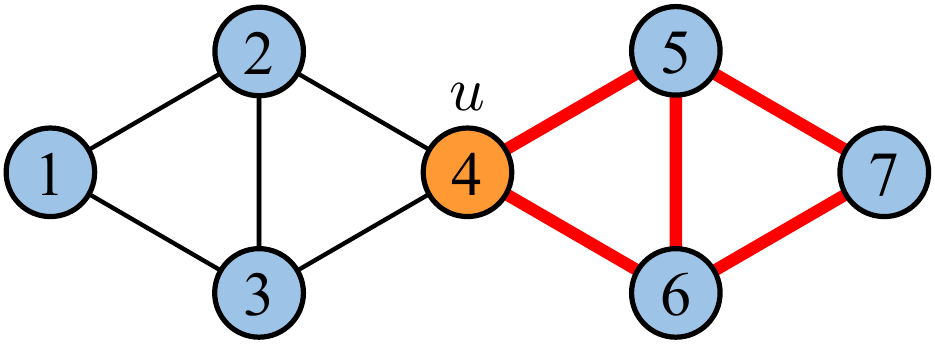} & \includegraphics[width=0.225\textwidth]{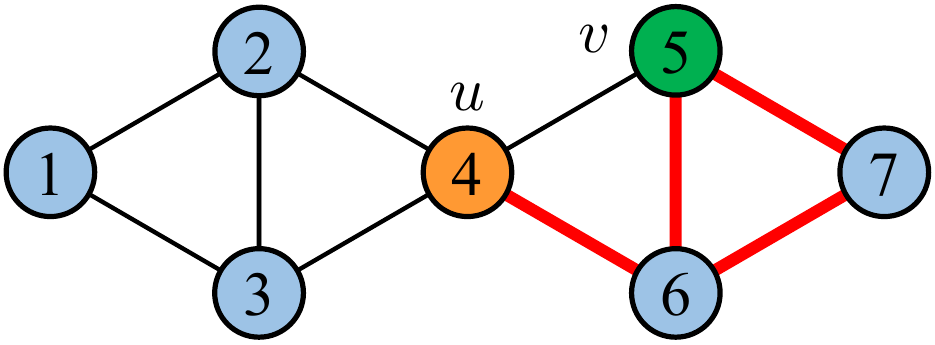} & \includegraphics[width=0.225\textwidth]{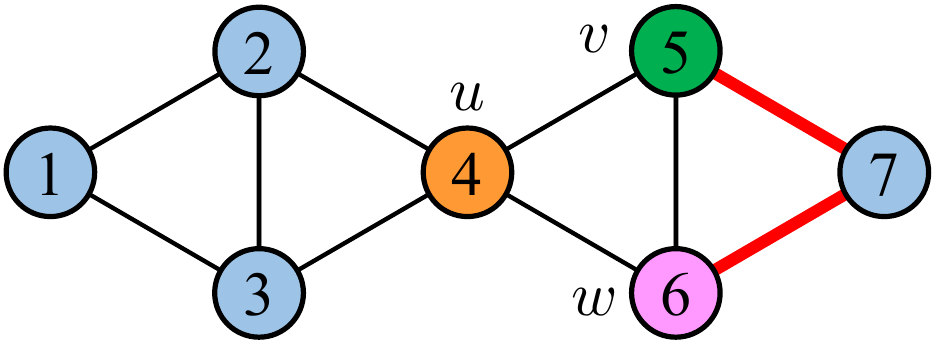}\\
        & (a) & (b) & (c)\\
         & \includegraphics[width=0.225\textwidth]{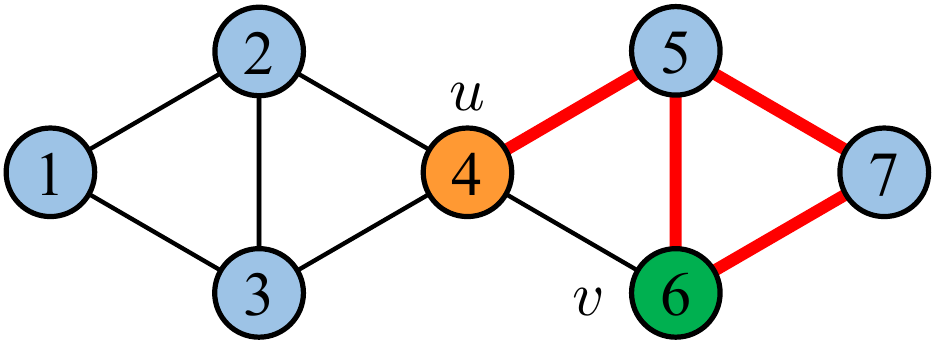} & \includegraphics[width=0.225\textwidth]{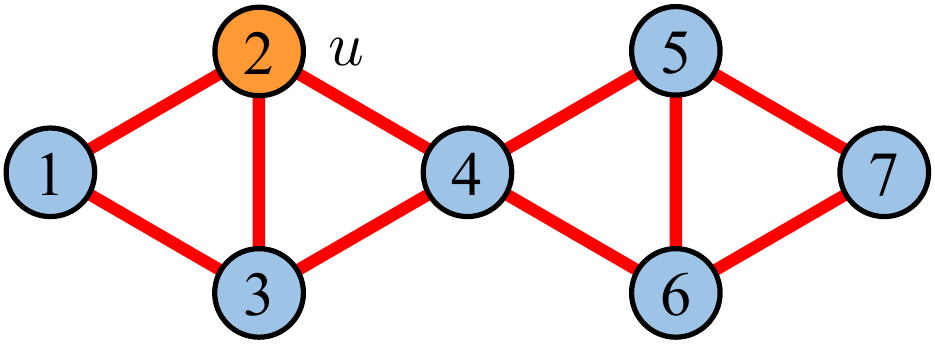} & \includegraphics[width=0.225\textwidth]{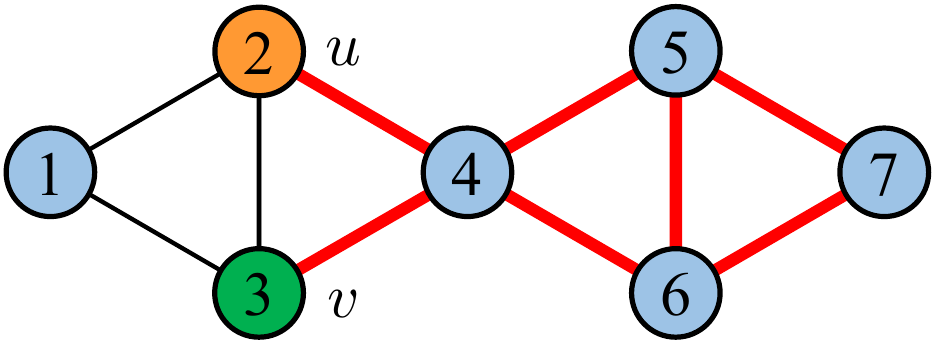}\\
        & (d) & (e) & (f)\\
        \includegraphics[width=0.225\textwidth]{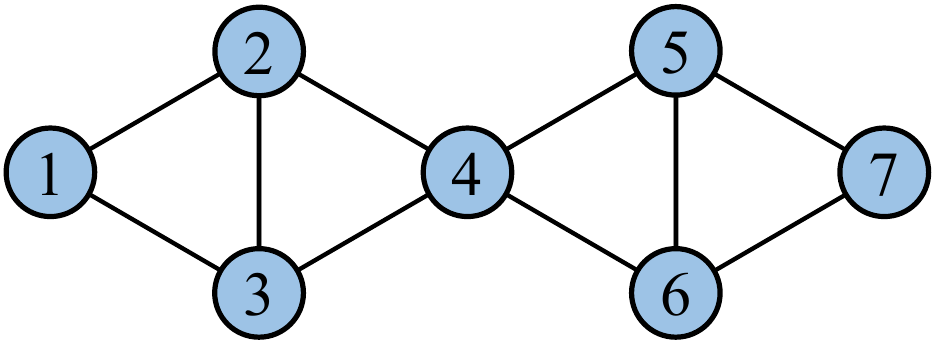} & \includegraphics[width=0.225\textwidth]{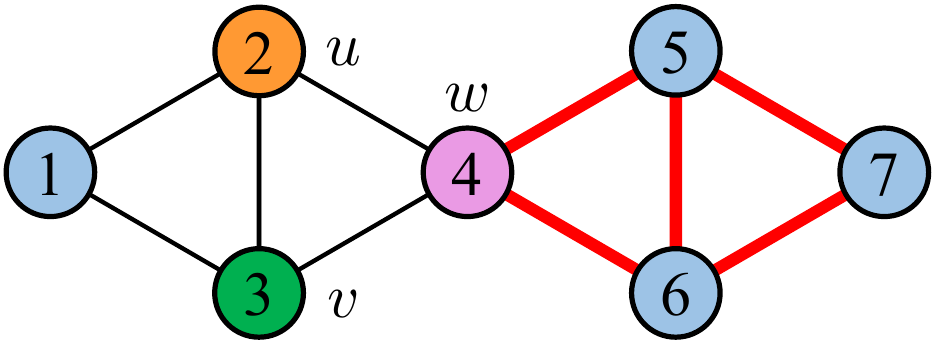} & \includegraphics[width=0.225\textwidth]{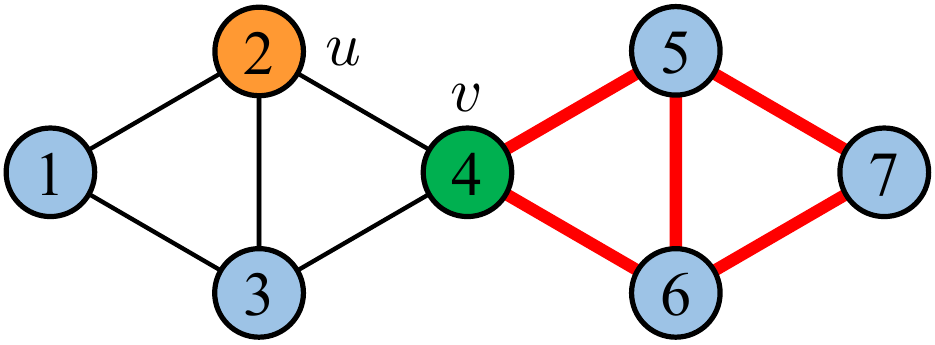} & \includegraphics[width=0.225\textwidth]{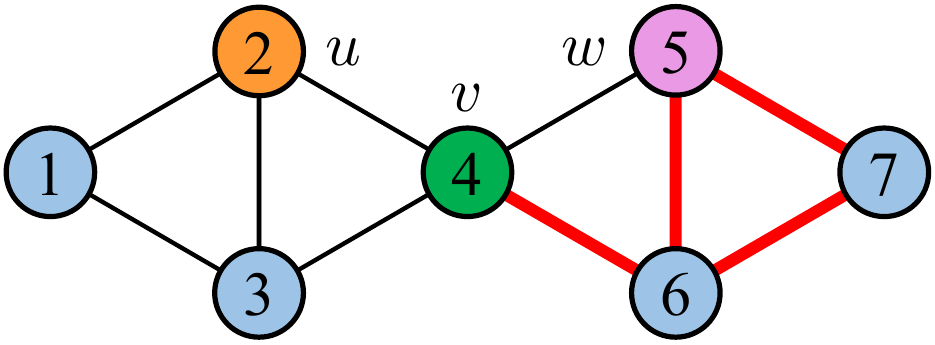} \\ 
        Base graph & (g) & (h) & (i)\\
         & \includegraphics[width=0.225\textwidth]{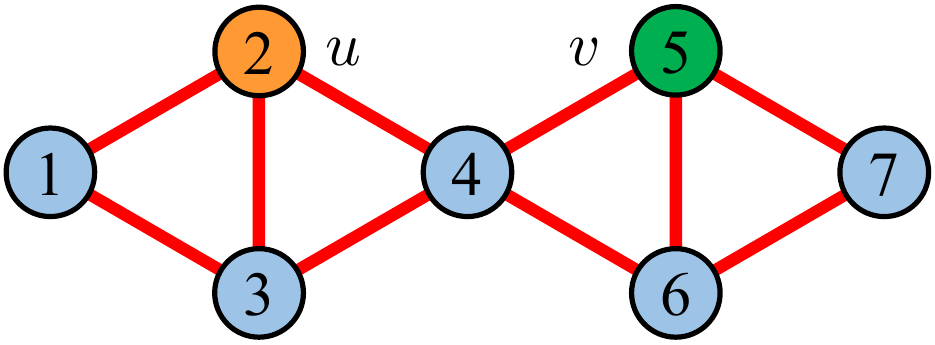} & \includegraphics[width=0.225\textwidth]{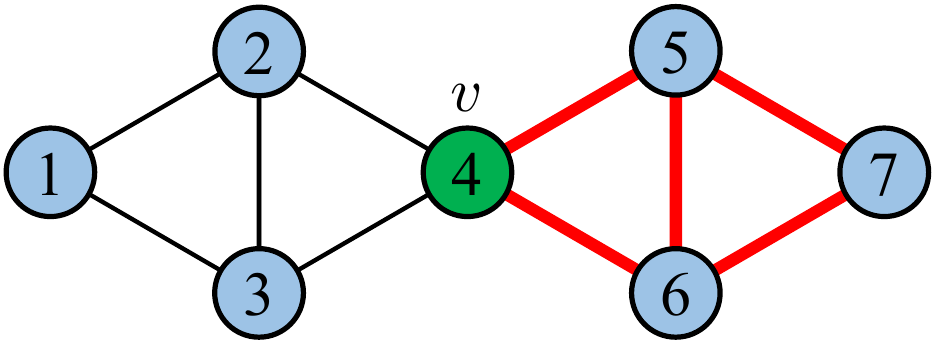} & \includegraphics[width=0.225\textwidth]{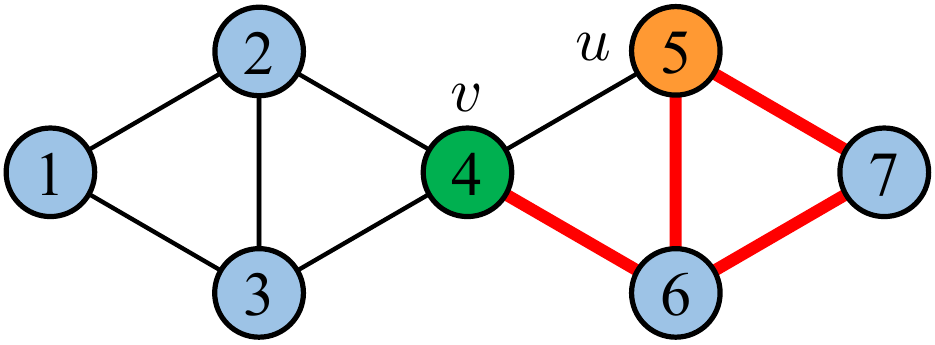} \\ 
        & (j) & (k) & (l)\\
         & \includegraphics[width=0.225\textwidth]{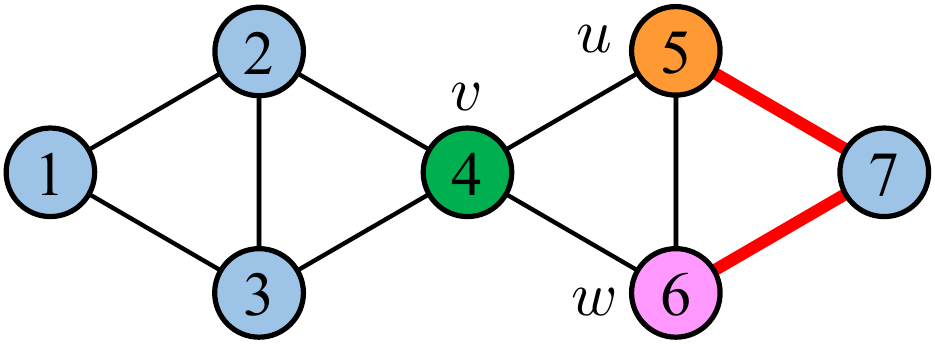} & \includegraphics[width=0.225\textwidth]{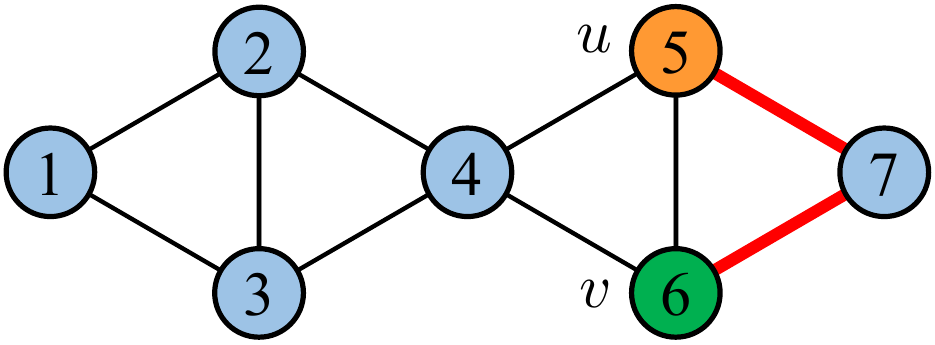} & \includegraphics[width=0.225\textwidth]{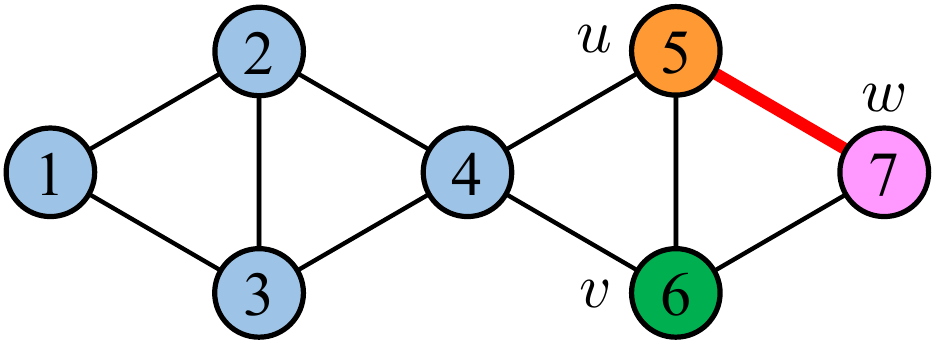} \\ 
        & (m) & (n) & (o)\\
    \end{tabular}
    \vspace{-5pt}
    \caption{Illustration of the proof of \cref{thm:counterexample_pswlsv}. When Duplicator follows her optimal strategy, the game process of $\mathsf{SWL(VS)}$ corresponds to a sequence of figures, such as (a, b, c, d, ...), (e, f, g, h, i, j, ...), or (e, h, i, j, ...), depending how Spoiler plays. In all cases, Spoiler cannot win. The game process of $\mathsf{PSWL(VS)}$ is similar. In contrast, the game process of $\mathsf{SWL(SV)}$ corresponds to figures (k, l, m, n, o) and Spoiler eventually wins as shown in figure (o).}
    \vspace{-5pt}
    \label{fig:counterexample_pswlsv}
\end{figure}
\begin{lemma}
\label{thm:counterexample_pswlsv}
    There exist two non-isomorphic graphs such that 
    \begin{itemize}[topsep=0pt]
    \setlength{\itemsep}{0pt}
        \item $\mathsf{SWL(SV)}$ can distinguish them;
        \item $\mathsf{SWL(VS)}$ cannot distinguish them;
        \item $\mathsf{PSWL(VS)}$ cannot distinguish them.
    \end{itemize}
\end{lemma}
\begin{proof}
    The base graph is constructed in \cref{fig:counterexample_pswlsv}. We separately consider each algorithm.

    We first analyze the simplified pebbling game for algorithm $\mathsf{SWL(VS)}$. Initially, Spoiler should first place pebble $u$ on some vertex. Due to the symmetry of the graph, there are three cases: vertex 4, vertex 2, and vertex 1. We separately consider each case below:
    \begin{itemize}[topsep=0pt]
    \setlength{\itemsep}{0pt}
        \item If Spoiler places pebble $u$ on vertex 4, then the graph is split into two connected components. By symmetry, without loss of generality suppose Duplicator selects the component at the right of $u$ (\cref{fig:counterexample_pswlsv}(a)). Next, Spoiler will place pebble $v$ on some vertex. Clearly, his best strategy is to choose vertex 5 (or equivalently, vertex 6), which can further split the connected component into two parts. Duplicator has to respond by choosing the larger part (\cref{fig:counterexample_pswlsv}(b)). In the next round, according to the game rule, Spoiler should place pebble $w$ adjacent to pebble $v$. He'd better place it on vertex 6. Duplicator can respond appropriately without losing the game (\cref{fig:counterexample_pswlsv}(c)). Then Spoiler swaps pebbles $v$ and $w$ and leaves $w$ outside the graph. It can be seen that multiple connected components are then merged into a larger one, yielding \cref{fig:counterexample_pswlsv}(d). Now the game state is equivalent to \cref{fig:counterexample_pswlsv}(b) by symmetry. It is easy to see that Spoiler can never win the game.
        \item If Spoiler places pebble $u$ on vertex 2, then the connected component remains unchanged, so Duplicator just does nothing (\cref{fig:counterexample_pswlsv}(e)). Next, Spoiler will place vertex $v$ on some vertex, e.g., vertex 3 or vertex 4. Regardless of where he places pebble $v$, Duplicator's strategy is always to choose the rightmost connected component (see \cref{fig:counterexample_pswlsv}(f) and \cref{fig:counterexample_pswlsv}(h) for the two cases). First consider the case when pebble $v$ is placed on vertex 3 (\cref{fig:counterexample_pswlsv}(f)). In the next round, Spoiler should place pebble $w$ adjacent to pebble $v$. He'd better place it on vertex 4 to further split the connected component. Duplicator just responds by selecting again the rightmost component as shown in \cref{fig:counterexample_pswlsv}(g). Spoiler then swaps pebbles $v$ and $w$ and leaves $w$ outside the graph. It can be seen that the game returns to \cref{fig:counterexample_pswlsv}(h)). When Spoiler continues to place pebble $w$ adjacent to pebble $v$, Duplicator again responds by updating the connected component (\cref{fig:counterexample_pswlsv}(i)). However, when Spoiler swaps pebbles $v$ and $w$ and leaves $w$ outside the graph, multiple connected components then merges into a whole, as shown in \cref{fig:counterexample_pswlsv}(j). Clearly, Spoiler cannot win the game as well.
        \item If Spoiler places pebble $u$ on vertex 1, we can similarly prove that Spoiler cannot win the game. Actually, placing pebble $u$ on vertex 1 is clearly not optimal.
    \end{itemize}
    We next analyze the simplified pebbling game for algorithm $\mathsf{PSWL(VS)}$, which is similar to $\mathsf{SWL(VS)}$ except that Spoiler has the additional ability to move pebble $u$ to the position of pebble $v$. However, when Spoiler performs this operation, the resulting game will simply be equivalent to the three cases studied above, e.g., \cref{fig:counterexample_pswlsv}(a) or \cref{fig:counterexample_pswlsv}(e), except that pebble $v$ is also present and coincides with $u$. As already proved above, Spoiler cannot win the game.

    We finally analyze the simplified pebbling game for algorithm $\mathsf{SWL(SV)}$. In the beginning, Spoiler can first place pebble $v$ on vertex 4, and suppose Duplicator chooses the connected component at the right of $v$ (\cref{fig:counterexample_pswlsv}(k)). Spoiler can then place pebble $u$ on vertex 5 to further split this connected component, and Duplicator has to respond by choosing the rightmost component (\cref{fig:counterexample_pswlsv}(l)). In the next round, Spoiler can place $w$ on vertex 6. Duplicator has not lost the game yet (see \cref{fig:counterexample_pswlsv}(m)). Then it comes to the major difference: when Spoiler swaps pebbles $v$ and $w$ and leaves $w$ outside the graph, the rightmost connected component is \emph{not} merged into a larger one due to the position of pebbles $u,v$ (see \cref{fig:counterexample_pswlsv}(n)). Therefore, in the next round, Spoiler can further use pebble $w$ to split the component as shown in \cref{fig:counterexample_pswlsv}(o), and Duplicator has no choice other than selecting the connected component $\{\{5,7\}\}$. Duplicator loses the game after this round.
\end{proof}

\textbf{Insight into \cref{thm:counterexample_pswlsv}}. The reason why $\mathsf{SWL(SV)}$ is stronger lies in the fact that Spoiler can specify the position of pebble $u$ \emph{after seeing Duplicator's response} because pebble $v$ is first placed before pebble $u$ is placed. In this way, Spoiler can exploit such information to better choose the position of pebble $u$. Importantly, note that pebble $u$ cannot be moved easily according to the game rule, therefore determining its position later may have additional benefits.

\begin{figure*}[!t]
    \small
    \centering
    \begin{tabular}{c|ccc}
         & \includegraphics[width=0.225\textwidth]{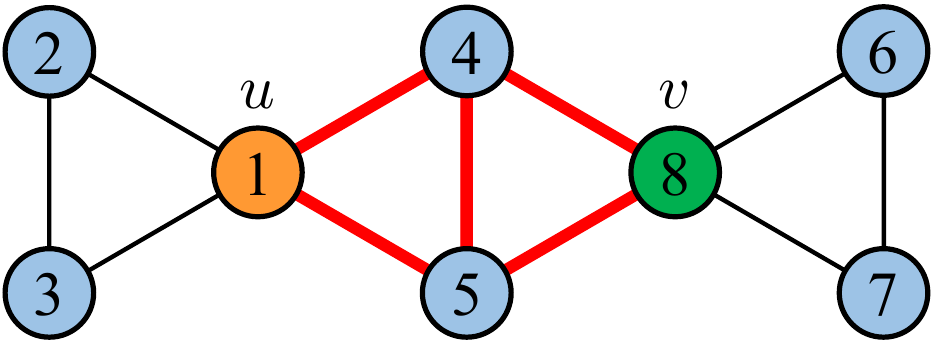} & \includegraphics[width=0.225\textwidth]{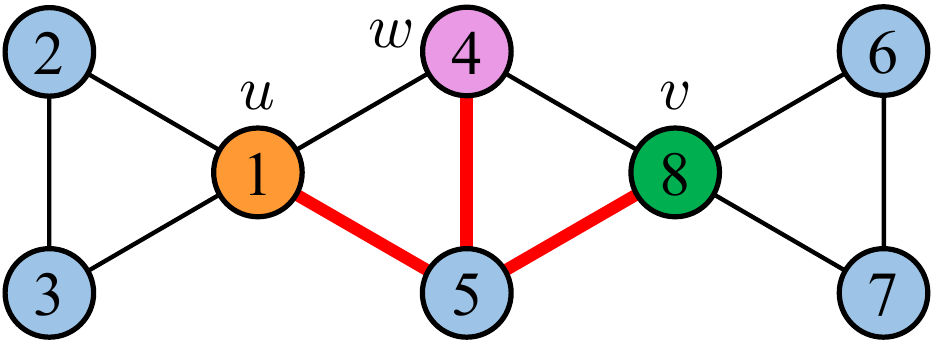} & \includegraphics[width=0.225\textwidth]{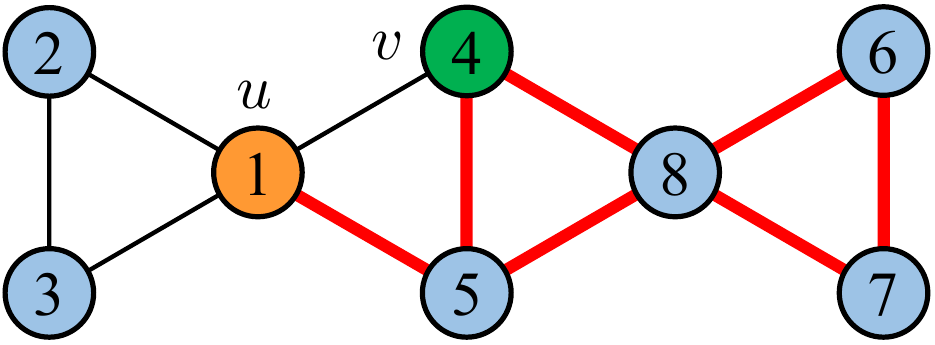}\\
        & (a) & (b) & (c)\\
        \includegraphics[width=0.225\textwidth]{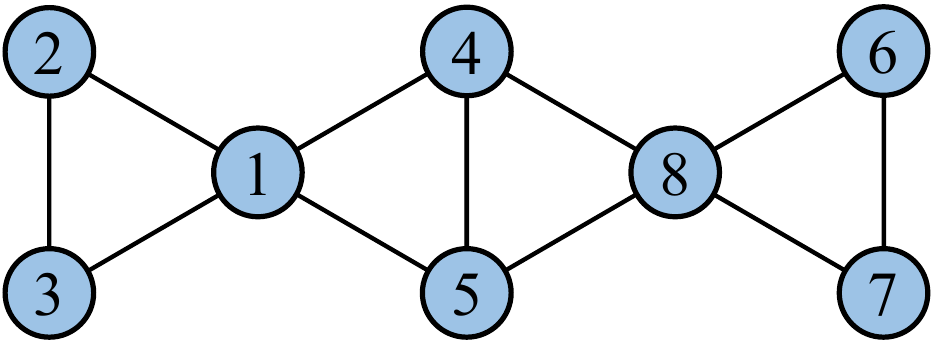} & \includegraphics[width=0.225\textwidth]{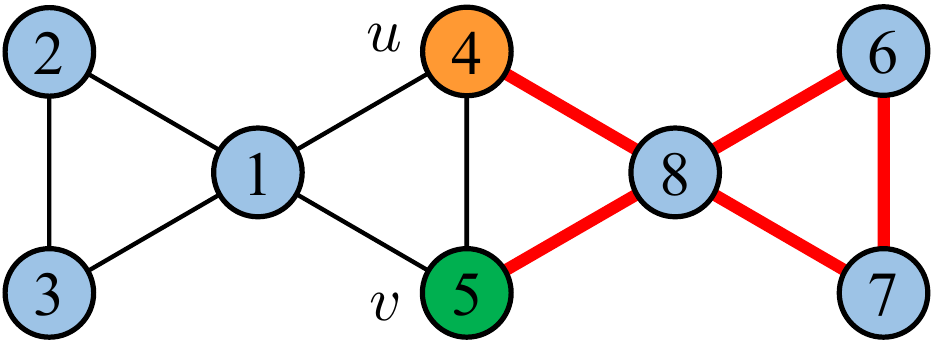} & \includegraphics[width=0.225\textwidth]{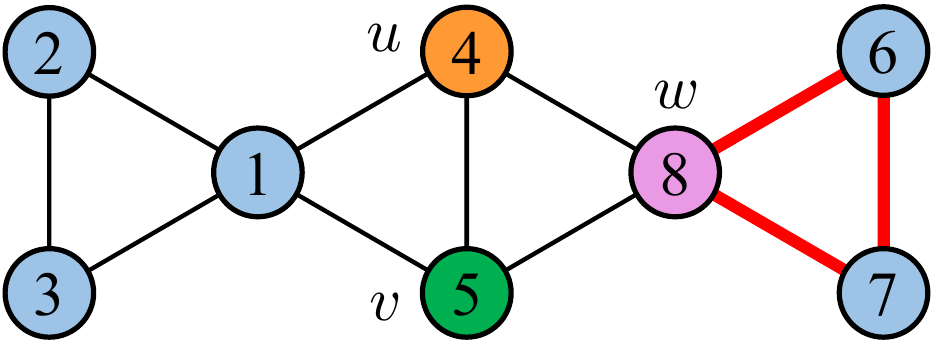} & \includegraphics[width=0.225\textwidth]{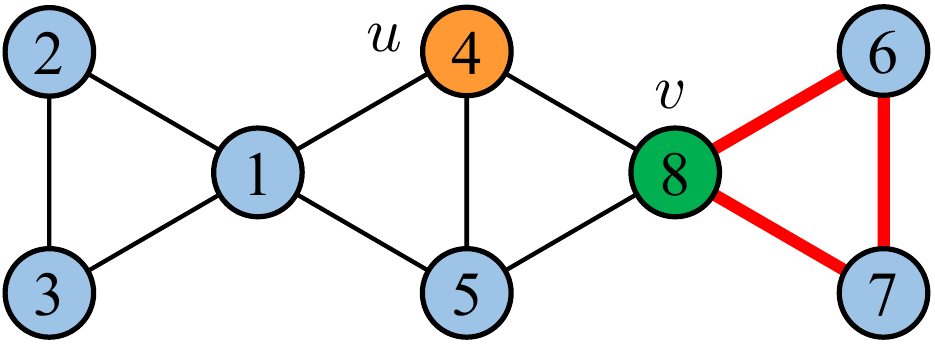}\\
        Base graph & (d) & (e) & (f)\\
         & \includegraphics[width=0.225\textwidth]{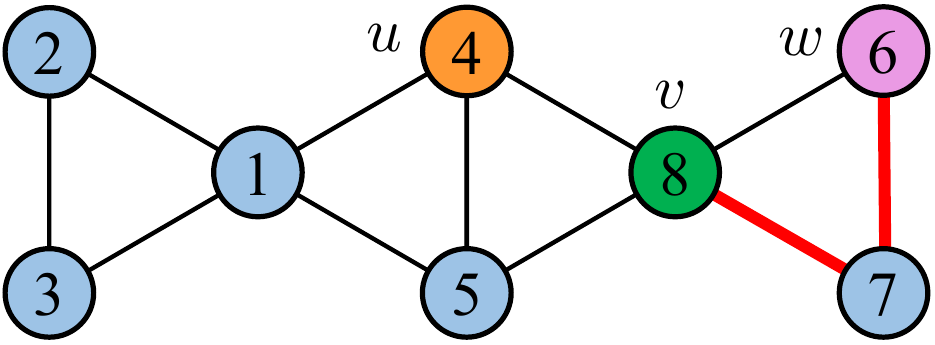} & \includegraphics[width=0.225\textwidth]{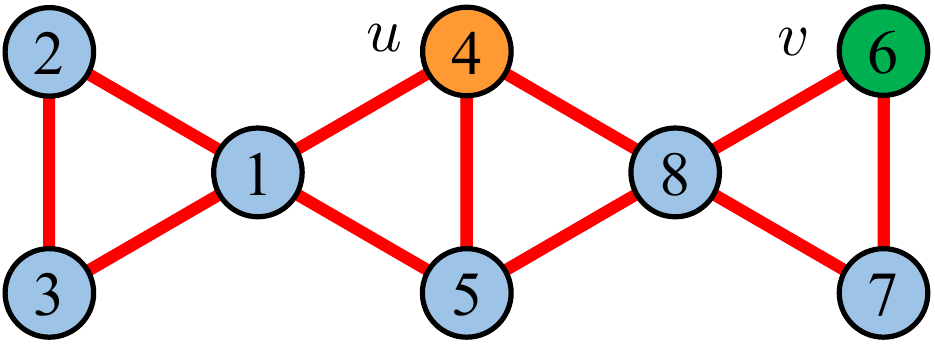} & \includegraphics[width=0.225\textwidth]{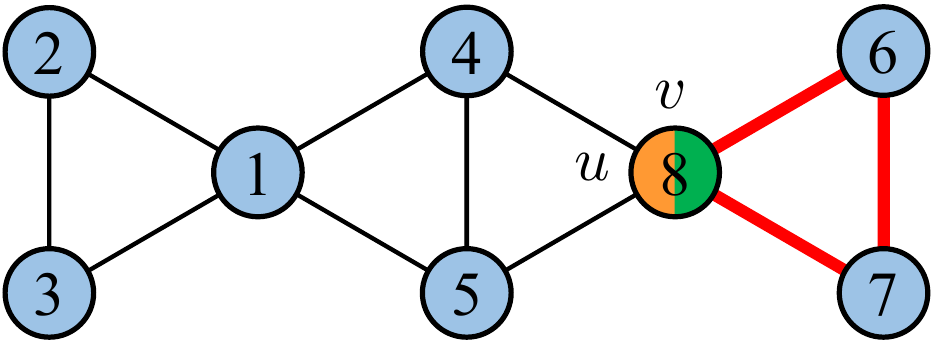} \\
         & (g) & (h) & (i)\\
         & \includegraphics[width=0.225\textwidth]{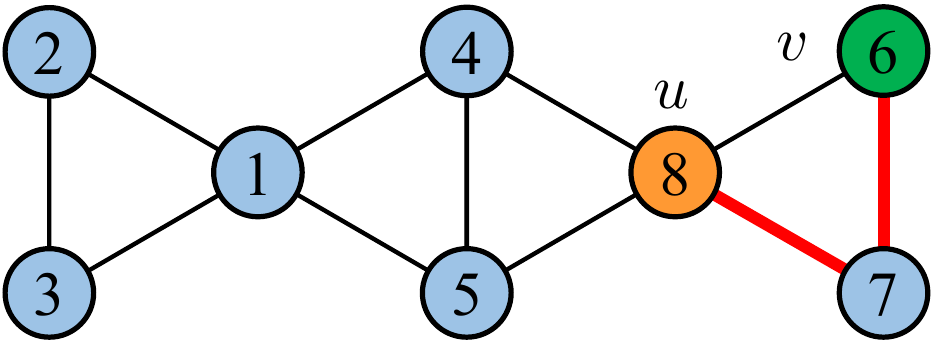} & \includegraphics[width=0.225\textwidth]{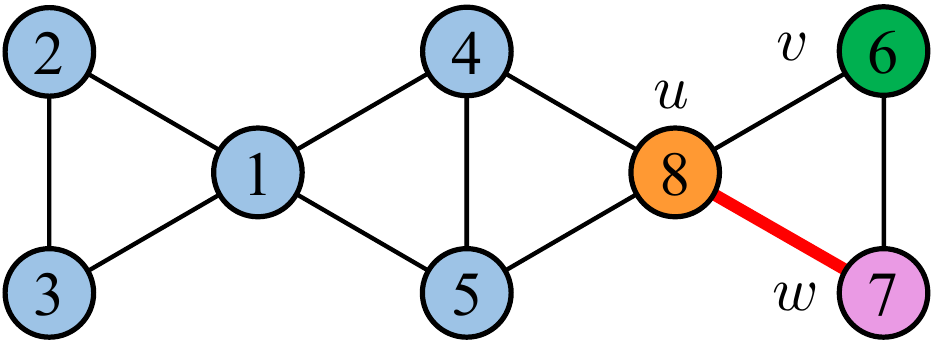} \\ 
         & (j) & (k)
    \end{tabular}
    \vspace{-5pt}
    \caption{Illustration of the proof of \cref{thm:counterexample_swlvs}. When Duplicator follows her optimal strategy, the game process of $\mathsf{SWL(VS)}$ (or $\mathsf{SWL(SV)}$) may correspond to figures (a, b, c, ...) or figures (d, e, f, g, h, ...) depending on how Spoiler chooses the initial positions of pebbles $u,v$. In both cases, Spoiler cannot win. In contrast, the game process of $\mathsf{PSWL(VS)}$ corresponds to figures (d, e, f, i, j, k) and Spoiler eventually wins in figure (k).}
    \label{fig:counterexample_swlvs}
    \vspace{-5pt}
\end{figure*}
\begin{lemma}
\label{thm:counterexample_swlvs}
    There exist two non-isomorphic graphs such that 
    \begin{itemize}[topsep=0pt]
    \setlength{\itemsep}{0pt}
        \item $\mathsf{PSWL(VS)}$ can distinguish them;
        \item $\mathsf{SWL(VS)}$ cannot distinguish them;
        \item $\mathsf{SWL(SV)}$ cannot distinguish them.
    \end{itemize}
\end{lemma}
\begin{proof}
    The base graph in constructed in \cref{fig:counterexample_swlvs}, which can be seen as a simple adaptation of \cref{fig:counterexample_pswlsv}. 
    
    We first analyze the simplified pebbling game for algorithm $\mathsf{SWL(VS)}$ or $\mathsf{SWL(SV)}$. Initially, Spoiler should choose the positions for pebbles $u,v$. We will show that it does not matter whether $u$ or $v$ is placed first. By symmetry, there are mainly two types of strategies which we separately investigate below. Other strategies are similar in analysis and we omit the proof for clarity.
    \begin{itemize}[topsep=0pt]
    \setlength{\itemsep}{0pt}
        \item Strategy 1: Spoiler places pebbles $u$ and $v$ on vertices 1 and 8, respectively. In this case, Duplicator's strategy is to ensure that the middle connected component is selected after pebbles $u,v$ are present, as shown in \cref{fig:counterexample_swlvs}(a). According to the game rule, Spoiler should then place pebble $w$ on some vertex adjacent to $v$, and clearly, he'd better place $w$ on vertex 4 (or vertex 5 by symmetry). Duplicator can respond appropriately without losing the game (shown in \cref{fig:counterexample_swlvs}(b)). When Spoiler swaps pebbles $v,w$ and leaves $w$ outside the graph, the chosen connected component will be merged (\cref{fig:counterexample_swlvs}(c)). It is easy to see that Spoiler can never win the game after any number of rounds.
        \item Strategy 2: Spoiler places pebbles $u$ and $v$ on vertices 4 and 5, respectively. By symmetry, suppose Duplicator chooses the connected component on the right (\cref{fig:counterexample_swlvs}(d)). Then Spoiler should place pebble $w$ on some vertex adjacent to $v$, and he'd clearly place $w$ on vertex 8. Duplicator must respond by choosing the rightmost triangle, resulting in \cref{fig:counterexample_swlvs}(e). When Spoiler swaps pebble $v,w$ and leaves $w$ outside the graph, the triangle component remains unchanged due to the presence of pebble $v$ (\cref{fig:counterexample_swlvs}(f)). In the next round, Spoiler should place pebble $w$ to further split the triangle, like \cref{fig:counterexample_swlvs}(g). However, he cannot win: when he swaps pebble $v,w$ and leaves $w$ outside the graph, all previous components merge into a whole as shown in \cref{fig:counterexample_swlvs}(h). Spoiler has no idea how to win.
    \end{itemize}

    We next turn to algorithm $\mathsf{PSWL(VS)}$. Initially, the game is the same as $\mathsf{SWL(VS)}$ until reaching the state of \cref{fig:counterexample_swlvs}(f). In the next round, Spoiler can resort to the game rule of $\agg^\mathsf{P}_\mathsf{vv}$ and move pebble $u$ to the position of pebble $v$ (\cref{fig:counterexample_swlvs}(i)). Now pebble $u$ becomes useful and Spoiler can easily win the remaining game, like \cref{fig:counterexample_swlvs}(j, k).
\end{proof}
\textbf{Insight into \cref{thm:counterexample_swlvs}}. The reason why $\mathsf{PSWL(VS)}$ is stronger lies in the fact that Spoiler can change the position of pebble $u$ throughout the game process. In contrast, for $\mathsf{SWL(VS)}$ and $\mathsf{SWL(SV)}$, pebble $u$ has to be kept fixed once it is placed on the graph, which severely limits the utility of the pebble $u$ in the subsequent game.

\begin{figure*}[!t]
    \small
    \centering
    \begin{tabular}{c|c@{}cc@{}c}
         & (a) & \includegraphics[width=0.29\textwidth]{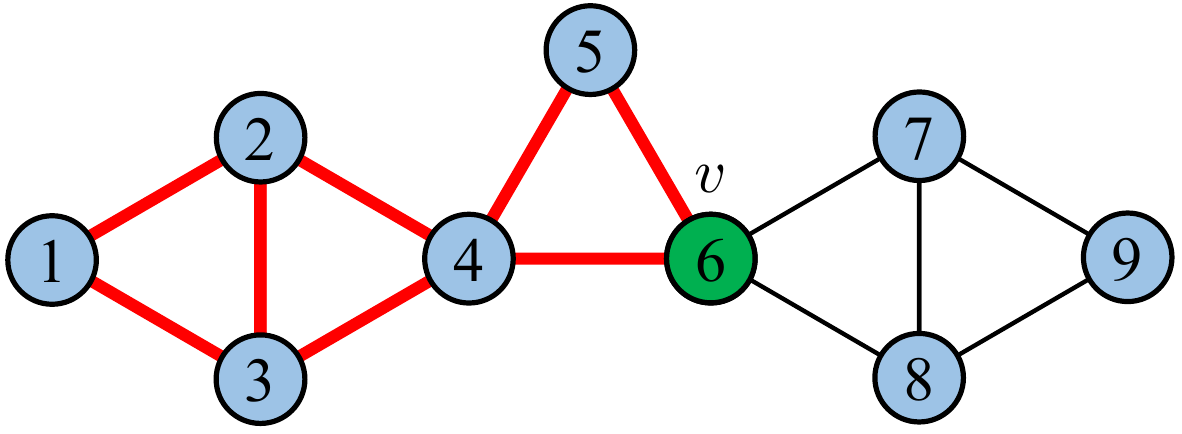} & (b) & \includegraphics[width=0.29\textwidth]{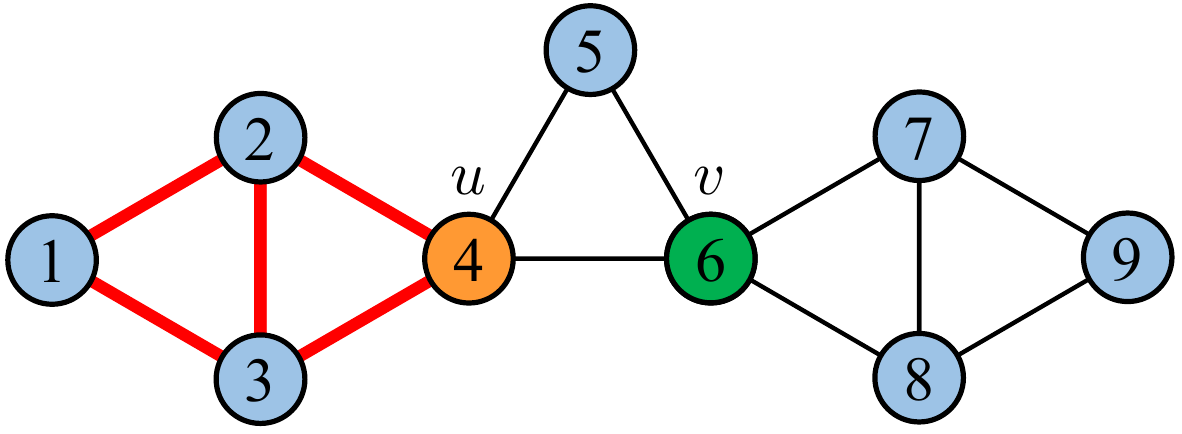}\\
         Base graph & (c) & \includegraphics[width=0.29\textwidth]{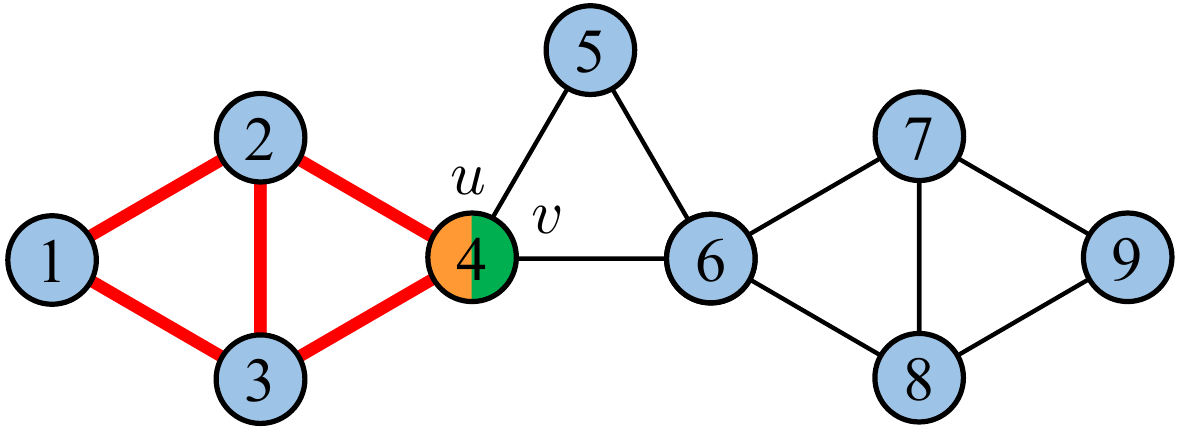} & (d) & \includegraphics[width=0.29\textwidth]{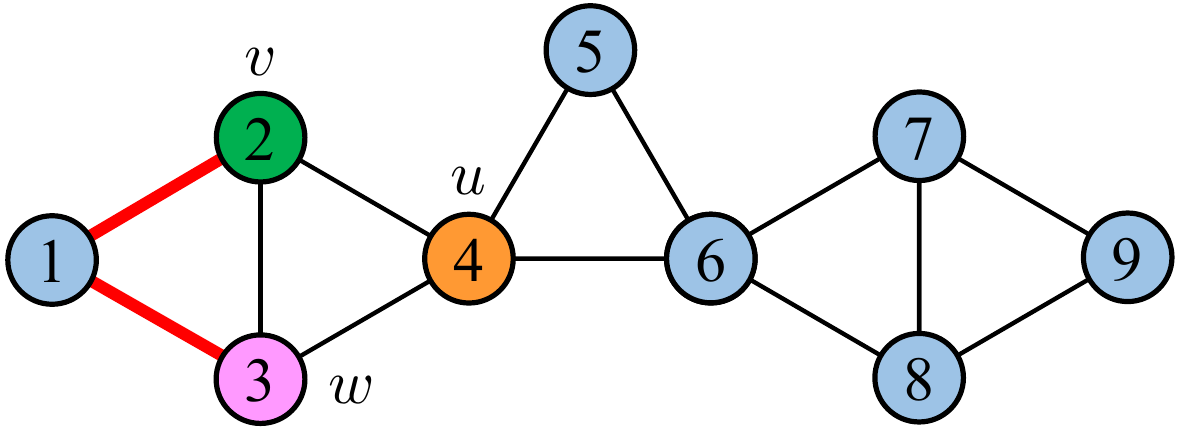}\\
         \includegraphics[width=0.29\textwidth]{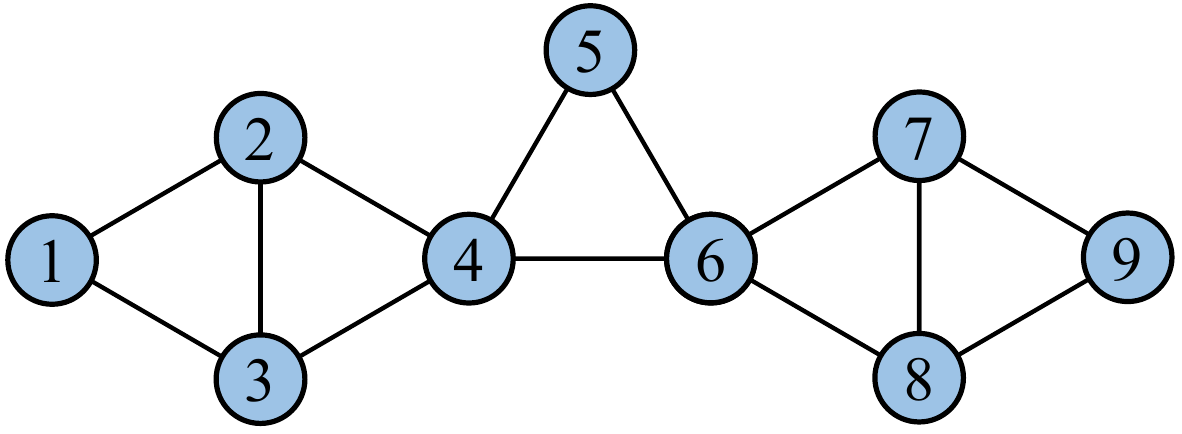} & (e) & \includegraphics[width=0.29\textwidth]{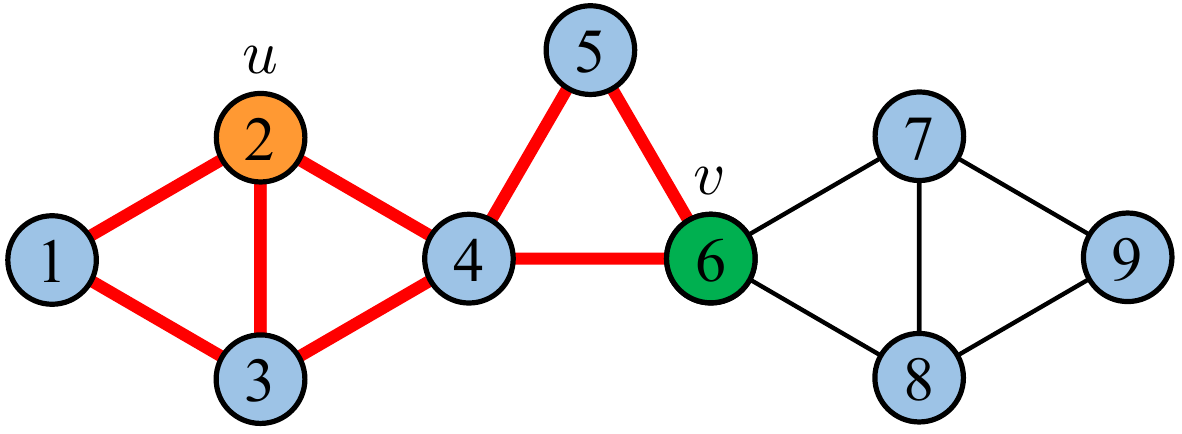} & (f) & \includegraphics[width=0.29\textwidth]{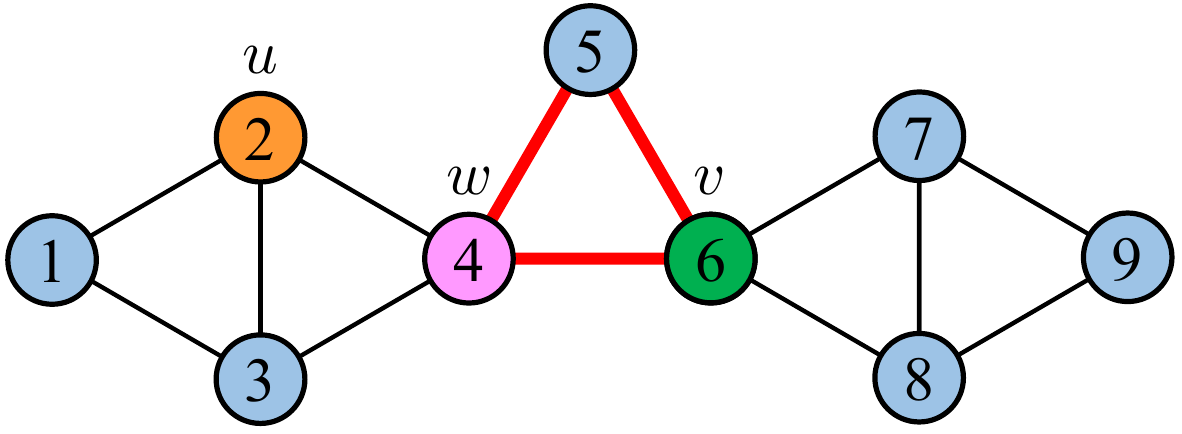}\\
        
        & (g) & \includegraphics[width=0.29\textwidth]{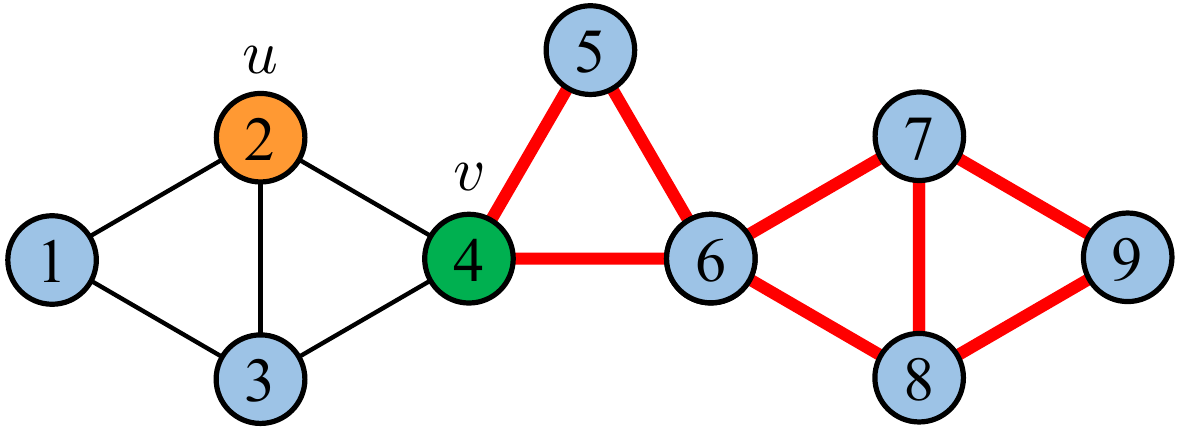} & (h) & \includegraphics[width=0.29\textwidth]{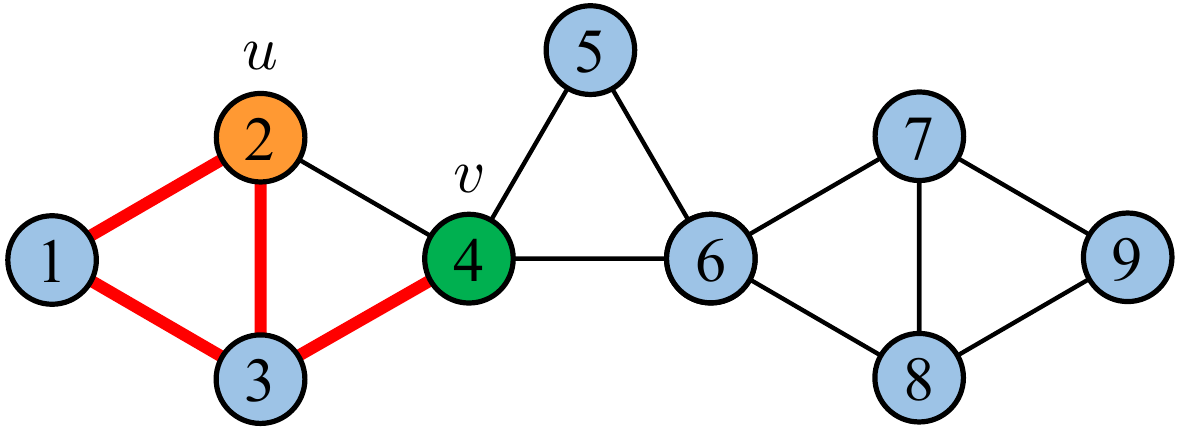}\\
        & (i) & \includegraphics[width=0.29\textwidth]{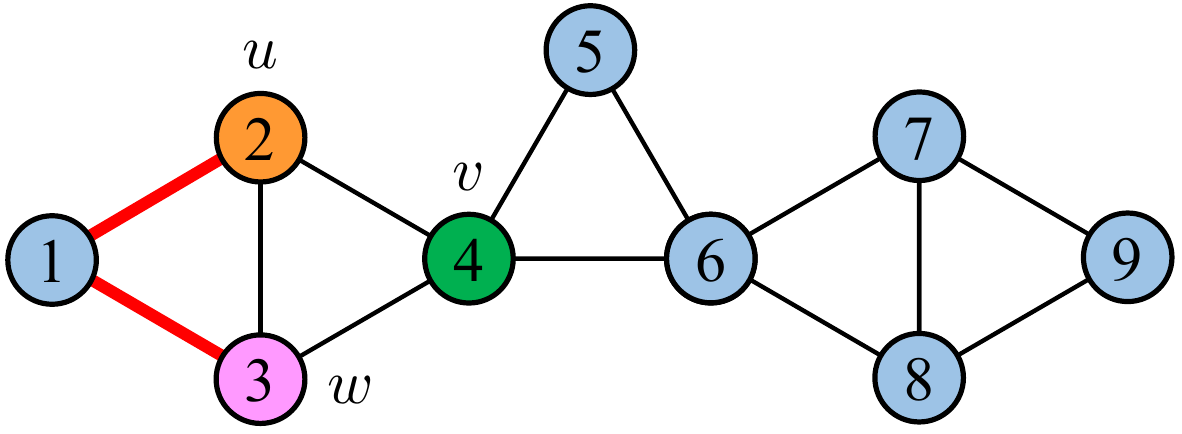} & (j) & \includegraphics[width=0.29\textwidth]{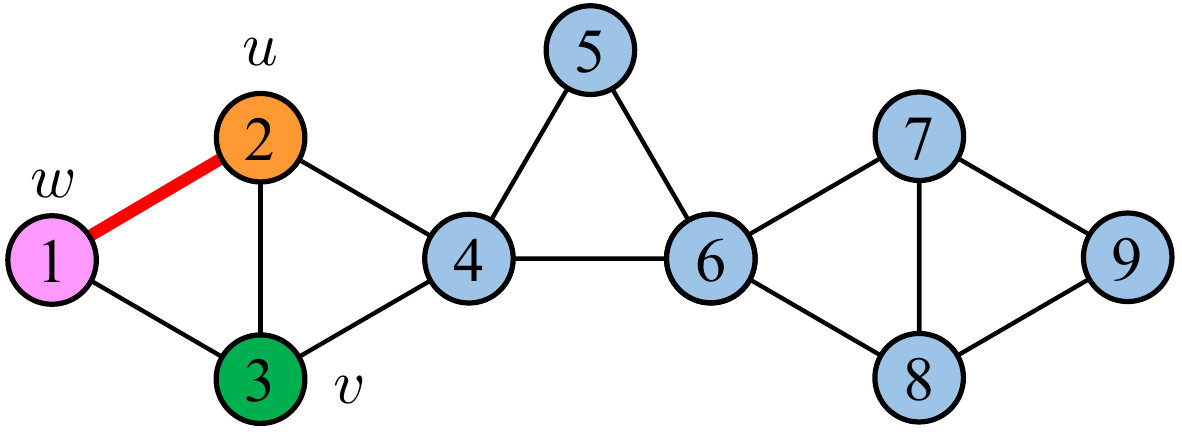}\\
    \end{tabular}
    \caption{Illustration of the proof of \cref{thm:counterexample_gswl}. When Duplicator follows her optimal strategy, the game process of $\mathsf{PSWL(SV)}$ may correspond to figures (a, b, c, d, ...) or figures (a, e, f, g, ...) depending on how to choose the initial position of pebble $u$. In both cases, Spoiler cannot win. In contrast, the game process of $\mathsf{GSWL}$ corresponds to figures (a, b, c, h, i, j) and Spoiler eventually wins in figure (j).}
    \label{fig:counterexample_gswl}
\end{figure*}

\begin{lemma}
\label{thm:counterexample_gswl}
    There exist two non-isomorphic graphs such that 
    \begin{itemize}[topsep=0pt]
    \setlength{\itemsep}{0pt}
        \item $\mathsf{GSWL}$ can distinguish them;
        \item $\mathsf{PSWL(SV)}$ cannot distinguish them.
    \end{itemize}
\end{lemma}
\begin{proof}
    The base graph in constructed in \cref{fig:counterexample_gswl}, which can be seen as a further extension of the counterexample in \cref{fig:counterexample_pswlsv}. 

    We first analyze the simplified pebbling game for algorithm $\mathsf{PSWL(SV)}$. Initially, Spoiler should place pebble $v$ on some vertex. We only consider the case of choosing vertex 6 (or equivalently, vertex 4), which is intuitively the best choice. Other choices can be similarly analyzed and we omit them for clarity. Since the presence of $v$ splits the graph into two connected components, Duplicator should select the larger one (\cref{fig:counterexample_gswl}(a)). Next, Spoiler should place pebble $u$ on some vertex.
    \begin{itemize}[topsep=0pt]
    \setlength{\itemsep}{0pt}
        \item We first consider the case when Spoiler places $u$ on vertex 4, which further splits the left connected component. In this case, Duplicator just selects the left diamond-shaped component (\cref{fig:counterexample_gswl}(b)). In the next round, Spoiler will play according to \cref{fig:counterexample_gswl}(c) by placing pebble $w$ on vertex 4 adjacent to $v$, swapping $v$ and $w$, and leaving pebble $w$ outside the graph. Duplicator just does nothing. The remaining game can be illustrated in \cref{fig:counterexample_gswl}(d), and the analysis is the same as the previous proof of \cref{thm:counterexample_pswlsv}. In short, Spoiler can never split the red connected component $\{\{1,2\},\{1,3\}\}$ shown in \cref{fig:counterexample_gswl}(d). Note that although Spoiler can additionally use the game rule of single-point aggregation $\agg^\mathsf{P}_\mathsf{vv}$, he should better not change the position of $u$: if he leaves pebble $u$ from vertex 4, the connected component will be merged. Therefore, Spoiler cannot win the game.
        \item Seeing why Spoiler cannot win in \cref{fig:counterexample_gswl}(d), let us restart from \cref{fig:counterexample_gswl}(a) with a different strategy. Suppose this time Spoiler places pebble $u$ on vertex 2 (shown in \cref{fig:counterexample_gswl}(e)). Since the red connected component remains unchanged, Duplicator does nothing. In the next round, Spoiler should place pebble $w$ on vertex 4 adjacent to $v$, which splits the red connected component in \cref{fig:counterexample_gswl}(e) into two parts. However, seeing the position of pebble $u$, this time Duplicator chooses a different strategy: she selects the upper triangle (\cref{fig:counterexample_gswl}(f)). When Spoiler swaps pebbles $v,w$ and leaves $w$ outside the graph, the upper triangle is merged into a larger connected component (see \cref{fig:counterexample_gswl}(g)). It is easy to see that Spoiler still cannot win the game after any number of rounds.
    \end{itemize}

    We next turn to algorithm $\mathsf{GSWL}$. Initially, the game is the same as $\mathsf{PSWL(SV)}$ until reaching the state of \cref{fig:counterexample_swlvs}(c). In the next round, Spoiler can choose a different way to play: according to the game rule of $\agg^\mathsf{G}_\mathsf{v}$, Spoiler can place pebble $w$ on vertex 2 and swap $w$ with $u$. Clearly, Duplicator has to respond by selecting the left connected component as shown in \cref{fig:counterexample_gswl}(h). Now the remaining game is easy for Spoiler. As illustrated in \cref{fig:counterexample_gswl}(i) and \cref{fig:counterexample_gswl}(j), Spoiler can finally win the game.
\end{proof}

\textbf{Insight into \cref{thm:counterexample_gswl}}. The proof of \cref{thm:counterexample_gswl} clearly shows why global aggregation is more powerful than the corresponding single-point aggregation (\cref{thm:aggregation}).

\begin{figure*}[!t]
    \small
    \centering
    \begin{tabular}{c|ccc}
         & \includegraphics[width=0.225\textwidth]{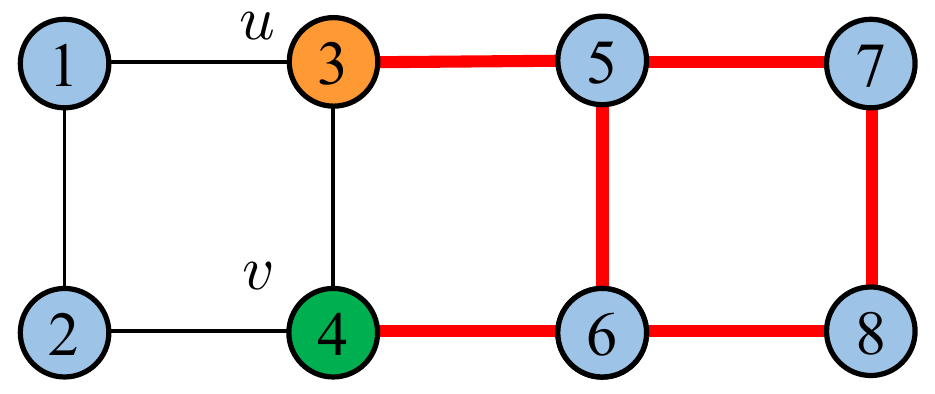} &\includegraphics[width=0.225\textwidth]{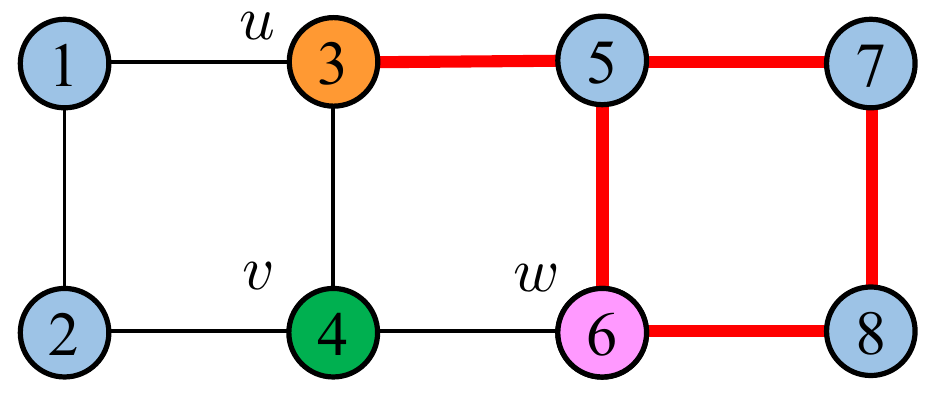} &\includegraphics[width=0.225\textwidth]{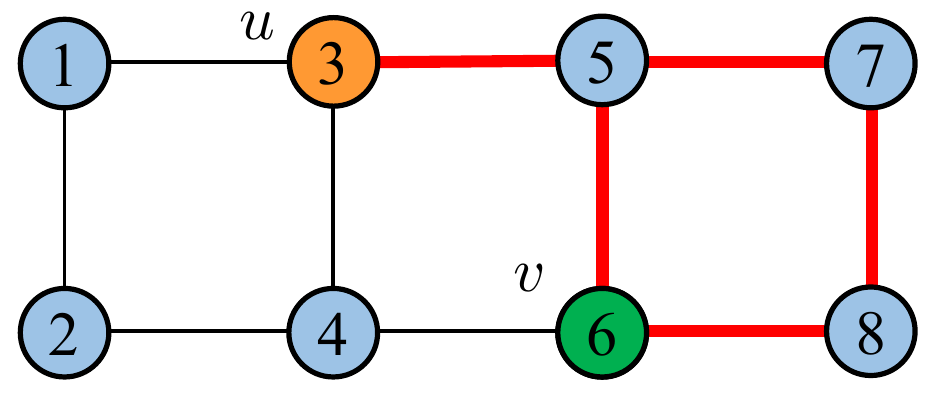}\\
         & (a) & (b) & (c)\\
         \includegraphics[width=0.225\textwidth]{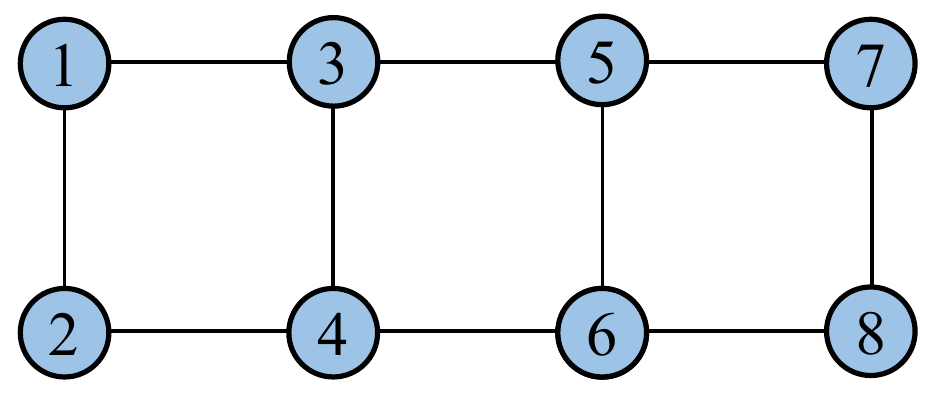} & \includegraphics[width=0.225\textwidth]{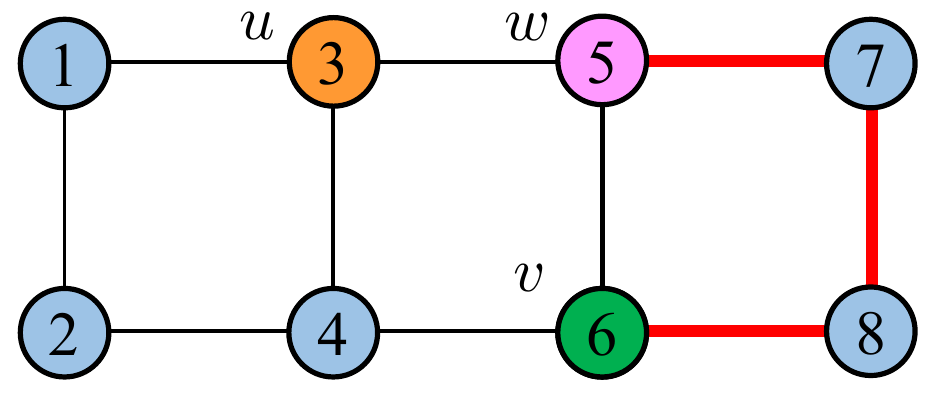} &\includegraphics[width=0.225\textwidth]{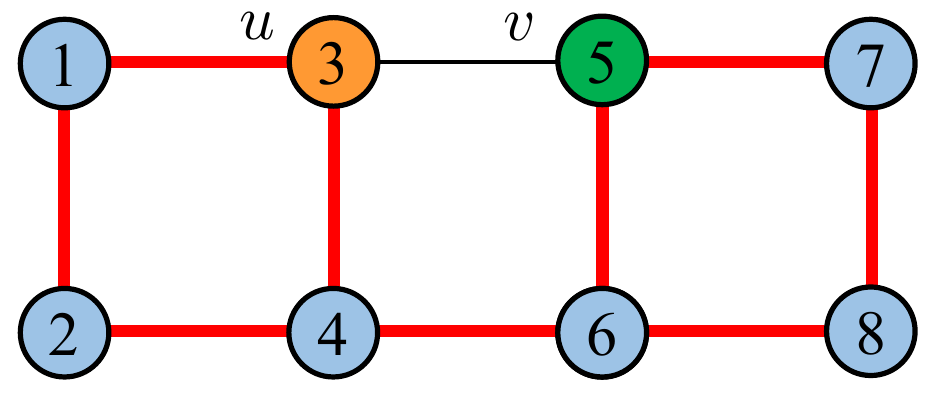} &\includegraphics[width=0.225\textwidth]{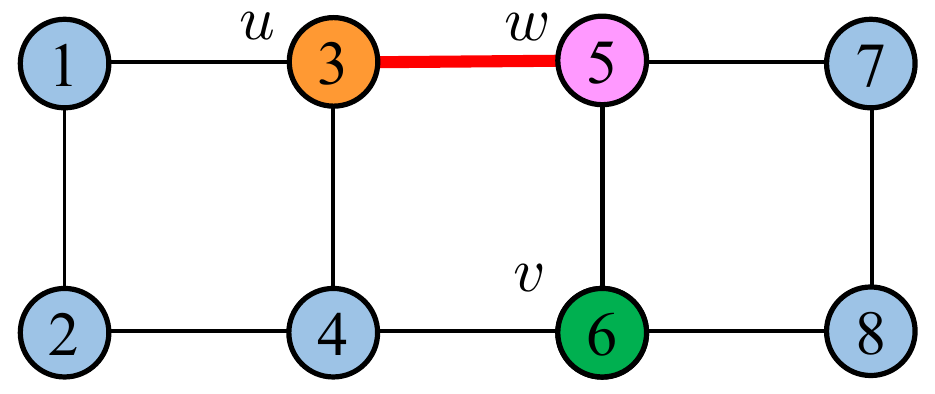}\\
         Base graph & (d) & (e) & (f)\\
         & \includegraphics[width=0.225\textwidth]{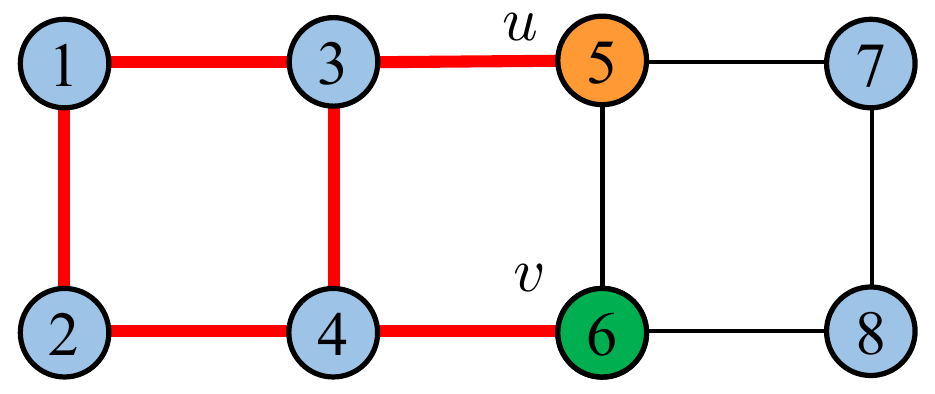} &\includegraphics[width=0.225\textwidth]{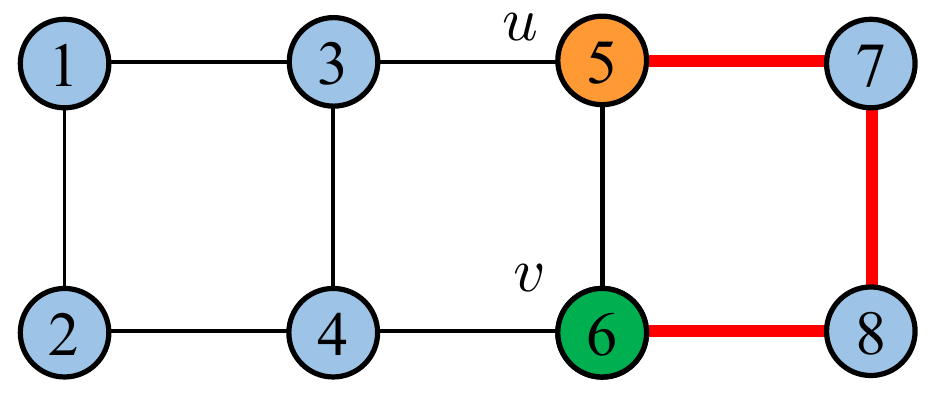} &\includegraphics[width=0.225\textwidth]{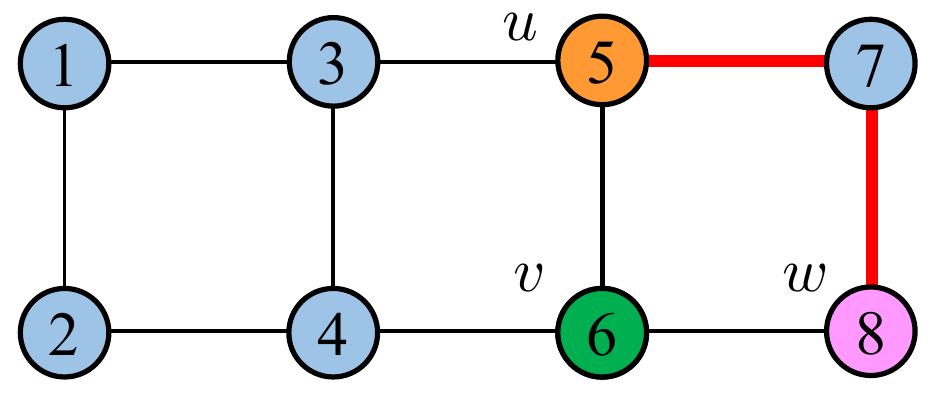}\\
         & (g) & (h) & (i)\\
    \end{tabular}
    \caption{Illustration of the proof of \cref{thm:counterexample_lfwl_sswl}. When Duplicator follows her optimal strategy, the game process of $\mathsf{GSWL}$ may correspond to figures (c, d, e, ...), (c, f, g, ...), (a, b, c, d, e, ...), or (a, b, c, f, g, ...), depending on Spoiler's strategy. In all cases, Spoiler cannot win. In contrast, the game process of $\mathsf{SSWL}$ corresponds to figures (a, b, c, d, h, i) and Spoiler eventually wins in figure (i). The game process of $\mathsf{LFWL(2)}$ is similar and Spoiler can also win.}
    \label{fig:counterexample_lfwl_sswl}
\end{figure*}
\begin{lemma}
\label{thm:counterexample_lfwl_sswl}
    There exist two non-isomorphic graphs such that 
    \begin{itemize}[topsep=0pt]
    \setlength{\itemsep}{0pt}
        \item $\mathsf{GSWL}$ cannot distinguish them;
        \item $\mathsf{SSWL}$ can distinguish them;
        \item $\mathsf{LFWL(2)}$ can distinguish them.
    \end{itemize}
\end{lemma}
\begin{proof}
    The base graph is constructed in \cref{fig:counterexample_gswl}, which is precisely the graph originally analyzed in \citet{furer2001weisfeiler} and is often called the F{\"u}rer grid graph \citep[Appendix D]{qian2022ordered}.

    We first analyze the simplified pebbling game for algorithm $\mathsf{GSWL}$. Depending on how Spoiler chooses the initial positions of pebbles $u$ and $v$, we main consider the two cases illustrated in \cref{fig:counterexample_lfwl_sswl}(a) and \cref{fig:counterexample_lfwl_sswl}(c) due to symmetry of the graph. Other cases are clearly not optimal. For the first case, Duplicator will select the larger connected component on the right (\cref{fig:counterexample_lfwl_sswl}(a)). In the next round, Spoiler may place pebble $w$ on vertex 6 adjacent to pebble $v$, and Duplicator updates her selected component accordingly (\cref{fig:counterexample_lfwl_sswl}(b)). Spoiler then swaps pebbles $v$ and $w$ and leaves $w$ outside the graph, returning to \cref{fig:counterexample_lfwl_sswl}(c). What follows is the central part of the proof. Spoiler should clearly place pebble $w$ on vertex 5 to further split the component selected by Duplicator, but he has two different ways to achieve this:
    \begin{itemize}[topsep=0pt]
    \setlength{\itemsep}{0pt}
        \item He plays according to the game rule of $\agg^\mathsf{L}_\mathsf{u}$. Duplicator knows this information and thus responds by selecting the connected component on the right (see \cref{fig:counterexample_lfwl_sswl}(d)). Then Spoiler should swap pebbles $w$ and $v$ and leave $w$ outside the graph. However, this will merge multiple component as shown in \cref{fig:counterexample_lfwl_sswl}(e). Clearly, Spoiler cannot win the subsequent game.
        \item A better choice would be to follow the game rule of $\agg^\mathsf{G}_\mathsf{v}$ because it can move pebble $u$ to vertex 5 after this round. However, Duplicator knows this information and thus responds differently: she selects the connected component of $\{\{3,5\}\}$ containing only one edge (see \cref{fig:counterexample_lfwl_sswl}(f)). Note that Duplicator does not lose the game, because for global aggregation Duplicator can freely choose a connected component of \emph{one edge} (while for local aggregation she cannot). Now, when Spoiler swaps pebbles $w$ and $u$ and leaves $w$ outside the graph, the component $\{\{3,5\}\}$ is merged into a larger component, as shown in \cref{fig:counterexample_lfwl_sswl}(g), which is equivalent to \cref{fig:counterexample_lfwl_sswl}(a) by symmetry. Again, Spoiler cannot win the subsequent game.
    \end{itemize}

    We next turn to algorithm $\mathsf{SSWL}$. Initially, the game is the same as $\mathsf{GSWL}$ until reaching the state of \cref{fig:counterexample_lfwl_sswl}(d). Now it comes to the major difference: Spoiler can play according to the game rule of $\agg^\mathsf{L}_\mathsf{v}$. This time Duplicator can no longer choose the connected component of $\{\{3,5\}\}$ since it is prohibited by the game rule. Therefore, her only choice is to select the rightmost component as shown in \cref{fig:counterexample_lfwl_sswl}(d). This then yields \cref{fig:counterexample_lfwl_sswl}(h) when Spoiler swaps pebbles $u,w$ and leaves $w$ outside the graph. The remaining game will be quite easy for Spoiler and it is easy to see that Spoiler can win when playing according to \cref{fig:counterexample_lfwl_sswl}(i).

    We finally turn to algorithm $\mathsf{LFWL(2)}$. Initially, the game is also the same as $\mathsf{GSWL}$ until reaching the state of \cref{fig:counterexample_lfwl_sswl}(d). Now it comes to the major difference: Duplicator does not know whether Spoiler will swap pebbles $u,w$ or swap pebbles $v,w$. Therefore, depending on Duplicator's response, Spoiler can adopt different strategies:
    \begin{itemize}[topsep=0pt]
    \setlength{\itemsep}{0pt}
        \item If Duplicator chooses the rightmost connected component, then Spoiler swaps pebbles $u,w$. This corresponds to \cref{fig:counterexample_lfwl_sswl}(h) when $w$ is left outside the graph, and we have proved that Spoiler can win.
        \item If Duplicator chooses the component containing only one edge $\{\{3,5\}\}$, then Spoiler swaps pebbles $v,w$. This corresponds to \cref{fig:counterexample_lfwl_sswl}(f) and Spoiler already wins after this round.
        \item Similarly, if Duplicator chooses the component containing only one edge $\{\{5,6\}\}$, then Spoiler swaps pebbles $u,w$ and wins after this round.
    \end{itemize}
    In all cases, Spoiler has a winning strategy.
\end{proof}

\textbf{Insight into \cref{thm:counterexample_lfwl_sswl}}. The proof of \cref{thm:counterexample_lfwl_sswl} shows why local aggregation is more powerful than global aggregation (\cref{thm:aggregation}). Importantly, in local aggregation there is an \emph{additional constraint} that Duplicator cannot choose the connected component containing the neighboring edge. The proof also reveals the power of FWL-type algorithms. Intuitively, in FWL-type algorithms Duplicator cannot ``see'' Spoiler's strategy before making choices, and thus Spoiler can gain an additional advantage by deliberately playing against Duplicator's strategy.

Based on the proof of \cref{thm:counterexample_lfwl_sswl}, curious readers may ask whether there is an expressivity relationship between $\mathsf{SSWL}$ and $\mathsf{LFWL(2)}$. However, below we will show that it is not the case: they are actually incomparable (due to \cref{thm:counterexample_slfwl,thm:counterexample_slfwl_2}).

\begin{figure*}[!t]
    \small
    \centering
    \begin{tabular}{c|cccc}
         \includegraphics[width=0.145\textwidth]{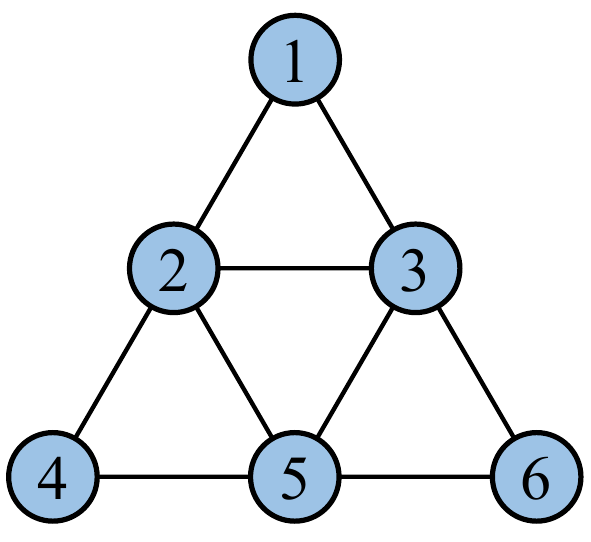} & \includegraphics[width=0.145\textwidth]{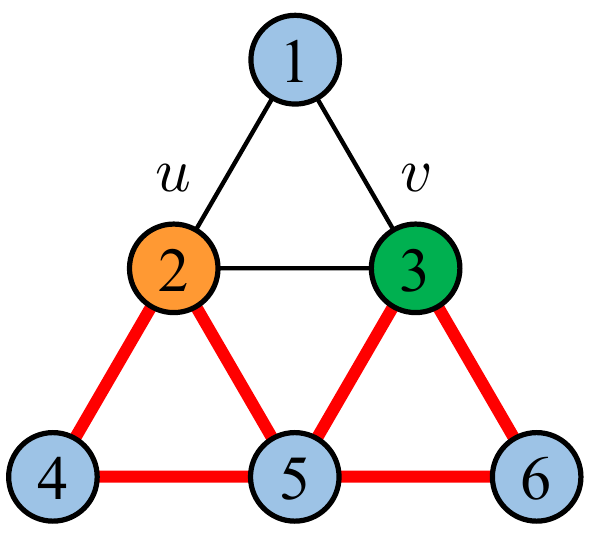} &\includegraphics[width=0.145\textwidth]{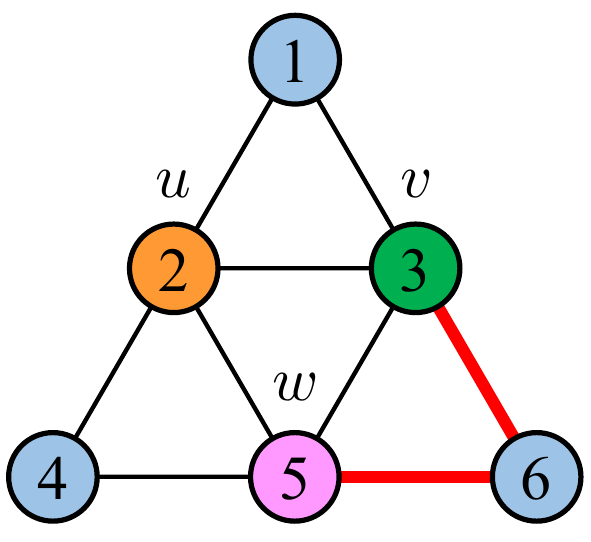} & \includegraphics[width=0.145\textwidth]{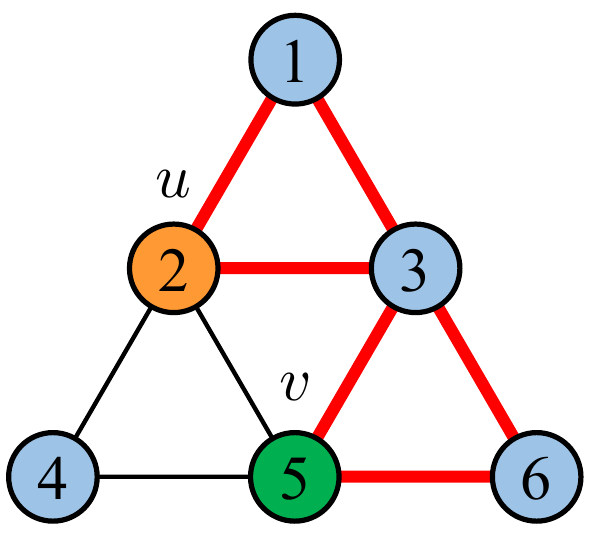} & \includegraphics[width=0.145\textwidth]{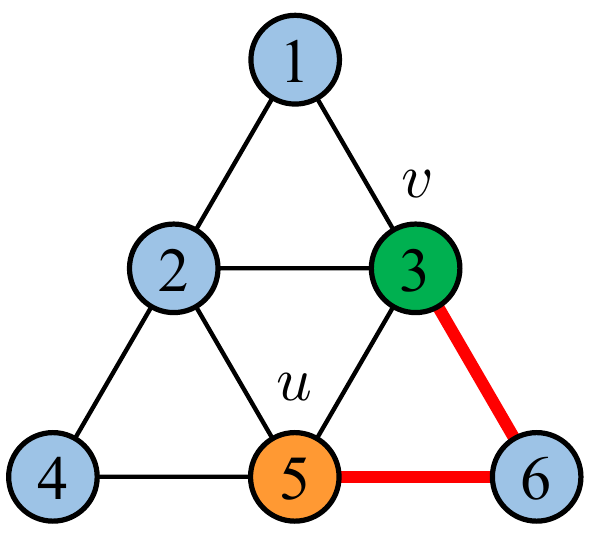}\\
         Base graph & (a) & (b) & (c) & (d)\\
    \end{tabular}
    \caption{Illustration of the proof of \cref{thm:counterexample_slfwl}. When Duplicator follows her optimal strategy, the game process of $\mathsf{SSWL}$ may correspond to figures (a, b, c, ...), and Spoiler cannot win. In contrast, the game process of $\mathsf{LFWL(2)}$ or $\mathsf{SLFWL(2)}$ corresponds to figures (a, b, d) and Spoiler can eventually win.}
    \label{fig:counterexample_slfwl}
\end{figure*}
\begin{lemma}
\label{thm:counterexample_slfwl}
    There exist two non-isomorphic graphs such that 
    \begin{itemize}[topsep=0pt]
    \setlength{\itemsep}{0pt}
        \item $\mathsf{SSWL}$ cannot distinguish them;
        \item $\mathsf{SLFWL(2)}$ can distinguish them;
        \item $\mathsf{LFWL(2)}$ can distinguish them.
    \end{itemize}
\end{lemma}
\begin{proof}
    The base graph is constructed in \cref{fig:counterexample_slfwl}.

    We first analyze the simplified pebbling game for algorithm $\mathsf{SSWL}$. Initially, Spoiler should place pebbles $u$ and $v$ on vertices of the graph. Due to symmetry, we can assume that Spoiler places $u$ on vertex 2 and places $v$ on vertex 3 (other nonequivalent cases are clearly not optimal). Duplicator then selects the bottom connected component split by pebbles $u,v$ (see \cref{fig:counterexample_slfwl}(a)). In the next round, Spoiler should place $w$ on vertex 5 to further split the connected component. By definition of $\mathsf{SSWL}$, he can play according to the rule of either $\agg^\mathsf{L}_\mathsf{u}$ or $\agg^\mathsf{L}_\mathsf{v}$. By symmetry, it suffices to analyze the case of $\agg^\mathsf{L}_\mathsf{u}$. Since Duplicator knows the information that Spoiler plays according to $\agg^\mathsf{L}_\mathsf{u}$, she selects the connected component in the lower right corner (see \cref{fig:counterexample_slfwl}(b)). Then, Spoiler should swap pebbles $v,w$ and leave $w$ outside the graph, which leads to the merging of multiple connected components. The resulting game, as illustrated in \cref{fig:counterexample_slfwl}(c), is equivalent to \cref{fig:counterexample_slfwl}(a) by symmetry. Therefore, Spoiler cannot win the game.

    We next analyze the simplified pebbling game for algorithm $\mathsf{LFWL(2)}$ or $\mathsf{SLFWL(2)}$. Initially, the game is the same as $\mathsf{SSWL}$ until reaching the state of \cref{fig:counterexample_slfwl}(b). Now it comes to the major difference: Duplicator does not know whether Spoiler will swap pebbles $u,w$ or swap pebbles $v,w$. Therefore, Duplicator can only choose either the bottom left component or the bottom right one \emph{at random}, which is equivalent by symmetry. Spoiler can then play against Duplicator's strategy and swap the pebbles so that after leaving pebble $w$ outside the graph, the connected component selected by Duplicator is not merged, as shown in \cref{fig:counterexample_slfwl}(d). Clearly, Spoiler can win the remaining game.
\end{proof}

\textbf{Insight into \cref{thm:counterexample_slfwl}}. \cref{thm:counterexample_slfwl} shows the inherent advantage of FWL-type algorithms compared with SWL algorithms, answering an open problem raised in \citet{frasca2022understanding}.

\begin{figure*}[!t]
    \small
    \centering
    \begin{tabular}{c|cccc}
         \includegraphics[width=0.15\textwidth]{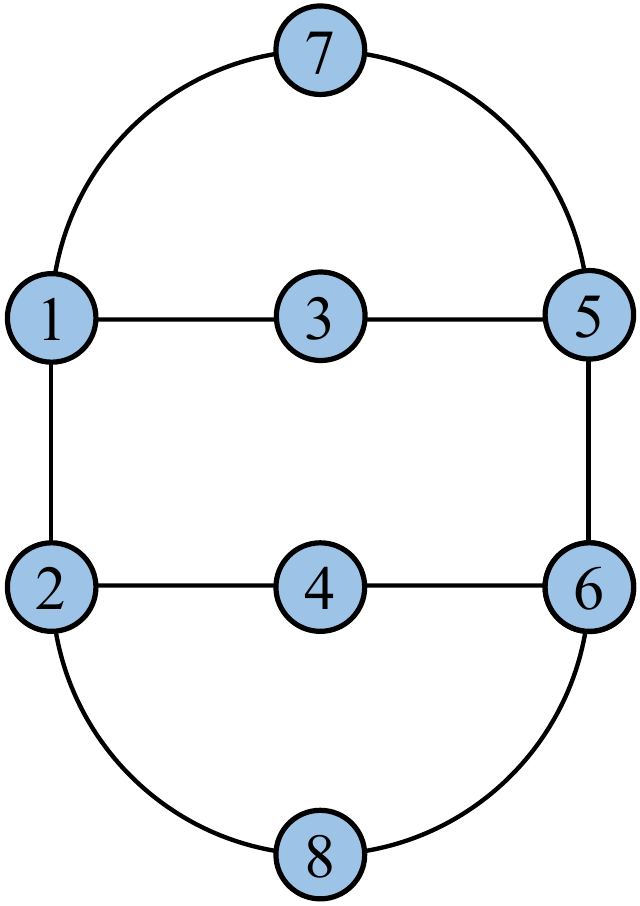} & \includegraphics[width=0.15\textwidth]{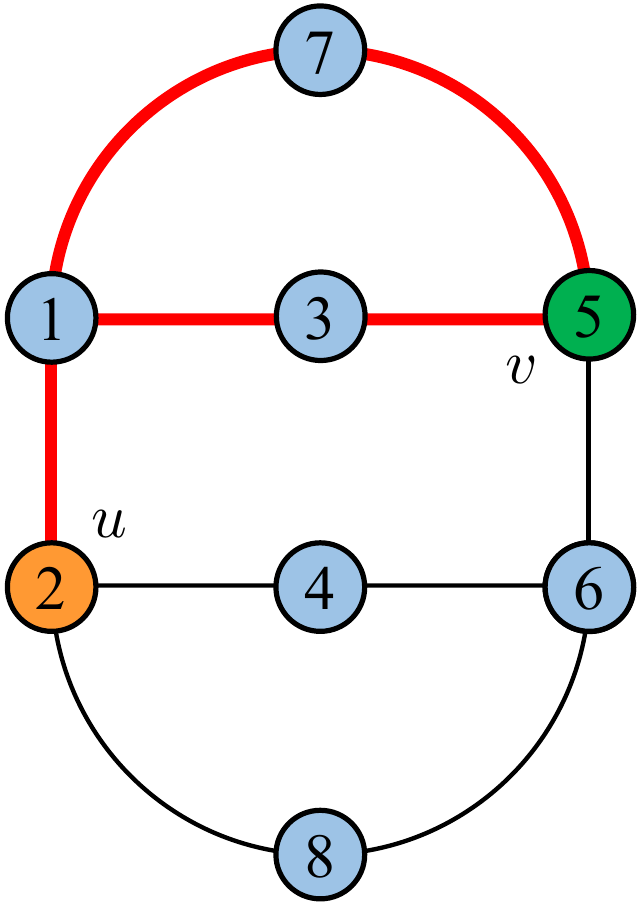} & \includegraphics[width=0.15\textwidth]{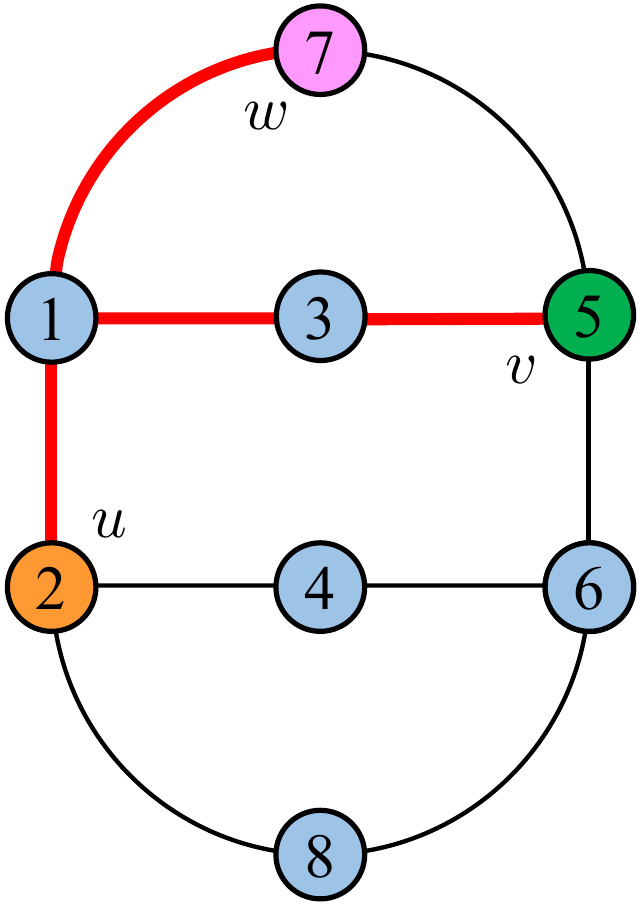} & \includegraphics[width=0.15\textwidth]{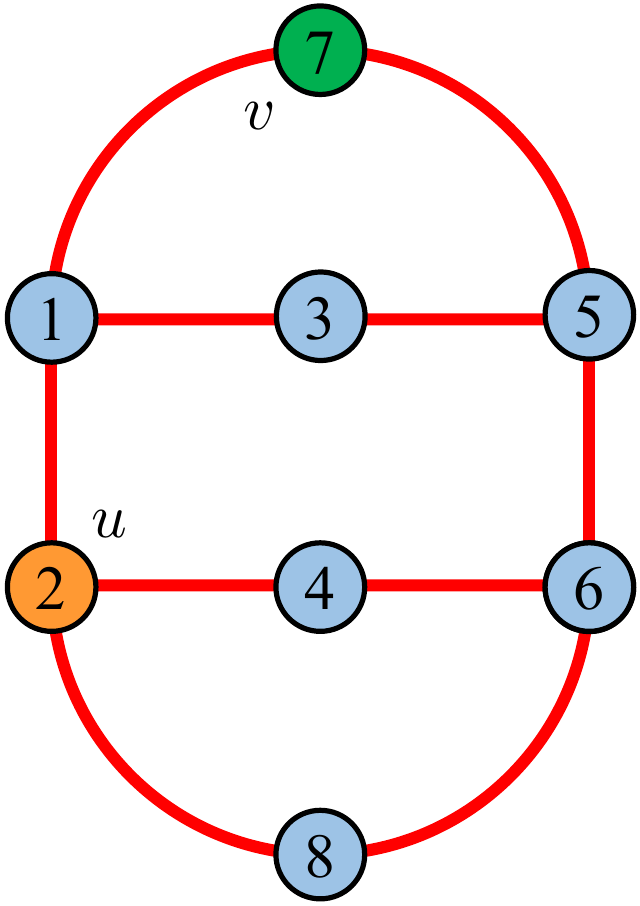} & \includegraphics[width=0.15\textwidth]{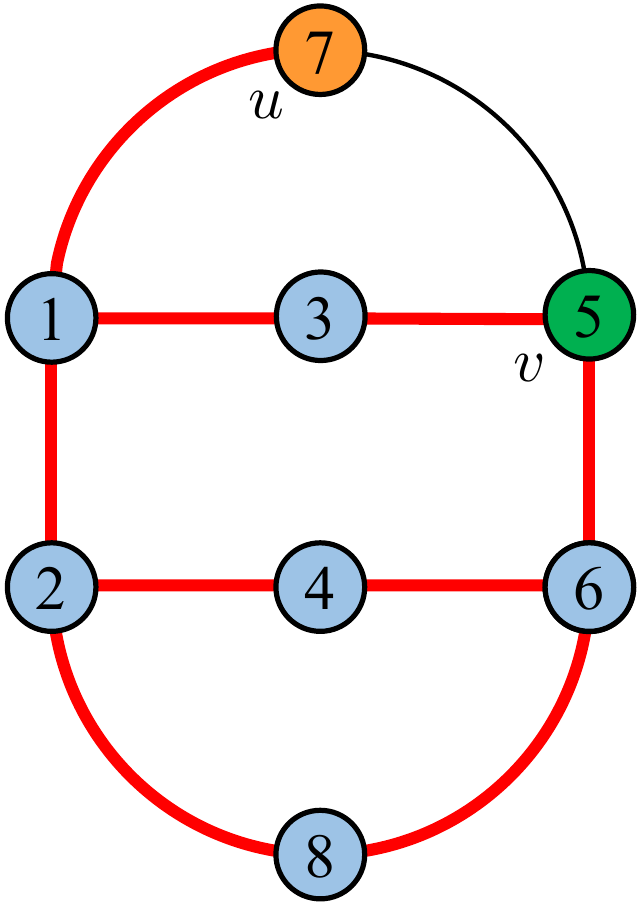}\\
         Base graph & (a) & (b) & (c) & (d) \\
         & \includegraphics[width=0.15\textwidth]{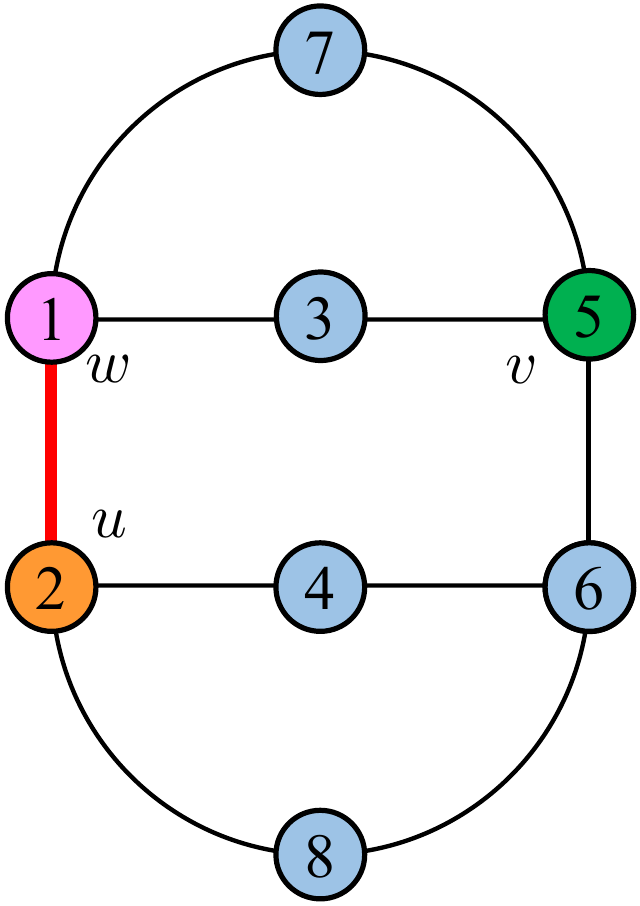} & \includegraphics[width=0.15\textwidth]{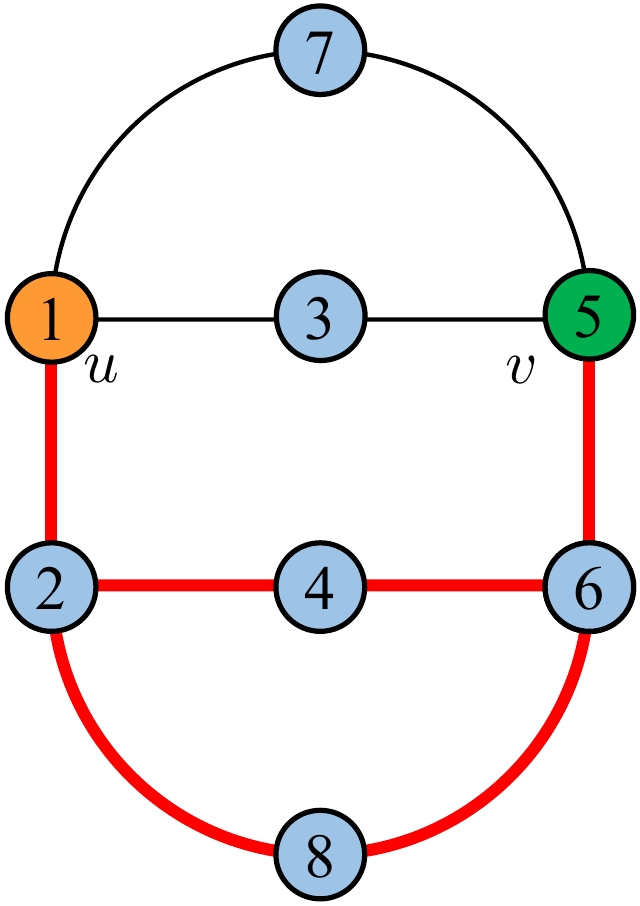} & \includegraphics[width=0.15\textwidth]{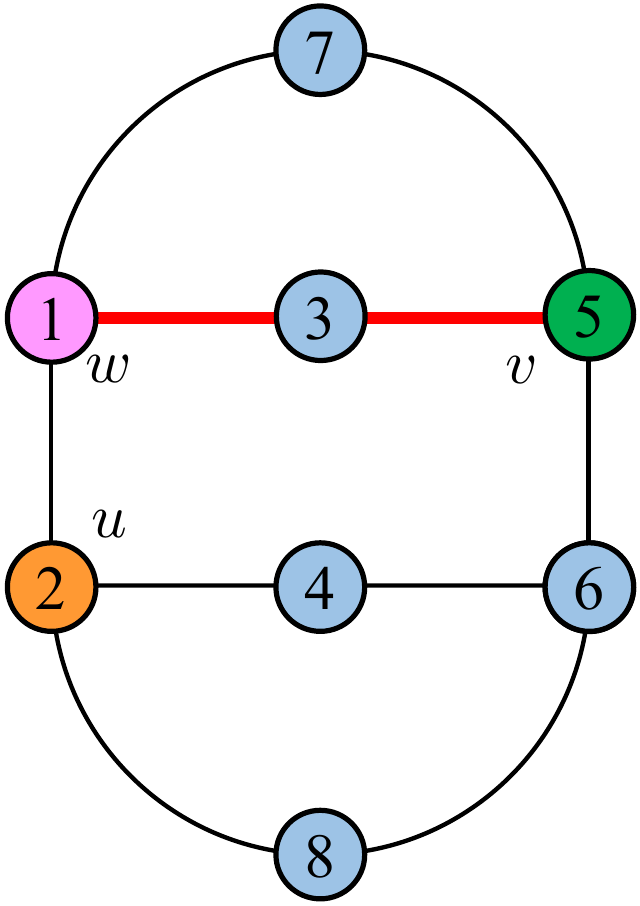} & \includegraphics[width=0.15\textwidth]{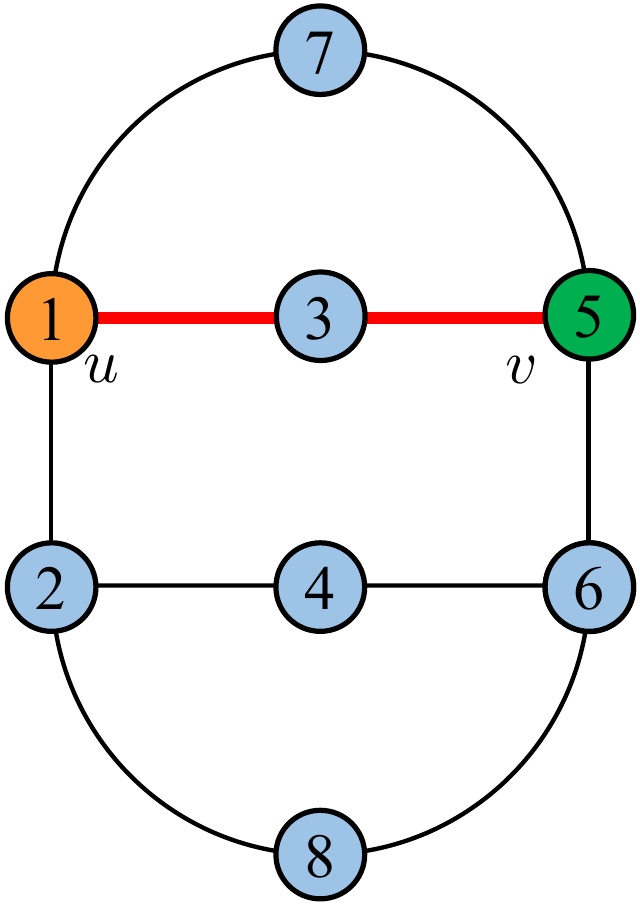}\\
         & (e) & (f) & (g) & (h)
    \end{tabular}
    \caption{Illustration of the proof of \cref{thm:counterexample_slfwl_2}. When Duplicator follows her optimal strategy, the game process of $\mathsf{LFWL(2)}$ may correspond to figures (a, b, c, ...) or figures (a, b, d, ...) depending on Spoiler's strategy. In both cases, Spoiler cannot win. Similarly, the game process of $\mathsf{GSWL}$ may correspond to figures (a, b, c, ...) or figures (a, e, f, ...), and Spoiler still cannot win. In contrast, the game process of $\mathsf{SSWL}$ or $\mathsf{SLFWL(2)}$ corresponds to figures (a, g, h) and Spoiler can eventually win.}
    \label{fig:counterexample_slfwl_2}
\end{figure*}
\begin{lemma}
\label{thm:counterexample_slfwl_2}
    There exist two non-isomorphic graphs such that 
    \begin{itemize}[topsep=0pt]
    \setlength{\itemsep}{0pt}
        \item $\mathsf{LFWL(2)}$ cannot distinguish them;
        \item $\mathsf{SLFWL(2)}$ can distinguish them;
        \item $\mathsf{GSWL}$ cannot distinguish them;
        \item $\mathsf{SSWL}$ can distinguish them.
    \end{itemize}
\end{lemma}
\begin{proof}
    The base graph is constructed in \cref{fig:counterexample_slfwl_2}.

    We first analyze the simplified pebbling game for algorithm $\mathsf{LFWL(2)}$. Initially, Spoiler should place pebbles $u$ and $v$ on vertices of the graph. Due to symmetry, there are mainly two cases we need to consider, as shown in \cref{fig:counterexample_slfwl_2}(a) and \cref{fig:counterexample_slfwl_2}(f), respectively. Here, we only consider the case of \cref{fig:counterexample_slfwl_2}(a), where pebble $u$ is placed on vertex 2 and pebble $v$ is placed on vertex 5; the other case is similar to analyze. Duplicator will respond by selecting the top connected component (\cref{fig:counterexample_slfwl_2}(a)). In the next round, Spoiler should place pebble $w$ adjacent to pebble $v$. Clearly, he should place it on vertex 3 or 7, which is equivalent. Assume that he places $w$ on vertex 7. Duplicator can easily respond by choosing the larger component (\cref{fig:counterexample_slfwl_2}(b)). According to the game rule of $\mathsf{LFWL(2)}$, he can either swap pebbles $v,w$ or swap pebbles $u,w$, and then leaves $w$ outside the graph. As shown in \cref{fig:counterexample_slfwl_2}(c) and \cref{fig:counterexample_slfwl_2}(d), in both cases the connected components get merged after leaving pebble $w$. Clearly, Spoiler cannot win.

    We next turn to algorithm $\mathsf{GSWL}$. Initially, the game is similar (\cref{fig:counterexample_slfwl_2}(a)). Now Spoiler has the additional choice to place pebble $w$ on vertex 1 according to the game rule of $\agg^\mathsf{G}_\mathsf{v}$. In this case, Duplicator responds by choosing the connected component $\{\{1,2\}\}$ which contains only one edge (see \cref{fig:counterexample_slfwl_2}(e)). Note that Duplicator does not lose the game. Then Spoiler will swap pebbles $u$ and $w$ and leaves $w$ outsides the graph, yielding \cref{fig:counterexample_slfwl_2}(f). Spoiler cannot win the game either.

    We finally turn to algorithm $\mathsf{SSWL}$. This time Spoiler can place pebble $w$ adjacent to pebble $u$ according to the game rule of $\agg^\mathsf{L}_\mathsf{v}$, and Duplicator can only choose either the connected component $\{\{1,3\},\{3,5\}\}$ or $\{\{1,7\},\{7,5\}\}$ (\cref{fig:counterexample_slfwl_2}(g)). Note that She cannot choose the component $\{\{1,2\}\}$ according to the game rule. After swapping pebbles $u,w$ and leaving $w$ outside the graph (\cref{fig:counterexample_slfwl_2}(h)), the remaining game is quite easy for Spoiler and he can eventually win.
\end{proof}

\textbf{Insight into \cref{thm:counterexample_slfwl_2}}. \cref{thm:counterexample_slfwl_2} shows the inherent advantages of ``symmetrized'' WL algorithms compared with WL algorithms that only aggregate local information of \emph{one} vertex.

\begin{figure*}[!t]
    \small
    \centering
    \begin{tabular}{c|ccccc}
         \includegraphics[width=0.145\textwidth]{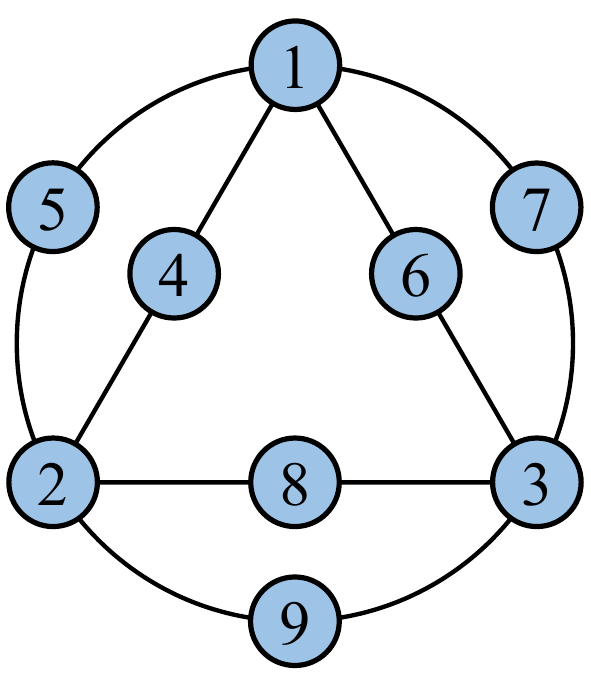} & \includegraphics[width=0.145\textwidth]{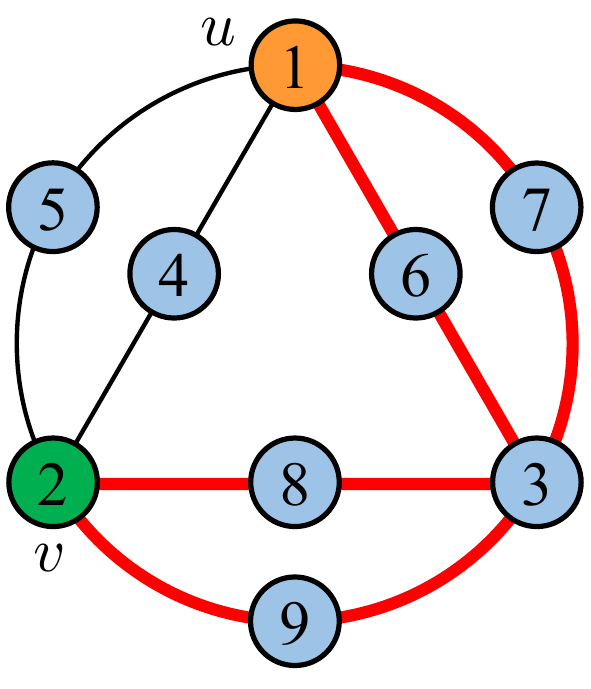} & \includegraphics[width=0.145\textwidth]{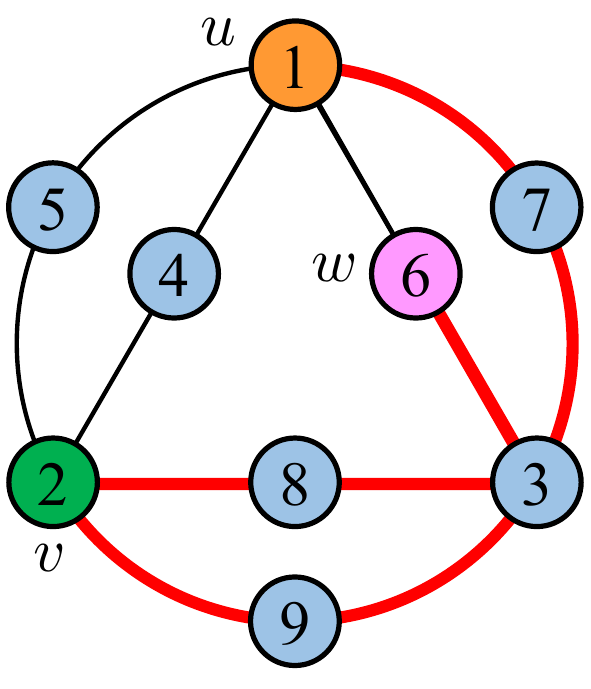} & \includegraphics[width=0.145\textwidth]{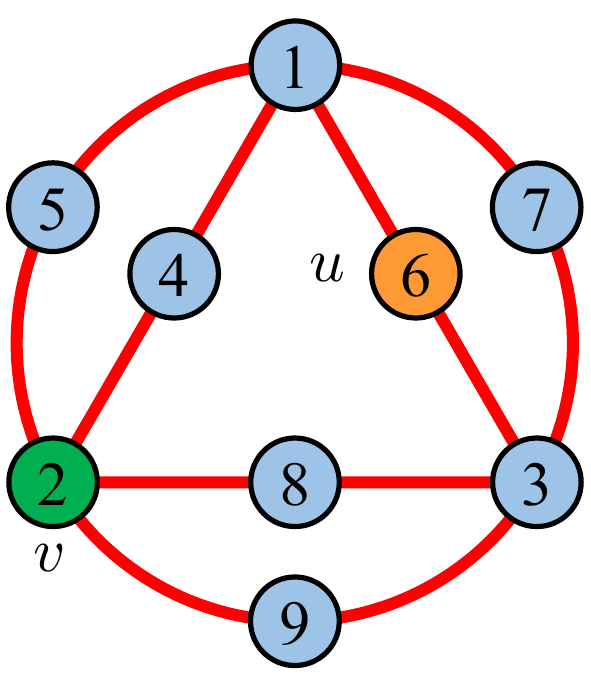} & \includegraphics[width=0.145\textwidth]{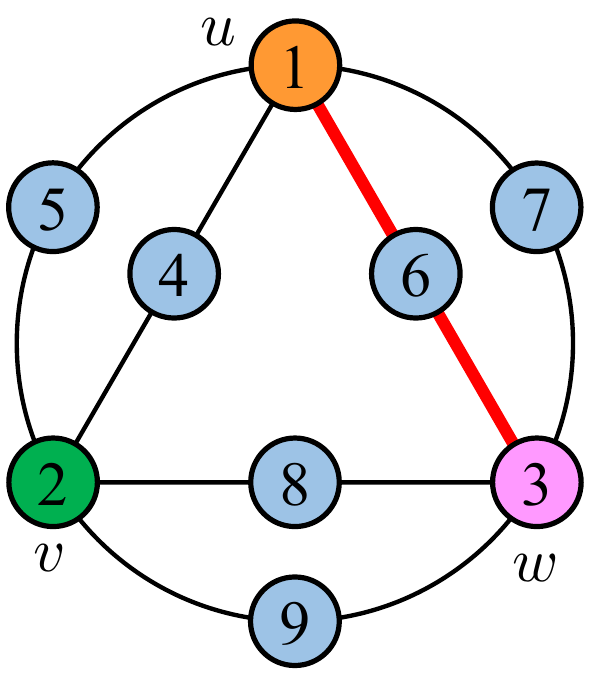} & \includegraphics[width=0.145\textwidth]{figure/counterexample_fwl_fwl2.pdf}\\
         Base graph & (a) & (b) & (c) & (d) & (e) \\
    \end{tabular}
    \caption{Illustration of the proof of \cref{thm:counterexample_fwl}. When Duplicator follows her optimal strategy, the game process of $\mathsf{SLFWL(2)}$ may correspond to figures (a, b, c, ...), and Spoiler cannot win. In contrast, the game process of $\mathsf{FWL(2)}$ corresponds to figures (a, d, e) and Spoiler can eventually win.}
    \label{fig:counterexample_fwl}
\end{figure*}
\begin{lemma}
\label{thm:counterexample_fwl}
    There exist two non-isomorphic graphs such that 
    \begin{itemize}[topsep=0pt]
    \setlength{\itemsep}{0pt}
        \item $\mathsf{SLFWL(2)}$ cannot distinguish them;
        \item $\mathsf{FWL(2)}$ can distinguish them.
    \end{itemize}
\end{lemma}
\begin{proof}
    The base graph in constructed in \cref{fig:counterexample_fwl}.
    
    We first analyze the simplified pebbling game for algorithm $\mathsf{SLFWL(2)}$. Initially, Spoiler should place pebbles $u$ and $v$ on vertices of the graph. Due to symmetry, we can assume that he places pebble $u$ on vertex 1 and places pebble $v$ on vertex 2. Other nonequivalent choices are clearly not optimal. Duplicator then responds by choosing the largest connect component as shown in \cref{fig:counterexample_fwl}(a). In the next round, Spoiler can place pebble $w$ on any vertex in $\gN_F^1(1)\cup\gN_F^1(2)$, namely, any vertex except vertex 3. Due to symmetry, we can assume that he places $w$ on vertex 6. Duplicator can easily respond according to \cref{fig:counterexample_fwl}(b). No matter how Spoiler swaps pebbles, as long as pebble $w$ is left outside the graph, multiple connected components then merge as shown in \cref{fig:counterexample_fwl}(c) and Spoiler has no idea how to win.

    We next turn to algorithm $\mathsf{FWL(2)}$. Starting from \cref{fig:counterexample_fwl}(a), this time Spoiler can place pebble $w$ on vertex 3. Then Duplicator should choose an odd number of connected components from the four components: $\{\{1,6\},\{6,3\}\}$, $\{\{1,7\},\{7,3\}\}$, $\{\{2,8\},\{8,3\}\}$, and $\{\{2,9\},\{9,3\}\}$. Regardless of Duplicator's choice, after swapping either pebbles $u,w$ or pebbles $v,w$, the chosen component are always surrounded by pebbles $u$ and $v$ (\cref{fig:counterexample_fwl}(e)). Therefore, Spoiler can easily win the remaining game.
\end{proof}

\textbf{Insight into \cref{thm:counterexample_fwl}}. \cref{thm:counterexample_fwl} shows there is an inherent gap between 2-FWL and all $O(nm)$-complexity algorithms considered in this paper. It can also be used to settle the open problem raised in \citet{frasca2022understanding}.

\begin{figure*}
    \small
    \centering
    \begin{tabular}{c|cc}
         & \includegraphics[width=0.3\textwidth]{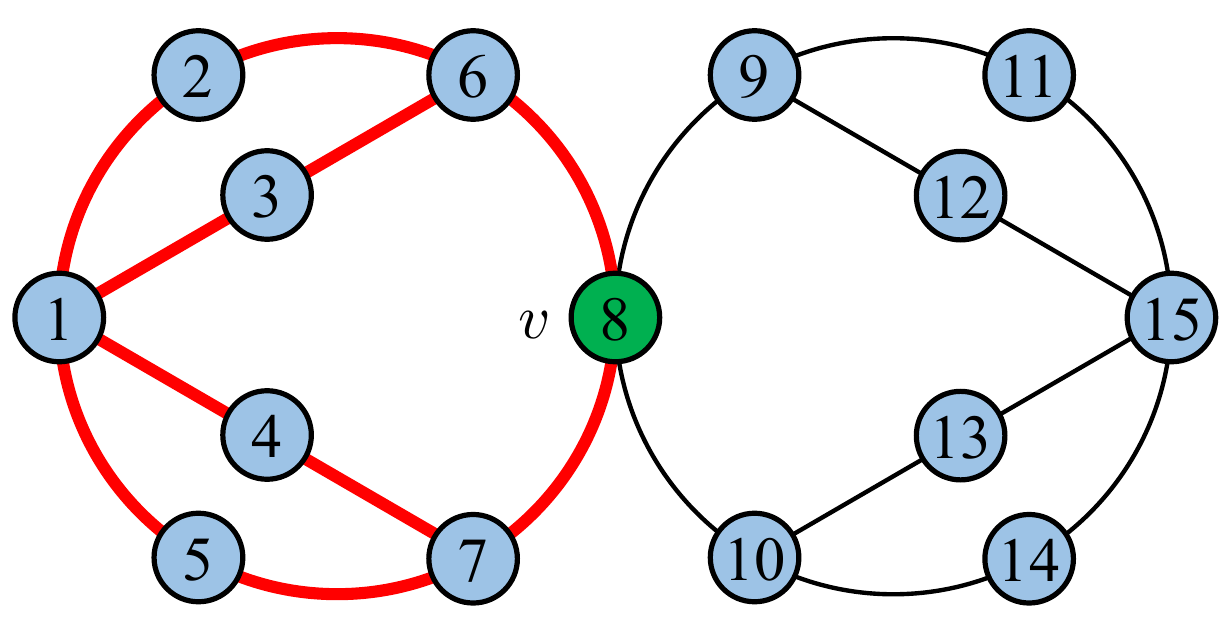} & \includegraphics[width=0.3\textwidth]{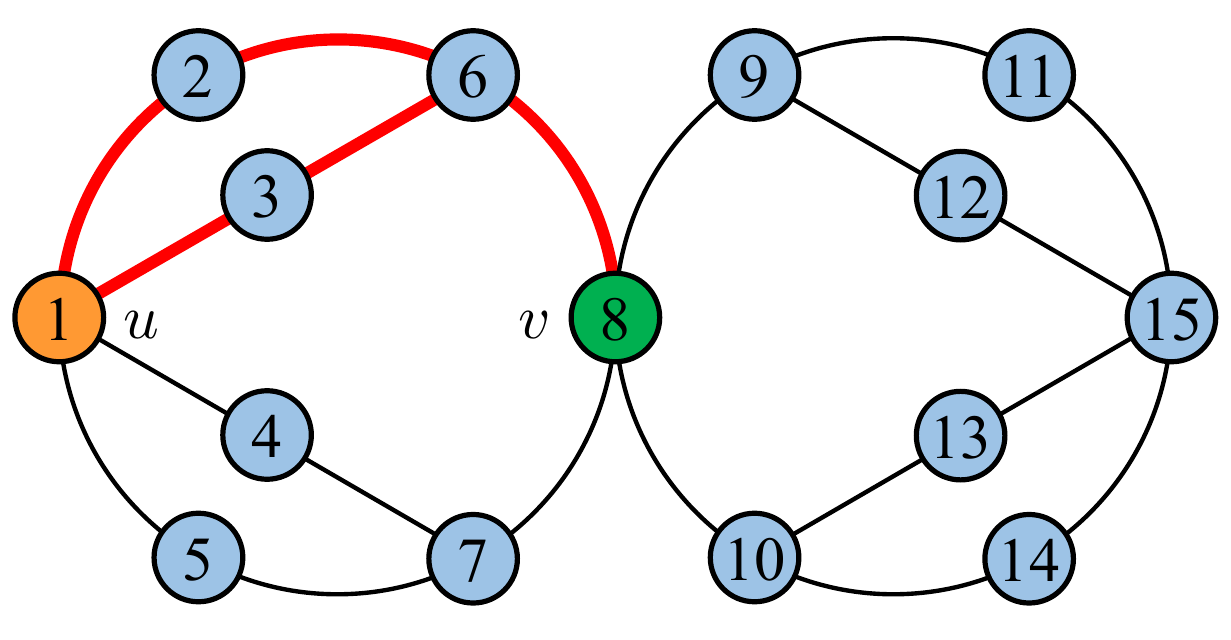}\\
         & (a) & (b) \\
         & \includegraphics[width=0.3\textwidth]{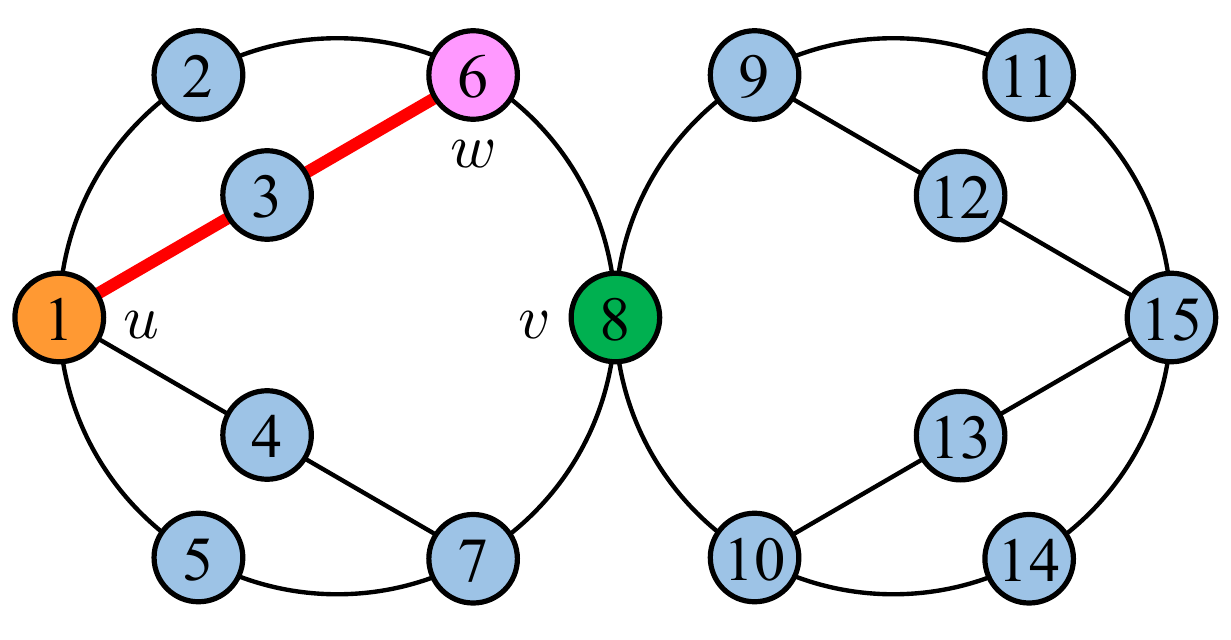} & \includegraphics[width=0.3\textwidth]{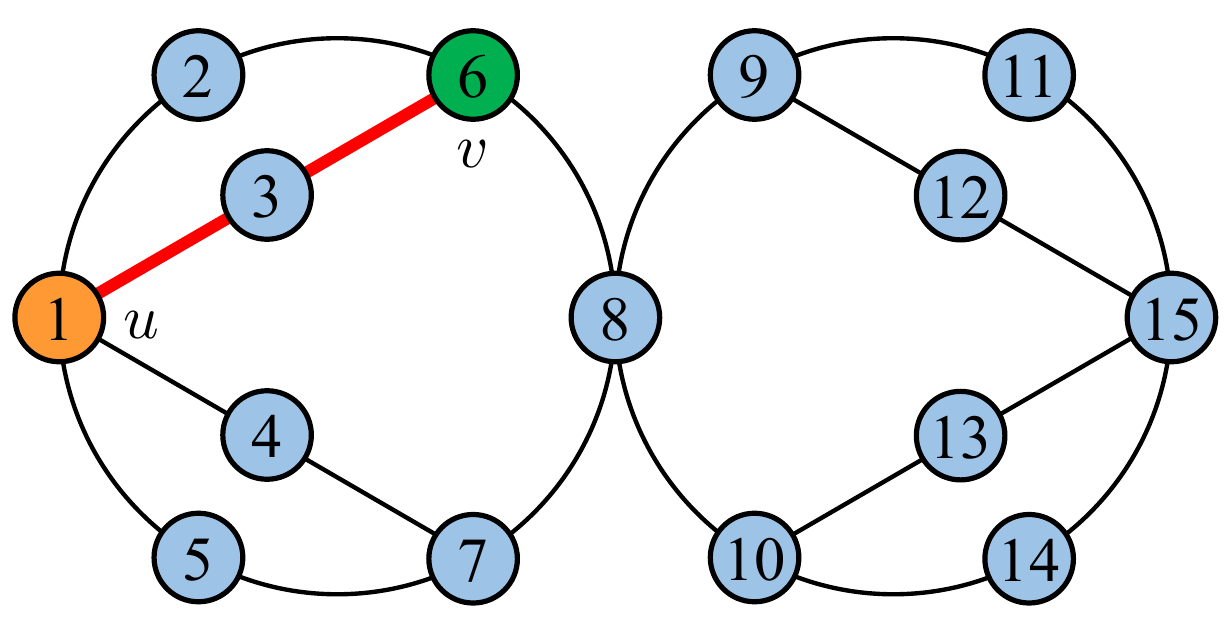}\\
         & (c) & (d)\\
         & \includegraphics[width=0.3\textwidth]{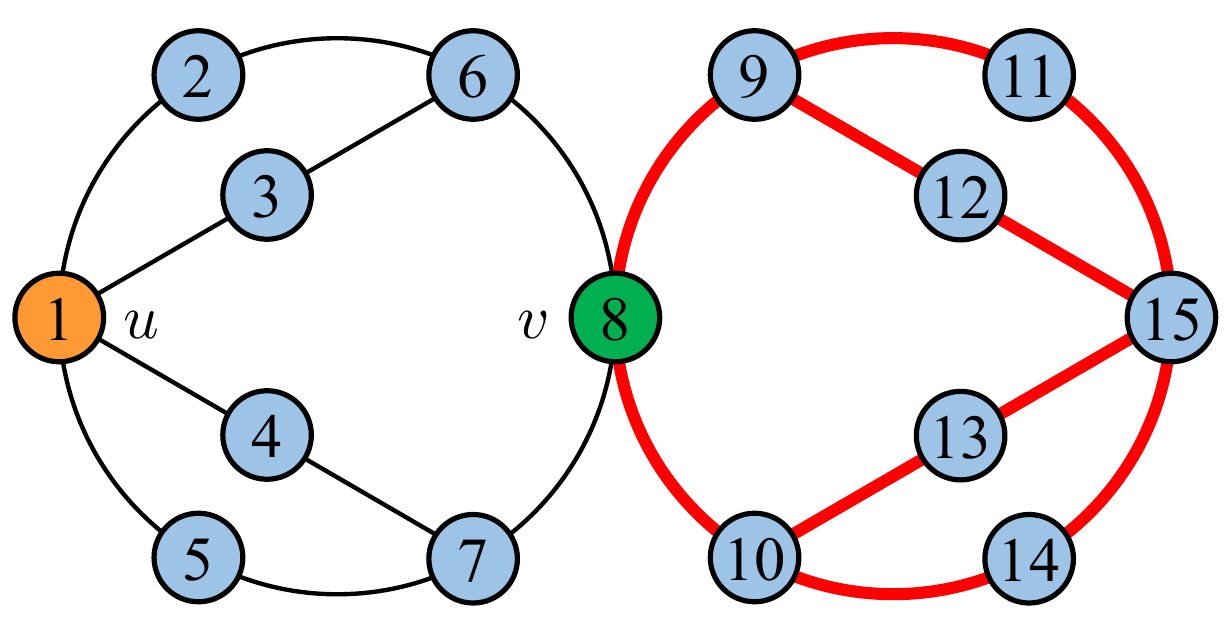} & \includegraphics[width=0.3\textwidth]{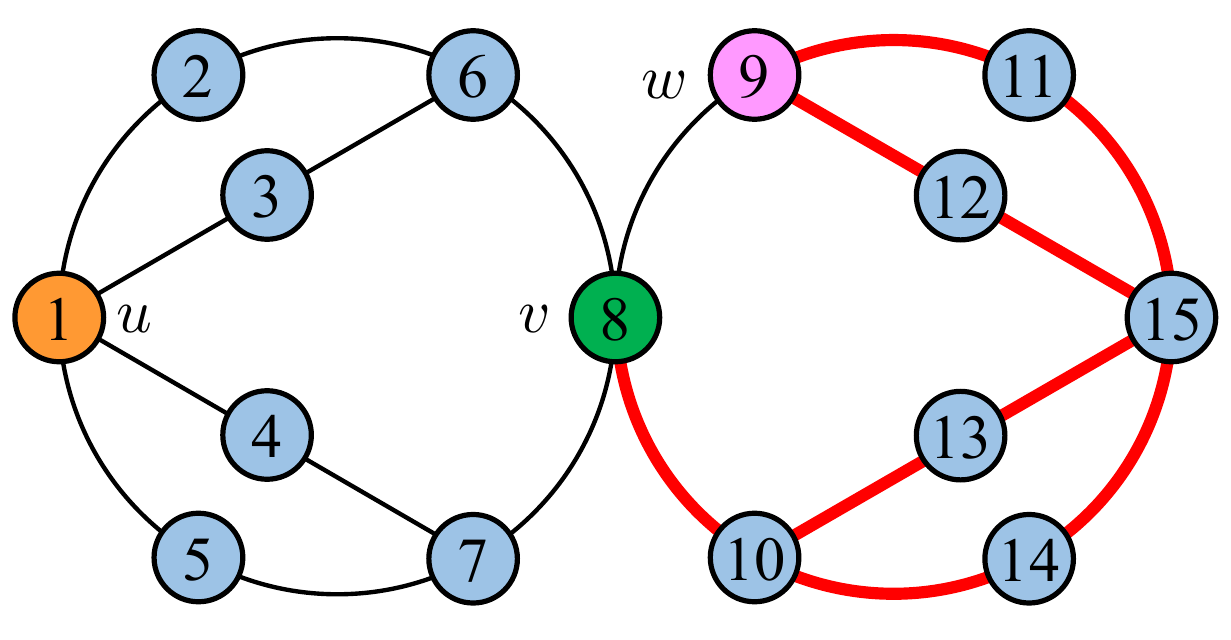}\\
         & (e) & (f) \\
         \includegraphics[width=0.3\textwidth]{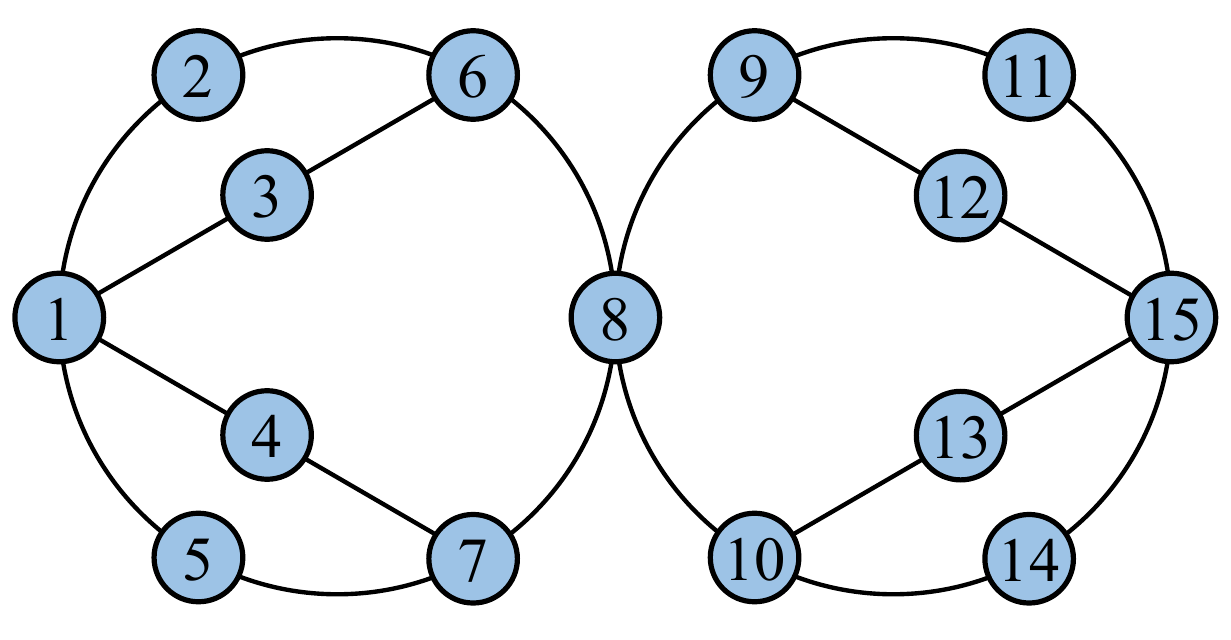} & \includegraphics[width=0.3\textwidth]{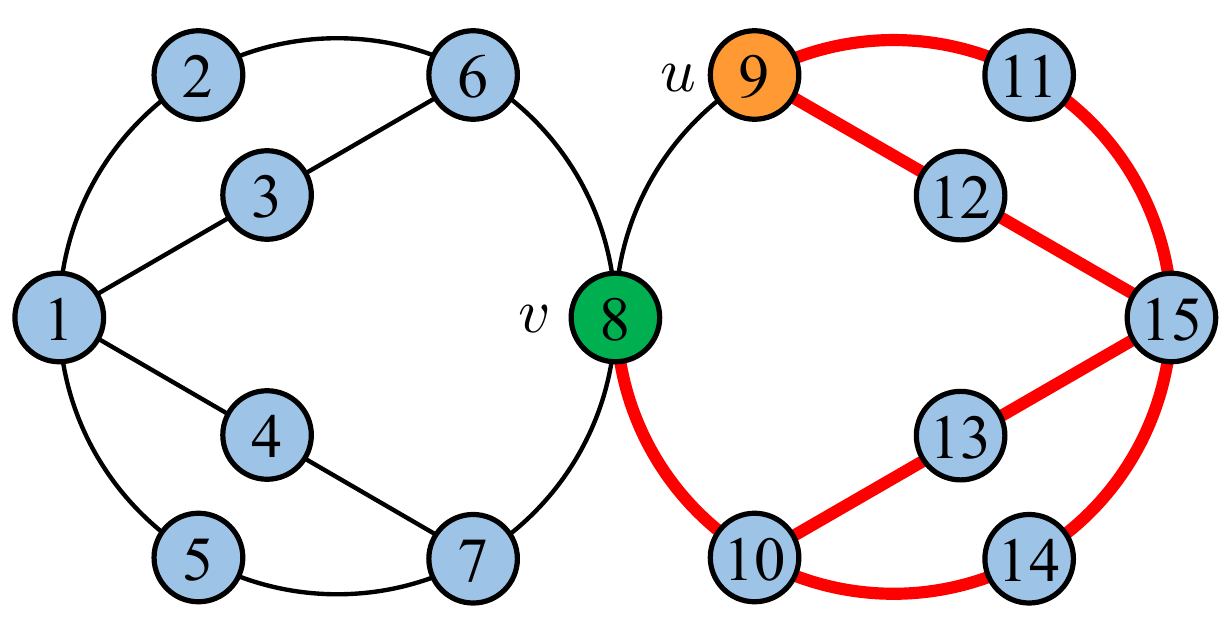} & \includegraphics[width=0.3\textwidth]{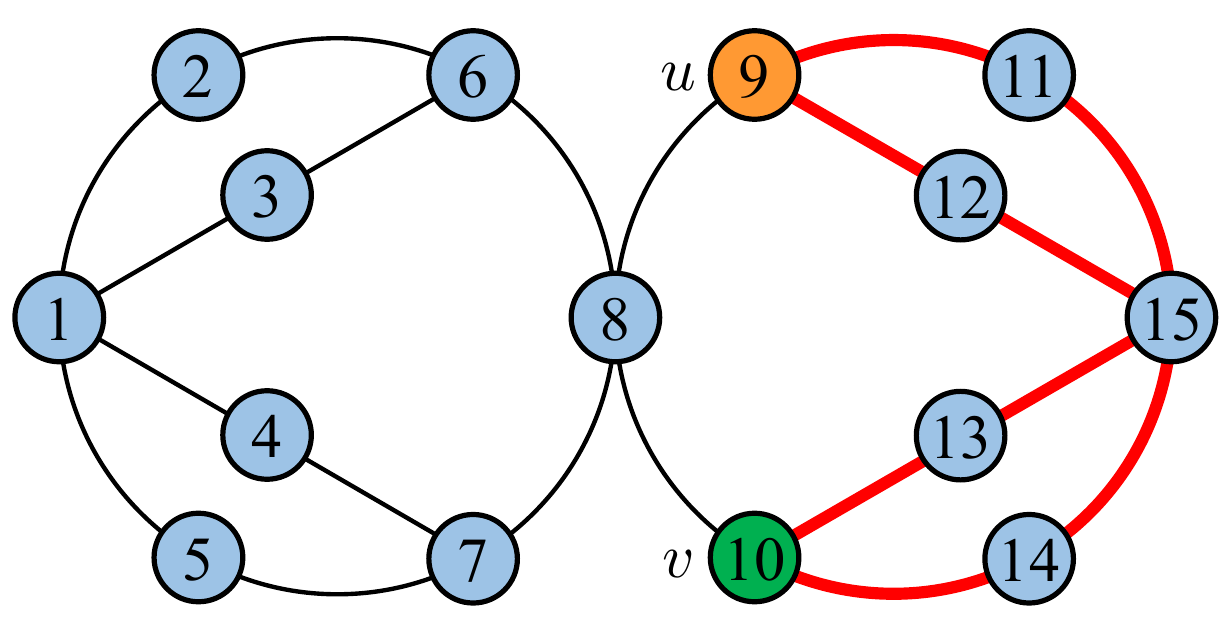}\\
         Base graph & (g) & (h) \\
         & \includegraphics[width=0.3\textwidth]{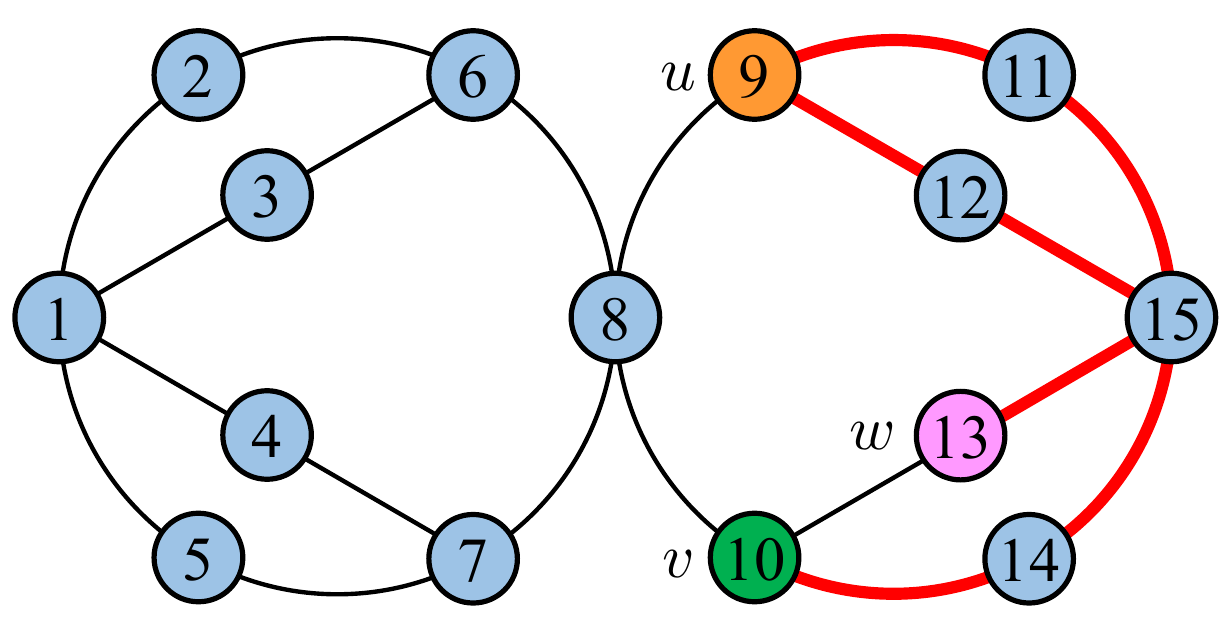} & \includegraphics[width=0.3\textwidth]{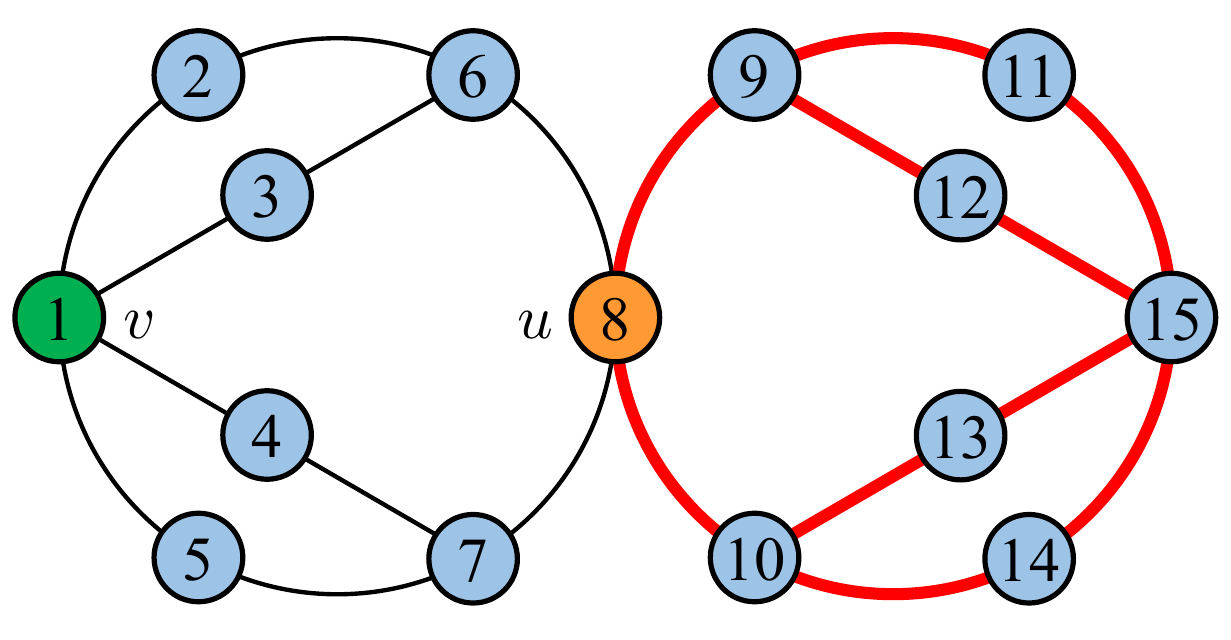} \\
         & (i) & (j) \\
         & \includegraphics[width=0.3\textwidth]{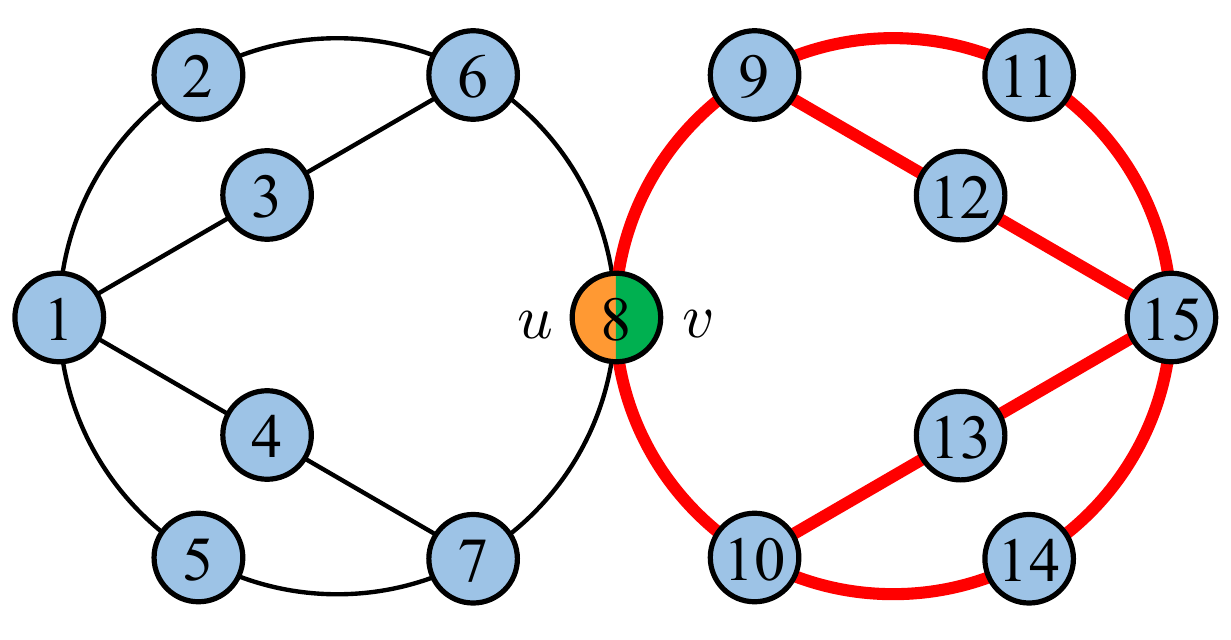} & \includegraphics[width=0.3\textwidth]{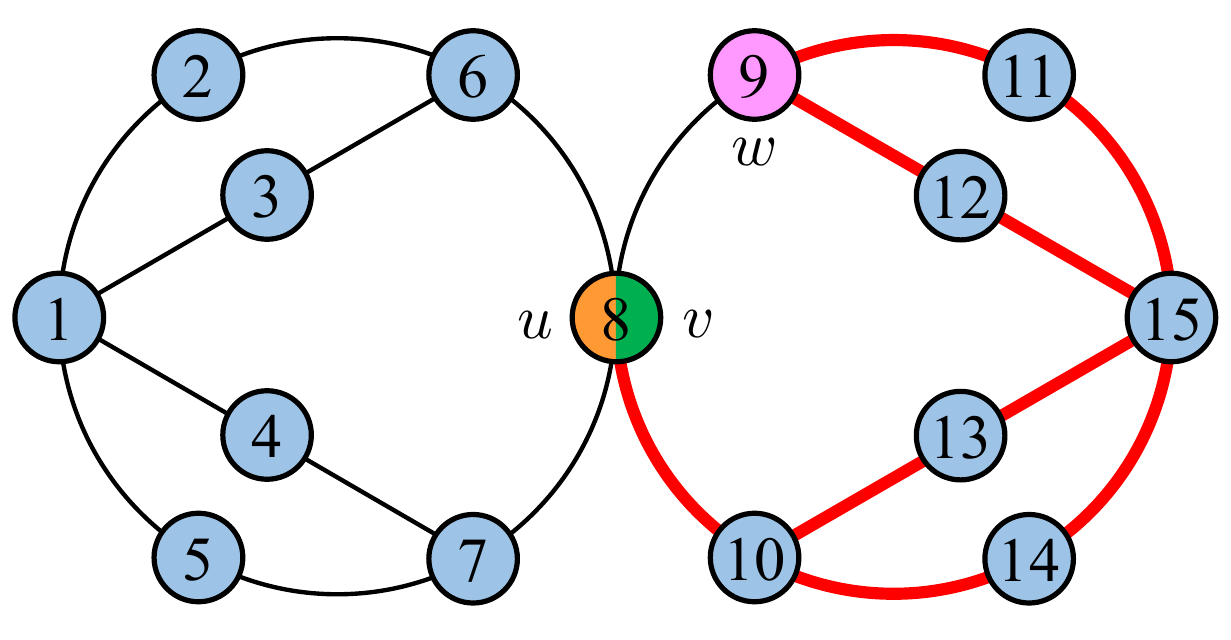}\\
         & (k) & (l) \\
         & \includegraphics[width=0.3\textwidth]{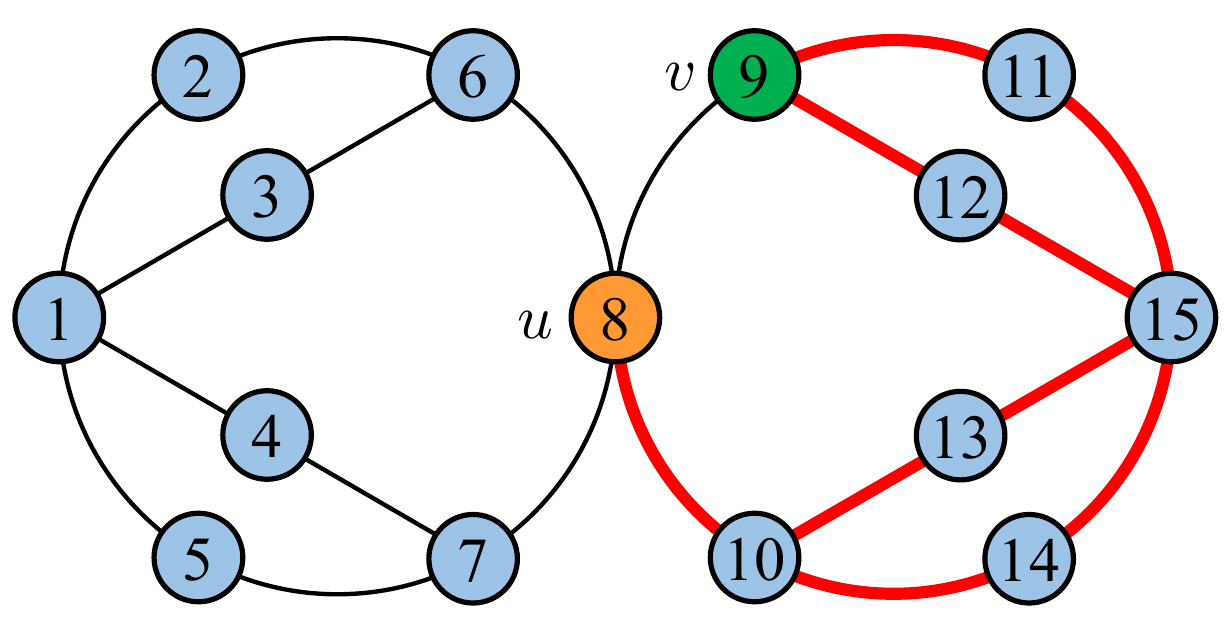} & \includegraphics[width=0.3\textwidth]{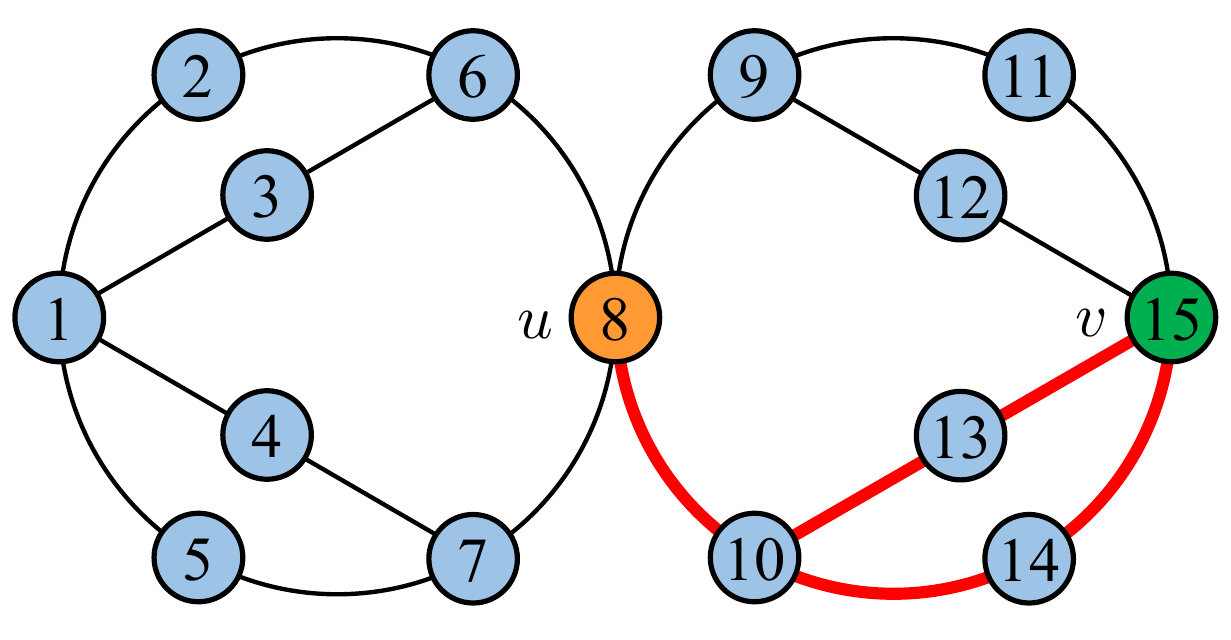}\\
         & (m) & (n) \\
    \end{tabular}
    \vspace{-5pt}
    \caption{Illustration of the proof of \cref{thm:counterexample_pooling}. When Duplicator follows her optimal strategy, the game process of $\mathsf{SWL(SV)}$ corresponds to figures (a, b, c, d), and Spoiler can eventually win the game. In contrast, the game process of $\mathsf{LFWL(2)}$ may correspond to figures (e, f, g, h, i, ...) or figures (j, k, l, g, h, i, ...) or figures (j, k, l, m, n, ...), depending how Spoiler swaps pebbles. In both cases, Spoiler cannot win.}
    \label{fig:counterexample_pooling}
\end{figure*}
\begin{lemma}
\label{thm:counterexample_pooling}
    There exist two non-isomorphic graphs such that 
    \begin{itemize}[topsep=0pt]
    \setlength{\itemsep}{0pt}
        \item $\mathsf{SWL(SV)}$ can distinguish them;
        \item $\mathsf{LFWL(2)}$ cannot distinguish them.
    \end{itemize}
\end{lemma}
\begin{proof}
    The base graph is constructed in \cref{fig:counterexample_pooling}.
    
    We first analyze the simplified pebbling game for algorithm $\mathsf{SWL(SV)}$. Initially, Spoiler places pebble $v$ on vertex 8, which splits the graph into two equal parts. Due to symmetry, suppose Duplicator selects the left component (\cref{fig:counterexample_pooling}(a)). Spoiler then places pebble $u$ on vertex 1 to further split the connected component. Due to symmetry, suppose Duplicator selects the top-left component (\cref{fig:counterexample_pooling}(b)). In the next round, Spoiler places pebble $w$ adjacent to $v$ on vertex 6. Duplicator should choose the connected component of either $\{\{1,3\},\{3,6\}\}$ or $\{\{1,2\},\{2,6\}\}$, which is equivalent by symmetry (see \cref{fig:counterexample_pooling}(c)). Spoiler then swaps pebbles $v,w$ and leaves $w$ outside the graph (\cref{fig:counterexample_pooling}(d)). The remaining game is straightforward to analyze and Spoiler can easily win.

    We next turn to the algorithm $\mathsf{LFWL(2)}$. Initially, Spoiler should simultaneously place pebbles $u$ and $v$ on two vertices of the graph. Without loss of generality, assume that he places one pebble on vertex 8 and places the other pebble on vertex 1 (other cases are similar to analyze). Depending on which pebble is placed on vertex 8, there are two cases:
    \begin{itemize}[topsep=0pt]
    \setlength{\itemsep}{0pt}
        \item Spoiler places pebble $u$ on vertex 1 and places pebble $v$ on vertex 8. In this case, Duplicator responds by choosing the connected component on the right (\cref{fig:counterexample_pooling}(e)). In the next round, Spoiler should better place pebble $w$ on vertex 9 (or vertex 10) adjacent to pebble $v$, and Duplicator can respond accordingly (\cref{fig:counterexample_pooling}(f)). Spoiler should then swap pebbles $u$ and $w$ and leave $w$ outside the graph (\cref{fig:counterexample_pooling}(g)). In the next round, Spoiler can similarly place pebble $w$ on vertex 10 adjacent to pebble $v$ to further split the connected component. This corresponds to \cref{fig:counterexample_pooling}(h) after swapping pebbles $v$ and $w$ and leaving $w$ outside the graph. In subsequent rounds, Spoiler can continue to place pebble $w$ adjacent to pebble $v$ (\cref{fig:counterexample_pooling}(i)). However, whether he swaps pebbles $u,w$ or pebbles $v,w$, as long as pebble $w$ is left outside the graph, multiple connected components then merge into a whole. Clearly, Spoiler cannot win the game.
        \item Spoiler places pebble $u$ on vertex 8 and places pebble $v$ on vertex 1. In this case, Duplicator similarly responds by choosing the connected component on the right (\cref{fig:counterexample_pooling}(j)). In subsequent rounds, Spoiler should gradually move pebble $v$ until it reaches the position of pebble $u$ (\cref{fig:counterexample_pooling}(k)). Next, Spoiler will continue to place pebble $w$ adjacent to pebble $v$ on vertex 9, and Duplicator can respond accordingly (\cref{fig:counterexample_pooling}(l)). Spoiler should then either swap pebbles $u,w$ or swap pebbles $v,w$ and leave $w$ outside the graph. The former case corresponds to (\cref{fig:counterexample_pooling}(g) and has been analyzed. Now consider the latter case, which corresponds to \cref{fig:counterexample_pooling}(m). In subsequent rounds, Spoiler can perform arbitrary operations, but he can never change the position of pebbles $u$. Otherwise, the lack of a pebble on vertex 8 will cause the merging of multiple connected components. With the position of pebble $u$ unchanged, the best state Spoiler can achieve is illustrated in \cref{fig:counterexample_pooling}(i). It is not hard to figure out that Spoiler cannot win the game, either.
    \end{itemize}
    In both cases, Spoiler cannot win.
\end{proof}
\textbf{Insight into \cref{thm:counterexample_pooling}}. The reason why $\mathsf{SWL(SV)}$ is stronger in this case is due to the $\mathsf{SV}$ pooling strategy. \cref{thm:counterexample_pooling} shows that $\mathsf{LFWL(2)}$ does not have the ability to implement the $\mathsf{SV}$ pooling strategy.

We are now ready to prove \cref{thm:separation}.

\begin{proof}[Proof of \cref{thm:separation}]
    \cref{thm:separation} is a direct consequence of \cref{thm:swl_hierarchy,thm:2lfwl,thm:counterexample_swlvs,thm:counterexample_pswlsv,thm:counterexample_gswl,thm:counterexample_lfwl_sswl,thm:counterexample_slfwl,thm:counterexample_slfwl_2,thm:counterexample_fwl,thm:counterexample_pooling}. 
\end{proof}

\section{Proof of Theorems in \cref{sec:practical_expressiveness}}
\label{sec:proof_practical_expressiveness}

\subsection{Proof of \cref{thm:pswl_gdwl}}
We first define several notations. Consider a path $P=(x_0,\cdots,x_d)$ (not necessarily simple) in graph $G$ of length $d\ge 1$. We say $P$ is a \emph{hitting path}, if $x_i\neq x_d$ for all $i\in\{0,1,\cdots,d-1\}$. Denote $\gQ^d_G(u,v)$ to be the set of all hitting paths from node $u$ to node $v$ of length $d$. Denote $\disH_G(u,v)$ as the \emph{hitting time distance} between vertices $u$ and $v$ in graph $G$, i.e, the average hitting time in a random walk from vertex $u$ to $v$. Then,
\begin{equation*}
    \disH_G(u,v)=\sum_{d=0}^\infty d\cdot\sum_{(x_0,\cdots,x_d)\in\gQ^d_G(u,v)} 1/\left(\prod_{i=0}^{d-1} \deg(x_i)\right).
\end{equation*}

Given a path $P=(x_0,\cdots,x_d)$, define $\omega(P):=(\deg_G(x_1),\cdots,\deg_G(x_{d-1}))$, which is a tuple of length $d-1$. Our proof if based on the following lemma:
\begin{lemma}
\label{thm:hitting_time}
    Let $G=(\gV_G,\gE_G)$ and $H=(\gV_H,\gE_H)$ be two graphs, and let $t\in\mathbb N_+$ be a positive integer. Consider any SWL algorithm $\mathsf{A}(\gA,\pool)$ with $\agg^\mathsf{L}_\mathsf{u}\in\gA$ and denote $\chi^{(t)}$ to be the color mapping after iteration $t$. Given nodes $u,v\in\gV_G$ and $x,y\in\gV_H$, if $\chi_G^{(t)}(u,v)=\chi_H^{(t)}(x,y)$, then $\ldblbrace \omega(Q):Q\in\gQ^{t}_G(v,u)\rdblbrace=\ldblbrace \omega(Q):Q\in\gQ^{t}_H(y,x)\rdblbrace$.
\end{lemma}
\begin{proof}
    The proof is based on induction over $t$. For the base case of $t=1$, it is easy to see that $\ldblbrace \omega(Q):Q\in\gQ^{1}_G(v,u)\rdblbrace$ depends only on whether $\{u,v\}\in\gE_G$ or not. Clearly, if $\chi_G^{(1)}(u,v)=\chi_H^{(1)}(x,y)$, then $\{u,v\}\in\gE_G\leftrightarrow\{x,y\}\in\gE_G$ holds (\cref{thm:node_marking_implies_distance}), implying $\ldblbrace \omega(Q):Q\in\gQ^{1}_G(v,u)\rdblbrace=\ldblbrace \omega(Q):Q\in\gQ^{1}_H(y,x)\rdblbrace$.

    Now assume that the lemma holds for all $t\le T$, and consider the case of $t= T+1$. When $\chi_G^{(T+1)}(u,v)=\chi_H^{(T+1)}(x,y)$, by definition of $\agg^\mathsf{L}_\mathsf{u}$ we have
    \begin{equation*}
        \ldblbrace \chi_G^{(T)}(u,w): w\in\gN_{G}(v)\rdblbrace=\ldblbrace \chi_H^{(T)}(x,z): z\in\gN_{H}(y)\rdblbrace.
    \end{equation*}
    Since $T\ge 1$, we have $\chi_G^{(T)}(u,w)=\chi_H^{(T)}(x,z)\implies \deg_G(w)=\deg_H(z)$ for any $w\in\gV_G$ and $z\in\gV_H$. Therefore,
    \begin{align*}
        \ldblbrace (\chi_G^{(T)}(u,w),\deg_G(w)): w\in\gN_{G}(v)\rdblbrace
        =\ldblbrace (\chi_H^{(T)}(x,z),\deg_H(z)): z\in\gN_{H}(y)\rdblbrace.
    \end{align*}
    By definition of node marking policy, we further obtain
    \begin{align*}
        \ldblbrace (\chi_G^{(T)}(u,w),\deg_G(w)): w\in\gN_{G}(v)\backslash\{u\}\rdblbrace
        =\ldblbrace (\chi_H^{(T)}(x,z),\deg_H(z)): z\in\gN_{H}(y)\backslash\{x\}\rdblbrace.
    \end{align*}
    By induction,
    \begin{align*}
        &\ldblbrace (\deg_G(w),\ldblbrace \omega(Q):Q\in\gQ^{T}_G(w,u)\rdblbrace): w\in\gN_{G}(v)\backslash\{u\}\rdblbrace\\
        =&\ldblbrace (\deg_H(z),\ldblbrace \omega(Q):Q\in\gQ^{T}_H(z,x)\rdblbrace): z\in\gN_{H}(y)\backslash\{x\}\rdblbrace.
    \end{align*}
    Therefore,
    $\ldblbrace \omega(Q):Q\in\gQ^{T+1}_G(v,u)\rdblbrace=\ldblbrace \omega(Q):Q\in\gQ^{T+1}_H(y,x)\rdblbrace$, concluding the induction step.
\end{proof}

\begin{corollary}
\label{thm:hitting_time_main}
    Consider any SWL algorithm $\mathsf{A}(\gA,\pool)$ with $\agg^\mathsf{L}_\mathsf{u}\in\gA$ and let $\chi$ be the stable color mapping. For any vertices $u,v\in\gV_G$ and $x,y\in\gV_H$, if $\chi_G(u,v)=\chi_H(x,y)$, then $\disH_G(v,u)=\disH_H(y,x)$.
\end{corollary}
\begin{proof}
    By definition of hitting time distance,
    $$\disH_G(v,u)=\sum_{i=0}^\infty i\cdot\sum_{Q\in\gQ^i_G(v,u)}1/q(Q),$$
    where $q(Q)=\deg_G(x_0)\prod_{i=1}^{d-1} \deg(x_i)$ for path $Q=(x_0,x_1,\cdots,x_d)$. Therefore, $q(Q)$ is fully determined by $\omega(Q)$ and $\deg_G(x_0)$. If $\disH_G(v,u)\neq\disH_H(y,x)$, then either $\deg_G(v)\neq \deg_H(y)$ or there exists a length $t$ such that $\ldblbrace \omega(Q):Q\in\gQ^{t}_G(v,u)\rdblbrace\neq\ldblbrace \omega(Q):Q\in\gQ^{t}_H(y,x)\rdblbrace$. Therefore, by using \cref{thm:hitting_time} we have $\chi_G(u,v)\neq\chi_H(x,y)$, as desired.
\end{proof}

We are now ready to prove the main theorem:

\begin{theorem}
\label{thm:pswl_gswl_variant}
    Define a variant of $\mathsf{GD}\text{-}\mathsf{WL}$ that incorporates the shortest path distance and the hitting time distance as follows:
    \begin{align*}
        \chi_G^{(t+1)}(u)=\Ldblbrace\left((\dis_G(v,u),\disH_G(v,u)),\chi_G^{(t)}(v)\right):v\in\gV_G\Rdblbrace.
    \end{align*}
    Then, $\mathsf{GD}\text{-}\mathsf{WL}\preceq\mathsf{PSWL(VS)}$.
\end{theorem}
\begin{proof}
    Denote $\chi^\mathsf{P}$ as the stable color mapping of $\mathsf{PSWL(VS)}$. Consider a pair of graphs $G$ and $H$ indistinguishable by $\mathsf{PSWL(VS)}$. Then, we clearly have
    $$\ldblbrace\chi_G^\mathsf{P}(u,v):u,v\in\gV_G\rdblbrace=\ldblbrace\chi_H^\mathsf{P}(x,y):x,y\in\gV_H\rdblbrace.$$
    By definition of the node marking policy,
    $$\ldblbrace\chi_G^\mathsf{P}(u,u):u\in\gV_G\rdblbrace=\ldblbrace\chi_H^\mathsf{P}(x,x):x\in\gV_H\rdblbrace.$$
    Now consider any vertices $u\in\gV$ and $x\in\gV_H$ satisfying $\chi_G^\mathsf{P}(u,u)=\chi_H^\mathsf{P}(x,x)$. Since $\agg^\mathsf{L}_\mathsf{u}$ is present in $\mathsf{PSWL(VS)}$, we can invoke \cref{thm:local_lemma}, which obtains that 
    $$\ldblbrace \chi_G^\mathsf{P}(u,w):w\in\gV_G\rdblbrace=\ldblbrace \chi_H^\mathsf{P}(x,z):z\in\gV_H\rdblbrace.$$
    Further using \cref{thm:node_marking_distance_corollary,thm:hitting_time_main} yields
    \begin{align*}
        \ldblbrace (\chi_G^\mathsf{P}(u,w),\dis_G(w,u),\disH_G(w,u)):w\in\gV_G\rdblbrace=\ldblbrace (\chi_H^\mathsf{P}(x,z),\dis_G(z,x),\disH_G(z,x)):z\in\gV_H\rdblbrace.
    \end{align*}
    Next, by definition of the aggregation $\agg^\mathsf{P}_\mathsf{vv}$, we have
    \begin{align*}
        \ldblbrace (\chi_G^\mathsf{P}(w,w),\dis_G(w,u),\disH_G(w,u)):w\in\gV_G\rdblbrace=\ldblbrace (\chi_H^\mathsf{P}(z,z),\dis_G(z,x),\disH_G(z,x)):z\in\gV_H\rdblbrace.
    \end{align*}
    The above equation shows that $\chi_G^\mathsf{P}$ induces a finer \emph{vertex} partition (i.e., by treating $\chi_G^\mathsf{P}(u):=\chi_G^\mathsf{P}(u,u)$) compared with the stable color mapping of $\mathsf{GD}\text{-}\mathsf{WL}$. Concretely, based on \cref{remark:finer}(c), we have $\chi_G^\mathsf{P}(u,u)=\chi_H^\mathsf{P}(x,x)\implies\chi_G(u)=\chi_H(x)$. This finally yields
    $$\ldblbrace\chi_G(u):u\in\gV_G\rdblbrace=\ldblbrace\chi_H(x):x\in\gV_H\rdblbrace,$$
    concluding the proof.
\end{proof}

\begin{remark}
    We note that the form of $\mathsf{GD}\text{-}\mathsf{WL}$ in \cref{thm:pswl_gswl_variant} slightly differs from \citet{zhang2023rethinking}, in that they use resistance distance instead of hitting time distance. Nevertheless, similar to resistance distance, hitting time distance also satisfies the following key property: for any vertices $u,v,w\in\gV_G$ in graph $G$,
    $\disH_G(u,v)=\disH_G(u,w)+\disH_G(w,v)$ if and only if $w$ is a cut vertex of $G$ (see \citet[Appendix C.5.1]{zhang2023rethinking}). This property is crucial to prove the expressivity for vertex-biconnectivity. Following the almost same proof, we can show that the variant of $\mathsf{GD}\text{-}\mathsf{WL}$ defined in \cref{thm:pswl_gswl_variant} is also fully expressive for vertex-biconnectivity.
\end{remark}

Finally, for the original $\mathsf{GD}\text{-}\mathsf{WL}$ defined in \citet{zhang2023rethinking} that incorporates SPD and RD, currently we can only prove the following result, which is a straightforward extension of \cref{thm:pswl_gswl_variant}:
\begin{theorem}
    Consider the WL algorithm $\mathsf{GD}\text{-}\mathsf{WL}$ that incorporates the shortest path distance and the resistance distance. Then, $\mathsf{GD}\text{-}\mathsf{WL}\preceq\mathsf{SSWL}$.
\end{theorem}
\begin{proof}
    Note that $\disR_G(u,v)=(\disH_G(u,v)+\disH_G(v,u))/2|\gE_G|$. The proof follows by noting that both $\agg^\mathsf{L}_\mathsf{u}$ and $\agg^\mathsf{L}_\mathsf{v}$ are present in $\mathsf{SSWL}$.
\end{proof}

However, it remains unclear whether $\mathsf{RD}\text{-}\mathsf{WL}\preceq\mathsf{GSWL}$ or $\mathsf{RD}\text{-}\mathsf{WL}\preceq\mathsf{PSWL(VS)}$ holds. We leave them as open problems.

\subsection{Counterexamples}
We first show that the vanilla SWL has additional power than $\mathsf{GD}\text{-}\mathsf{WL}$.
\begin{lemma}
\label{thm:counterexample_gdwl_1}
    There exist two non-isomorphic graphs such that 
    \begin{itemize}[topsep=0pt]
    \setlength{\itemsep}{0pt}
        \item $\mathsf{SWL(VS)}$ can distinguish them;
        \item $\mathsf{RD}\text{-}\mathsf{WL}$, $\mathsf{HTD}\text{-}\mathsf{WL}$, $\mathsf{SPD}\text{-}\mathsf{WL}$ cannot distinguish them.
    \end{itemize}
\end{lemma}
\begin{figure*}
    \small
    \centering
    \begin{tabular}{c|cccc}
         \includegraphics[width=0.145\textwidth]{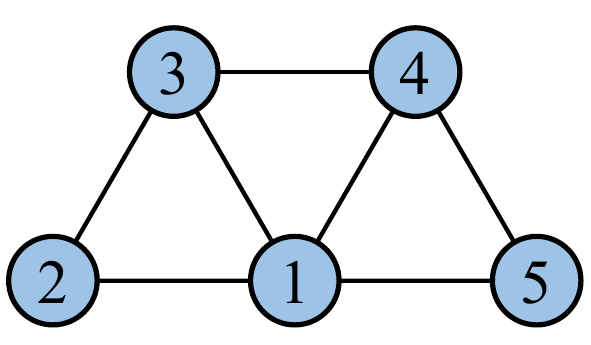} & \includegraphics[width=0.145\textwidth]{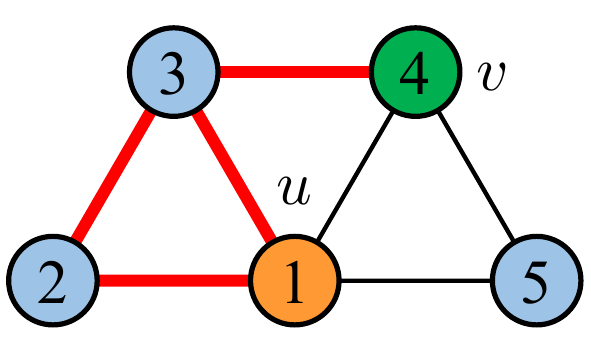} &\includegraphics[width=0.145\textwidth]{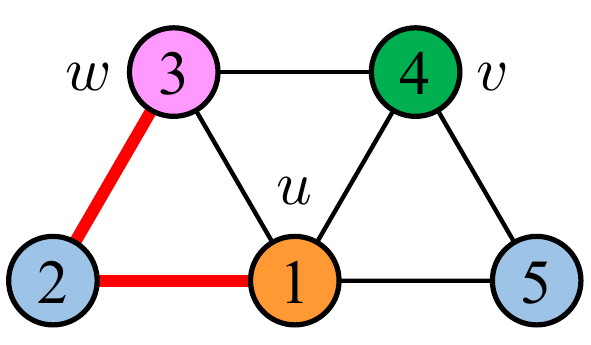} & \includegraphics[width=0.145\textwidth]{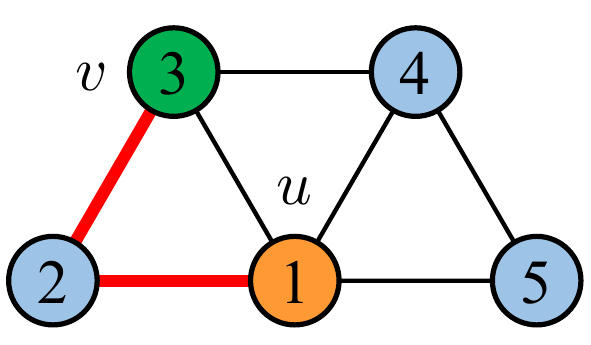}\\
         Base graph & (a) & (b) & (c)\\
    \end{tabular}
    \caption{Illustration of the proof of \cref{thm:counterexample_gdwl_1}. When Duplicator follows her optimal strategy, the game process of $\mathsf{SWL(VS)}$ corresponds to figures (a, b, c) and Spoiler can eventually win.}
    \label{fig:counterexample_gdwl}
\end{figure*}
\begin{proof}
    The proof is based on \cref{sec:proof_separation_part3} using generalized F{\"u}rer graphs. The base graph is constructed in \cref{fig:counterexample_gdwl}. 
    
    We first consider $\mathsf{SWL(VS)}$ and analyze the simplified pebbling game developed in \cref{sec:pebbling_furer}. At the beginning, Spoiler just places pebble $u$ on vertex 1, and Duplicator does nothing. Spoiler then places pebble $v$ on vertex 4, splitting the connected component into three parts. It is easy to see that Duplicator should select the largest connected component on the left (\cref{fig:counterexample_gdwl}(a)). In the next round, Spoiler places pebbles $w$ adjacent to pebble $v$ on vertex 3. Duplicator has to respond according to \cref{fig:counterexample_gdwl}(b). Spoiler then swaps pebbles $v,w$ and leaves $w$ outside the graph (\cref{fig:counterexample_gdwl}(c)). It is easy to see that Spoiler can win the remaining game.

    We next consider $\mathsf{GD}\text{-}\mathsf{WL}$, for which we do not have a corresponding game. Nevertheless, a good news is that the corresponding (twisted) F{\"u}rer graph has only 20 vertices. We can thus directly verify that the stable colors of the F{\"u}rer graph match those of the twisted F{\"u}rer graph. A deep understanding of why $\mathsf{GD}\text{-}\mathsf{WL}$ cannot distinguish the two graphs is left for future work.
\end{proof}

Conversely, we next show that $\mathsf{GD}\text{-}\mathsf{WL}$ also has additional power than the vanilla SWL.
\begin{lemma}
\label{thm:counterexample_gdwl_2}
    There exist two non-isomorphic graphs such that 
    \begin{itemize}[topsep=0pt]
    \setlength{\itemsep}{0pt}
        \item $\mathsf{SWL(SV)}$ cannot distinguish them;
        \item $\mathsf{SPD}\text{-}\mathsf{WL}$ can distinguish them;
        \item $\mathsf{RD}\text{-}\mathsf{WL}$ can distinguish them;
        \item $\mathsf{HTD}\text{-}\mathsf{WL}$ can distinguish them.
    \end{itemize}
\end{lemma}
\begin{figure*}
    \small
    \centering
    \begin{tabular}{ccc}
         \includegraphics[width=0.2\textwidth]{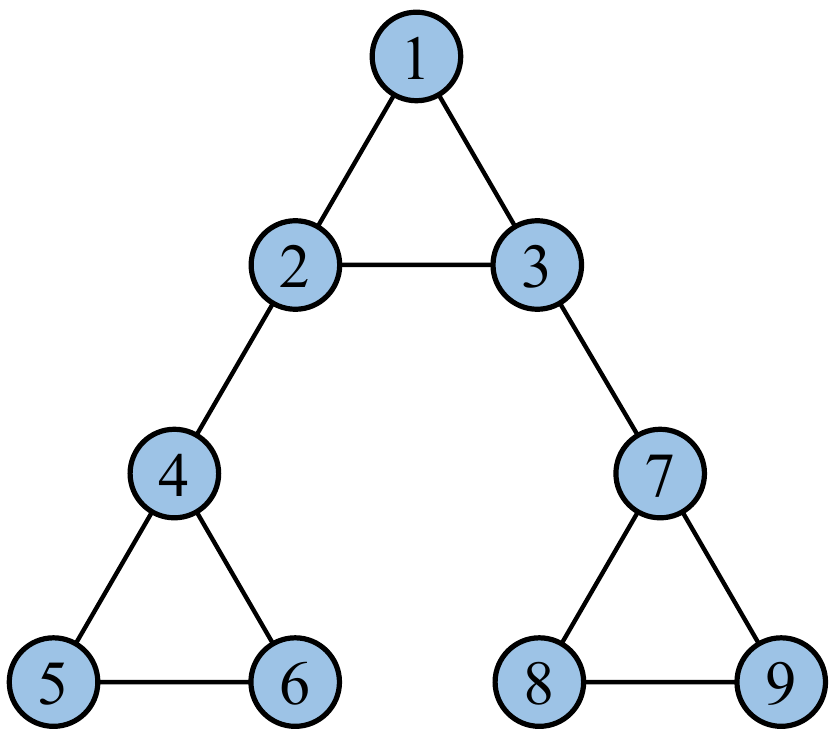} & \includegraphics[width=0.2\textwidth]{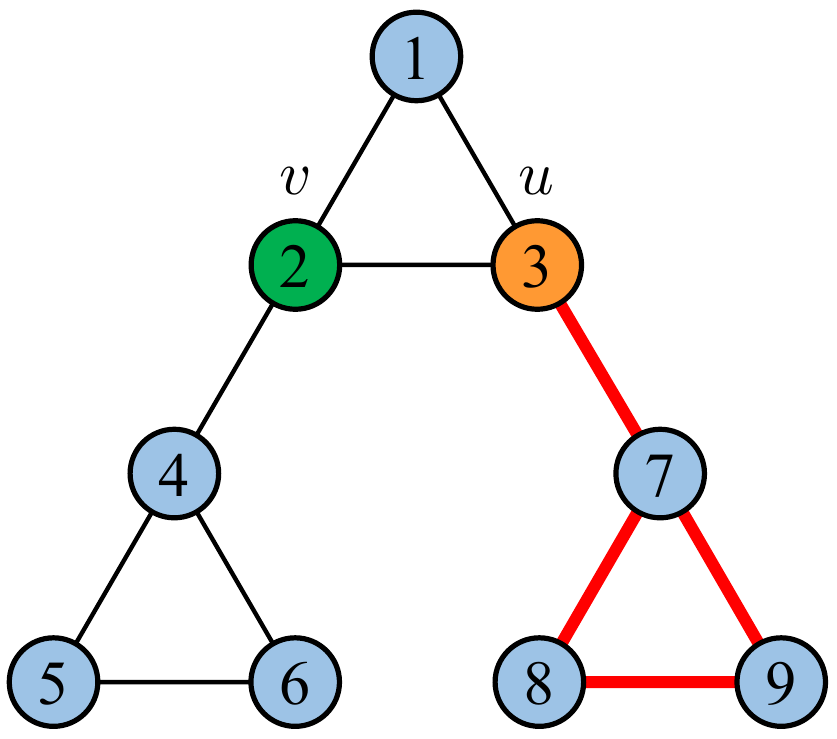} & \includegraphics[width=0.2\textwidth]{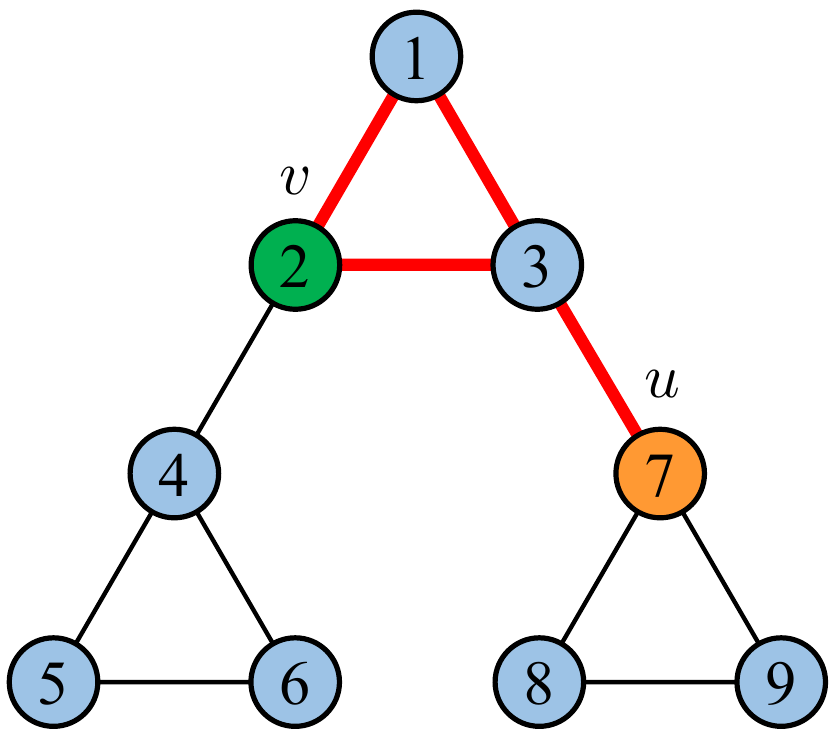}\\
         Base graph & (a) & (b)\\
    \end{tabular}
    \caption{Illustration of the proof of \cref{thm:counterexample_gdwl_2}. For $\mathsf{SWL(SV)}$, when Duplicator follows her optimal strategy, Spoiler can never win the game.}
    \label{fig:counterexample_gdwl2}
\end{figure*}
\begin{proof}
    The proof is based on \cref{sec:proof_separation_part3} using generalized F{\"u}rer graphs. The base graph is constructed in \cref{fig:counterexample_gdwl2}. 

    We first consider $\mathsf{SWL(SV)}$ and analyze the simplified pebbling game developed in \cref{sec:pebbling_furer}. At the beginning, Spoiler should place pebble $v$ on some vertex. Regardless of Spoiler's choice, Duplicator just selects the largest connected component after pebble $v$ is placed. Next, Spoiler should place pebble $u$ on some vertex. Now Duplicator's strategy is to select a connected component such that it contains a triangle with no pebbles. It is easy to see that Duplicator can always achieve her goal (see \cref{fig:counterexample_gdwl2}(a) and \cref{fig:counterexample_gdwl2}(b) for two representative cases). The remaining game is easy to analyze: since pebble $u$ cannot be moved throughout the game, there is a triangle that holds at most two pebbles and cannot be split into three single-edge components. Clearly, Duplicator can always respond without losing the game.

    We next turn to $\mathsf{SPD}\text{-}\mathsf{WL}$, for which we do not have a corresponding game. Nevertheless, a good news is that the corresponding (twisted) F{\"u}rer graph has only 26 vertices. We can thus directly verify that the stable colors of the F{\"u}rer graph do not match those of the twisted F{\"u}rer graph. The case of $\mathsf{RD}\text{-}\mathsf{WL}$ and $\mathsf{HTD}\text{-}\mathsf{WL}$ can be similarly verified. A deep understanding of why these algorithms can distinguish the two graphs is left for future work. 
\end{proof}

\section{Experimental Details}
\label{sec:exp_details}

We conduct experiments on three standard benchmark datasets: ZINC \citep{dwivedi2020benchmarking}, Counting Substructure \citep{zhao2022stars,frasca2022understanding}, and OGBG-molhiv \citep{hu2020open}. ZINC is a standard benchmark for molecular property prediction, where the task is to predict the constrained solubility of a molecule, which is an important chemical property for drug discovery. We train and evaluate our proposed $\mathsf{GNN}\text{-}\mathsf{SSWL}$ and $\mathsf{GNN}\text{-}\mathsf{SSWL}\text{+}$ on both ZINC (consisting of 250k molecular graph) and ZINC-subset (a 12K-subset selected as in \citet{dwivedi2020benchmarking}). Counting Substructure is a widely-used synthetic task in the expressive GNN community, where the task is to predict the number of a given substructure (such as cycle or star) in an input graph. We use the same dataset in \citet{zhao2022stars,frasca2022understanding} and further extend it to include the setting of counting 5/6-cycles motivated by \citet{huang2022boosting}. Finally, we additionally consider the OGBG-molhiv dataset in \cref{sec:ogbg}.

\subsection{Model details}
We implement our model using Pytorch \citep{paszke2019pytorch} and Pytorch Geometric \citep{fey2019fast} (available respectively under the BSD and MIT license). All experiments are run on a single NVIDIA Tesla V100 GPU. Our code will be released at \texttt{\href{https://github.com/subgraph23/SWL}{https://github.com/subgraph23/SWL}}.

Motivated by \cref{thm:node_marking,thm:policy_expand}, for all SWL models, the graph generation policy is chosen as the distance encoding on the original graph. Such a policy achieves the maximal power among other policies (as expressive as node marking) while explicitly introducing inductive biases, which may be beneficial for real-world tasks. Concretely, we initialize the feature of node $v$ in subgraph $G^u$ by summing the atom embedding $h^\mathsf{atom}(v)$ and the distance encoding $h^\mathsf{dis}(\dis_G(u, v))$, where the atom embedding is a learnable vector determined by the atom type of $v$, and the distance encoding is a learnable vector determined by the shortest path distance between $u$ and $v$. Mathematically, $h_G^{(0)}(u,v)= h^\mathsf{atom}(v) + h^\mathsf{dis}(\dis_G(u, v))$. Distances exceeding \texttt{max\_dis} (including infinity) are all encoded as a shared embedding. For tasks without atom features (e.g., Counting Substructure dataset), we set $h^\mathsf{atom}(v)$ to zero.

As an instance of \cref{def:layer}, our subgraph GNN layer can be written in the following form:
\begin{equation}
\label{eq:layer_detail}
\begin{aligned}
    h_G^{(l+1)}(u, v)=\mathsf{ReLU}\left(\sum_{i=1}^r\mu^{(l+1,i)}\left(h_G^{(l)}(u,v),\op_i(u, v, G, h_G^{(l)})\right)\right),
\end{aligned}
\end{equation}
where each $\op_i$ can take one of the following forms, depending on the atomic aggregations in the SWL algorithm:
\begin{itemize}[topsep=0pt]
\raggedright
\setlength{\itemsep}{0pt}
    \item For $\agg^\mathsf{P}_\mathsf{uu}$: $\op_i(u, v, G, h)=h(u,u)$;
    \item For $\agg^\mathsf{P}_\mathsf{vv}$: $\op_i(u, v, G, h)=h(v,v)$;
    \item For $\agg^\mathsf{G}_\mathsf{u}$: $\op_i(u, v, G, h)=\sum_{w\in\gV_G}h(u,w)$;
    \item For $\agg^\mathsf{G}_\mathsf{v}$: $\op_i(u, v, G, h)=\sum_{w\in\gV_G}h(w,v)$;
    \item For $\agg^\mathsf{L}_\mathsf{u}$: $\op_i(u, v, G, h)=\sum_{w\in\gN_G(v)}\mathsf{ReLU}((h(u,w)+g(w,v))$;
    \item For $\agg^\mathsf{L}_\mathsf{v}$: $\op_i(u, v, G, h)=\sum_{w\in\gN_G(u)}\mathsf{ReLU}(h(w,v)+g(u,w))$.
\end{itemize}
Note that we have included the single-point aggregation $\agg^\mathsf{P}_\mathsf{uv}$ directly in the update formula (\ref{eq:layer_detail}). For the last two local aggregations, we further encode the edge embedding $g(w,v)$ for edge $\{w,v\}\in\gE_G$ (or $g(u,w)$ for edge $\{w,u\}\in\gE_G$) when there is additional information for each edge (e.g., the bond information in the ZINC dataset). In the above equation, each $\mu^{(l+1,i)}$ is implemented by a GIN base encoder \citep{xu2019powerful}:
\begin{align*}
    \mu^{(l+1,i)}(h_1,h_2)=\mathsf{MLP}^{(l+1,i)}\left((1+\epsilon^{(l+1,i)})h_1+h_2\right),
\end{align*}
where $\epsilon^{(l+1,i)}$ is a learnable scalar and $\mathsf{MLP}^{(l+1,i)}$ is a one-hidden-layer MLP with hidden size equal to the input dimension. Batch Normalization \citep{ioffe2015batch} is adopted in the hidden layer of each MLP as well as in (\ref{eq:layer_detail}) before taking ReLU.

The final pooling layer is implemented as an $\mathsf{MLP}$ over the summation of $v$-dimension for $\mathsf{VS}$-pooling scheme (or $u$ for $\mathsf{SV}$-pooling) and a global mean pooling, namely
\begin{align*}
    f(G) = \frac 1 {|\gV_G|} \sum_{u \in \gV_G} \mathsf{MLP}\left(\sum_{v \in \gV_G} h^{(L)}_G(u,v)\right)
\end{align*}
Batch Normalization is adopted similarly.

\subsection{Training details}

\textbf{ZINC}. Throughout all experiments, we set the number of layers $L=6$, similar to \citet{frasca2022understanding}. To constrain the parameter budget within 500k, the feature dimension of each layer is set to 96. The initial atom embedding, distance embedding, and edge embedding is also set to 96. The hyper-parameter \texttt{max\_dis} is set to 5. We adopt the Adam optimizer \citep{kingma2014adam} with an initial learning rate $0.001$. The learning rate will be decayed by a factor of 0.5 when the MAE on \emph{validation set} plateaus for 20 epochs (similar to \citet{frasca2022understanding}). The batch size is set to 128. On ZINC-12K subset, the model is trained for 400 epochs according to \citet{frasca2022understanding}, and it takes roughly 1 to 2 hours for a single run. On ZINC-250K full set, we find that the model still does not converge after 400 epochs, so we adjust the configuration to 500 epochs, which takes about 40 hours. For each setting, we run the model 10 times with different seeds from 1 to 10 and report both the mean value and the standard deviation of MAE.

We also compare our model performance with various subgraph GNN baselines. The performance numbers of these baselines in \cref{tab:exp_zinc} are generally brought from the original works. For some baselines such that $\mathsf{GNN}\text{-}\mathsf{AK}$ and $\mathsf{GNN}\text{-}\mathsf{AK}\text{-}\mathsf{ctx}$, the results are obtained from \citet{frasca2022understanding}. The $\mathsf{NGNN}$ result is obtained from \citet{huang2022boosting}. Since other baseline models did not present the results on ZINC-full, we obtain their performance by running the code provided on the authors' official GitHub repo. We run each model 10 times with different seeds and report the mean performance and standard deviation. For $\mathsf{ESAN}$ and $\mathsf{SUN}$, we tried both the $k$-ego network policy and the $k$-ego network policy with marking and reported the better performance among the two policies. We find that the results are almost the same for $\mathsf{SUN}$, and the $k$-ego network policy with marking is slightly better for $\mathsf{ESAN}$.

\textbf{Counting Substructure}. Throughout all experiments, we simply follow the same model configuration as ZINC to use a 6-layer GNN with a hidden size of 96. The hyper-parameter \texttt{max\_dis} is also set to 5. We adopt the Adam optimizer \citep{kingma2014adam} with an initial learning rate $0.002$. The learning rate is decayed with cosine annealing. The batch size is set to 512. The models are trained for 600 epochs. Note that unlike prior works, \emph{we use the same model/training hyper-parameters for all six substructures}. For each setting, we run the model 5 times with different seeds from 1 to 5 and report the mean performance of MAE. We found that the standard deviation is very small.

We also compare our model performance with various subgraph GNN baselines. The performance numbers of GNN-AK and SUN in \cref{tab:exp_substructure} are brought from the \citet{zhao2022stars} and \citet{frasca2022understanding}, respectively. For the tasks of counting 5/6-cycles, we obtain their performance by running the code provided on the authors' official GitHub repo. To make the SUN result convincing, we grid search the network width from $\{64, 96, 110\}$, depth from $\{5, 6\}$, and search the $k$-hop ego net with $k\in\{2,3\}$ as suggested by \citet{frasca2022understanding}. We consider both the ego network policy with and without marking and find that $k$-ego network with marking is better.

\subsection{Ablation study on ZINC}

In this subsection, we present a set of ablation results to investigate the effect of different aggregation operations in $\mathsf{GNN}\text{-}\mathsf{SSWL}\text{+}$. We fix all model details and hyper-parameters presented above, and remove one or both of the additional two operations $\agg^\mathsf{L}_\mathsf{v}, \agg^\mathsf{P}_\mathsf{vv}$ in $\mathsf{GNN}\text{-}\mathsf{SSWL}\text{+}$ or change the pooling paradigm. This results in several types of models with expressivity corresponding to $\mathsf{SSWL}$, $\mathsf{PSWL(VS)}$, $\mathsf{PSWL(SV)}$, $\mathsf{SWL(SV)}$, and $\mathsf{SWL(VS)}$, respectively. The results are shown in \cref{table:ablation}. It can be seen that both of these additional aggregations $\agg^\mathsf{L}_\mathsf{v}, \agg^\mathsf{P}_\mathsf{vv}$ provide a significant improvement in performance. Moreover, $\mathsf{SV}$ pooling is significantly better than $\mathsf{VS}$ for vanilla SWL.

We further design an ablation experiment to verify that our introduced local aggregation $\agg^\mathsf{L}_\mathsf{v}$ is actually crucial and cannot be replaced by other aggregations. Here, we consider the model $\mathsf{SUN}$ proposed in \citet{frasca2022understanding}, which comprises a large number of basic aggregation operations but without $\agg^\mathsf{L}_\mathsf{v}$. $\mathsf{SUN}$ uses different parameters to compute the on-diagonal (i.e. $h_G^{(l)}(u,u)$) and off-diagonal (i.e. $h_G^{(l)}(u,v),v\neq u$) features in order to further enhance the model flexibility. We rerun their code by changing the graph generation policy to distance encoding (on the original graph) and use exactly the same training configuration as our model. Notably, the feature dimension is increased to 96 (instead of 64), resulting in a larger model with roughly 1163k parameters. We can see in \cref{table:ablation} that, even with distance encoding policy and larger model size, there is still a remarkable gap between $\mathsf{SUN}$ and $\mathsf{GNN}\text{-}\mathsf{SSWL}\text{+}$ (less than 400k parameters). This confirms that the introduced aggregation $\agg^\mathsf{L}_\mathsf{v}$ in $\mathsf{SSWL}$ not only theoretically improves the expressive power of the GNN model, but also leads to significantly better performance on real datasets.

\begin{table}[h]
\centering
\small
\caption{Ablation study of $\mathsf{GNN}\text{-}\mathsf{SSWL}\text{+}$ on ZINC-subset.}
\label{table:ablation}
\begin{tabular}[b]{lcc}
    \hline
    Method & Pooling & Test MAE $\downarrow$ \\
    \hline
    $\mathsf{GNN}\text{-}\mathsf{SSWL}\text{+}$ & $\mathsf{VS}$ & \textbf{0.0703 ± 0.0046} \\ 
    w/o $\agg_\mathsf{vv}$ ($\mathsf{GNN}\text{-}\mathsf{SSWL}$) & $\mathsf{VS}$ & 0.0822  ± 0.0029 \\ 
    w/o $\agg_\mathsf{v}^\mathsf{L}$ & $\mathsf{VS}$ & 0.0765  ± 0.0028 \\ 
    w/o $\agg_\mathsf{v}^\mathsf{L}$ & $\mathsf{SV}$ & 0.0758  ± 0.0037 \\ 
    w/o $\agg_\mathsf{v}^\mathsf{L}$ and $\agg_\mathsf{vv}$ & $\mathsf{VS}$ & 0.1103  ± 0.0090 \\ 
    w/o $\agg_\mathsf{v}^\mathsf{L}$ and $\agg_\mathsf{vv}$ & $\mathsf{SV}$ & 0.0999  ± 0.0044 \\ 
    $\mathsf{SUN}$ (Distance Encoding) & - & 0.0802 ± 0.0024 \\
    \hline
\end{tabular}
\end{table}

\subsection{Additional experiments on OGBG-molhiv}
\label{sec:ogbg}

We further run experiments on the OGBG-molhiv dataset. Following \citet{frasca2022understanding}, we use a 2-layer $\mathsf{GNN}\text{-}\mathsf{SSWL}\text{+}$ model with a network width of 64, and add residual connection between different layers. The hyper-parameter \texttt{max\_dis} is also set to 5. To prevent overfitting, we similarly use the ASAM optimizer \citep{kwon2021asam} with a batch size of 32, a learning rate of 0.01, and a dropout ratio of 0.3. Moreover, we change each MLP to a linear layer following \citet{frasca2022understanding}. We train the model for 100 epochs. We run our model 8 times with different seeds ranging from 1 to 8 and report the average ROC AUC as well as the standard deviation. The result is presented in \cref{tab:exp_ogbg}. 

\begin{table}[h]
\centering
\caption{Performance comparison on OGBG-molhiv.}
\small
\label{tab:exp_ogbg}
\vspace{2pt}
\begin{tabular}{llc}
\Xhline{1pt}
Model & Reference & Test ROC-AUC (\%) \\ \Xhline{0.75pt}
GCN         & \citet{kipf2017semisupervised} & 76.06±0.97 \\
GIN         & \citet{xu2019powerful}         & 75.58±1.40 \\
PNA         & \citet{corso2020principal}     & 79.05±1.32 \\
GSN         & \citet{bouritsas2022improving} & 80.39±0.90 \\
CIN         & \citet{bodnar2021cellular}     & \textbf{80.94±0.57} \\ \hline
Recon. GNN  & \citet{cotta2021reconstruction}& 76.32±1.40 \\
DS-GNN (EGO+) & \citet{bevilacqua2022equivariant} & 77.40±2.19 \\
DSS-GNN (EGO+) & \citet{bevilacqua2022equivariant} & 76.78±1.66 \\
GNN-AK+    & \citet{zhao2022stars}           & 79.61±1.19 \\
SUN        & \citet{frasca2022understanding} & 80.03±0.55 \\\hline
GNN-SSWL+  & This paper                      & 79.58±0.35 \\ \Xhline{1pt}
\end{tabular}
\end{table}
\end{document}